\documentclass{article}
\usepackage{iclr2019_conference,times}

\usepackage[utf8]{inputenc} % allow utf-8 input
\usepackage[T1]{fontenc}    % use 8-bit T1 fonts
\usepackage{hyperref}       % hyperlinks
\usepackage{url}            % simple URL typesetting
\usepackage{booktabs}       % professional-quality tables
\usepackage{amsfonts}       % blackboard math symbols
\usepackage{nicefrac}       % compact symbols for 1/2, etc.
\usepackage{microtype}      % microtypography

\usepackage{tikz}
\usetikzlibrary{arrows,shapes,backgrounds,fit,positioning}

\usepackage{subcaption}
\usepackage[inline]{enumitem}

\usepackage{amsmath, amsthm, amssymb, amsbsy, mathtools}
\usepackage[export]{adjustbox}

\newcommand{\overridecommand}[2]{
  \providecommand{#1}{}
  \renewcommand{#1}{#2}
}

\overridecommand{\aa}{\mathbf{a}}
\overridecommand{\bb}{\mathbf{b}}
\overridecommand{\cc}{\mathbf{c}}
\overridecommand{\dd}{\mathbf{d}}
\overridecommand{\ee}{\mathbf{e}}
\overridecommand{\ff}{\mathbf{f}}
\overridecommand{\bmg}{\mathbf{g}}   % \gg for >>
\overridecommand{\hh}{\mathbf{h}}
\overridecommand{\ii}{\mathbf{i}}
\overridecommand{\jj}{\mathbf{j}}
\overridecommand{\kk}{\mathbf{k}}
\overridecommand{\bml}{\mathbf{l}}   % \ll for <<
\overridecommand{\mm}{\mathbf{m}}
\overridecommand{\nn}{\mathbf{n}}
\overridecommand{\oo}{\mathbf{o}}
\overridecommand{\pp}{\mathbf{p}}
\overridecommand{\qq}{\mathbf{q}}
\overridecommand{\rr}{\mathbf{r}}
\overridecommand{\ss}{\mathbf{s}}
\overridecommand{\tt}{\mathbf{t}}
\overridecommand{\uu}{\mathbf{u}}
\overridecommand{\vv}{\mathbf{v}}
\overridecommand{\ww}{\mathbf{w}}
\overridecommand{\xx}{\mathbf{x}}
\overridecommand{\yy}{\mathbf{y}}
\overridecommand{\zz}{\mathbf{z}}

\overridecommand{\BB}{\mathbf{B}}
\overridecommand{\CC}{\mathbf{C}}
\overridecommand{\DD}{\mathbf{D}}
\overridecommand{\EE}{\mathbf{E}}
\overridecommand{\FF}{\mathbf{F}}
\overridecommand{\GG}{\mathbf{G}}
\overridecommand{\HH}{\mathbf{H}}
\overridecommand{\II}{\mathbf{I}}
\overridecommand{\JJ}{\mathbf{J}}
\overridecommand{\KK}{\mathbf{K}}
\overridecommand{\LL}{\mathbf{L}}
\overridecommand{\MM}{\mathbf{M}}
\overridecommand{\NN}{\mathbf{N}}
\overridecommand{\OO}{\mathbf{O}}
\overridecommand{\PP}{\mathbf{P}}
\overridecommand{\QQ}{\mathbf{Q}}
\overridecommand{\RR}{\mathbf{R}}
\overridecommand{\SS}{\mathbf{S}}
\overridecommand{\TT}{\mathbf{T}}
\overridecommand{\UU}{\mathbf{U}}
\overridecommand{\VV}{\mathbf{V}}
\overridecommand{\WW}{\mathbf{W}}
\overridecommand{\XX}{\mathbf{X}}
\overridecommand{\YY}{\mathbf{Y}}
\overridecommand{\ZZ}{\mathbf{Z}}

\overridecommand{\aalpha}{\boldsymbol{\alpha}}
\overridecommand{\bbeta}{\boldsymbol{\beta}}
\overridecommand{\ggamma}{\boldsymbol{\gamma}}
\overridecommand{\ddelta}{\boldsymbol{\delta}}
\overridecommand{\eepsilon}{\boldsymbol{\epsilon}}
\overridecommand{\vvarepsilon}{\boldsymbol{\varepsilon}}
\overridecommand{\zzeta}{\boldsymbol{\zeta}}
\overridecommand{\eeta}{\boldsymbol{\eta}}
\overridecommand{\ttheta}{\boldsymbol{\theta}}
\overridecommand{\vvartheta}{\boldsymbol{\vartheta}}
\overridecommand{\iiota}{\boldsymbol{\iota}}
\overridecommand{\kkappa}{\boldsymbol{\kappa}}
\overridecommand{\llambda}{\boldsymbol{\lambda}}
\overridecommand{\mmu}{\boldsymbol{\mu}}
\overridecommand{\nnu}{\boldsymbol{\nu}}
\overridecommand{\xxi}{\boldsymbol{\xi}}
\overridecommand{\ppi}{\boldsymbol{\pi}}
\overridecommand{\vvarpi}{\boldsymbol{\varpi}}
\overridecommand{\rrho}{\boldsymbol{\rho}}
\overridecommand{\vvarrho}{\boldsymbol{\varrho}}
\overridecommand{\ssigma}{\boldsymbol{\sigma}}
\overridecommand{\vvarsigma}{\boldsymbol{\varsigma}}
\overridecommand{\ttau}{\boldsymbol{\tau}}
\overridecommand{\uupsilon}{\boldsymbol{\upsilon}}
\overridecommand{\pphi}{\boldsymbol{\phi}}
\overridecommand{\vvarphi}{\boldsymbol{\varphi}}
\overridecommand{\cchi}{\boldsymbol{\chi}}
\overridecommand{\ppsi}{\boldsymbol{\psi}}
\overridecommand{\oomega}{\boldsymbol{\omega}}
\overridecommand{\GGamma}{\boldsymbol{\Gamma}}
\overridecommand{\DDelta}{\boldsymbol{\Delta}}
\overridecommand{\TTheta}{\boldsymbol{\Theta}}
\overridecommand{\LLambda}{\boldsymbol{\Lambda}}
\overridecommand{\XXi}{\boldsymbol{\Xi}}
\overridecommand{\PPi}{\boldsymbol{\Pi}}
\overridecommand{\SSigma}{\boldsymbol{\Sigma}}
\overridecommand{\UUpsilon}{\boldsymbol{\Upsilon}}
\overridecommand{\PPhi}{\boldsymbol{\Phi}}
\overridecommand{\PPsi}{\boldsymbol{\Psi}}
\overridecommand{\OOmega}{\boldsymbol{\Omega}}

\newcommand{\D}{\mathcal{D}}  % dataset
\newcommand{\E}{\mathbb{E}}   % expectation
\newcommand{\F}{\mathcal{F}}
\newcommand{\I}{\mathcal{I}}
\renewcommand{\L}{\mathcal{L}}   % loss
\newcommand{\M}{\mathcal{M}}

\renewcommand{\P}{\mathcal{P}}
\newcommand{\R}{\mathbb{R}}   % real number
\renewcommand{\S}{\mathcal{S}}
\newcommand{\0}{\mathbf{0}}   % 0 vector
\newcommand{\1}{\mathbf{1}}   % 1 vector

\newcommand{\ind}[1]{\mathbf{1}\{#1\}}

\newcommand{\mis}{\textup{mis}}
\newcommand{\obs}{\textup{obs}}

\newtheorem{theorem}{Theorem}
\newtheorem{proposition}{Proposition}
\newtheorem{corollary}{Corollary}

\newcommand{\misgan}{MisGAN}

\newcommand{\desc}[1]{\textbf{#1}\hspace{.5em}}

\title{MisGAN: Learning from Incomplete Data with Generative Adversarial Networks}
\author{%
  Steven Cheng-Xian Li \\
  University of Massachusetts Amherst \\
  \texttt{cxl@cs.umass.edu}
  \And
  Bo Jiang%
  \\
  Shanghai Jiao Tong University \\
  \texttt{bjiang@sjtu.edu.cn}
  \And
  Benjamin M. Marlin
  \\
  University of Massachusetts Amherst \\
  \texttt{marlin@cs.umass.edu}
}

\iclrfinalcopy % Uncomment for camera-ready version, but NOT for submission.

\begin{document}

\maketitle

\begin{abstract}
  Generative adversarial networks (GANs)
have been shown to provide an effective way to model
complex distributions and have obtained impressive results on various
challenging tasks.
However, typical GANs require fully-observed data during training.
In this paper, we present a GAN-based framework for learning
from complex, high-dimensional incomplete data.
The proposed framework learns a complete
data generator along with a mask generator that models the
missing data distribution.
We further demonstrate how to impute
missing data by equipping our framework with an adversarially
trained imputer.
We evaluate the proposed framework using a series of experiments
with several types of missing data processes under the
missing completely at random assumption.\footnote{
Our implementation is available at \url{https://github.com/steveli/misgan}}

\end{abstract}

\section{Introduction}%
\label{sec:intro}

Generative adversarial networks (GANs) \citep{goodfellow2014generative}
provide a powerful modeling framework for learning complex high-dimensional
distributions.
Unlike likelihood-based methods, GANs are referred to as
implicit probabilistic models \citep{mohamed2016learning}.
They represent a probability distribution through a
generator that learns to directly produce samples from the desired distribution.
The generator is trained adversarially by optimizing a minimax objective
together with a discriminator.
In practice,
GANs have been shown to be very successful in a range of applications including
generating photorealistic images \citep{karras2017progressive}.
Other than generating samples, many downstream tasks
require a good generative model, such as image inpainting
\citep{pathakCVPR16context,yeh2017semantic}.

Training GANs normally requires access to a large collection of fully-observed
data.
However, it is not always possible to obtain a large amount of
fully-observed data.
Missing data is well-known to be prevalent in many real-world application
domains where different data cases might have different missing entries.
This arbitrary missingness poses a significant challenge to
many existing machine learning models.

Following \citet{little2014statistical},
the generative process for incompletely observed
data can be described as shown below where $\xx\in\mathbb{R}^n$
is a complete data vector and $\mm\in\{0,1\}^n$
is a binary mask\footnote{
The complement $\bar{\mm}$ is usually referred to as
the missing data indicator in the literature.}
that determines which entries in $\xx$ to reveal:
\begin{equation}
  \xx\sim p_\theta(\xx), \quad \mm\sim p_\phi(\mm|\xx).
  \label{eq:genprocess}
\end{equation}
Let $\xx_\obs$ denote the observed elements of $\xx$, and
$\xx_\mis$ denote the missing elements according to the mask $\mm$.
In addition, let $\theta$ denote the unknown parameters of the data distribution,
and $\phi$ denote the unknown parameters for the mask distribution, which
are usually assumed to be independent of $\theta$.
In the standard maximum likelihood setting,
the unknown parameters are estimated by
maximizing the following marginal likelihood,
integrating over the unknown missing data values:
\[
  p(\xx_\obs,\mm) =
  \int p_\theta(\xx_\obs,\xx_\mis)p_\phi(\mm|\xx_\obs,\xx_\mis)d\xx_\mis.
\]

\citet{little2014statistical}
characterize the missing data mechanism $p_\phi(\mm|\xx_\obs,\xx_\mis)$
in terms of independence relations between the complete data
$\xx=[\xx_\obs,\xx_\mis]$ and the masks $\mm$:
\begin{itemize}
  \item Missing completely at random (MCAR): $p_\phi(\mm|\xx)=p_\phi(\mm)$,
  \item Missing at random (MAR): $p_\phi(\mm|\xx)=p_\phi(\mm|\xx_\obs)$,
  \item Not missing at random (NMAR):
    $\mm$ depends on $\xx_\mis$ and possibly also $\xx_\obs$.
\end{itemize}
Most work on incomplete data assumes MCAR or MAR
since under these assumptions $p(\xx_\obs,\mm)$ can be factorized into
$p_\theta(\xx_\obs)p_\phi(\mm|\xx_\obs)$.
With such decoupling,
the missing data mechanism can be ignored when learning the
data generating model while yielding correct estimates for $\theta$.
When $p_\theta(\xx)$ does not admit efficient marginalization over $\xx_\mis$,
estimation of $\theta$ is usually performed by maximizing a variational lower bound,
as shown below, using the EM algorithm or a more general approach
\citep{little2014statistical,ghahramani1994supervised}:
\begin{equation}
  \log p_\theta(\xx_\obs)
  \ge
  \E_{q(\xx_\mis|\xx_\obs)}
  \left[
    \log p_\theta(\xx_\obs, \xx_\mis) - \log q(\xx_\mis|\xx_\obs)
  \right].
  \label{eq:elbo}
\end{equation}
The primary contribution of this paper is the development of
a GAN-based framework for learning
high-dimensional data distributions in the presence of incomplete observations.
Our framework introduces an auxiliary GAN for learning a mask distribution
to model the missingness.
The masks are used to ``mask'' generated
complete data by filling the indicated missing entries with a constant value.
The complete data generator is trained so that the resulting masked data
are indistinguishable from real incomplete data that are masked similarly.

Our framework builds on the ideas of AmbientGAN \citep{bora2018ambientgan}.
AmbientGAN modifies the discriminator of a GAN to distinguish
corrupted real samples from corrupted generated samples under a range of
corruption processes (or measurement processes).
For images, examples of the measurement processes include
random dropout, blur, block-patch, and so on.
Missing data can be seen as a special type of corruption,
except that we have access to the missing pattern in addition to the corrupted
measurements.
Moreover,
AmbientGAN assumes the measurement process is known or parameterized only
by a few parameters, which is not the case in general missing data problems.

We provide empirical evidence that the proposed framework is able to
effectively learn complex, high-dimensional data distributions
from highly incomplete data when the GAN generator
incorporates suitable priors on the data generating process.
We further show how the architecture can be used to generate high-quality
imputations.

\section{{\misgan}: A GAN for missing data}%
\label{sec:model}

%In this section, we describe our framework named {\misgan} for learning
%the generative process from incomplete data.
In the missing data problem, we know exactly which entries in
each data examples are missing.
Therefore, we can represent an incomplete data case as a pair of
a partially-observed data vector $\xx\in\R^n$ and
a corresponding mask $\mm\in\{0,1\}^n$
that indicates which entries in $\xx$ are observed:
$x_d$ is observed if $m_d=1$ otherwise $x_d$ is missing and might
contain an arbitrary value that we should ignore.
With this representation,
an incomplete dataset is denoted
$\D=\{(\xx_i,\mm_i)\}_{i=1,\dots,N}$
(we assume instances are i.i.d.~samples).
We choose this representation instead of $\xx_\obs$ because
it leads to a cleaner description of the proposed {\misgan} framework.
It also suggests how {\misgan} can be implemented efficiently
in practice as both $\xx$ and $\mm$ are fixed-length vectors.

We begin by defining a masking operator $f_\tau$
that fills in missing entries with a
constant value $\tau$:
\begin{equation}
  f_\tau(\xx, \mm) = \xx \odot \mm + \tau\bar{\mm},
  \label{eq:patch}
\end{equation}
where $\bar{\mm}$ denotes the complement of $\mm$
and $\odot$ denotes element-wise multiplication.

\begin{figure}[t]
  \centering
  \includegraphics[width=3.6in]{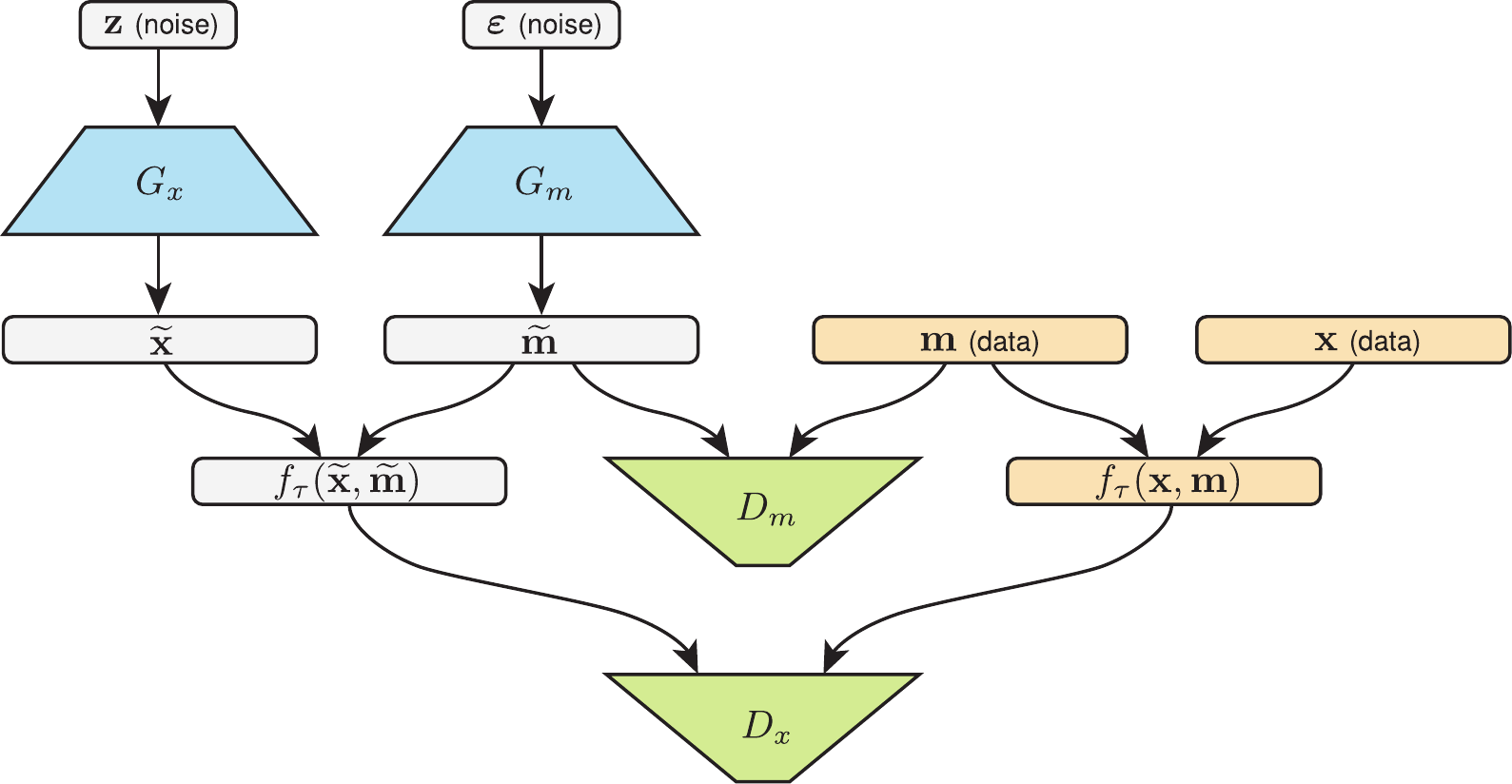}
  \caption{Overall structure of the {\misgan} framework}%
  \label{fig:archmisgan}
\end{figure}

Two key ideas underlie the {\misgan} framework.
First,
in addition to the complete data generator,
we explicitly model the missing data process using a mask generator.
Since the masks in the incomplete dataset are fully observed,
we can estimate their distribution using a standard GAN.
Second, we train the complete data generator adversarially
by masking its outputs using generated masks and $f_\tau$ and comparing
to real incomplete data that are similarly masked by $f_\tau$.

Specifically, we use two generator-discriminator pairs
$(G_m, D_m)$ and $(G_x, D_x)$
for the masks and data respectively.
In this paper, we focus on the missing completely
at random (MCAR) case, where the two generators are independent of each other
and have their own noise distributions $p_z$ and $p_\varepsilon$.
We define the following two loss functions,
one for the masks and the other for the data:
\begin{align}
  \L_m(D_m, G_m) &=
  \E_{(\xx,\mm)\sim p_\D}\left[D_m(\mm)\right] -
  \E_{\vvarepsilon\sim p_\varepsilon}\left[D_m(G_m(\vvarepsilon))\right],
  \label{eq:lossmask}
  \\
  \L_x(D_x, G_x, G_m) &=
  \E_{(\xx,\mm)\sim p_\D}\left[D_x(f_\tau(\xx,\mm))\right] -
  \E_{\vvarepsilon\sim p_\varepsilon, \zz\sim p_z}\left[
  D_x\left(f_\tau\left(G_x(\zz),G_m(\vvarepsilon)\right)\right)\right].
  \label{eq:lossdata}
\end{align}
%\begin{align}
%  \L_m(D_m, G_m) &=
%  \E_{(\xx,\mm)\sim p_\D}\left[\log D_m(\mm)\right] -
%  \E_{\vvarepsilon\sim p_\varepsilon}\left[
%  \log\left(1 - D_m(G_m(\vvarepsilon))\right)\right],
%  \label{eq:lossmask}
%  \\
%  \L_x(D_x, G_x, G_m) &=
%  \E_{(\xx,\mm)\sim p_\D}\left[\log D_x(f_\tau(\xx,\mm))\right] \notag\\
%  &\qquad
%  -
%  \E_{\vvarepsilon\sim p_\varepsilon, \zz\sim p_z}\left[
%  \log\left(1 - D_x\left(f_\tau\left(G_x(\zz),G_m(\vvarepsilon)\right)\right)
%  \right)\right].
%  \label{eq:lossdata}
%\end{align}
The losses above follow the
Wasserstein GAN formulation \citep{arjovsky2017wasserstein},
although the proposed framework is compatible with many GAN variations
\citep{goodfellow2014generative,berthelot2017began,
gulrajani2017improved}.
We optimize the generators and the discriminators
according to the following objectives:
%\begin{equation}
%  \min_{G_m, G_x}\max_{D_m\in\F_m, D_x\in\F_x} \L_m(D_m, G_m)
%  + \alpha \L_x(D_x, G_x, G_m)
%\end{equation}
\begin{align}
  &\min_{G_x}\max_{D_x\in\F_x} \L_x(D_x, G_x, G_m),
  \label{eq:objdata}
  \\
  &\min_{G_m}\max_{D_m\in\F_m} \L_m(D_m, G_m) + \alpha \L_x(D_x, G_x, G_m),
  \label{eq:objmisgan}
\end{align}
where $\F_x, \F_m$ are defined such that $D_x, D_m$ are both 1-Lipschitz
for Wasserstein GANs \citep{arjovsky2017wasserstein}.
Practically, we follow the common practice of
alternating between a few steps of optimizing the discriminators
and one step of optimizing the generators~\citep{%
goodfellow2014generative,arjovsky2017wasserstein,gulrajani2017improved}.
The coefficient $\alpha$ is introduced when optimizing
the mask generator $G_m$ with the aim of minimizing
a combination of $\L_m$ and $\L_x$.
Although in theory we could choose $\alpha=0$ to train
$G_m$ and $D_m$ without using the data,
we find that choosing a small value such as $\alpha=0.2$
improves performance.
This encourages the generated masks to match the distribution
of the real masks and the masked generated complete samples
to match masked real data.
The overall structure of {\misgan} is illustrated in
Figure~\ref{fig:archmisgan}.

Note that the data discriminator $D_x$ takes as input the masked
samples as if the data are fully-observed.
This allows us to use any existing architecture designed for complete data
to construct the data discriminator.
There is no need to develop
customized neural network modules for dealing with missing data.
For example, $D_x$ can be a standard convolutional network
for image applications.

Note that the masks are binary-valued.
Since discrete data generating processes have zero gradient almost everywhere,
to carry out gradient-based training for GANs, we relax the output
of the mask generator $G_m$ from $\{0,1\}^n$ to $[0, 1]^n$.
%Similar to the gating mechanism used in LSTMs \citep{hochreiter1997long},
We employ a sigmoid activation
$\sigma_\lambda(x) = 1 / (1 + \exp(-x / \lambda))$
with a low temperature $0 < \lambda < 1$ to encourage saturation
and make the output closer to zero or one.
% Alternatively, we can also use the hard sigmoid activation
% $\tilde{\sigma}_\lambda(x) = \max(0, \min(1, x/\lambda))$ to
% saturate the output all the way to exactly zero or one
% \citep{hubara2016binarized}.
% \citet{louizos2017learning} proposed to combine the two and use
% $\hat{\sigma}_\lambda(x) = \max(0, \min(1, \sigma_\lambda(x)(b-a) + a))$,
% which stretches the sigmoid output to the $(a,b)$ interval, with $a < 0$
% and $b > 1$, followed by the hard sigmoid.
% Empirically,
% we found that $\sigma_\lambda$ is marginally worse than
% $\hat{\sigma}_\lambda$,
% but better than $\tilde{\sigma}_\lambda$.

Finally, we note that
the discriminator $D_x$ in {\misgan} is unaware of which entries
are missing in the masked input samples, and does not
even need to know which value $\tau$ is used for masking.
In the next section, we present a theoretical analysis
providing support for the idea that this type of masking process does not necessarily
make it more difficult to recover the complete data distribution.
The experiments provide compelling empirical evidence for the
effectiveness of the proposed framework.

\section{Theoretical results}
\label{sec:theory}

In Section~\ref{sec:model}
we described how the discriminator $D_x$ in {\misgan}
takes as input the masked samples using \eqref{eq:patch}
without knowing what value $\tau$ is used
or which entries in the input vector are missing.
In this section, we discuss the following two important
questions:
\begin{enumerate*}[label={\roman*)}]
%\begin{itemize}
  \item
    Does the choice of the filled-in value $\tau$
    affect the ability to recover the data distribution?
    %How should we choose the filled-in value $\tau$?
  \item Does information about the location of missing values
    affect the ability to recover the data distribution?
%\end{itemize}
\end{enumerate*}

We address these questions in a simplified scenario
where each dimension of the data vector takes values from a finite set $\P$.
For $n$-dimensional data, let $\M=\{0,1\}^n$ be the set of
all possible masks and $\I=\P^n$ be the set of all possible data vectors.
Also let $\D_\M$ and $\D_\I$ be the set of all possible probability
distributions on $\M$ and $\I$ respectively,
whose elements are non-negative and sum to one.
We first discuss the case where the filled-in value $\tau$ is chosen from $\P$.
%as it is the more convenient case in practice.

Given $\tau\in\P$ and $\qq\in\D_\M$,
we can construct a left transition matrix
$T_{\qq,\tau}\in\R^{\I\times\I}$ defined below where
the $(\tt,\ss)$-th entry specifies the transition probability from
a data vector $\ss\in\I$ to
an outcome $\tt\in\I$ masked by $f_\tau$,
which involves all possible masks under which $\ss$
is converted into $\tt$ by filling in the indicated missing entries
with $\tau$:
\[
  T_{\qq,\tau}(\tt,\ss)
  = \sum_{\mm\in\M: f_\tau(\ss,\mm)=\tt}\qq(\mm).
\]

Let $\pp_x^*\in\D_\I$ be the unknown true data distribution we want to
estimate.
In the presence of missing data specified by $\qq$, the masked samples
then follow the distribution $\pp_y = T_{\qq,\tau}\pp_x^*$.
%Without making use of any prior knowledge about the underlying application,
Without imposing extra application-specific constraints,
{\misgan} with a fixed mask generator
%$\alpha=0$ in~\eqref{eq:objmisgan}
can be viewed as solving the linear system
$\pp_y = T_{\qq,\tau}\pp_x$, where $\pp_x\in\D_\I$
is the unknown data distribution to solve for.
Here we assume that $\pp_y$ and $T_{\qq,\tau}$ are given, as those can be
estimated separately from a collection of fully-observed
masks and masked samples.

Note that a  transition matrix preserves the sum of the vectors it is applied to
since $\1^\top T_{\qq,\tau}=\1^\top$. For $\pp_x$ to be a valid
distribution vector, we only need the non-negativity constraint
because any solution $\pp_x$ automatically sums to one.
That is, estimating the data generating process
in the presence of missing data
based on the masking scheme used in {\misgan} %with $\tau\in\P$
is equivalent to
solving the linear system
\begin{equation}
  T_{\qq,\tau}\pp_x = \pp_y\ \text{ subject to }\pp_x\succeq\0.
\label{eq:linsys}
\end{equation}

In Theorem \ref{thm:nullspace}, we state a key property of the transition matrix
$T_{\qq,\tau}$ that leads to the answer to our questions.
The proof of Theorem \ref{thm:nullspace} is in Appendix~\ref{sec:proof}.
\begin{theorem}%
  \label{thm:nullspace}
  Given $\qq\in\D_\M$,
  all transition matrices
  $T_{\qq,\tau}$ with $\tau\in\P$ have the same null space.
\end{theorem}

Theorem~\ref{thm:nullspace} implies that
if the solution to the constrained linear system~\eqref{eq:linsys} is not
unique for a given $\tau_0\in\P$, that is,
there exists some non-negative $\pp_x\ne\pp_x^*$ such that
$T_{\qq,\tau_0}\pp_x = T_{\qq,\tau_0}\pp_x^*$,
then we must have
$T_{\qq,\tau}\pp_x = T_{\qq,\tau}\pp_x^*$ for all $\tau\in\P$.
In other words, we have the following corollary:
%that suggests we could choose an arbitrary value for $\tau\in\P$ in {\misgan}.
\begin{corollary}%
  Whether the true data distribution is uniquely recoverable
  is independent of the choice of the filled-in value $\tau$.
\end{corollary}

Here we only discuss the case when the probability of observing all features
$\qq(\1)$ is zero, where $\qq(\1)$ denotes the scalar entry of $\qq$
indexed by $\1\in\M$.
Otherwise, the linear system is uniquely solvable
as the transition matrix $T_{\qq,\tau_0}$ has full rank.
With the non-negativity constraint,
it is possible that the solution for the linear system~\eqref{eq:linsys}
is unique when the true data distribution $\pp_x^*$ is sparse.
Specifically,
if there exists two indices $\ss_1,\ss_2\in\I$ such that
$\pp_x^*(\ss_1) = \pp_x^*(\ss_2) = 0$
and also
$\vv(\ss_1) > 0$ and $\vv(\ss_2) < 0$ for all
$\vv\in\text{Null}(T_{\qq,\tau})\setminus\{\0\}$,
then the solution to~\eqref{eq:linsys} is unique.
%where $\vv(\ss)$ denotes the scalar entry in $\vv$ indexed by $\ss$.

Sparsity of the data distribution is a reasonable assumption in many
situations.
For example, natural images are typically considered to lie on a
low dimensional manifold, which means most of the instances in $\I$
should have almost zero probability.
On the other hand, when the missing rate is high, that is,
if the masks in $\M$ that have many zeros are more probable,
the null space of $T_{\qq,\tau}$
will be larger and therefore it is more likely that the
non-negative solution is not unique.
\citet{bruckstein2008uniqueness} proposed a sufficient condition
on the sparsity of the non-negative solutions to a general
underdetermined linear system that guarantees unique optimality.
%However, the analysis does not take the structure of the transition matrix
%into account and therefore the bound on the sparsity is too loose to be
%useful in our case.

Next we note that in the case of $\tau\in\P$,
an entry with value $\tau$ in a masked sample $\tt\in\I$
may come either from an observed entry with value $\tau$ in the unmasked sample
or from an unobserved entry through the masking operation in \eqref{eq:patch}.
One might wonder if this prevents an algorithm from recovering the
true distribution when it is otherwise possible to do so. In other words,
if we take the location of the missing values into account, would
that make the missing data problem less ill-posed?
However, this is not the case, as we state in Corollary~\ref{cor:newtau}.
The proof is in Appendix~\ref{sec:detailedcor}
where we discuss the case of $\tau\notin\P$.
\begin{corollary}%
  \label{cor:newtau}
  %Without any inductive bias,
  %if an algorithm that is unaware of
  %where the missing entries are cannot uniquely recover the true distribution,
  %any algorithm that takes the location of the missing values
  %into account cannot uniquely recover the true distribution either.
  If the linear system $T_{\qq,\tau}\pp_x = T_{\qq,\tau}\pp_x^*$
  does not have a unique non-negative solution,
  then for this missing data problem,
  we cannot uniquely recover the true data distribution even if
  we take the location of the missing values into account.
\end{corollary}

Note that the analysis in this section
characterizes how difficult the missing data
problem is, which is independent of the choice of the algorithm
that solves it.
In practice,
it is useful to incorporate application-specific prior knowledge
into the model to regularize the problem when it is ill-posed.
For example, for modeling natural images,
convolutional networks are commonly used
to exploit the local structure of the data.
%It may be useful for selecting the solution that is closer to the
%true distribution from a large set of solutions to the
%given linear system if the prior makes sense.
In addition, decoder-based deep generative models such as GANs
implicitly enforce some
sparsity constraints due to the use of low dimensional latent codes
in the generator, which also helps to regularize the problem.

Finally, the following theorem justifies the training objective
\eqref{eq:objdata} of {\misgan} for the missing data problem
(see Appendix~\ref{sec:proof} for details).
\begin{theorem}
  \label{thm:marginals}
  Given a mask distribution $p_\phi(\mm)$,
  two distributions $p_\theta(\xx)$ and $p_{\theta'}(\xx)$
  induce the same distribution
  for $f_\tau(\xx,\mm)$
  %under the missing data generating process \eqref{eq:genprocess}
  if and only if they have the same marginals
  $p_\theta(\xx_\obs | \mm) = p_{\theta'}(\xx_\obs | \mm)$
  for all masks $\mm$ with $p_\phi(\mm) > 0$.\footnote{
    $p_\theta(\xx_\obs | \mm)$ is technically equivalent to $p_\theta(\xx_\obs)$
    as the random variable $\xx_\obs = \{x_d : m_d = 1\}$
    is defined with a known mask $\mm$.
  }
\end{theorem}

\section{Missing data imputation}
\label{sec:impute}
Missing data imputation is an important task when dealing with
incomplete data.
In this section,
we show how to impute missing data according to $p(\xx_\mis|\xx_\obs)$
by equipping {\misgan} with an imputer $G_i$ accompanied by
a corresponding discriminator $D_i$.
The imputer is a function of the incomplete example $(\xx,\mm)$ and
a random vector $\oomega$ drawn from a noise distribution $p_\omega$.
It outputs the completed sample with the observed part in $\xx$ kept
intact.

To train the imputer-equipped {\misgan},
we define the loss for the imputer in addition to \eqref{eq:lossmask} and
\eqref{eq:lossdata}:
\[
  \L_i(D_i, G_i, G_x) =
  \E_{\zz\sim p_z}\left[D_i(G_x(\zz))\right] -
  \E_{(\xx,\mm)\sim p_\D, \oomega\sim p_\omega}
  \left[D_i(G_i(\xx,\mm,\oomega))\right].
  \label{eq:lossimputer}
\]
We jointly learn the data generating process and
the imputer according to the following objectives:
\begin{align}
\min_{G_i}\max_{D_i\in\F_i}\ &\L_i(D_i, G_i, G_x),
\label{eq:objimputer}\\
\min_{G_x}\max_{D_x\in\F_x}\ &\L_x(D_x, G_x, G_m) + \beta \L_i(D_i, G_i, G_x),
\\
\min_{G_m}\max_{D_m\in\F_m}\ &\L_m(D_m, G_m) + \alpha \L_x(D_x, G_x, G_m),
\notag
\label{eq:objm}
\end{align}
where we use $\beta=0.1$ in the experiments when optimizing $G_x$.
This encourages the generated complete data to match the distribution
of the imputed real data in addition to
having the masked generated data match the masked real data.
The overall structure for {\misgan} imputation is illustrated in
Figure~\ref{fig:archimputer}.

\begin{figure}
  \centering
  \includegraphics[width=4.68in]{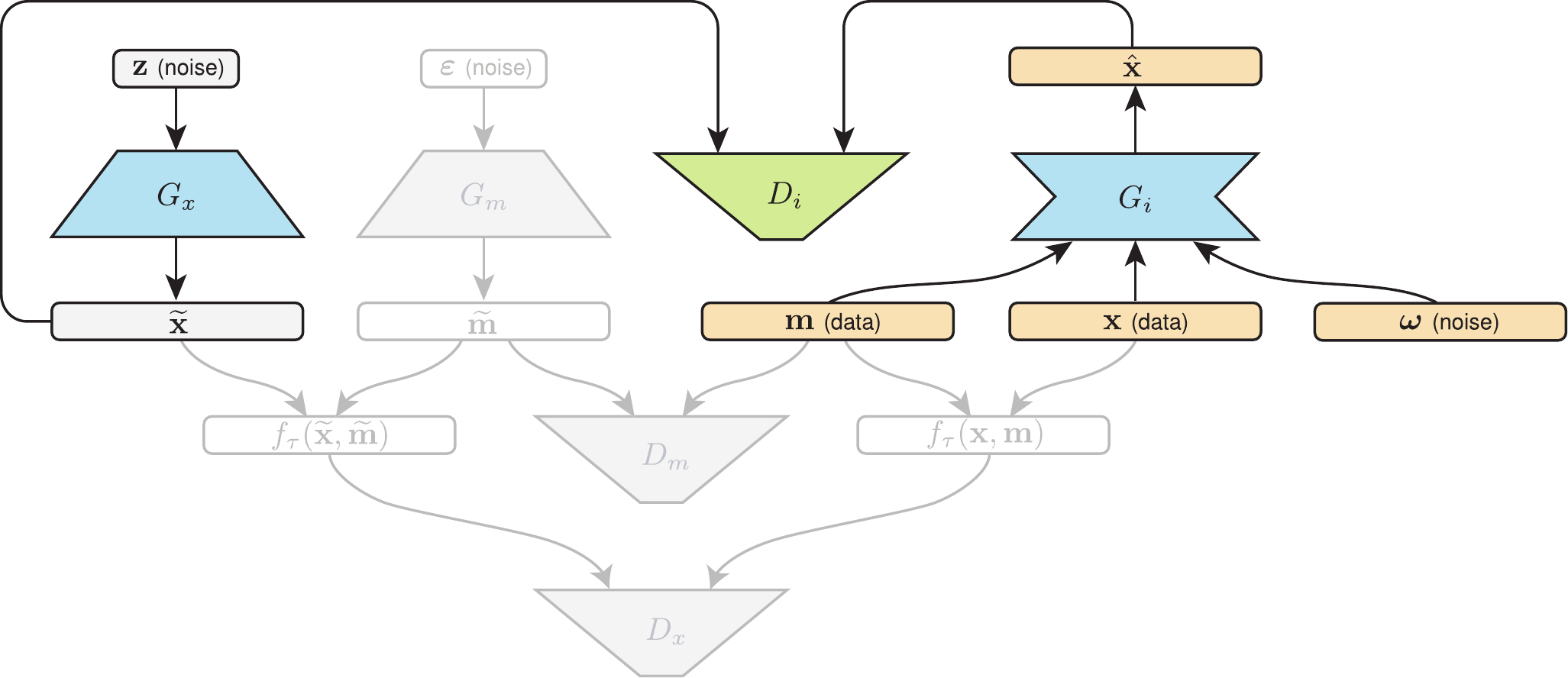}
  \caption{Architecture for {\misgan} imputation.
    The complete data generator $G_x$ and the imputer $G_i$ can be trained
    jointly with all the components.
    We can also independently train the imputer $G_i$
    without the faded parts
    if the data generator $G_x$ has been pre-trained.
  }
  \label{fig:archimputer}
\end{figure}

We can also train a stand-alone imputer using only \eqref{eq:objimputer}
with a pre-trained data generator $G_x$. The architecture is as shown in
Figure~\ref{fig:archimputer} with the faded parts removed.
Moreover, it is also possible to train the imputer to target a different
missing distribution $p_m$ with a pre-trained data generator $G_x$ alone
without access to the original (incomplete) training data:
\begin{equation}
  \min_{G_i}\max_{D_i\in\F_i}
  \E_{\zz\sim p_z}\left[D_i(G_x(\zz))\right] -
  \E_{\mm\sim p_m, \zz\sim p_z, \oomega\sim p_\omega}
  \left[D_i(G_i(G_x(\zz),\mm,\oomega))\right].
  \label{eq:imputeronly}
\end{equation}
We construct the imputer $G_i(\xx,\mm,\oomega)$ as follows:
\begin{equation}
G_i(\xx,\mm,\oomega)
=\xx\odot\mm+\widehat{G}_i(\xx\odot\mm + \oomega\odot\bar{\mm})\odot\bar{\mm},
\label{eq:imputer}
\end{equation}
where $\widehat{G}_i$ generates the imputed result
with the same dimensionality as its input,
$\xx\odot\mm + \oomega\odot\bar{\mm}$, which could be
implemented by a deep neural network.
The masking outside of $\widehat{G}_i$
ensures that the observed part of $\xx$ stays the same
in the output of the imputer $G_i$.
The similar masking on the input of $\widehat{G}_i$,
$\xx\odot\mm + \oomega\odot\bar{\mm}$,
ensures that the amount of noise injected to $\widehat{G}_i$
scales with the number of missing dimensions.
This is intuitive in the sense that when a data case is
almost fully-observed, we expect less variety in $p(\xx_\mis|\xx_\obs)$
and vice versa.
Note that the noise $\oomega$ needs to have the same dimensionality as $\xx$.

\section{Experiments}%
\label{sec:experiments}

In this section, we first assess various properties of {\misgan}
on the MNIST dataset:
we demonstrate qualitatively how {\misgan} behaves under different missing
patterns and different architectures.
We then conduct an ablation study to justify the construction of {\misgan}.
Finally, we compare {\misgan} with various baseline methods on the
missing data imputation task over three datasets under a series of
missingness settings.

\desc{Data}
We evaluate {\misgan} on three datasets:
MNIST, CIFAR-10 and CelebA.
MNIST is a dataset of %single-channel
handwritten digits images of size 28$\times$28
\citep{lecun1998gradient}. We use the provided 60,000 training examples
for the experiments.
CIFAR-10 is a dataset of 32$\times$32 color images from 10 classes
\citep{krizhevsky2009learning}.
Similarly, we use 50,000 training examples for the experiments.
CelebA is a large-scale face attributes dataset
\citep{liu2015faceattributes} that
contains 202,599 face images, where we use the provided
aligned and cropped images and resize them to 64$\times$64.
For all three datasets, the range of pixel values of each
image is rescaled to $[0,1]$.

\desc{Missing data distributions}
We consider three types of missing data distribution:
\begin{enumerate*}[label={\roman*)}]
  \item \emph{Square observation}:
    all pixels are missing except for a square occurring at a random
    location on the image.
  \item \emph{Dropout}:
    each pixel is independently missing according to a Bernoulli distribution.
  \item \emph{Variable-size rectangular observation}:
    all pixels are missing except for a rectangular observed region.
    The width and height of the rectangle are independently drawn from
    25\% to 75\% of the image length uniformly at random,
    which results in a 75\% missing rate on average.
    In this missing data distribution,
    each example may have a different number of missing pixels.
    The highest per-example missing data rate under this mechanism is 93.75\%.
\end{enumerate*}

\desc{Evaluation metric}
We use the Fr\'echet Inception Distance (FID) \citep{heusel2017gans}
to evaluate the quality of the learned generative model.
For MNIST,
instead of the Inception network trained on ImageNet
\citep{salimans2016improved}, we use a basic LeNet
model \footnote{%
\url{https://github.com/pytorch/examples/tree/master/mnist}}
trained on the complete MNIST training set, and then
take the 50-dimensional output from the second-to-last
fully-connected layer as the features to compute the FID.
For CIFAR-10 and CelebA, we follow the procedure described in
\citet{heusel2017gans} to compute the FID using
the pretrained Inception-v3 model.
When evaluating generative models using the FID, we use the same
number of generated samples as the size of the training set.

\subsection{Empirical study of {\misgan} on MNIST}
\label{sec:mnist}

In this section, we study various properties of {\misgan}
using the MNIST dataset.

\desc{Architectures}
We consider two kinds of architecture for {\misgan}:
convolutional networks and fully connected networks.
We follow the DCGAN architecture~\citep{radford2015unsupervised}
for (de)convolutional
generators and discriminators to exploit
the local structures of images. We call this model Conv-{\misgan}.

To demonstrate the performance of {\misgan} in the absence
of the implicit structural regularization provided by the
use of a convolutional network, we construct another {\misgan} with
only fully-connected layers for both the generators and
the discriminators, which we call FC-{\misgan}.

In the experiments, both Conv-{\misgan} and FC-{\misgan} are trained
using the improved procedure for the Wasserstein GAN with gradient penalty
\citep{gulrajani2017improved}.
Throughout we use $\tau=0$ for the masking operator and
the temperature $\lambda=0.66$ for the mask activation $\sigma_\lambda(x)$
described in Section~\ref{sec:model}.

\desc{Baseline}
We compare {\misgan} to a baseline model
that is capable of learning from large-scale incomplete data:
the generative convolutional arithmetic circuit (ConvAC)
\citep{sharir2016tractable}.
ConvAC is an expressive mixture model similar to sum-product networks
\citep{Poon2011} with
a compositional structure similar to deep convolutional networks.
Most importantly, ConvAC admits tractable marginalization
due to the product form of the base distributions for the mixtures,
which makes it readily capable of learning with missing data.

\def\figwidth{1.3in}
\begin{figure}
  \centering
  \begin{subfigure}[b]{0.24\textwidth}
    \centering
    \includegraphics[width=\figwidth]{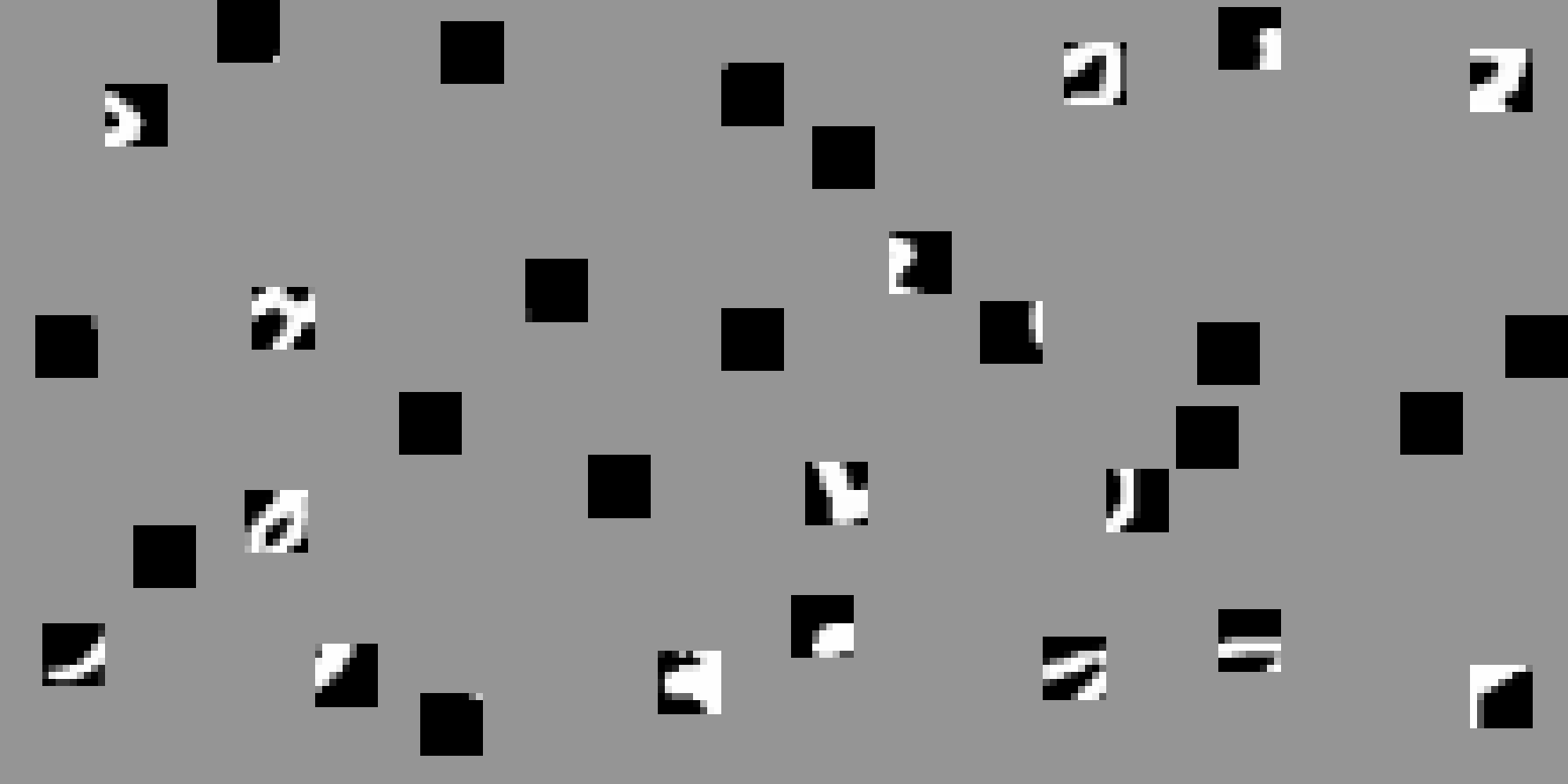} \\
    \includegraphics[width=\figwidth]{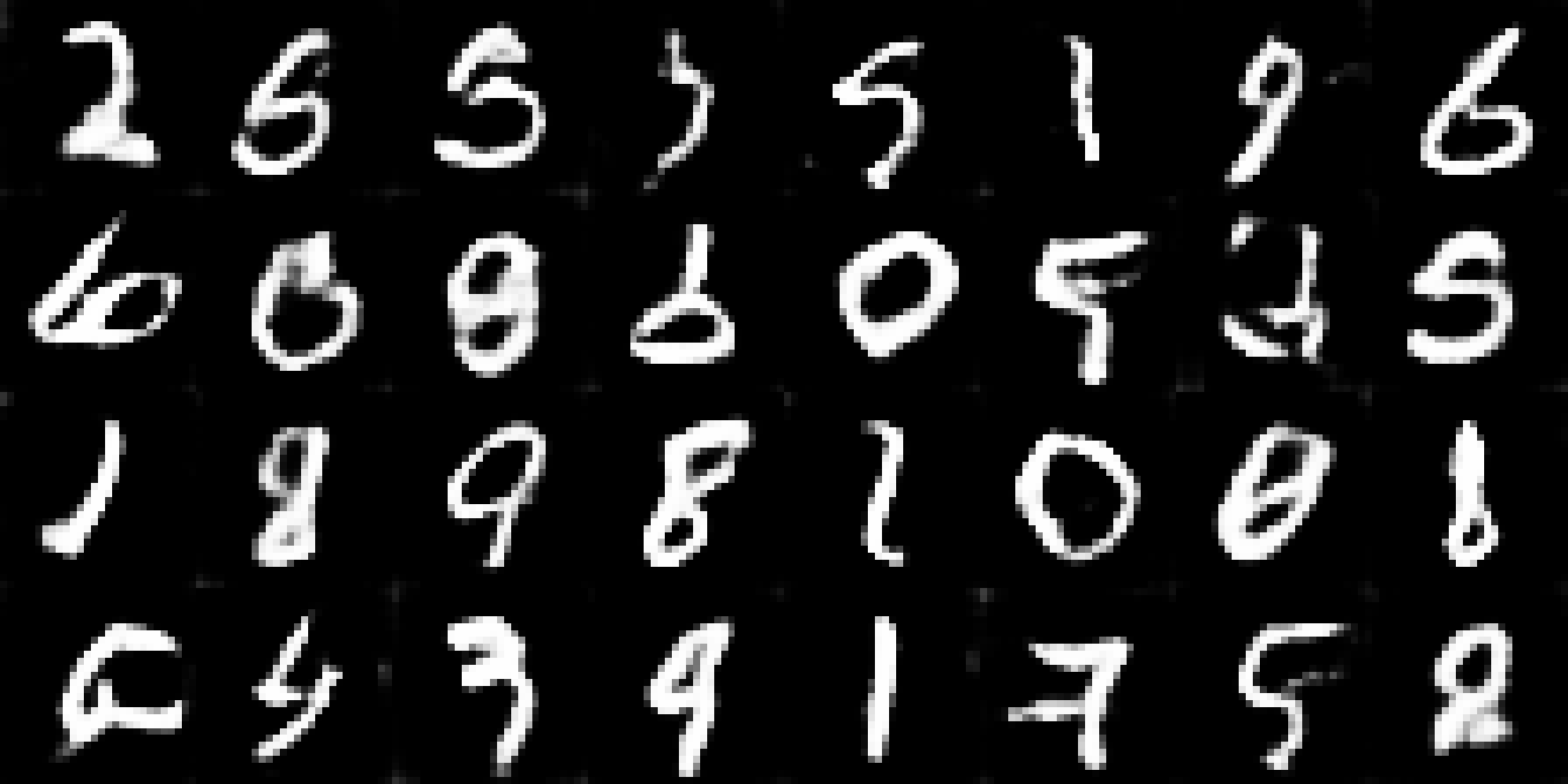} \\
    \includegraphics[width=\figwidth]{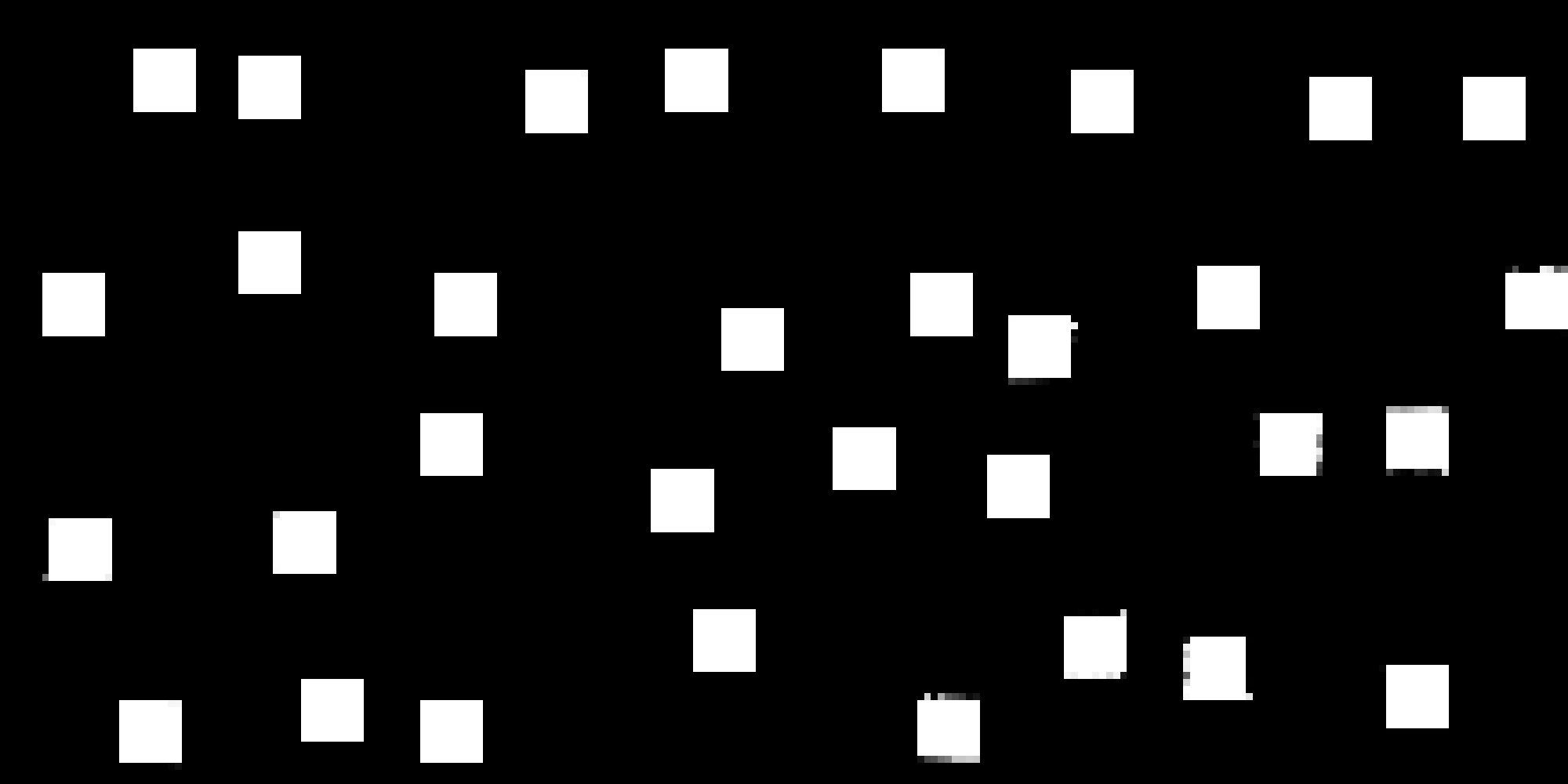}
    \caption{9$\times$9 (90\% missing)}
  \end{subfigure}
  \begin{subfigure}[b]{0.24\textwidth}
    \centering
    \includegraphics[width=\figwidth]{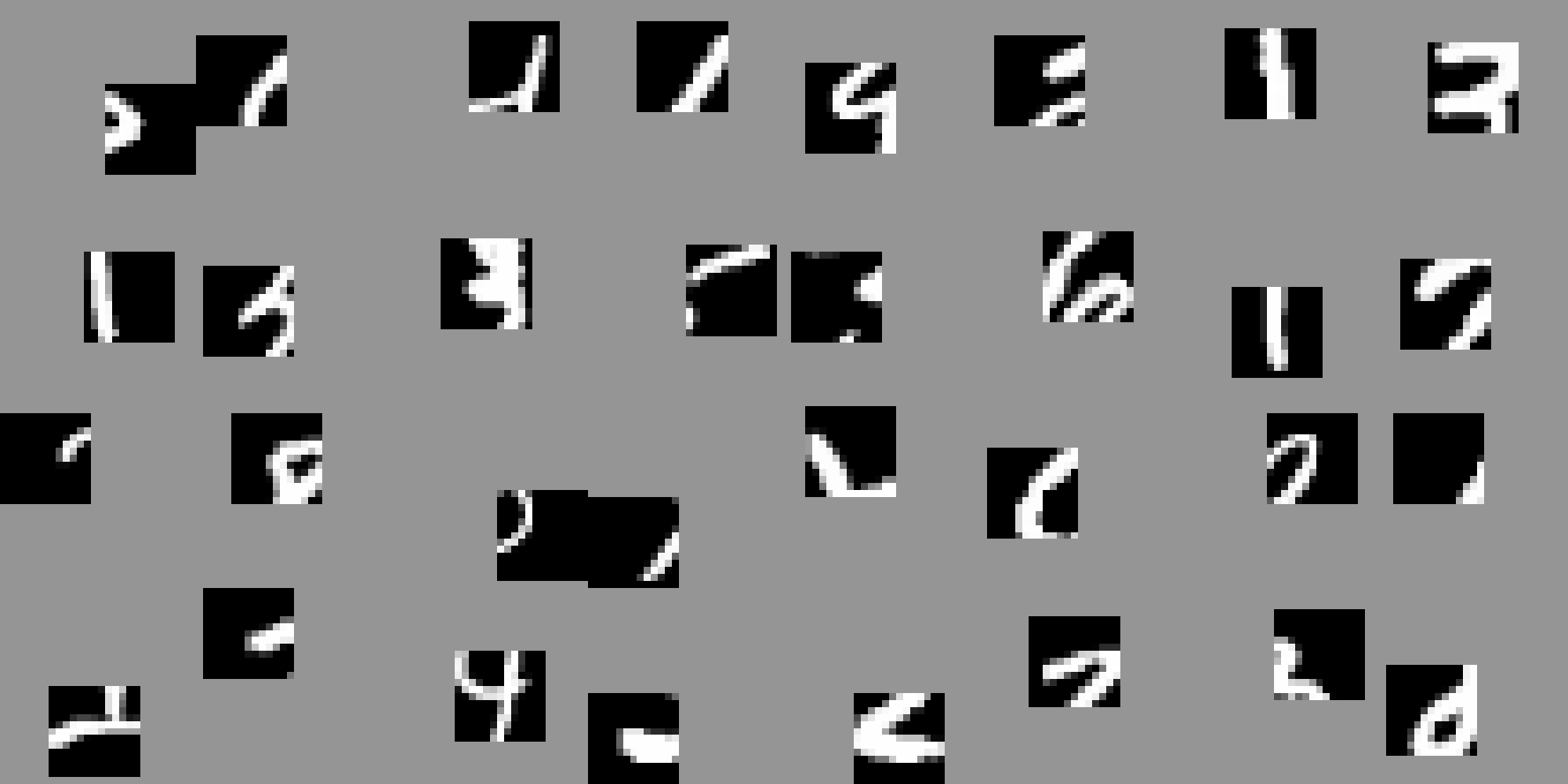} \\
    \includegraphics[width=\figwidth]{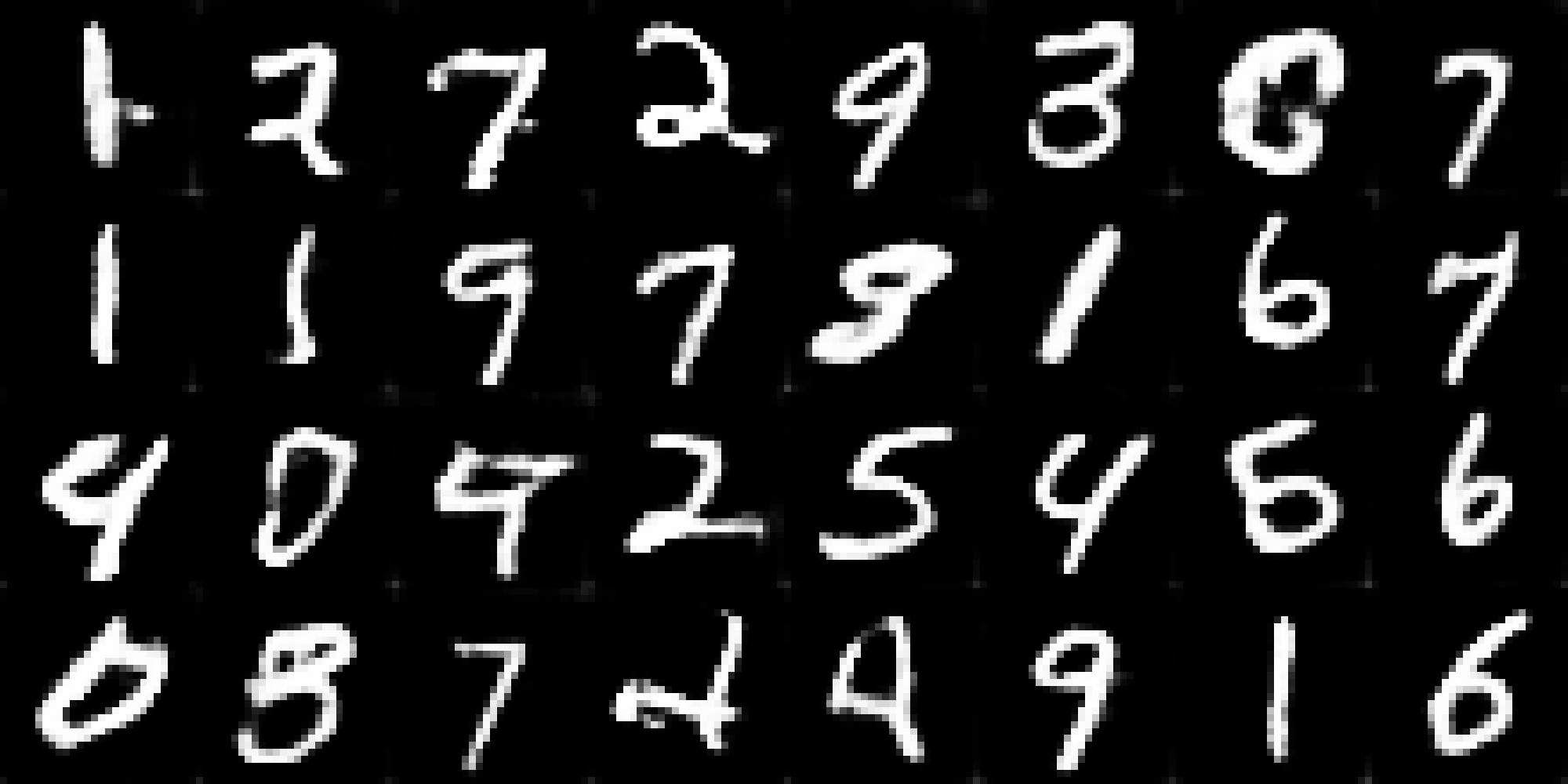} \\
    \includegraphics[width=\figwidth]{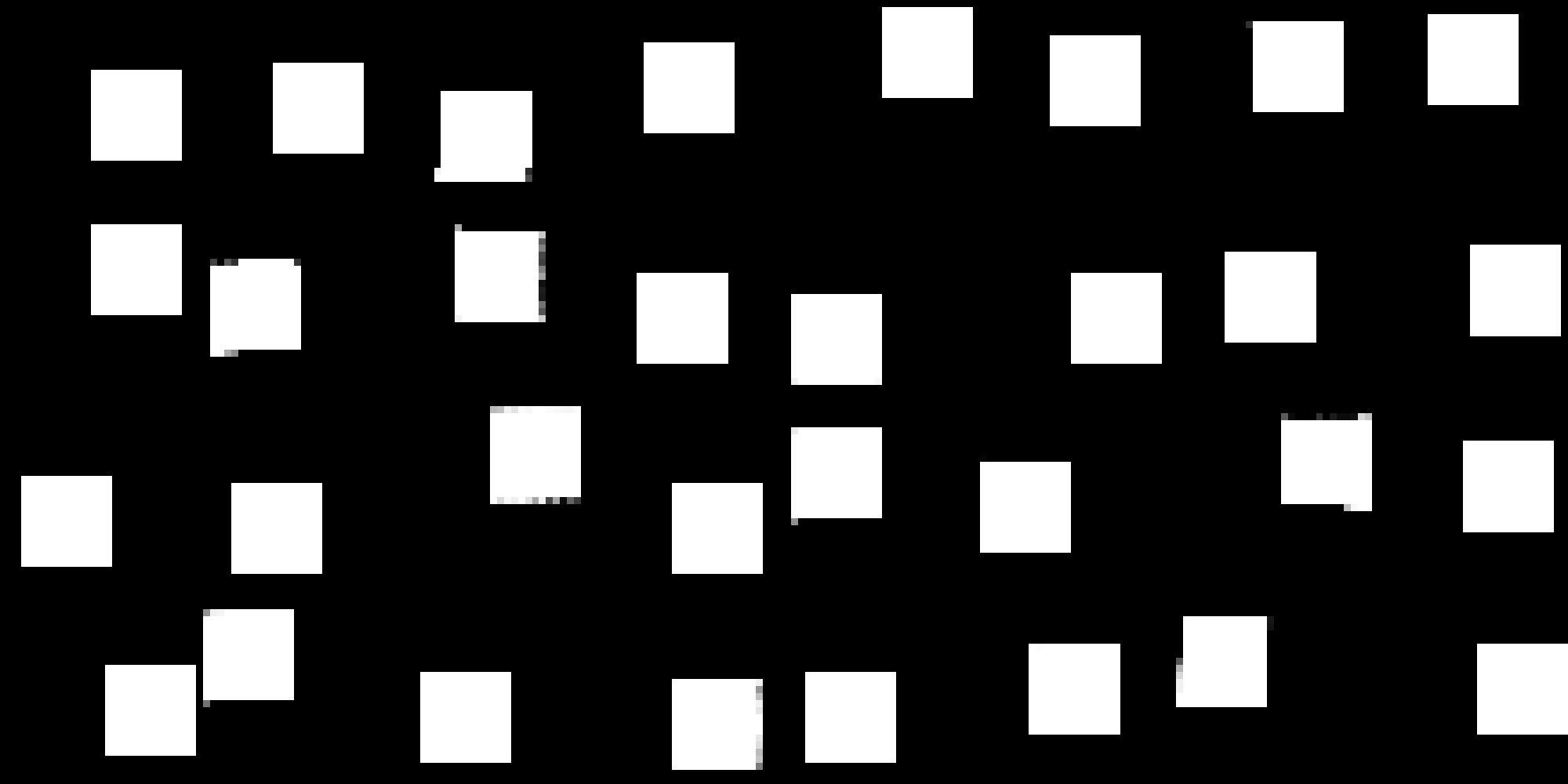}
    \caption{13$\times$13 (80\% missing)}
  \end{subfigure}
  \begin{subfigure}[b]{0.24\textwidth}
    \centering
    \includegraphics[width=\figwidth]{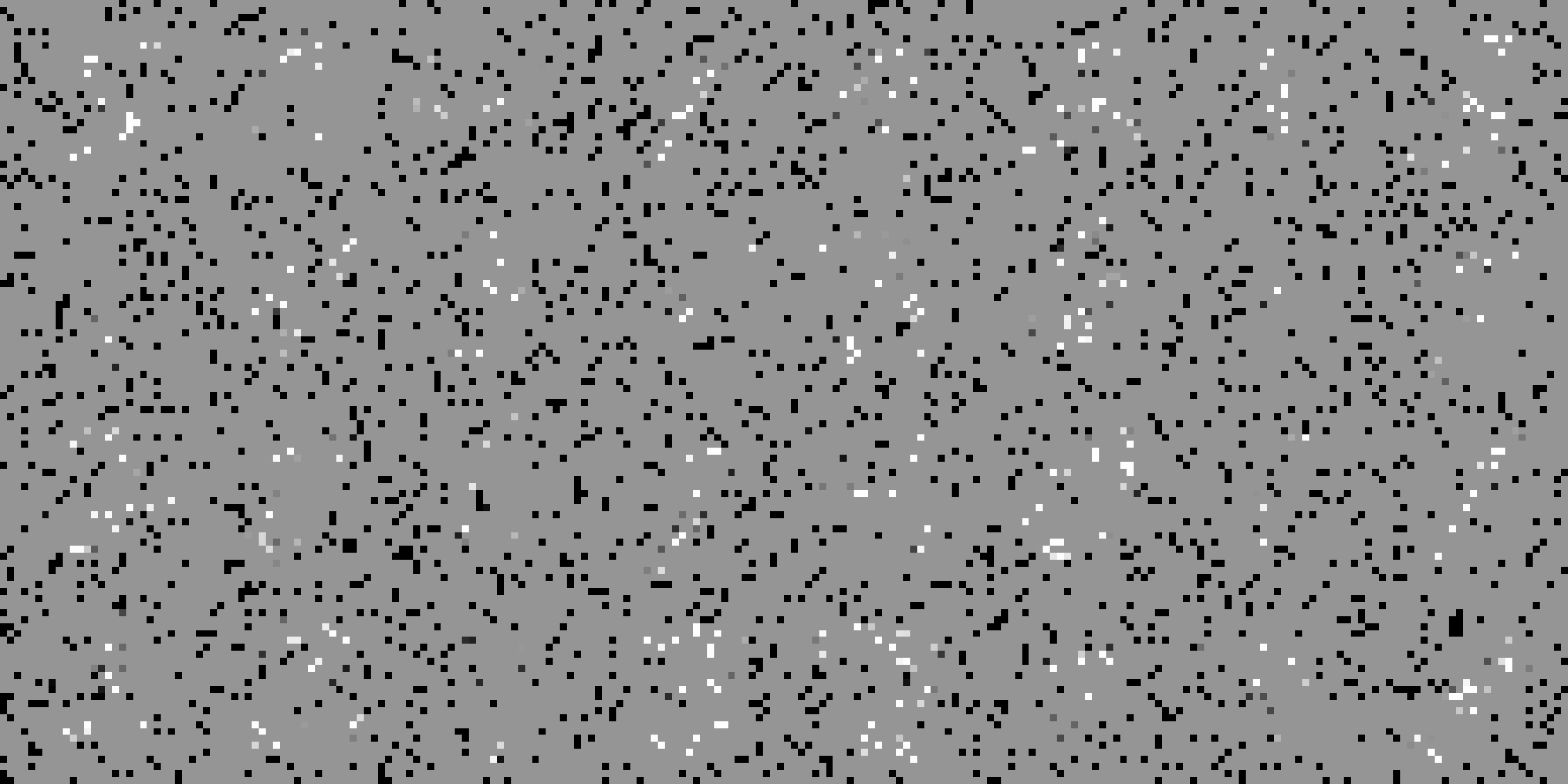} \\
    \includegraphics[width=\figwidth]{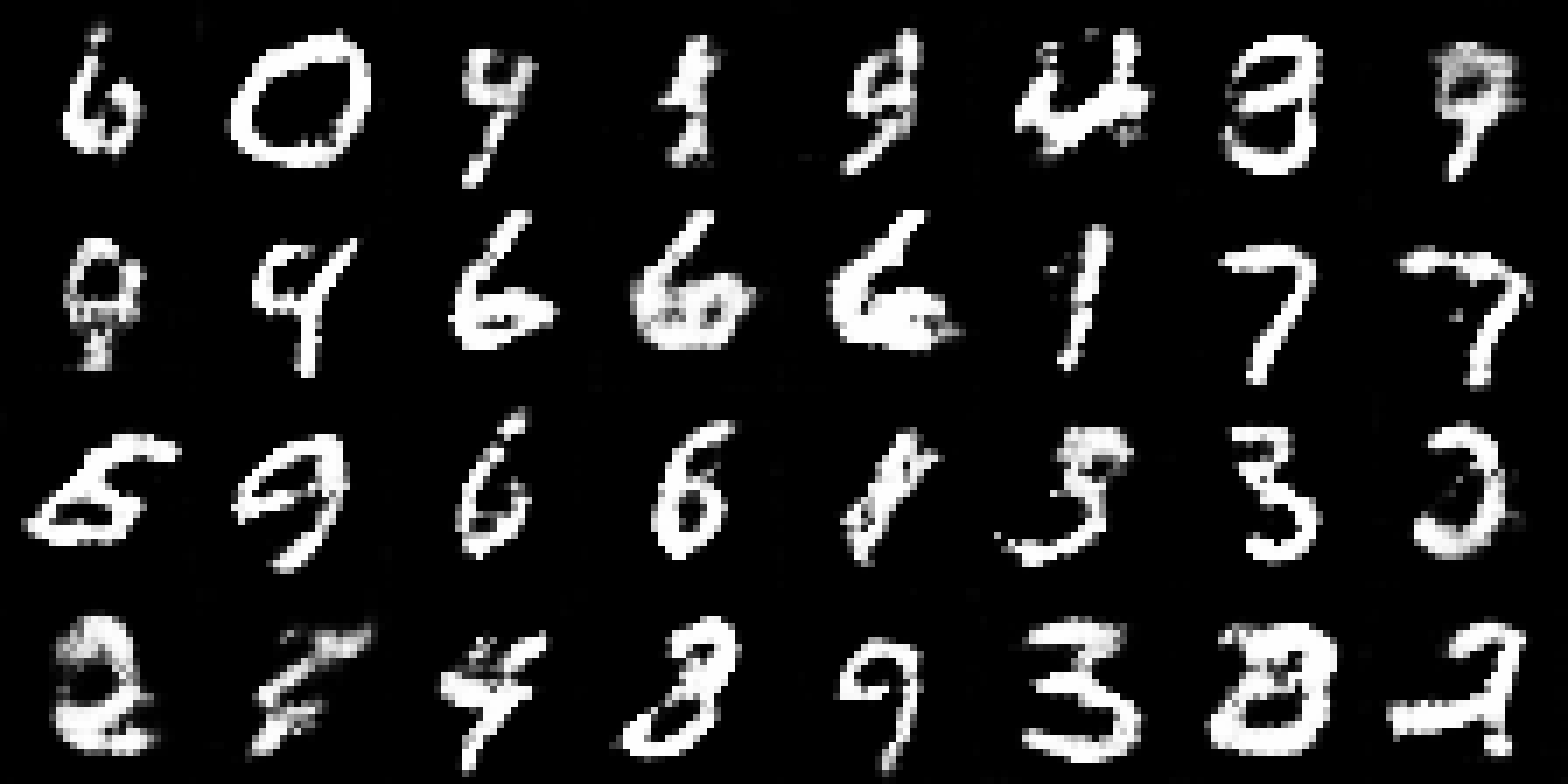} \\
    \includegraphics[width=\figwidth]{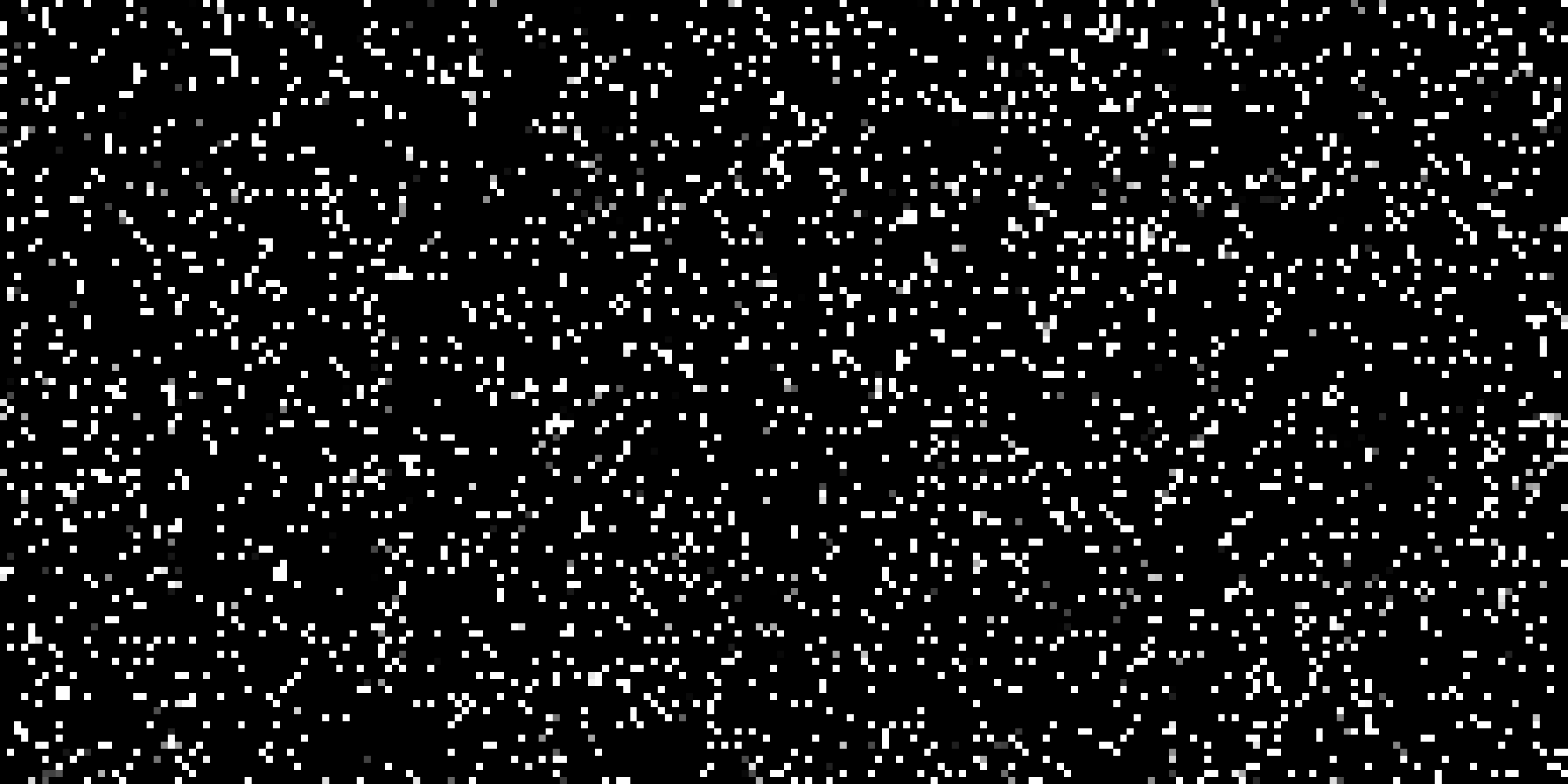}
    \caption{90\% dropout}
  \end{subfigure}
  \begin{subfigure}[b]{0.24\textwidth}
    \centering
    \includegraphics[width=\figwidth]{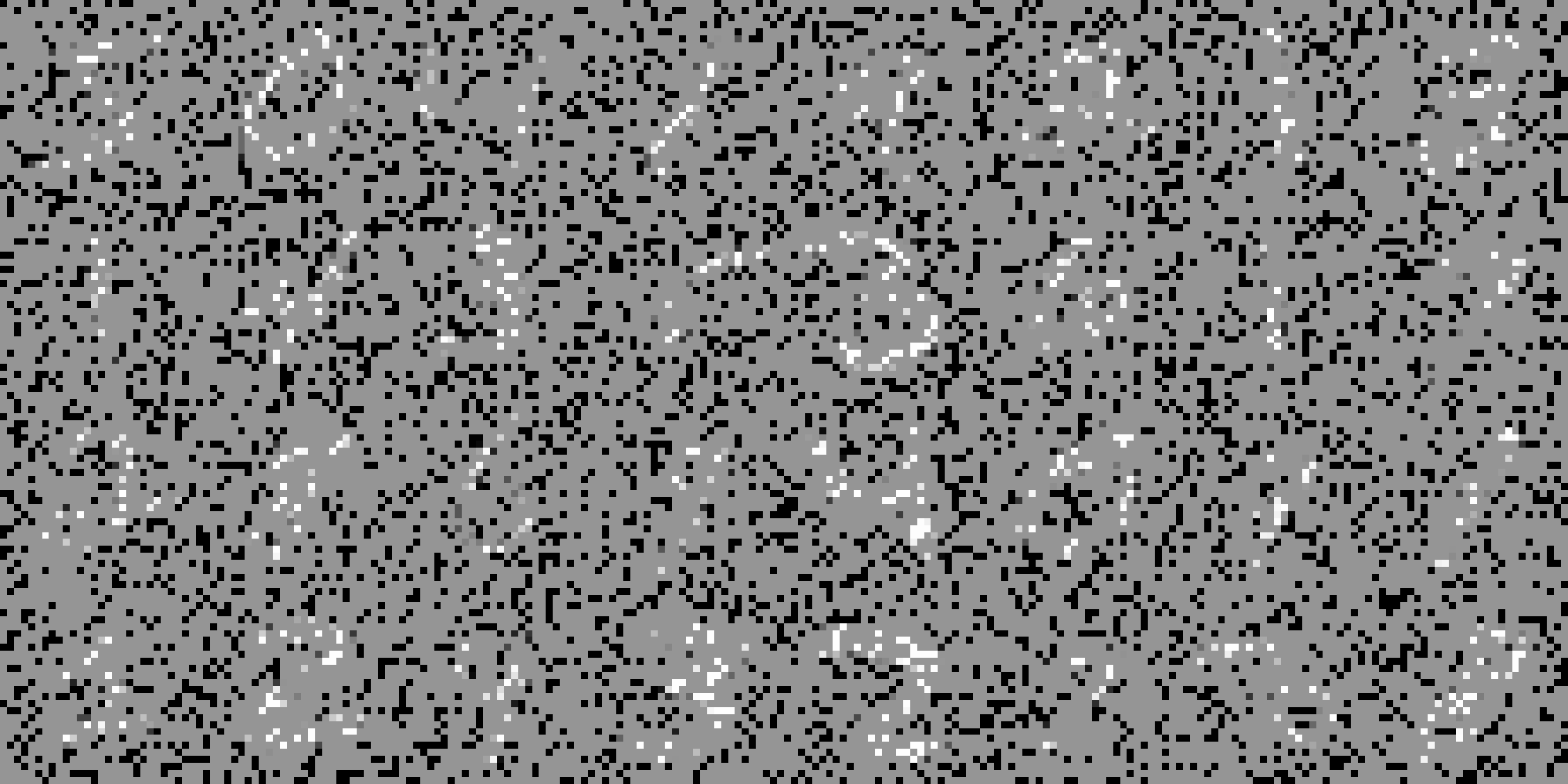} \\
    \includegraphics[width=\figwidth]{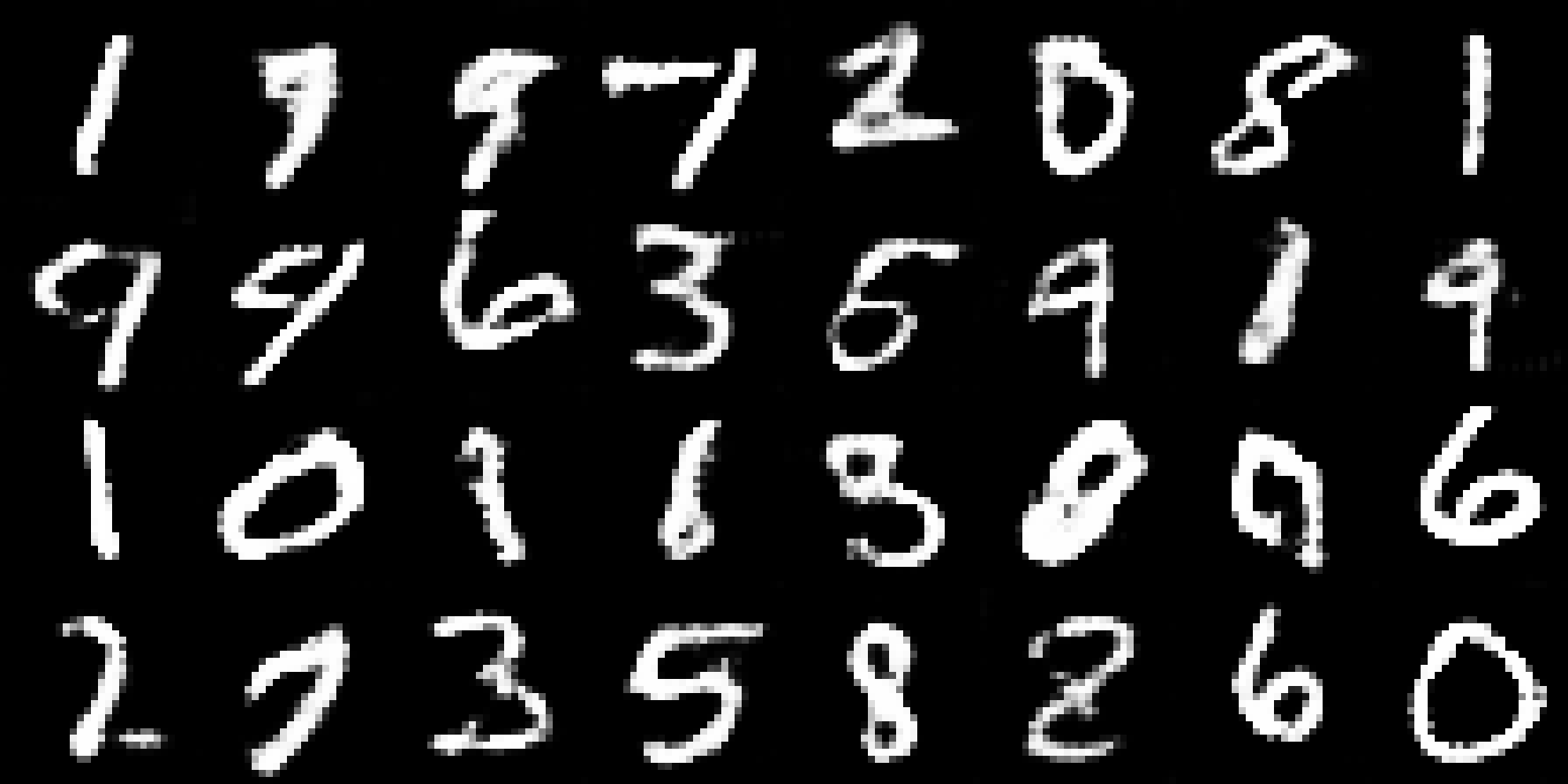} \\
    \includegraphics[width=\figwidth]{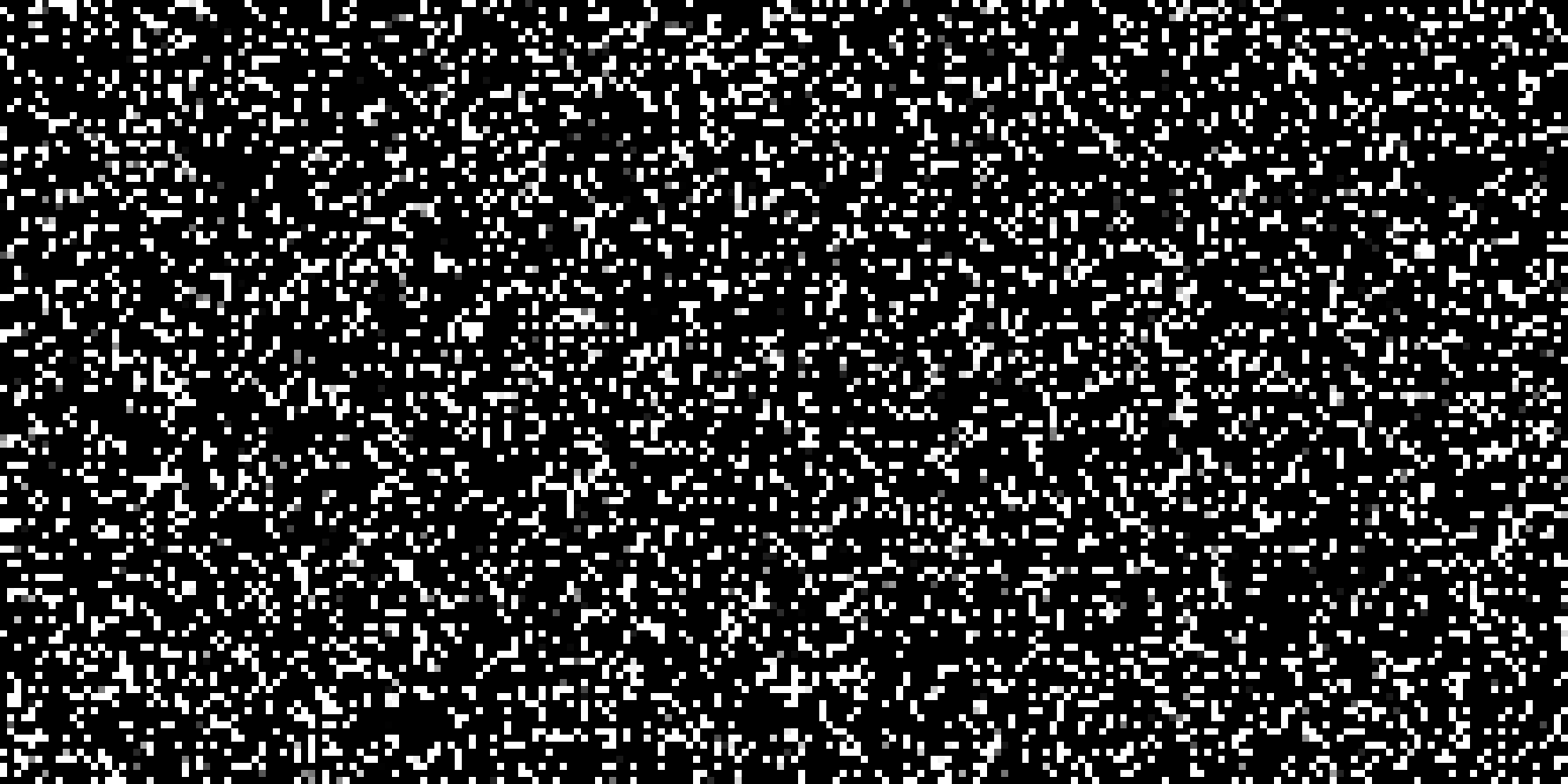}
    \caption{80\% dropout}
  \end{subfigure}
  \caption{Conv-{\misgan} results under different missing data
    processes.
    Top: training samples where gray pixels indicate missing data.
    Middle: data samples generated by $G_x$.
  Bottom: mask samples generated by $G_m$.}%
  \label{fig:convmisgan}
\end{figure}

\def\figwidth{1.3in}
\begin{figure}
  \centering
  \begin{subfigure}[b]{0.24\textwidth}
    \centering
    \includegraphics[width=\figwidth]{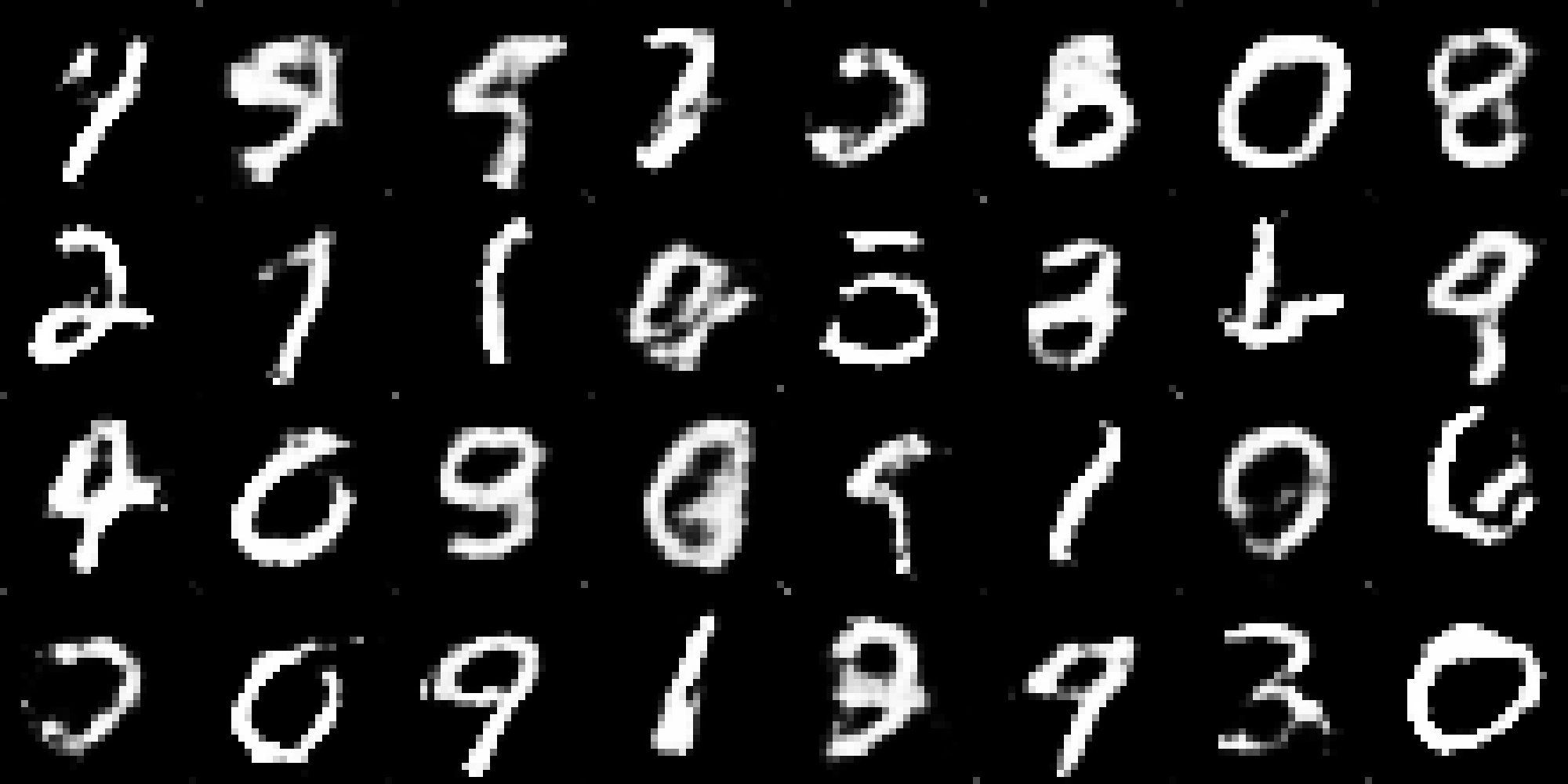}
    \caption{9$\times$9 (90\% missing)}
  \end{subfigure}
  \begin{subfigure}[b]{0.24\textwidth}
    \centering
    \includegraphics[width=\figwidth]{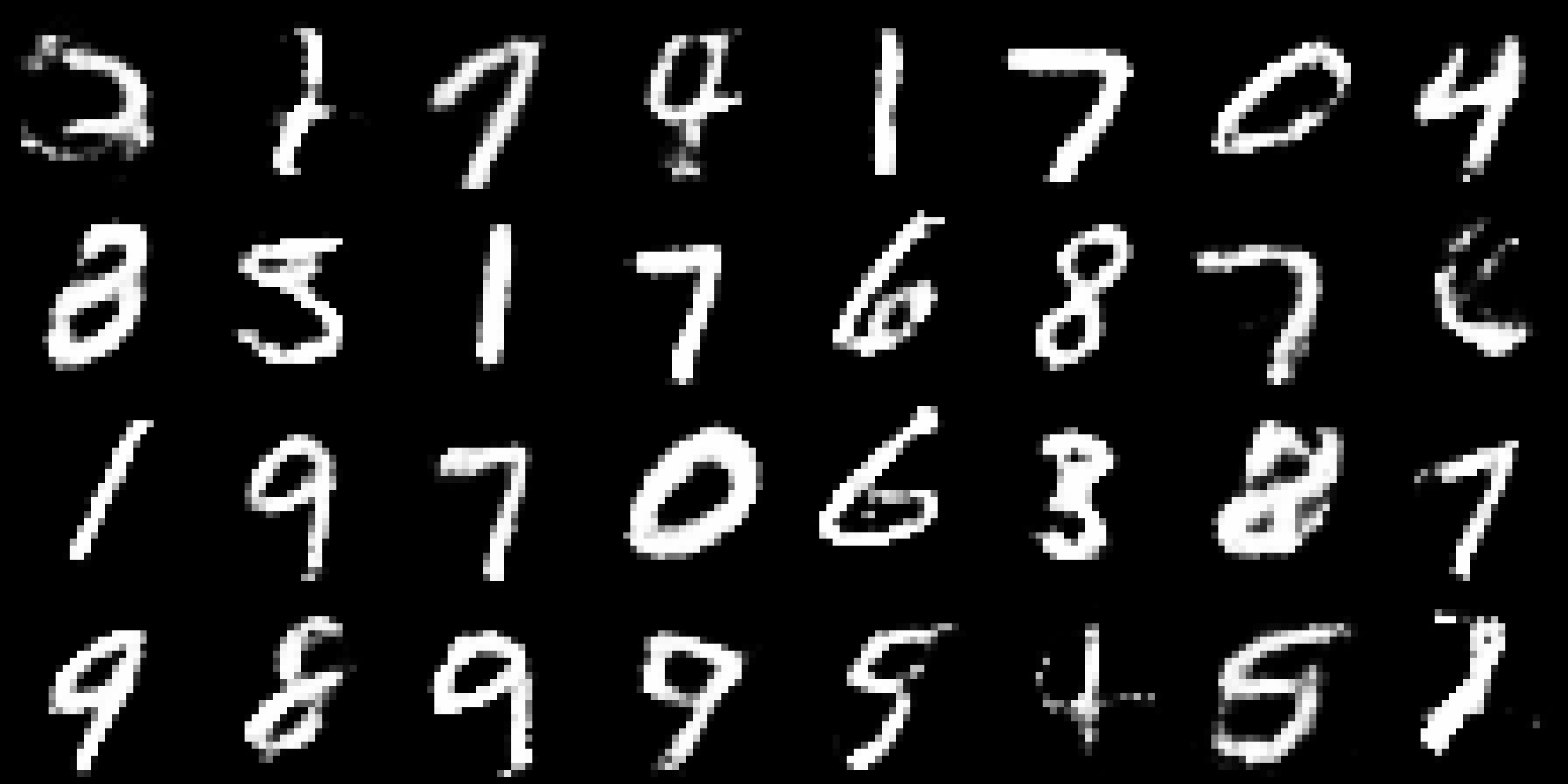}
    \caption{13$\times$13 (80\% missing)}
  \end{subfigure}
  \begin{subfigure}[b]{0.24\textwidth}
    \centering
    \includegraphics[width=\figwidth]{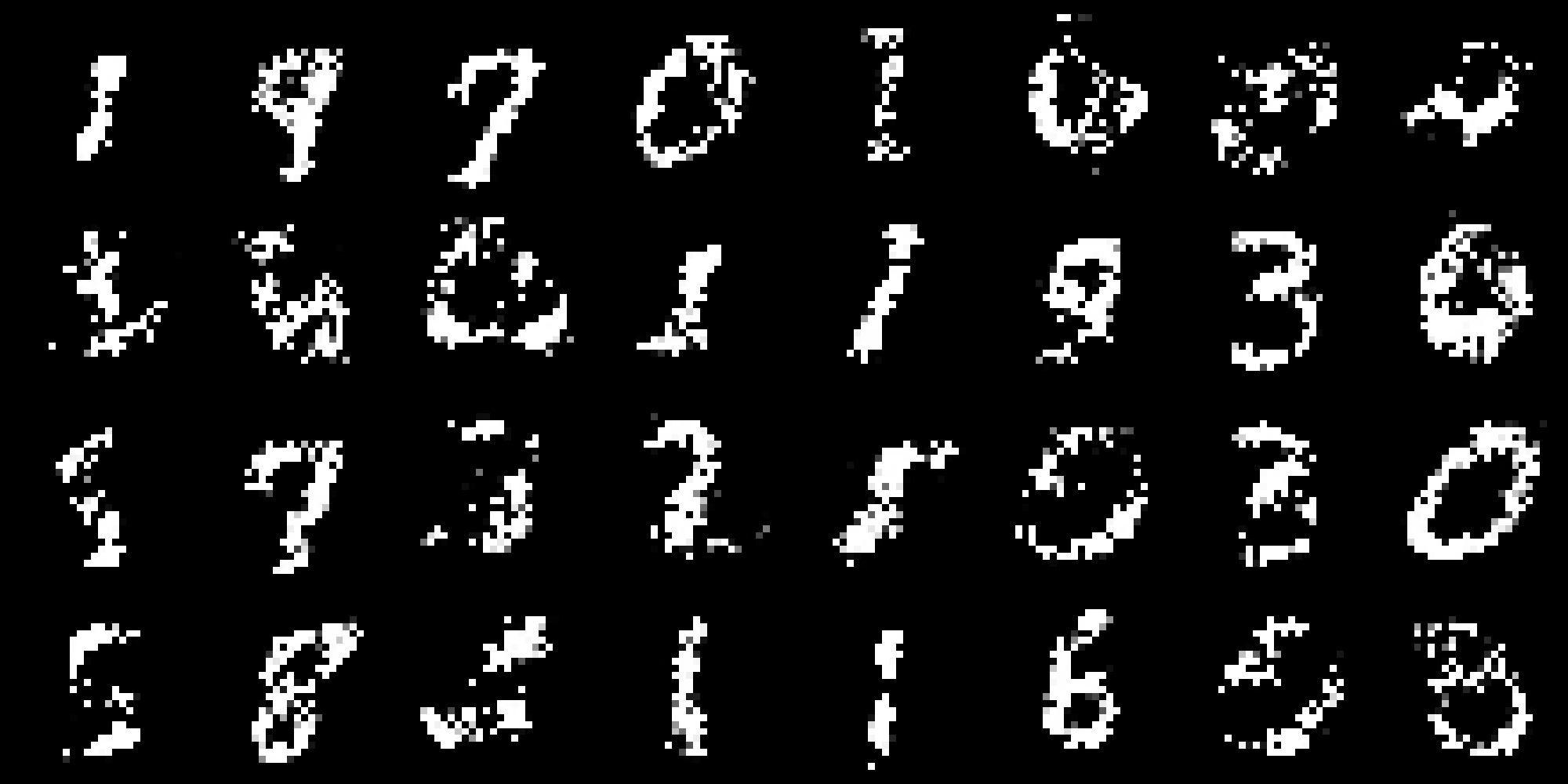}
    \caption{90\% dropout}
  \end{subfigure}
  \begin{subfigure}[b]{0.24\textwidth}
    \centering
    \includegraphics[width=\figwidth]{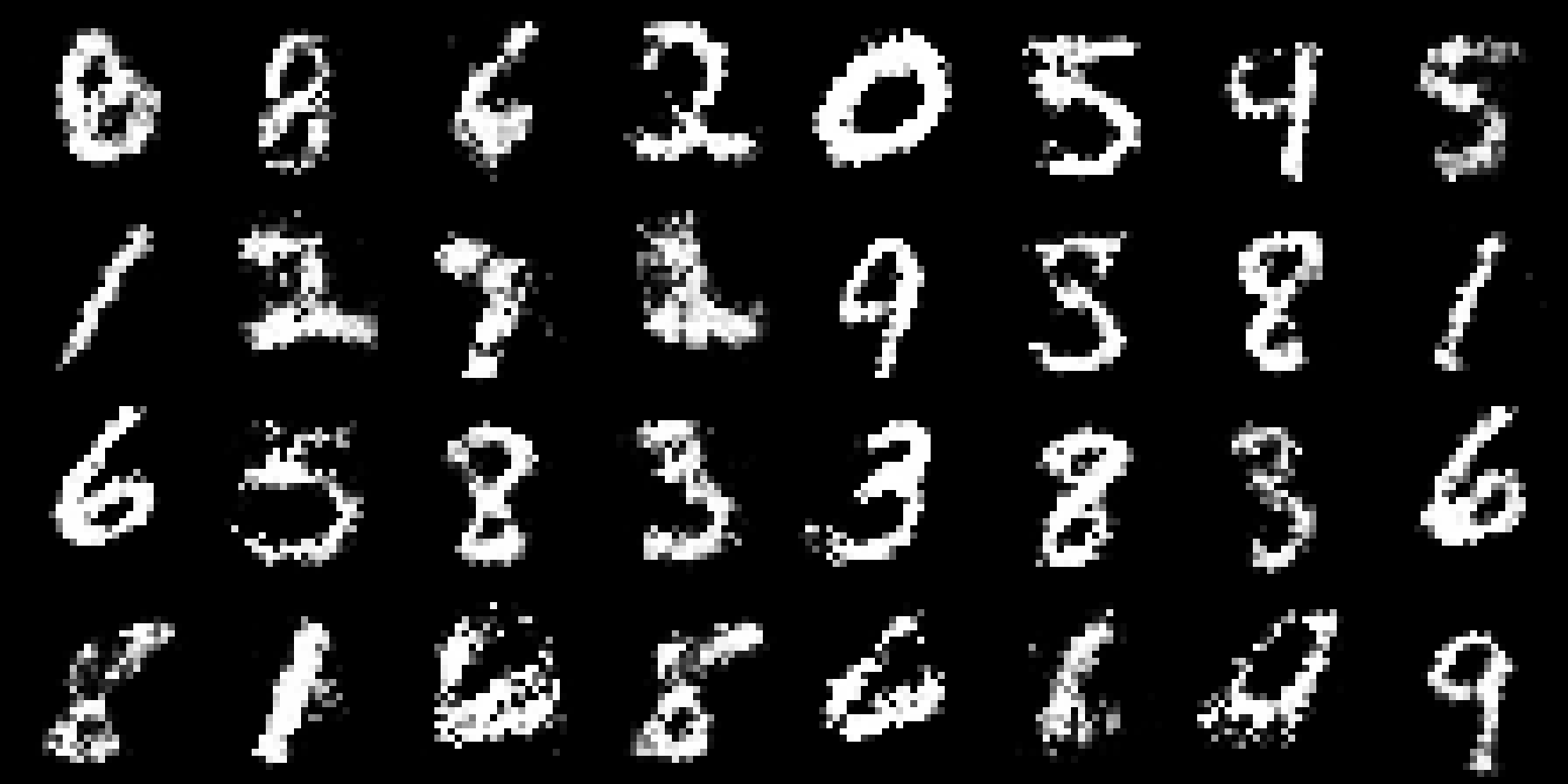}
    \caption{80\% dropout}
  \end{subfigure}
  \caption{Data samples generated by FC-{\misgan}.}%
  \label{fig:fcmisgan}
\end{figure}

\desc{Results}
Figures~\ref{fig:convmisgan} and~\ref{fig:fcmisgan}
show the generated data samples as well as
the learned mask samples produced by Conv-{\misgan} and FC-{\misgan}
under the square observation and independent dropout missing mechanisms.
From these results, we can see that Conv-{\misgan}
produces visually better samples than FC-{\misgan} on this problem.
On the other hand, under the same missing rate, independent dropout leads to
worse samples than square observations.
Samples generated by ConvAC are shown in Figure~\ref{fig:convac} in
Appendix~\ref{sec:convac}.

We quantitatively evaluate Conv-{\misgan}, FC-{\misgan} and
ConvAC under two missing patterns with
missing rates from 10\% to 90\% with a step of 10\%.
Figure~\ref{fig:sparsityfid} shows that {\misgan} in general outperforms ConvAC
as ConvAC tends to generate samples with aliasing artifacts as
shown in Figure~\ref{fig:convac}.
It also shows that in the square observation case,
Conv-{\misgan} and FC-{\misgan} have similar performance in terms of their FIDs.
However, under independent dropout, the performance of FC-{\misgan}
degrades significantly as the missing rate increases compared to
Conv-{\misgan}.
This is because independent dropout with high missing rate
makes the problem more challenging as it induces less overlapping
co-occurrence among pixels, which degrades the signal for understanding
the overall structure.

This is illustrated in Figure~\ref{fig:partition}
where the observed pattern comes from one of four equally probable
14$\times$14 square quadrants with no overlap.
Clearly this missing data problem is ill-posed and we could never
uniquely determine the correlation between pixels across different quadrants
without additional assumptions.
The samples generated by the FC-{\misgan} produce obvious discontinuity
across the boundary of the quadrants as it does not rely on any prior
knowledge about how pixels are correlated.
The discontinuity artifact is less severe with Conv-{\misgan}
since the convolutional layers encourage local smoothness.
This shows the importance of incorporating prior knowledge into the model when
the problem is highly ill-posed.

\def\figwidth{1.3in}
\begin{figure}
  \centering
  \begin{subfigure}[b]{0.75\textwidth}
    \includegraphics[width=2in]{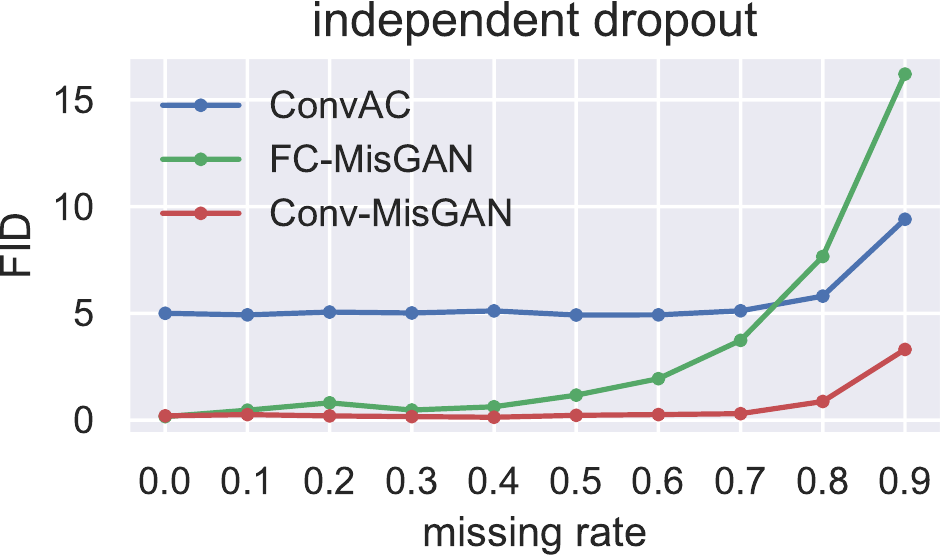}
    %\hspace{1em}
    \includegraphics[width=2in]{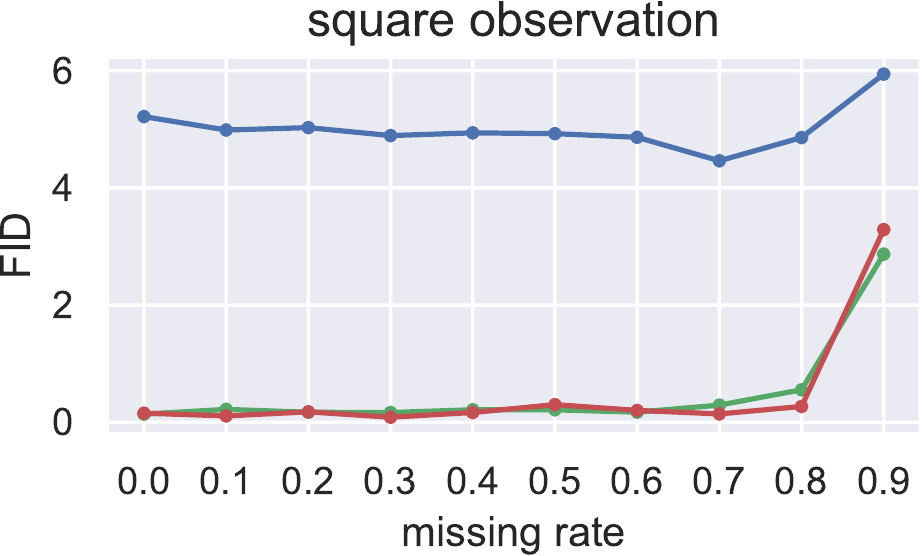}
  \end{subfigure}
  \begin{subfigure}[b]{0.24\textwidth}
    \centering
    \includegraphics[width=\figwidth]{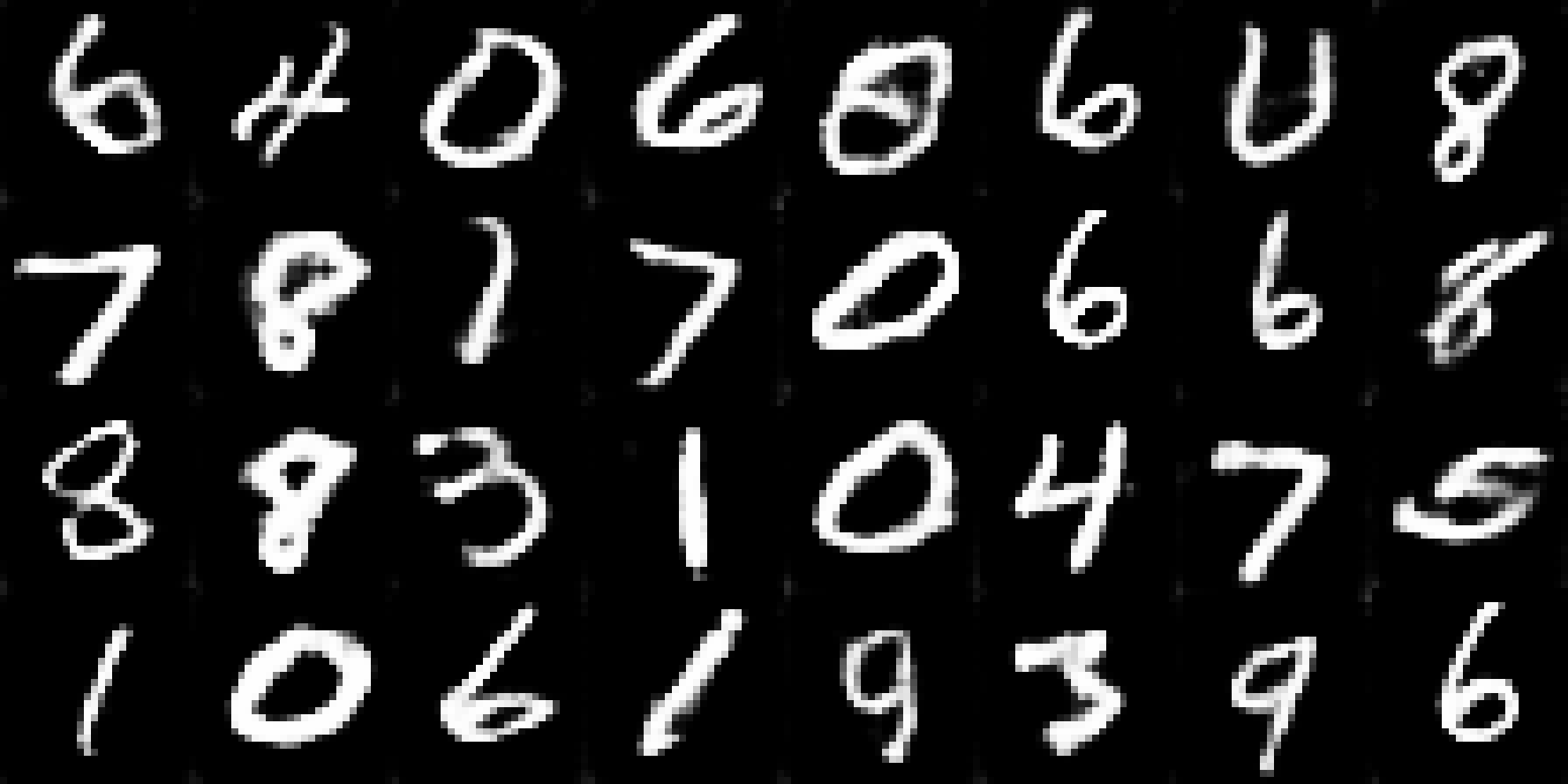} \\
    \includegraphics[width=\figwidth]{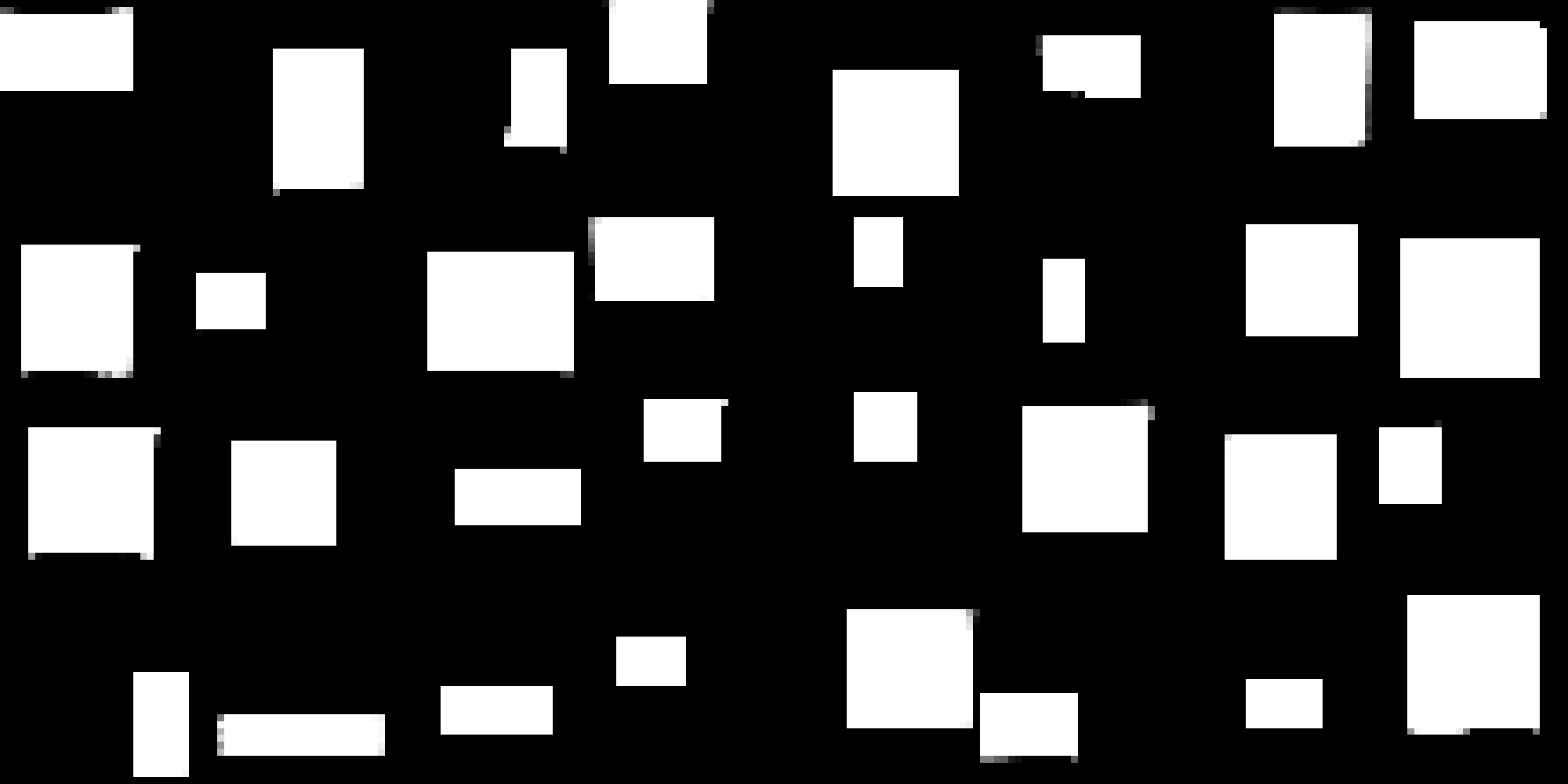}
    %\caption{Conv-{\misgan}: variable-size}
  \end{subfigure}
  \caption{Left \& Middle: Missing rate versus FID (The lower the better)
  with different missing data processes.
  Right: Data samples (top) and mask samples (bottom)
  generated by Conv-{\misgan} learned with variable-size observations.}%
  \label{fig:sparsityfid}
\end{figure}

%\def\figwidth{1.3in}
%\begin{figure}
%  \centering
%  \begin{subfigure}[b]{0.24\textwidth}
%    \centering
%    \includegraphics[width=\figwidth]{convac/scaled/block09-convac.png}
%    \caption{ConvAC: 9$\times$9 (90\%)}
%  \end{subfigure}
%  \begin{subfigure}[b]{0.24\textwidth}
%    \centering
%    \includegraphics[width=\figwidth]{convac/scaled/block13-convac.png}
%    \caption{ConvAC: 13$\times$13 (80\%)}
%  \end{subfigure}
%  \hspace{.5em}
%  \begin{subfigure}[b]{0.48\textwidth}
%    \centering
%    \includegraphics[width=\figwidth]{img/scaled/blockvarsize-conv.png}
%    \includegraphics[width=\figwidth]{img/scaled/blockvarsize-conv-mask.png}
%    \caption{Conv-{\misgan}: variable-size observations}
%  \end{subfigure}
%  \caption{
%    (a)(b): Samples of ConvAC trained with square observations.
%    (c): Results of Conv-{\misgan} with variable-size observations.
%    Left: data samples by $G_x$. Right: mask samples by $G_m$.
%  }
%  \label{fig:convac}
%\end{figure}

\begin{figure}
  \centering
  \begin{subfigure}[b]{0.3\textwidth}
    \centering
    \includegraphics[width=\textwidth]{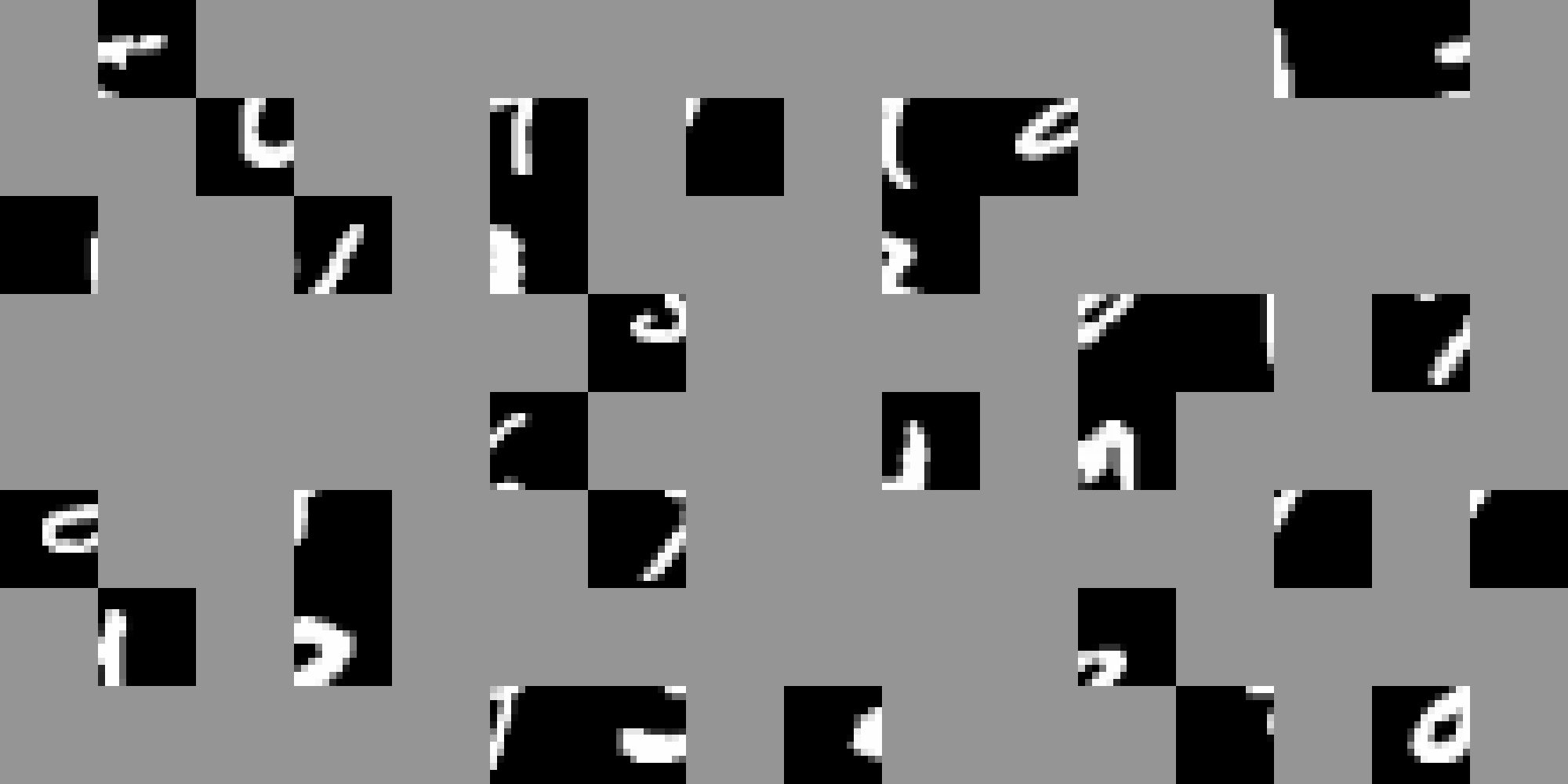}
    \caption{incomplete training data}
  \end{subfigure}
  \begin{subfigure}[b]{0.3\textwidth}
    \centering
    \includegraphics[width=\textwidth]{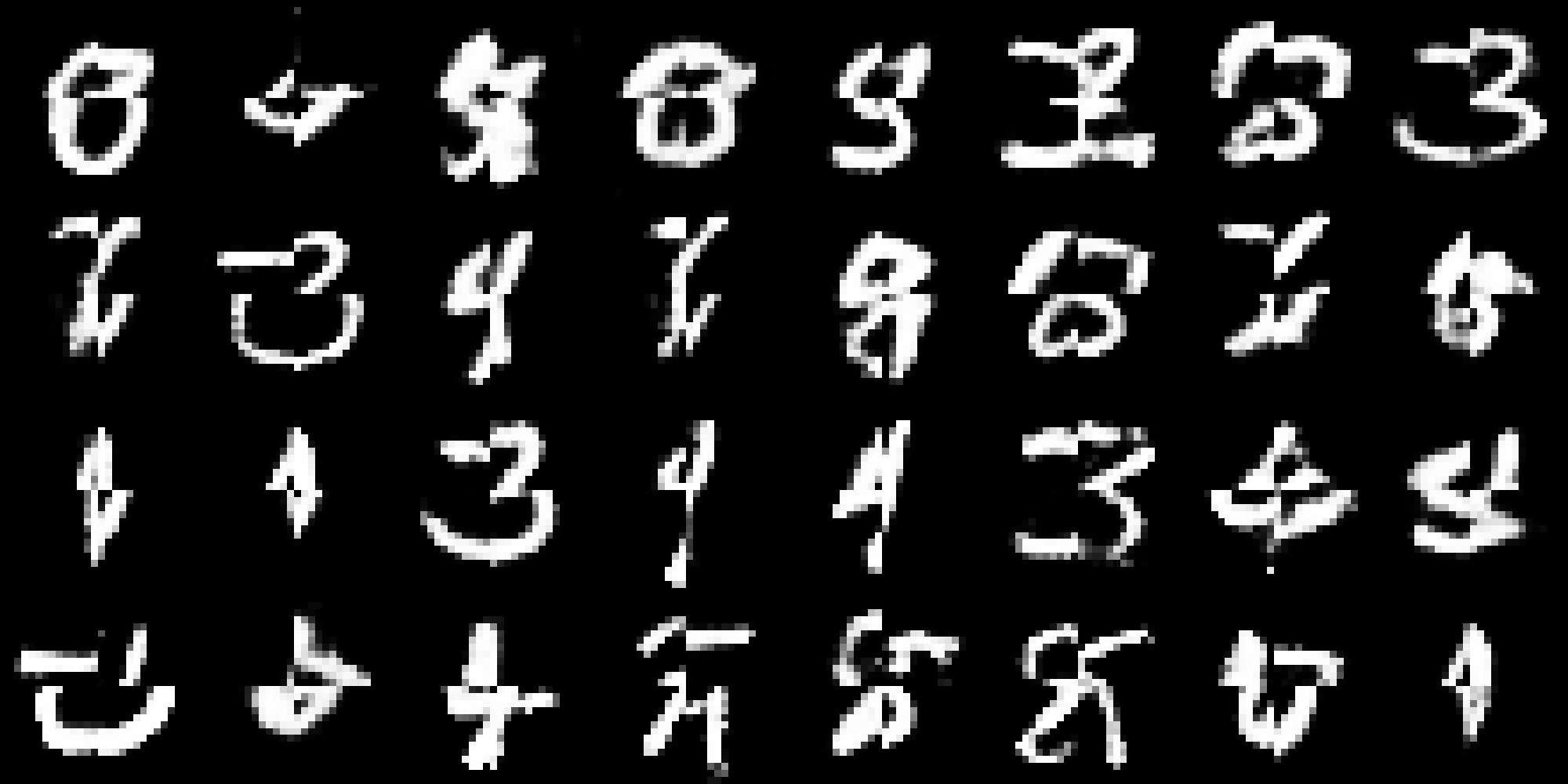}
    \caption{FC-{\misgan} (FID: 20.6)}
  \end{subfigure}
  \begin{subfigure}[b]{0.3\textwidth}
    \centering
    \includegraphics[width=\textwidth]{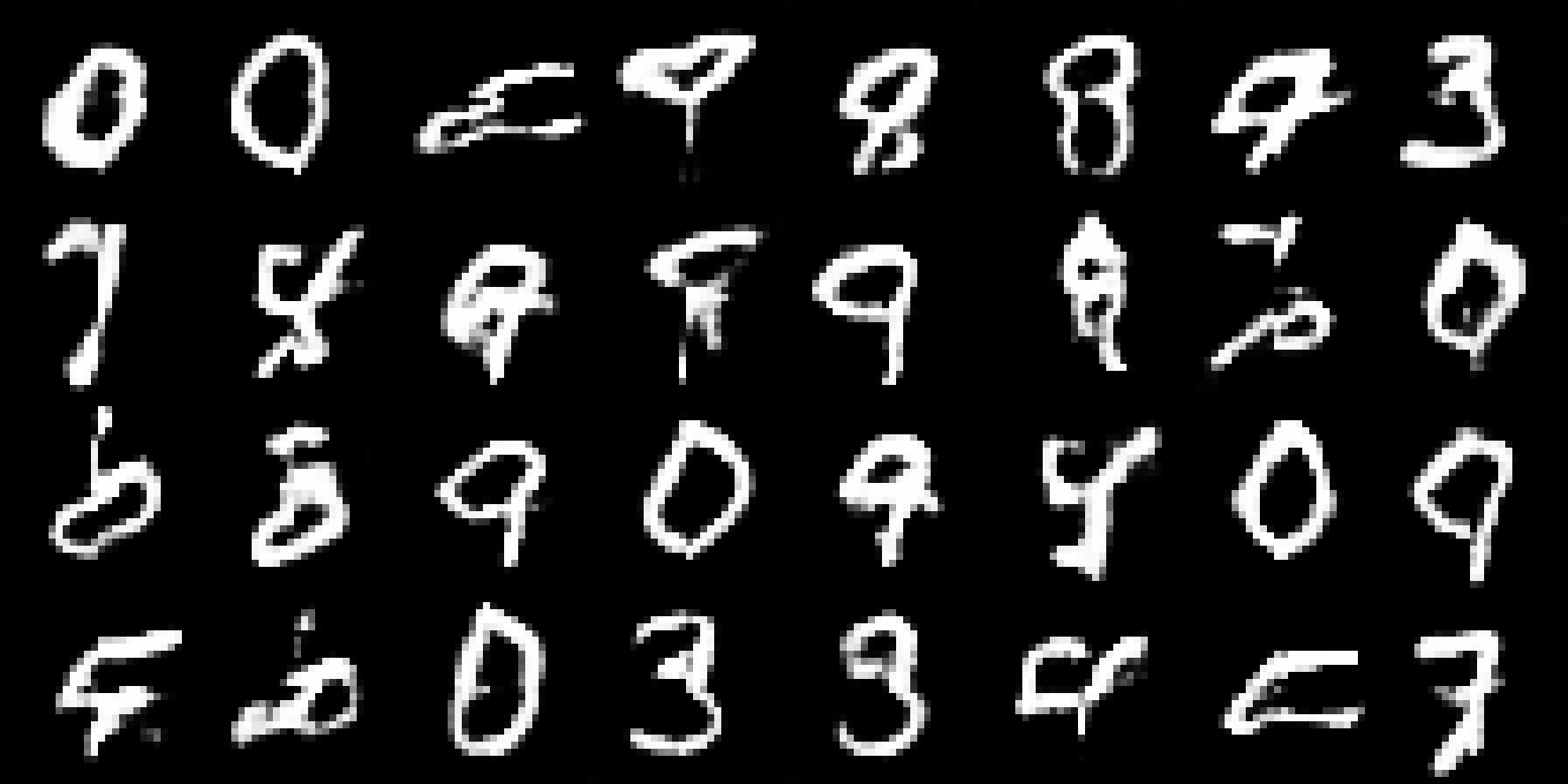}
    \caption{Conv-{\misgan} (FID: 10.8)}
  \end{subfigure}
  \caption{Data samples generated by {\misgan} when trained on
    missing data distributions with non-overlapping samples
    (square quadrants).
  }%
  \label{fig:partition}
\end{figure}

%\begin{figure}
%  \centering
%  \begin{subfigure}[b]{0.32\textwidth}
%    \centering
%    \includegraphics[width=\textwidth]{mask_img/scaled/blockvarsize.png}
%    \caption{incomplete training data}
%  \end{subfigure}
%  \begin{subfigure}[b]{0.32\textwidth}
%    \centering
%  \includegraphics[width=\textwidth]{img/scaled/fail-blockvarsize.png}
%    \caption{learned data samples}
%  \end{subfigure}
%  \begin{subfigure}[b]{0.32\textwidth}
%    \centering
%  \includegraphics[width=\textwidth]{img/scaled/fail-blockvarsize-mask.png}
%    \caption{learned masks}
%  \end{subfigure}
%  \caption{Failed case}
%  \label{fig:fail}
%\end{figure}

\desc{Ablation study}
We point out that the mask discriminator in {\misgan} is important
for learning the correct distribution robustly.
Figure~\ref{fig:fail} shows two common failure scenarios
that frequently happen with an AmbientGAN,
which is essentially equivalent to a {\misgan} without the mask discriminator.
Figure~\ref{fig:fail} (left) shows a case where AmbientGAN produces
perfectly consistent masked outputs, but the learned mask distribution
is completely wrong.
Since we use $f_{\tau=0}(\xx,\mm)=\xx\odot\mm$, it
makes the role of $\xx$ and $\mm$ interchangeable when considering
only the masked outputs.
Even if we rescale the range of pixel values from $[0,1]$ to $[-1,1]$
to avoid this situation,
AmbientGAN still fails often as shown in Figure~\ref{fig:fail} (right).
%Figure~\ref{fig:fail} right shows another frequently occurred failure case
%for AmbientGAN which learns a data generator that produces patched-looking
%samples and the model thinks the data are always fully-observed.
In contrast,
{\misgan} avoids learning such degenerate solutions due to explicitly
modeling the mask distribution.

%\begin{figure}
%  \def\figwidth{1.5in}
%  \centering
%  \begin{subfigure}[b]{0.4\textwidth}
%    \centering
%    \includegraphics[width=\figwidth]{img/scaled/fail-blockvarsize.png} \\
%    \includegraphics[width=\figwidth]{img/scaled/fail-blockvarsize-mask.png}
%    %\caption{learned data samples}
%  \end{subfigure}
%  \begin{subfigure}[b]{0.4\textwidth}
%    \centering
%    \includegraphics[width=\figwidth]{img/scaled/fail-rescale.png}
%    \includegraphics[width=\figwidth]{img/white.png}
%    %\caption{learned masks}
%  \end{subfigure}
%  \caption{Two failure cases of AmbientGAN.
%    Top: samples of $G_x$. Bottom: samples of $G_m$.
%    Top-right: the range of pixel values are rescaled to $[-1, 1]$
%    so the gray pixels correspond to $\tau=0$.
%  }
%  \label{fig:fail}
%\end{figure}

\begin{figure}
  \def\figwidth{1.3in}
  \centering
  \begin{subfigure}[b]{\textwidth}
    \centering
    \includegraphics[width=\figwidth]{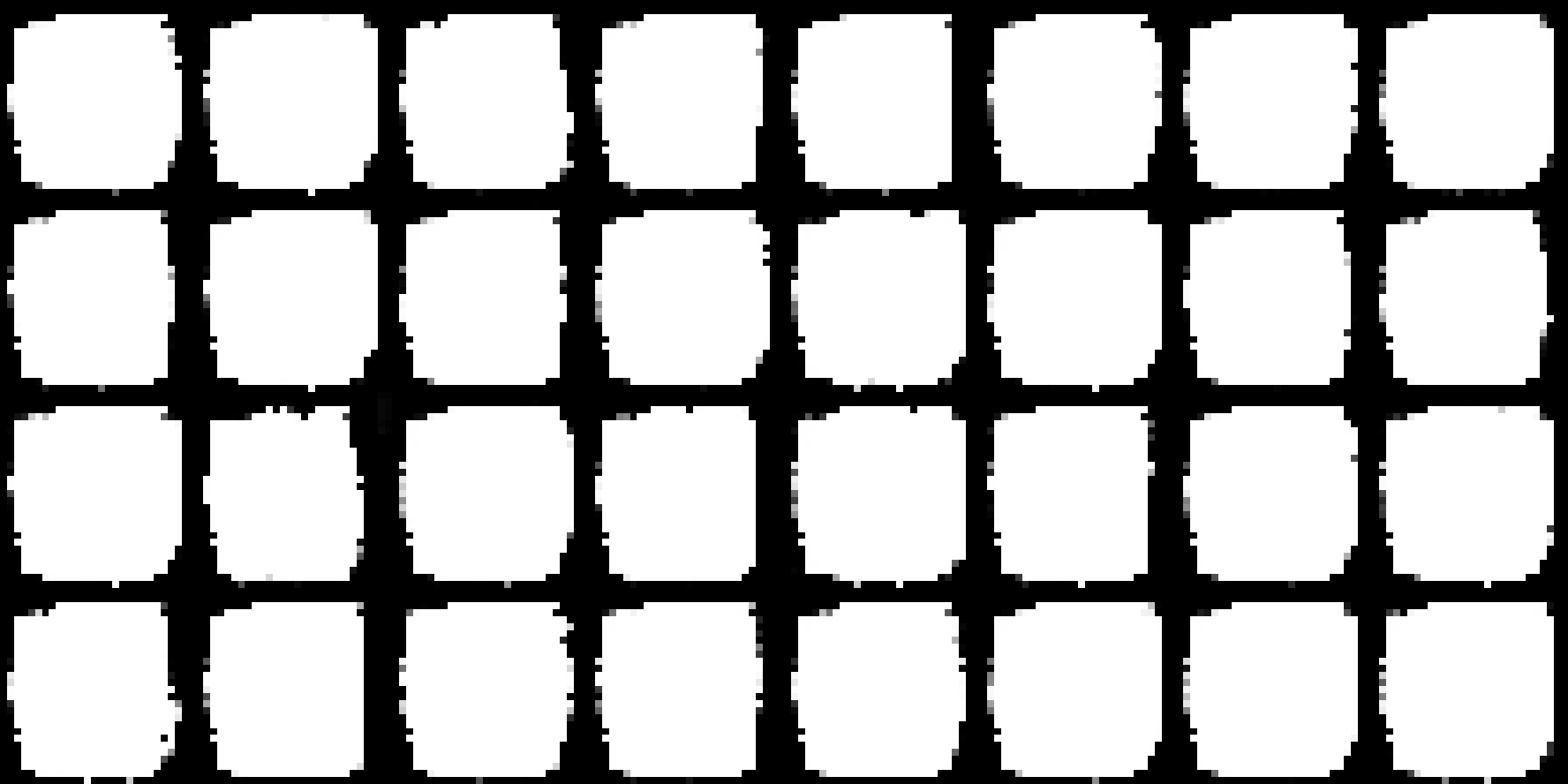}
    \includegraphics[width=\figwidth]{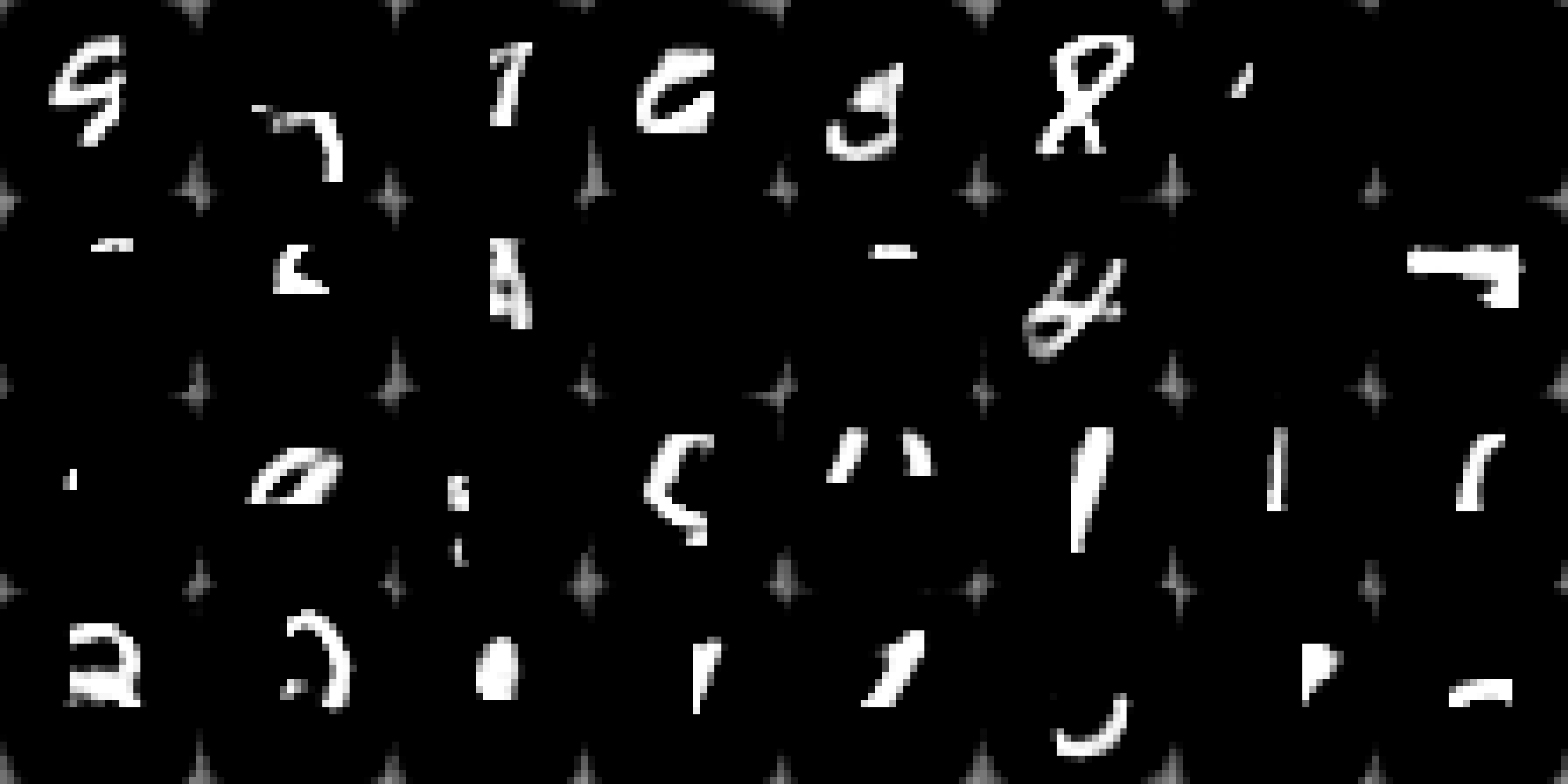}
    \hspace{1em}
    \includegraphics[width=\figwidth]{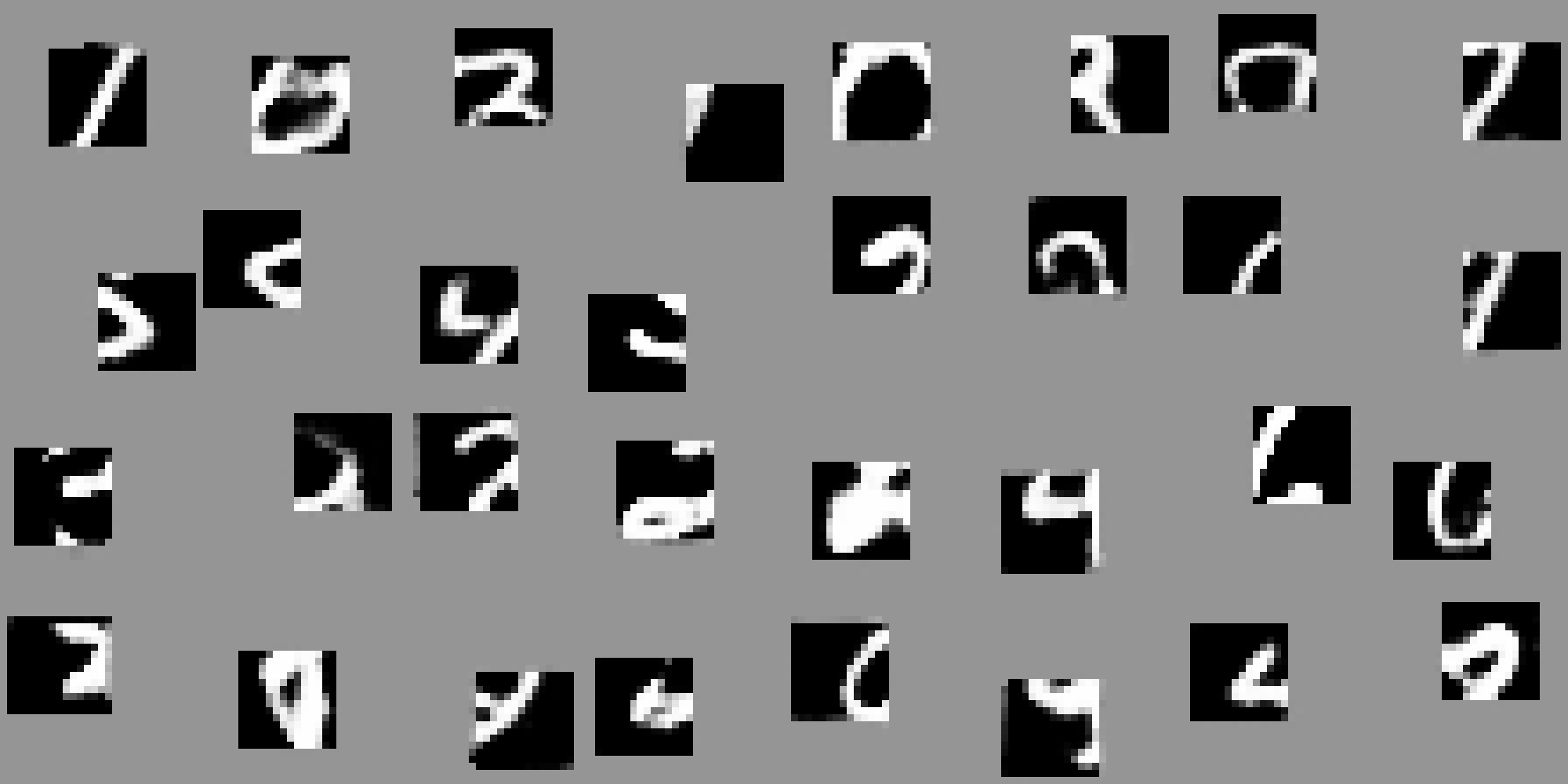}
    \includegraphics[width=\figwidth]{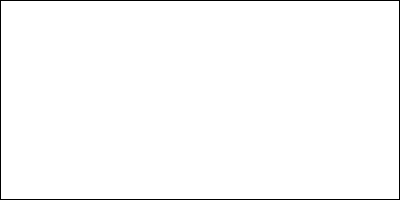}
    %\caption{learned masks}
  \end{subfigure}
  \caption{Two failure cases of AmbientGAN.
    In each pair, data samples produced by $G_x$ are on the left,
    mask samples from $G_m$ are on the right.
    In the right panels,
    the range of pixel values is rescaled to $[-1, 1]$ so
    gray pixels correspond to $\tau=0$. It learns the masks with all ones.
  }%
  \label{fig:fail}
\end{figure}

\desc{Missing data imputation}
We construct the imputer network $\widehat{G}_i$
defined in~\eqref{eq:imputer}
using a three-layer fully-connected network
with 500 hidden units in the middle layers.
Figure~\ref{fig:imputation} (left) shows the imputation results
on different examples applying novel masks randomly drawn according to
the same distribution.
Figure~\ref{fig:imputation} (right) shows the imputation results where
each row corresponds to the same incomplete input.
It demonstrates that the imputer can produce a variety of different
imputed results due to the random noise input to the imputer.
We also note that if we modify~\eqref{eq:imputeronly} to train
the imputer together with the data generator from scratch without
the mask generator/discriminator, it fails most of the time for a similar
reason to why AmbientGAN fails. The learning problem is highly ill-posed
without the agreement on the mask distribution.

\begin{figure}
  \centering
  \begin{subfigure}[b]{0.55\textwidth}
    \centering
    \includegraphics[width=2.5in]{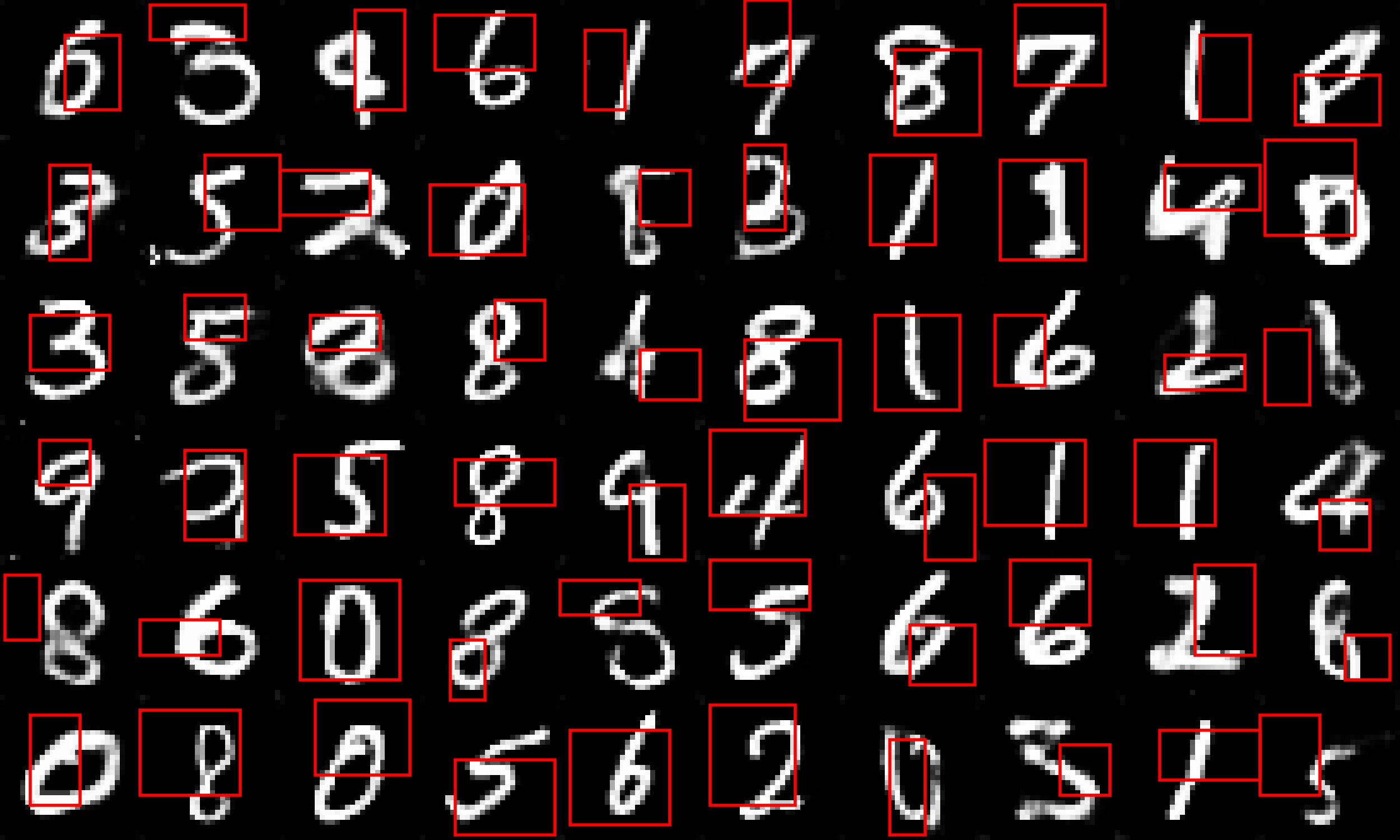}
  \end{subfigure}
  %\hspace{.5in}
  \begin{subfigure}[b]{0.44\textwidth}
    \centering
    \def\imgw{1.62in}
    \includegraphics[width=\imgw]{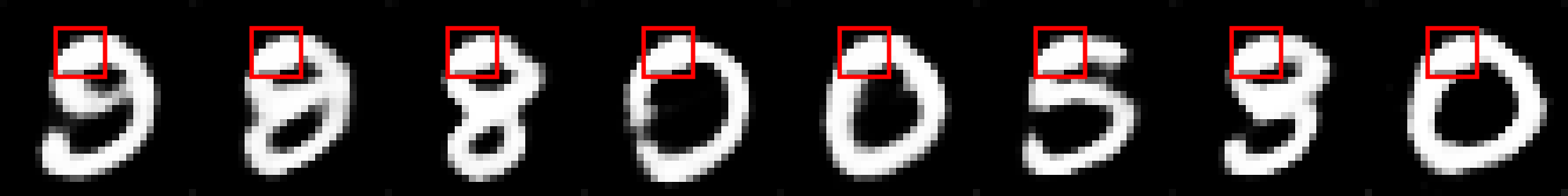}
    \\
    \includegraphics[width=\imgw]{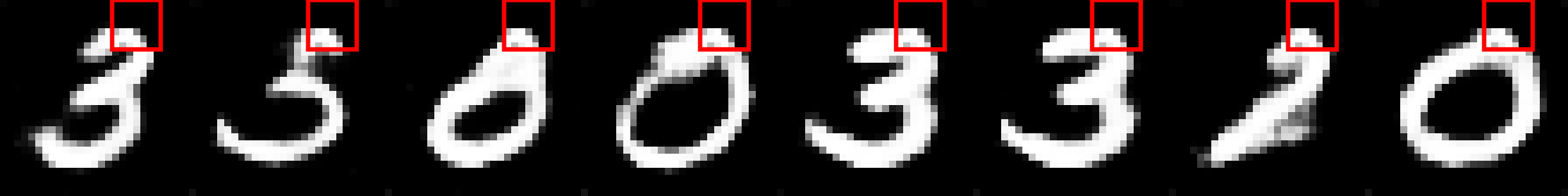}
    \\
    \includegraphics[width=\imgw]{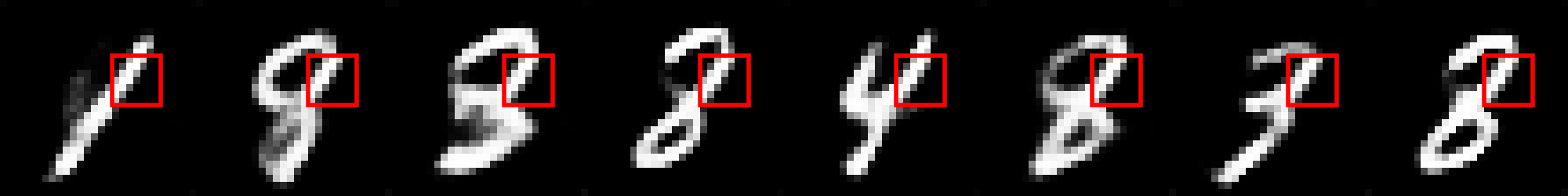}
    \\
    \includegraphics[width=\imgw]{multi-impute/7-4-8-18-3.png}
    \\
    \includegraphics[width=\imgw]{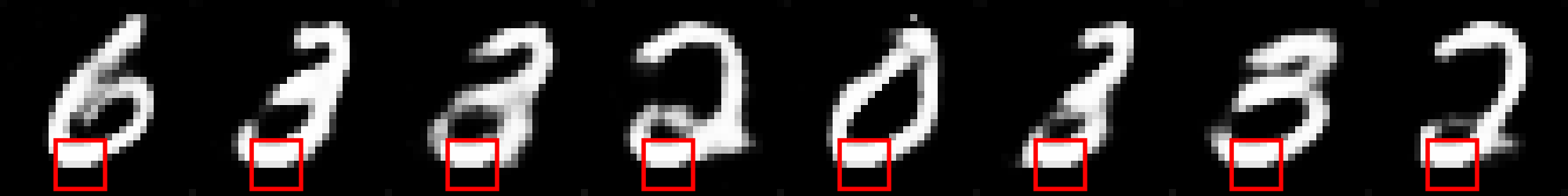}
    \\
    \includegraphics[width=\imgw]{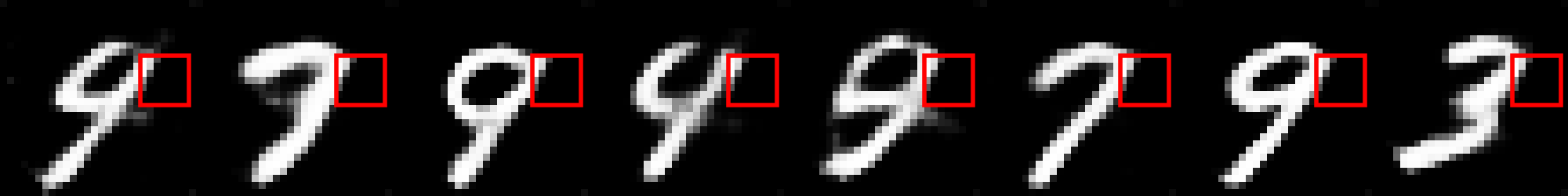}
    \\
    \includegraphics[width=\imgw]{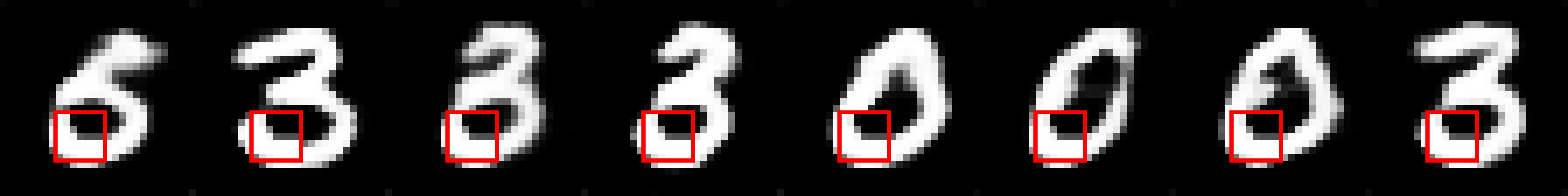}
    %\\
    %\includegraphics[width=\imgw]{multi-impute/7-4-20-00-2.png}
  \end{subfigure}
  \caption{%
    Inside of each red box are the
    observed pixels; the pixels outside of the box
    are generated by the imputer.
    Right: each row corresponds to the same incomplete input,
    marked by the red box.
  }%
  \label{fig:imputation}
\end{figure}

\subsection{Quantitative evaluation}
In this section, we quantitatively evaluate the performance of
{\misgan} on three datasets: MNIST, CIFAR-10, and CelebA.
We focus on evaluating {\misgan} on the missing data
imputation task as it is widely studied and many baseline methods are
available.

\begin{figure}[t]
  \def\imgw{.35\textwidth}
  \hspace*{-.25in}
  \centering
  \includegraphics[width=\imgw]{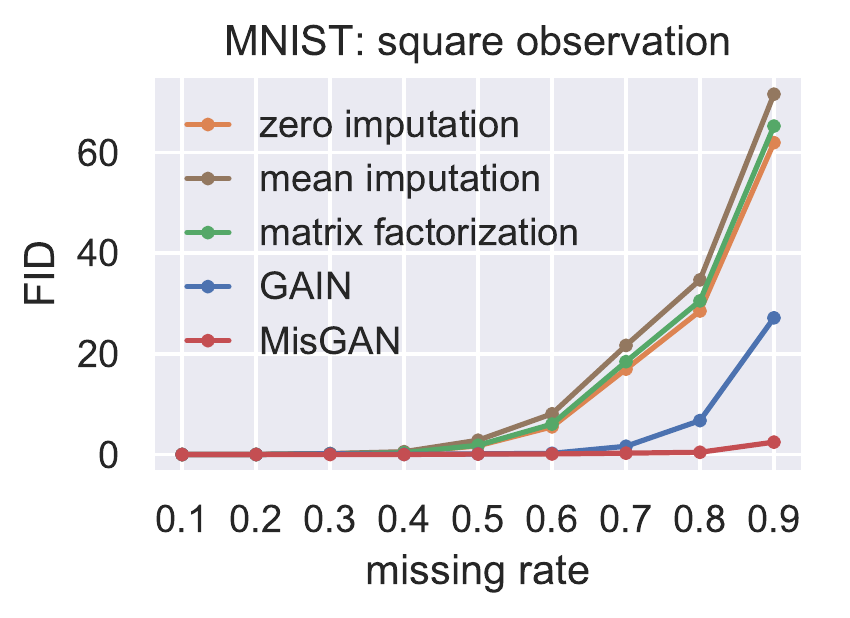}
  \hspace*{-.1in}
  \includegraphics[width=\imgw]{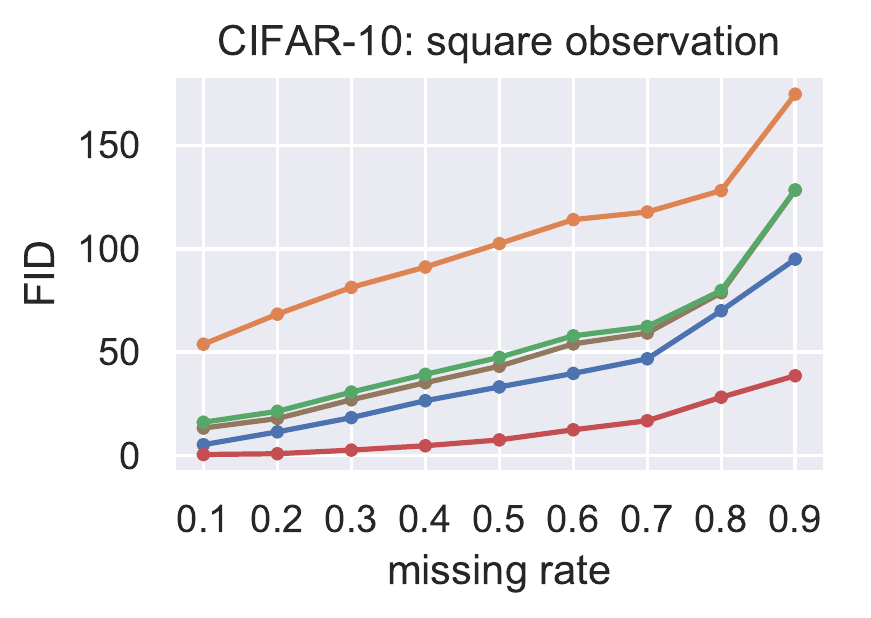}
  \hspace*{-.1in}
  \includegraphics[width=\imgw]{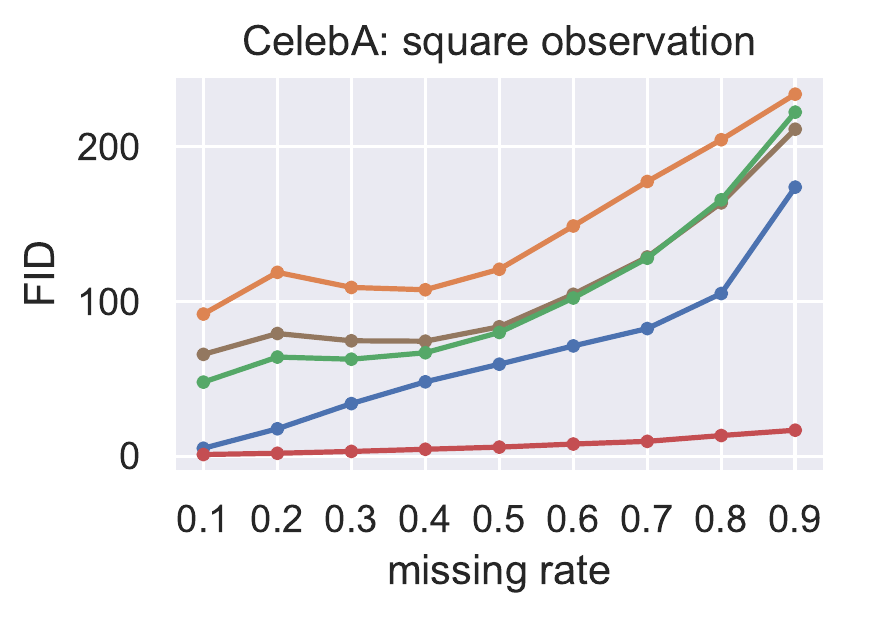}
  \\[.1em]
  \hspace*{-.25in}
  \centering
  \includegraphics[width=\imgw]{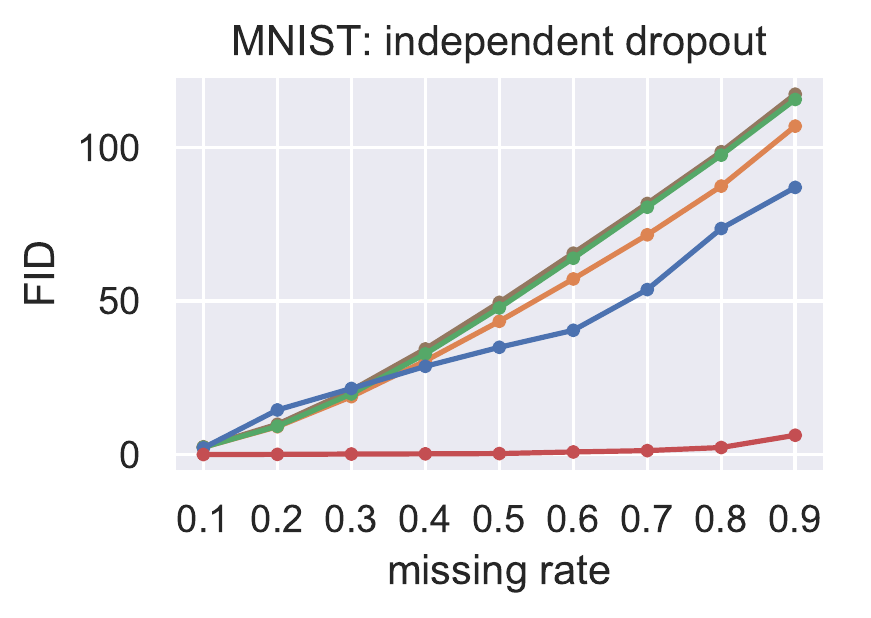}
  \hspace*{-.1in}
  \includegraphics[width=\imgw]{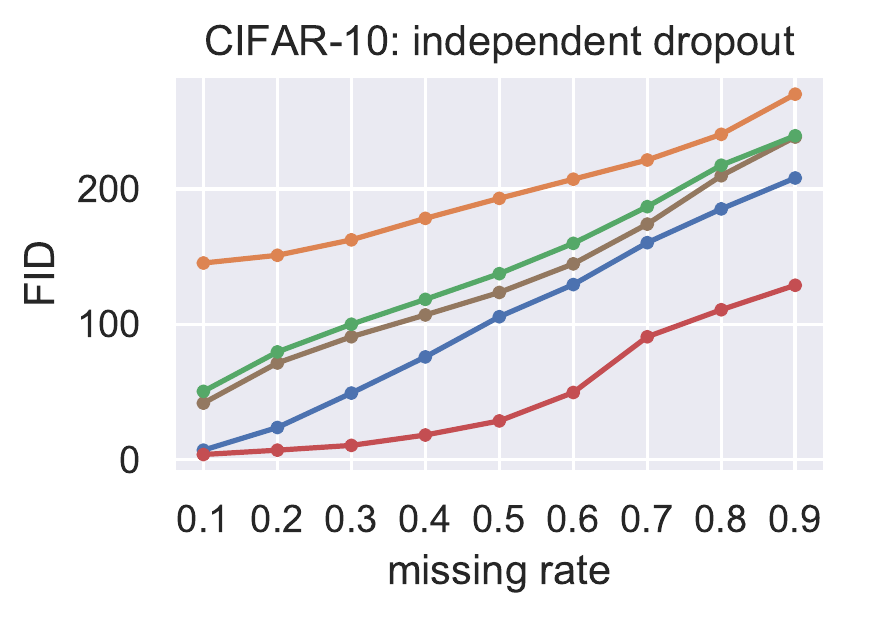}
  \hspace*{-.1in}
  \includegraphics[width=\imgw]{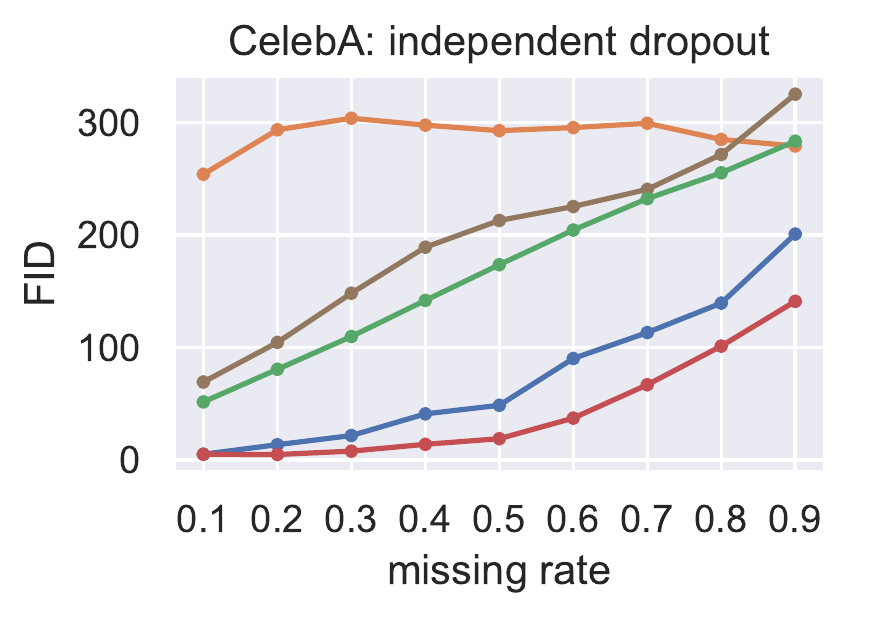}
  \caption{%
    Comparison of FID across different missing rates.
  }%
  \label{fig:fidcompare}
\end{figure}

\desc{Baselines}
We compare the {\misgan} imputer to a range of baseline methods
including the basic zero/mean imputation, matrix factorization,
and the recently proposed
Generative Adversarial Imputation Network (GAIN) \citep{yoon2018gain}.
GAIN is an imputation model that employs an imputer network to complete the
missing data. It is trained adversarially with a discriminator that
determines which entries in the completed data were actually observed
and which were imputed.
It has shown to outperform many state-of-the-art imputation methods.

\desc{Evaluation of imputation}
We impute all of the incomplete examples in the training set and
use the FID between the imputed data and the original fully-observed data
as the evaluation metric.\footnote{
See Appendix~\ref{sec:rmse} for a discussion of why we favor this
metric over evaluating metrics like RMSE between the imputed
missing values and the ground truth.}

\desc{Architecture}
We use convolutional generators and discriminators for {\misgan}
for all experiments in this section.
For MNIST, we use the same fully-connected imputer network
as described in the previous section;
for CIFAR-10 and CelebA, we use a five-layer U-Net architecture
\citep{ronneberger2015u}
for the imputer network $\widehat{G}_i$ in {\misgan}.

\desc{Results}
We compare all the methods under two missing patterns,
square observation and independent dropout, with
missing rates from 10\% to 90\%.
Figure~\ref{fig:fidcompare} shows that
{\misgan} consistently outperforms other methods
in all cases, especially under high missing rates.
In our experiments, we found GAIN training to be quite unstable
for the block missingness.
We also observed that there is a ``sweet spot''
for the number of training epochs when training GAIN.
%In our experiments, we train GAINs for only 50 epochs for the square
%observation case.
If trained longer, the imputation behavior will gradually become similar
to constant imputation (see Appendix~\ref{sec:gain} for details).
On the other hand,
we find that training {\misgan} is more stable than training GAIN
across all scenarios
in the experiments.
The imputation results of {\misgan} and GAIN are shown in
Appendix~\ref{sec:cifar10misgan}, \ref{sec:celebamisgan}, and \ref{sec:gain}.

\section{Discussion and future work}%
\label{sec:discussion}

This work presents and evaluates a highly flexible framework
for learning standard GAN data generators in the presence
of missing data.
Although we only focus on the MCAR case in this work, {\misgan}
can be easily extended to cases where the
output of the data generator is provided to the mask generator.
These modifications can capture both MAR and NMAR mechanisms.
The question of learnability requires further
investigation as the analysis in Section~\ref{sec:theory} no longer
holds due to dependence between the transition matrix and
the data distribution under MAR and NMAR.
We have tried this modified architecture in our experiments
and it showed similar results as to the original {\misgan}.
This suggests that the extra dependencies may not adversely
affect learnability.
We leave the formal evaluation of this modified framework
for future work.

\bibliography{paper}
\bibliographystyle{iclr2019_conference}

\newpage
\clearpage
\appendix
\label{appendix}
%\section{Necessary condition for the unique optimality}

\section{Proof of Theorem~\ref{thm:nullspace} and Theorem~\ref{thm:marginals}}
\label{sec:proof}

Let $\P$ be the finite set of feature values.
For the $n$-dimensional case,
let $\M=\{0,1\}^n$ be the set of
masks and $\I=\P^n$ be the set of all possible feature vectors.
Also let $\D_\M$ be the set of probability
distributions on $\M$, which implies
$\mm\succeq\0$ and $\sum_{\vv\in\I}\mm(\vv)=1$
for all $\mm\in\M$, where $\mm(\vv)$ denotes the entry of $\mm$ indexed
by $\vv$.

Given $\tau\in \P$ and $\qq\in \D_\M$, define the transformation
\begin{align*}
  T_{\qq,\tau} : \R^\I & \to \R^\I \\
  \xx &\mapsto \yy = T_{\qq,\tau} \xx
\end{align*}
by
\begin{equation}\label{eq:T}
\yy(\vv) =
(T_{\qq,\tau} \xx)(\vv)=\sum_{\mm\in \M}
\sum_{\uu\in \I} \qq(\mm) \xx(\uu) \ind{\uu\odot\mm + \tau \bar\mm = \vv},
\
\text{ for all }\vv\in\I
\end{equation}
where $\odot$ is the entry-wise multiplication and
$\1\{\cdot\}$ is the indicator function.

Given $\mm\in \M$,
define an equivalent relation $\sim_\mm$ on $\I$ by $\vv\sim_\mm \uu$
iff $\vv\odot \mm = \uu\odot \mm$, and denote by $[\vv]_\mm$
the equivalence class containing $\vv$.

Given $\qq\in \D_\M$, let $\S_\qq\subset \M$
be the support of $\qq$, that is,
\[
  \S_\qq = \{\mm\in \M: \qq(\mm) > 0\}.
\]

Given $\tau\in \P$ and $\vv\in \I$,
let $\M_{\tau,\vv}$ denote the set of masks
consistent with $\vv$ in the sense that
$\qq(\mm)>0$ and $\vv\odot \bar \mm= \tau \bar
\mm$, that is,
\[
  \M_{\tau,\vv} = \{\mm\in \S_\qq: \vv\odot \bar \mm= \tau \bar \mm\}.
\]

\begin{proposition}
  \label{prop:p1}
  \mbox{
    For any $\qq\in \D_\M$ and $\xx\in \R^\I$, the collection of marginals
    $\left\{\xx([\vv]_\mm): \vv\in\I, \mm\in \S_\qq\right\}$
  }
  determines $T_{\qq,\tau}\xx$ for all $\tau\in \P$
  where $\xx([\vv]_\mm):=\sum_{\uu\in [\vv]_\mm} \xx(\uu)$.
\end{proposition}

\begin{proof} This is clear from the following equation
  \begin{equation}
    \label{eq:T2} T_{\qq,\tau}\xx(\vv) =\sum_{\mm\in   \M_{\tau,\vv}} \qq(\mm)
    \xx([\vv]_\mm),
  \end{equation}
  which can be obtained from \eqref{eq:T} as follows,
  \begin{align*}
    T_{\qq,\tau}\xx(\vv)
    &=
    \sum_{\mm\in \S_\qq} \sum_{\uu\in \I} \qq(\mm) \xx(\uu)
    \ind{\uu\odot \mm = \vv\odot \mm} \ind{\tau \bar \mm = \vv\odot \bar \mm}
    \\
    &=
    \sum_{\mm\in \S_\qq}
    \qq(\mm) \ind{\tau \bar \mm = \vv\odot \bar \mm}
    \sum_{\uu\in \I}\xx(\uu)
    \ind{\uu\odot \mm = \vv\odot \mm}
    \\
    &=
    \sum_{\mm\in \S_\qq}  \qq(\mm) \ind{\tau \bar \mm = \vv\odot \bar \mm}
    \xx([\vv]_\mm)
    \\
    &=
    \sum_{\mm\in  \M_{\tau,\vv}} \qq(\mm) \xx([\vv]_\mm).
\end{align*} \end{proof}

\begin{proposition}
  \label{prop:p2}
  For any $\tau\in \P$, $\qq\in \D_\M$ and $\xx\in \R^\I$,
  the vector $T_{\qq,\tau}\xx$ determines the collection of marginals
  $\{\xx([\vv]_\mm):\vv\in\I, \mm\in \S_\qq\}$.
\end{proposition}

\begin{proof}
  Fix $\tau\in \P$, $\qq\in \D_\M$ and $\xx\in \R^\I$.
  Since $\vv\odot \mm + \tau\bar \mm \in [\vv]_\mm$,
  it suffices to show that we can solve for $\xx([\vv]_\mm)$
  in terms of $T_{\qq,\tau}\xx$ for $\mm\in \M_{\tau,\vv}\neq \emptyset$.
  We use induction on the size of $\M_{\tau,\vv}$.

  First consider the base case $|\M_{\tau,\vv}| = 1$.
  Consider $\vv_0\in\I$ with $\M_{\tau, \vv_0} = \{\mm_0\}$.
  By \eqref{eq:T2},
  \[
    T_{\qq,\tau} \xx(\vv_0) = \qq(\mm_0) \xx([\vv_0]_{\mm_0}).
  \]
  Hence $\xx([\vv_0]_{\mm_0}) =T_{\qq,\tau}\xx(\vv_0)/\qq(\mm_0)$,
  which proves the base case.

  Now assume we can solve for $\xx([\vv]_\mm)$
  in terms of $T_{\qq,\tau}\xx$ for $\mm\in \S_\qq$
  and $\vv\in\I$ with $|\M_{\tau,\vv}|\leq k$.
  Consider $\vv_0\in\I$ with $|\M_{\tau,\vv_0}| = k+1$;
  if no such $\vv_0$ exists, the conclusion holds trivially.
  Let $\M_{\tau, \vv_0} = \{\mm_0, \mm_1, \dots, \mm_k\}$.
  We need to show that $T_{\qq,\tau}\xx$ determines
  $\xx([\vv_0]_{\mm_\ell})$ for $\ell=0,1,\dots,k$.
  By \eqref{eq:T2} again,
  \begin{equation}\label{eq:T3}
    T_{\qq,\tau}\xx(\vv_0) =\sum_{\ell=0}^k \qq(\mm_\ell)
    \xx([\vv_0]_{\mm_\ell}).
  \end{equation}
  Let $\mm = \bigwedge_{\ell=0}^k \mm_{\ell}$,
  which may or may not belong to $\S_\qq$.
  Note that
  \[
    \xx([\vv_0]_\mm)
    = \sum_{\vv\in [\vv_0]_{\mm\vee \bar \mm_\ell}} \xx([\vv]_{\mm_\ell})
    = \xx([\vv_0]_{\mm_\ell}) +
      \sum_{\vv\in [\vv_0]_{\mm\vee \bar \mm_\ell} \setminus \{\vv_0\}}
    \xx([\vv]_{\mm_\ell}),
  \]
  and hence
  \begin{equation}\label{eq:m_ell}
    \xx([\vv_0]_{\mm_\ell})
    = \xx([\vv_0]_\mm) -
    \sum_{\vv\in [\vv_0]_{\mm\vee \bar \mm_\ell} \setminus \{\vv_0\}}
    \xx([\vv]_{\mm_\ell}).
  \end{equation}
  Plugging \eqref{eq:m_ell} into \eqref{eq:T3} yields
  \begin{equation}\label{eq:m}
    \xx([\vv_0]_\mm)
    = \frac{1}{\sum_{\ell'=0}^k \qq(\mm_{\ell'})}
    \left[T_{\qq,\tau}
      \xx(\vv_0) + \sum_{\ell=0}^k \qq(\mm_\ell)
      \sum_{\vv\in [\vv_0]_{\mm\vee \bar \mm_\ell} \setminus
    \{\vv_0\}} \xx([\vv]_{\mm_\ell})\right].
  \end{equation}
  Note that $\M_{\tau,\vv} \subset \M_{\tau,\vv_0} \setminus \{\mm_\ell\}$
  for $\vv\in [\vv_0]_{\mm\vee \bar \mm_\ell} \setminus \{\vv_0\}$,
  so $|\M_{\tau,\vv}|\leq k$.
  By the induction hypothesis,
  $\xx([\vv]_{\mm_\ell})$ is determined by $T_{\qq,\tau} \xx$.
  It follows from \eqref{eq:m} and \eqref{eq:m_ell} that $\xx([\vv_0]_\mm)$
  and $\xx([\vv_0]_{\mm_\ell})$ are also determined by
  $T_{\qq,\tau}\xx$. This completes the induction step.
\end{proof}

Theorem~\ref{thm:nullspace} is a direct consequence of
Proposition~\ref{prop:p1} and Proposition~\ref{prop:p2}
as the collection of marginals
$\left\{\xx([\vv]_\mm): \vv\in\I, \mm\in \S_\qq\right\}$
is independent of $\tau$.
Therefore,
if $\xx_1,\xx_2\in\R^\I$ satisfy
$T_{\qq,\tau_0}\xx_1=T_{\qq,\tau_0}\xx_2$ for some $\tau_0\in\P$,
then $T_{\qq,\tau}\xx_1=T_{\qq,\tau}\xx_2$ for all $\tau\in\P$.
Theorem~\ref{thm:nullspace} is a special case when $\xx_1=\0$.

Moreover, Proposition~\ref{prop:p2}
also shows that {\misgan} overall learns the distribution
$p(\xx_\obs, \mm)$, as
$\xx([\vv]_\mm)$ is equivalent to $p(\xx_\obs | \mm)$ and
$T_{\qq,\tau}\xx$ is essentially the distribution of
$f_\tau(\xx, \mm)$ under the optimally learned missingness $\qq=p(\mm)$.
Theorem~\ref{thm:marginals} basically restates Proposition~\ref{prop:p1}
and Proposition~\ref{prop:p2}.
This is also true when $\tau\notin\P$ according to
Appendix~\ref{sec:detailedcor}.

\section{Proof of Corollary~\ref{cor:newtau}}
\label{sec:detailedcor}

Corollary~\ref{cor:newtau}
can be shown by augmenting the set of feature values
by $\P' = \P\cup\{\psi\}$ with
a novel symbol $\psi\notin\P$.
If we choose $\tau=\psi$ for the masking operator,
whenever we spot a $\psi$ in a masked sample,
we know that it corresponds to a missing entry.
We can also construct the corresponding transition matrix
$T_{\qq,\psi}'\in\R^{\I'\times\I'}$ where $\I'=(\P')^n$
given the mask distribution $\qq\in\D_\M$ before.
In this setting, the generative model for missing data is equivalent to
solving the linear system $T_{\qq,\psi}\pp_x' = T_{\qq,\psi}{\pp_x^*}'$
so that $\pp_x'\in\R^{\I'}$ is non-negative and
$\pp_x'(\ss)=0$ for all $\ss\in\I'\setminus\I$,
where the true distribution ${\pp_x^*}'$ is given by
${\pp_x^*}'(\ss)=\pp_x^*(\ss)$ for all $\ss\in\I$ and zeros elsewhere.
Theorem~\ref{thm:nullspace} implies that if the solution to
original problem \eqref{eq:linsys} is not unique, the
non-negative solution to the augmented linear system
with the extra constraint on $\I'\setminus\I$
with $\tau=\psi$ is not unique either.

\section{Evaluation of imputation using root mean square error}
\label{sec:rmse}
Root mean square error (RMSE) is a commonly used metric for evaluating
the performance of missing data imputation, which
computes the RMSE of the imputed missing values against the ground truth.
However, in a complex system,
the conditional distribution $p(\xx_\mis | \xx_\obs)$
is likely to be highly multimodal.
It's not guaranteed that the ground truth of the missing values
in the incomplete dataset created under
the missing completely at random (MCAR) assumption
correspond to the global mode of $p(\xx_\mis | \xx_\obs)$.
A good imputation model might produce samples from
$p(\xx_\mis | \xx_\obs)$ associated with a higher density
than the ground truth (or from other modes that are similarly probable).
In this case, it will lead to a large error in terms of metrics like RMSE
as multiple modes might be far away from each other in a complex distribution.
Therefore, we instead compute the FID between the distribution
of the completed data and the distribution of the
originally fully-observed data as our evaluation metric.
This provides a practical way
to assess how close a model imputes according to $p(\xx_\mis | \xx_\obs)$
by comparing two groups of samples collectively.

As a concrete example, Figure~\ref{fig:fidrmse} compares the two evaluation
metrics on MNIST, our distribution-based FID and the ground truth-based RMSE.
It shows that the rankings on most of the missing rates
are not consistent across the two metrics.
In particular, under 90\% missing rate,
MisGAN is worse than GAIN and matrix factorization in terms of RMSE,
but significantly better in terms of FID.
Figure~\ref{fig:rmse_samples}
plots the imputation results of the three methods mentioned above.
We can clearly see that MisGAN produces the best completion
even though its RMSE is much higher than the other two.
It's not surprising as the mean of $p(\xx_\mis | \xx_\obs)$
minimizes the squared error in expectation,
even if the mean might have low density.
This probably explains why the blurry completion results produced by
matrix factorization achieve the lowest RMSE.

\begin{figure}[tb]
  \def\imgw{.49\textwidth}
  \def\imgh{1.8in}
  \hspace*{-.3in}
  \centering
  \includegraphics[height=\imgh]{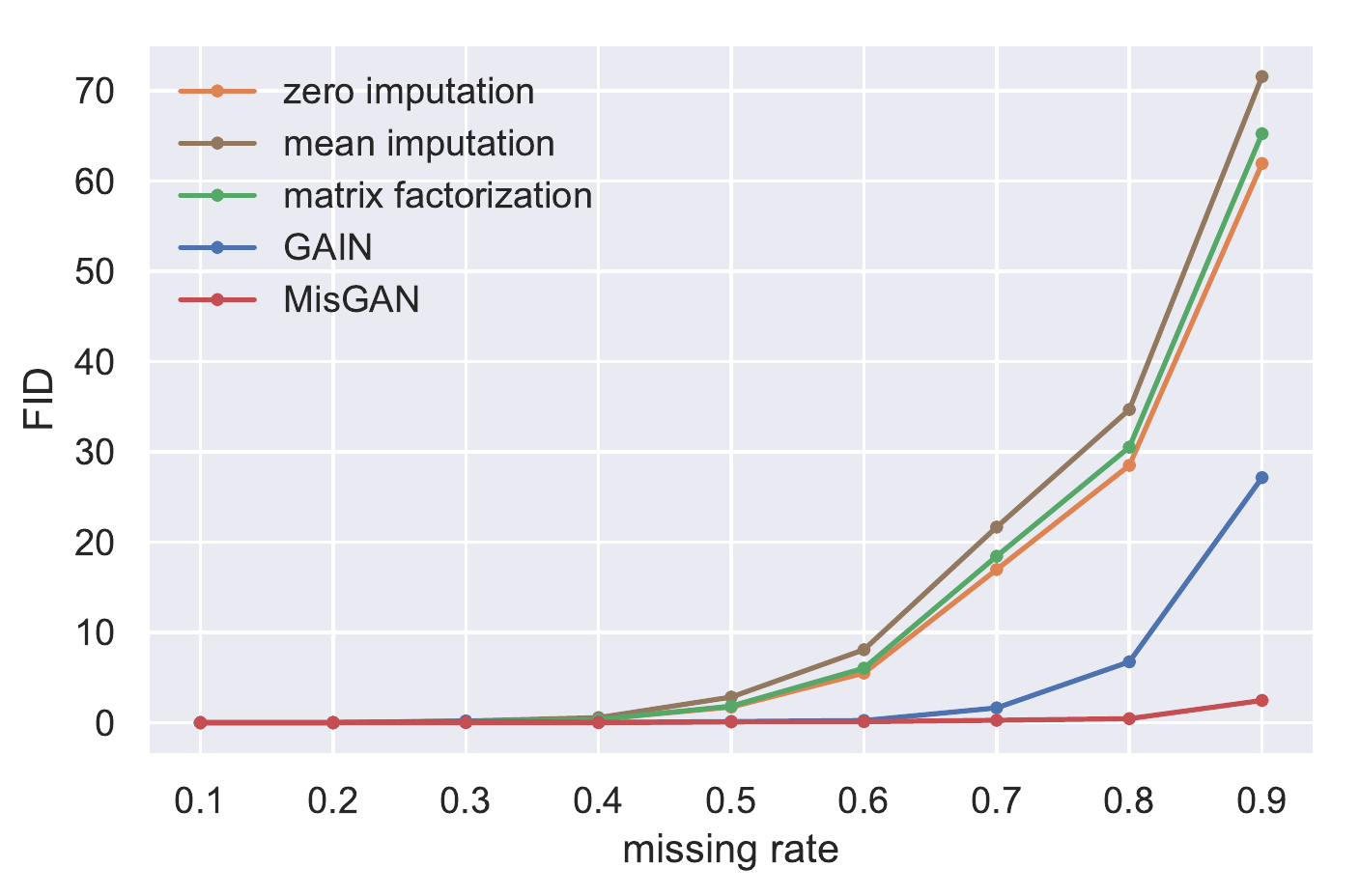}
  \includegraphics[height=\imgh]{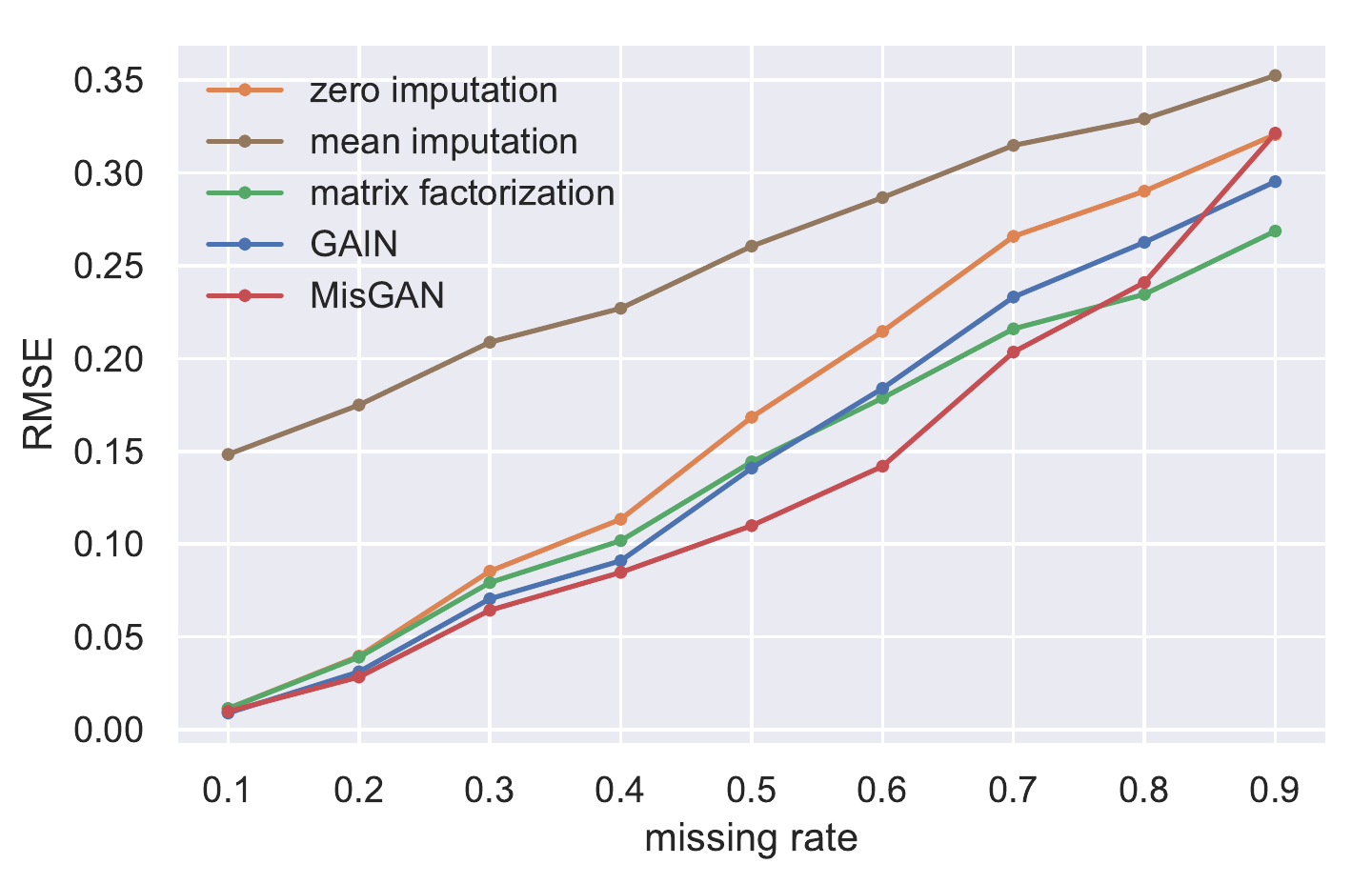}
  \caption{%
    Comparison of evaluating imputation using FID and RMSE
    (both the lower the better)
    on the MNIST dataset with block observation missingness.
    The rankings of the imputation methods are not consistent across
    the two metrics under most of the assessed missing rates.
  }%
  \label{fig:fidrmse}
\end{figure}

\begin{figure}[tb]
  \def\figwidth{.42\textwidth}
  \def\figgap{\figwidth+3em}
  %\hspace*{-.25in}
  \centering
  \begin{tikzpicture}
    \node[inner sep=0pt,
    label=below:{\small Ground truth samples}] (11) {
      \includegraphics[width=\figwidth]{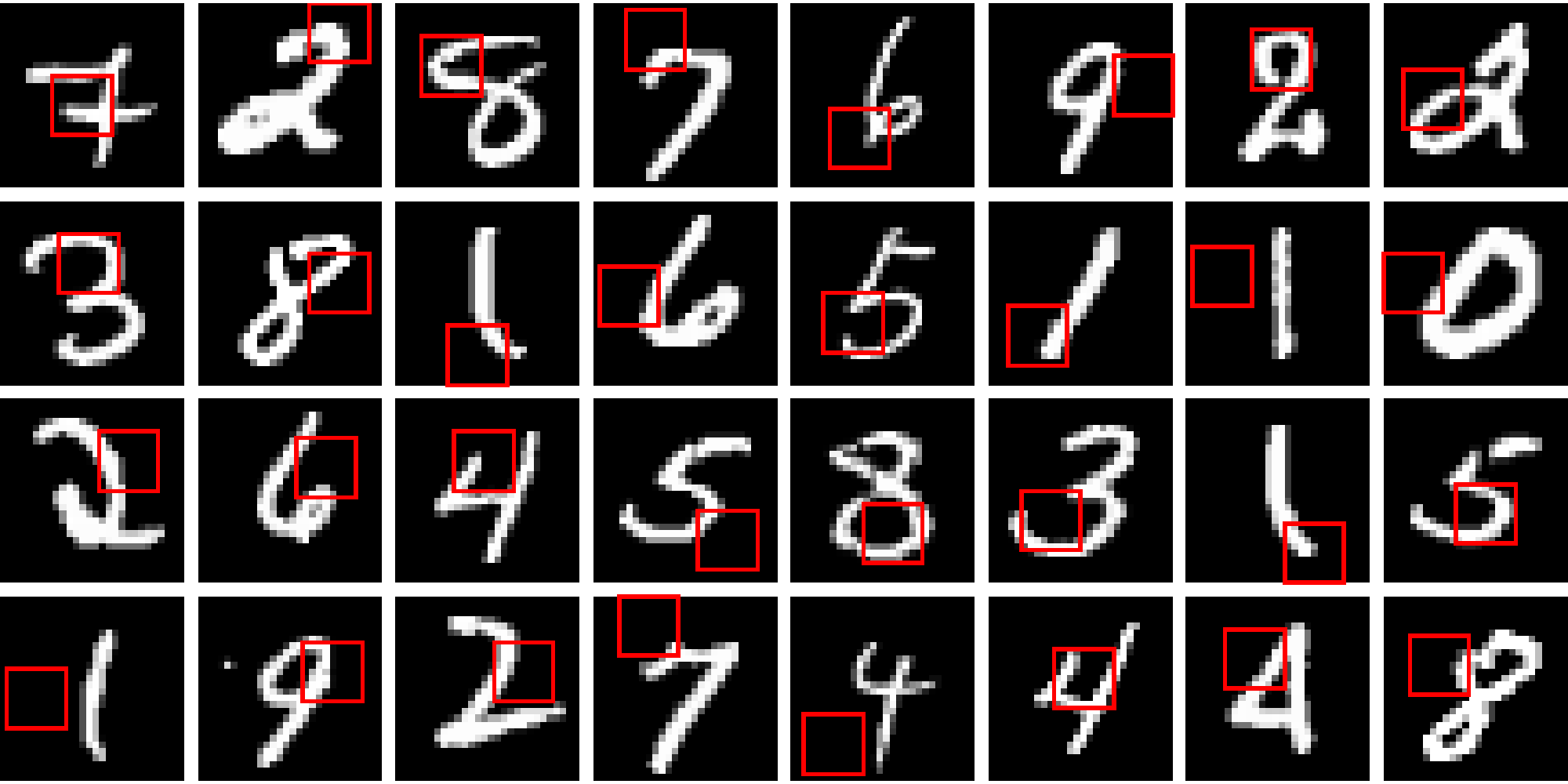}
    };
    \node[inner sep=0pt,right of=11,node distance=\figgap,
    label=below:{\small MisGAN (RMSE: 0.3214)}] (12) {
      \includegraphics[width=\figwidth]{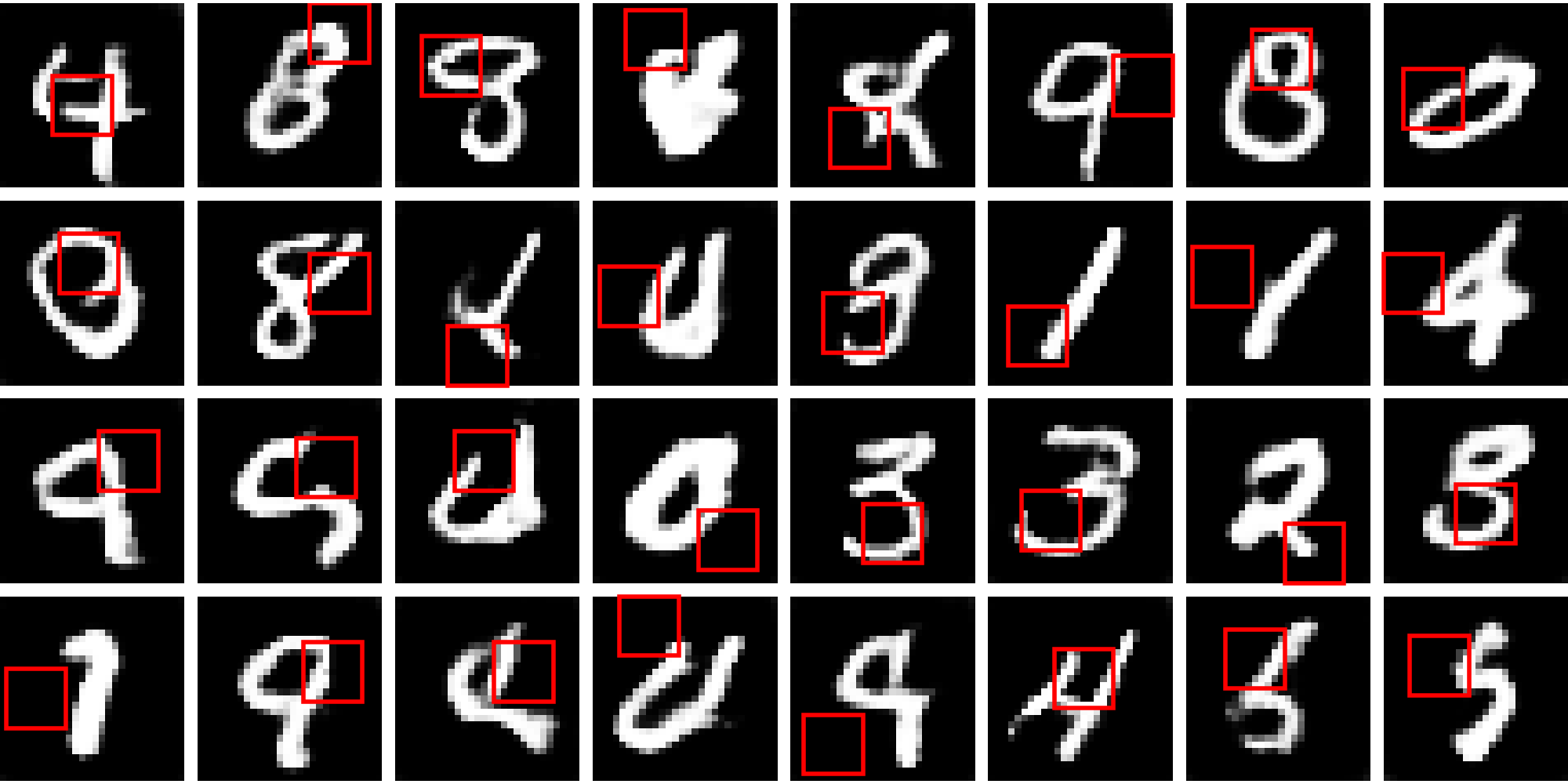}
    };
  \end{tikzpicture}
  \vspace{1em}

  \begin{tikzpicture}
    \node[inner sep=0pt,
    label=below:{\small GAIN (RMSE: 0.2953)}] (21) {
      \includegraphics[width=\figwidth]{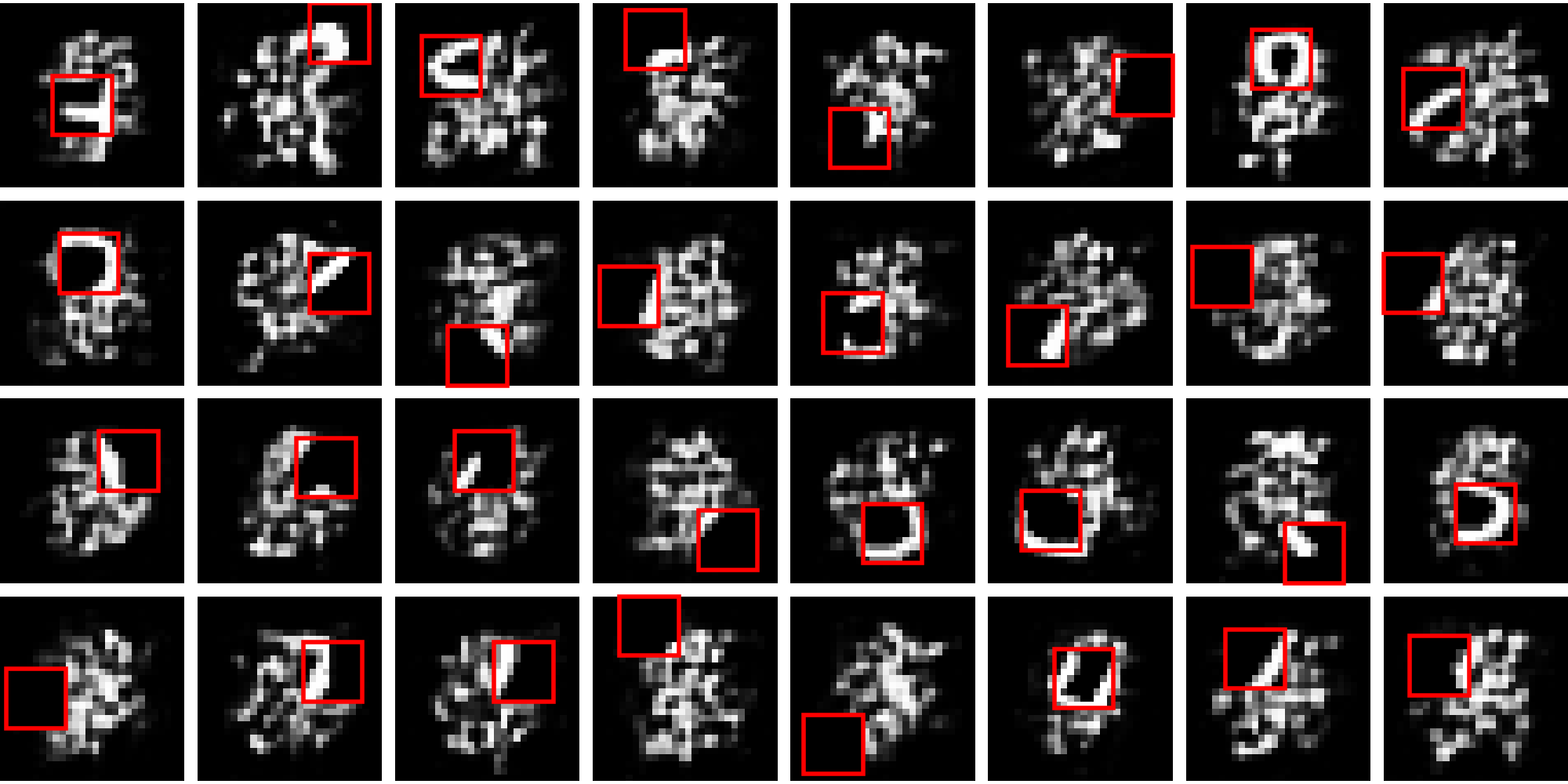}
    };
    \node[inner sep=0pt,right of=21,node distance=\figgap,
    label=below:{\small Matrix factorization (RMSE: 0.2686)}] (22) {
      \includegraphics[width=\figwidth]{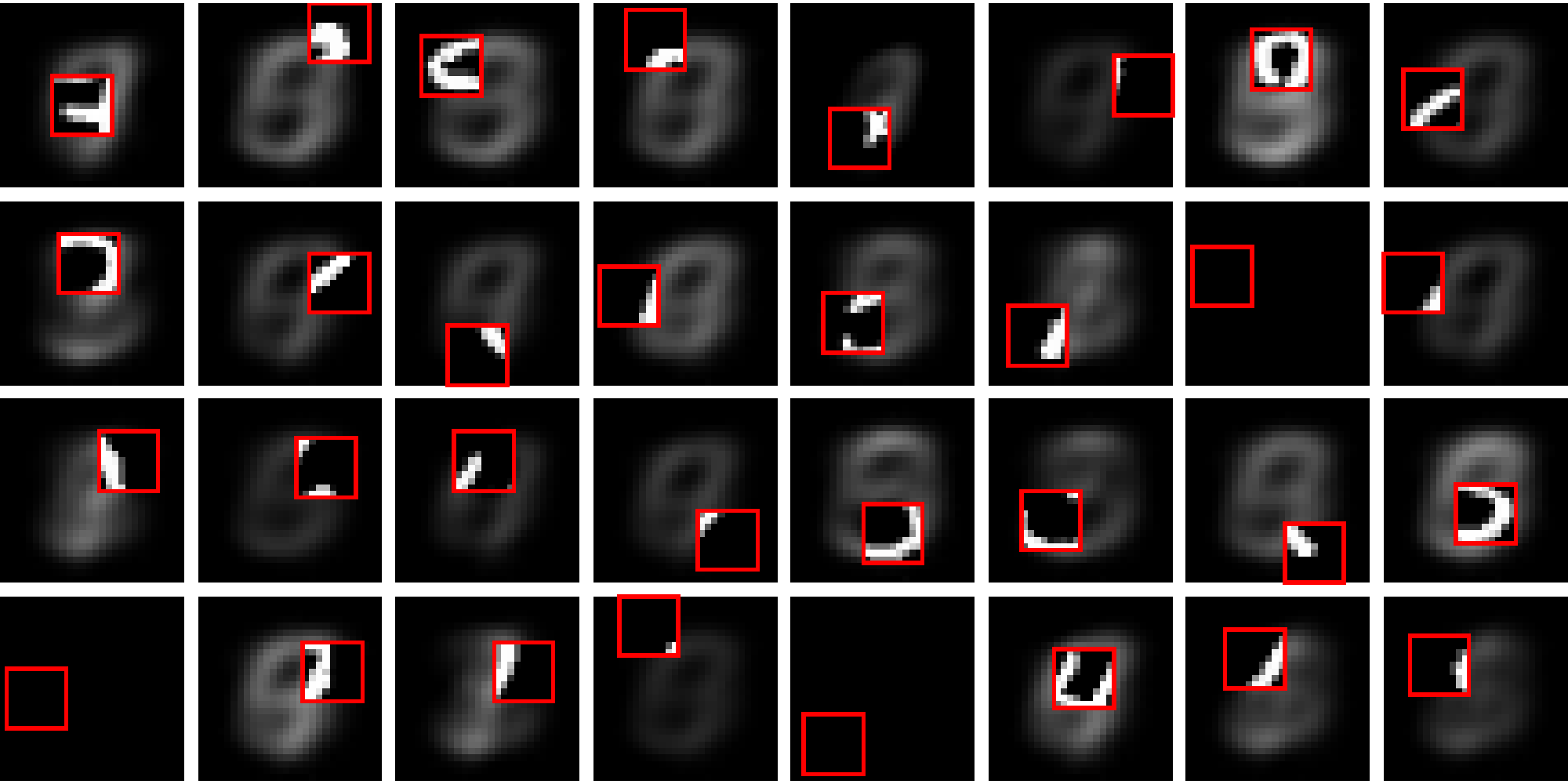}
    };
  \end{tikzpicture}
  \caption{%
    Imputation results by MisGAN, GAIN and matrix factorization
    along with the corresponding RMSE with block observation missingness
    under 90\% missing rate.
    Inside of each red box are the
    observed pixels; the pixels outside of the box
    are generated by the imputation methods.
  }%
  \label{fig:rmse_samples}
\end{figure}

\section{Architectural details and hyperparameters}
All of the generators and discriminators in Conv-{\misgan} follow the
architecture used by the DCGAN model
\citep{radford2015unsupervised} with 128-dimensional latent code.

As for FC-{\misgan}, the architecture of the generators is
\[
\text{FC}(128,256)\text{--}\text{FC}(256,512)\text{--}\text{FC}(512,784)
\]
with ReLUs in between.
The discriminators are of the structure
\[
\text{FC}(784,512)\text{--}\text{FC}(512,256)\text{--}\text{FC}(256,
128)\text{--}\text{FC}(128,1)
\]
also with ReLUs in between.

For the imputer network for {\misgan} trained on CIFAR-10 and CelebA,
we follow the U-Net implementation of the CycleGAN and pix2pix work\footnote{
\url{https://github.com/junyanz/pytorch-CycleGAN-and-pix2pix}}.
In the experiments, we use 5-layer U-Nets for both CIFAR-10 and CelebA.

For training Wasserstein GAN with gradient penalty,
We use all the default hyperparameters reported in
\citet{gulrajani2017improved}.
For all the datasets, {\misgan} is trained for 300 epochs.
We train {\misgan} imputer for 1000 epochs for MNIST and CIFAR-10 as
the networks are smaller and 600 epochs for CelebA.

For ConvAC, we use the same architecture described in
\citet{sharir2016tractable}.
We train ConvAC for 1000 epochs using Adam optimizer with learning rate
$10^{-4}$.

\section{{\misgan} on CIFAR-10}
\label{sec:cifar10misgan}
Figure~\ref{fig:cifar10block}, \ref{fig:cifar10indep} and
\ref{fig:cifar10impute}
show the results of {\misgan} trained on CIFAR-10
for the two extreme missing rates, namely
90\% and 80\%, as well as the case of 10\% that is
close to full observation.

\begin{figure}
  \def\figwidth{.32\textwidth}
  \def\figgap{\figwidth+.3em}

  \centering
  \begin{subfigure}[b]{\textwidth}
    \centering
    \begin{tikzpicture}
      \node[inner sep=0pt,
            label=below:{\small training samples}] (11) {
        \includegraphics[width=\figwidth]{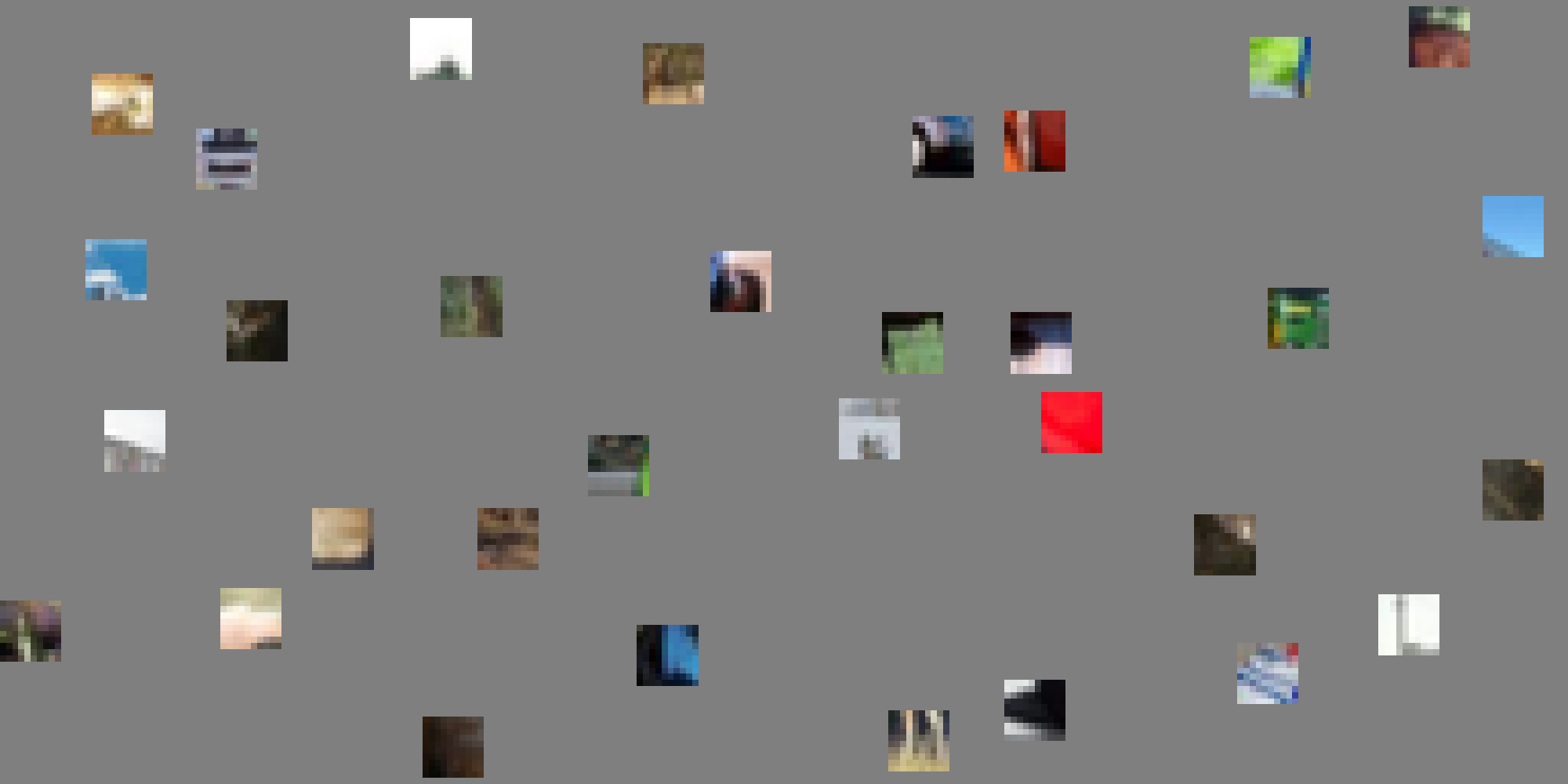}
      };
      \node[inner sep=0pt,right of=11,node distance=\figgap,
            label=below:{\small $G_x$ samples}] (12) {
        \includegraphics[width=\figwidth]{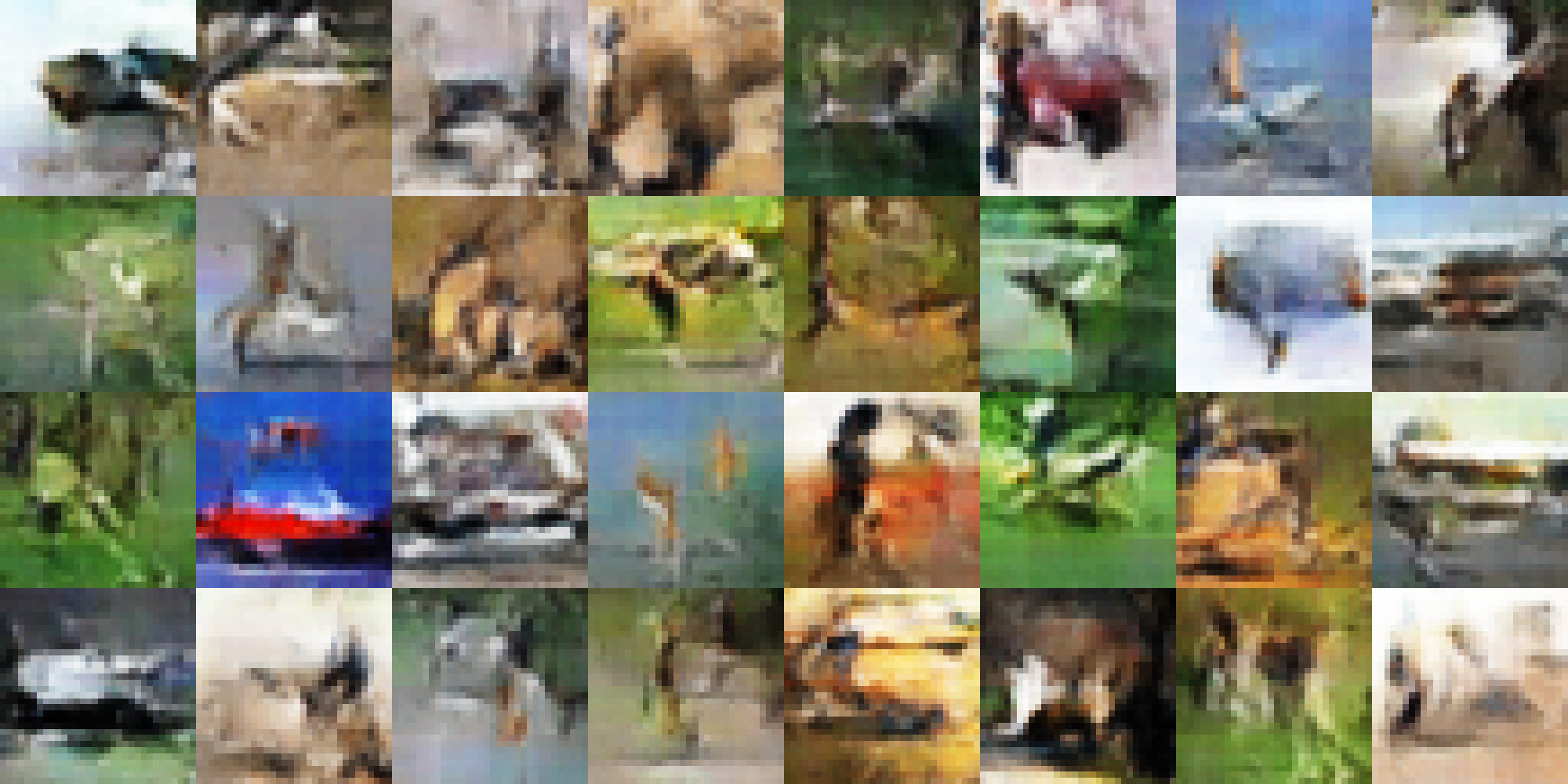}
      };
      \node[inner sep=0pt,right of=12,node distance=\figgap,
            label=below:{\small $G_m$ samples}] (13) {
        \includegraphics[width=\figwidth]{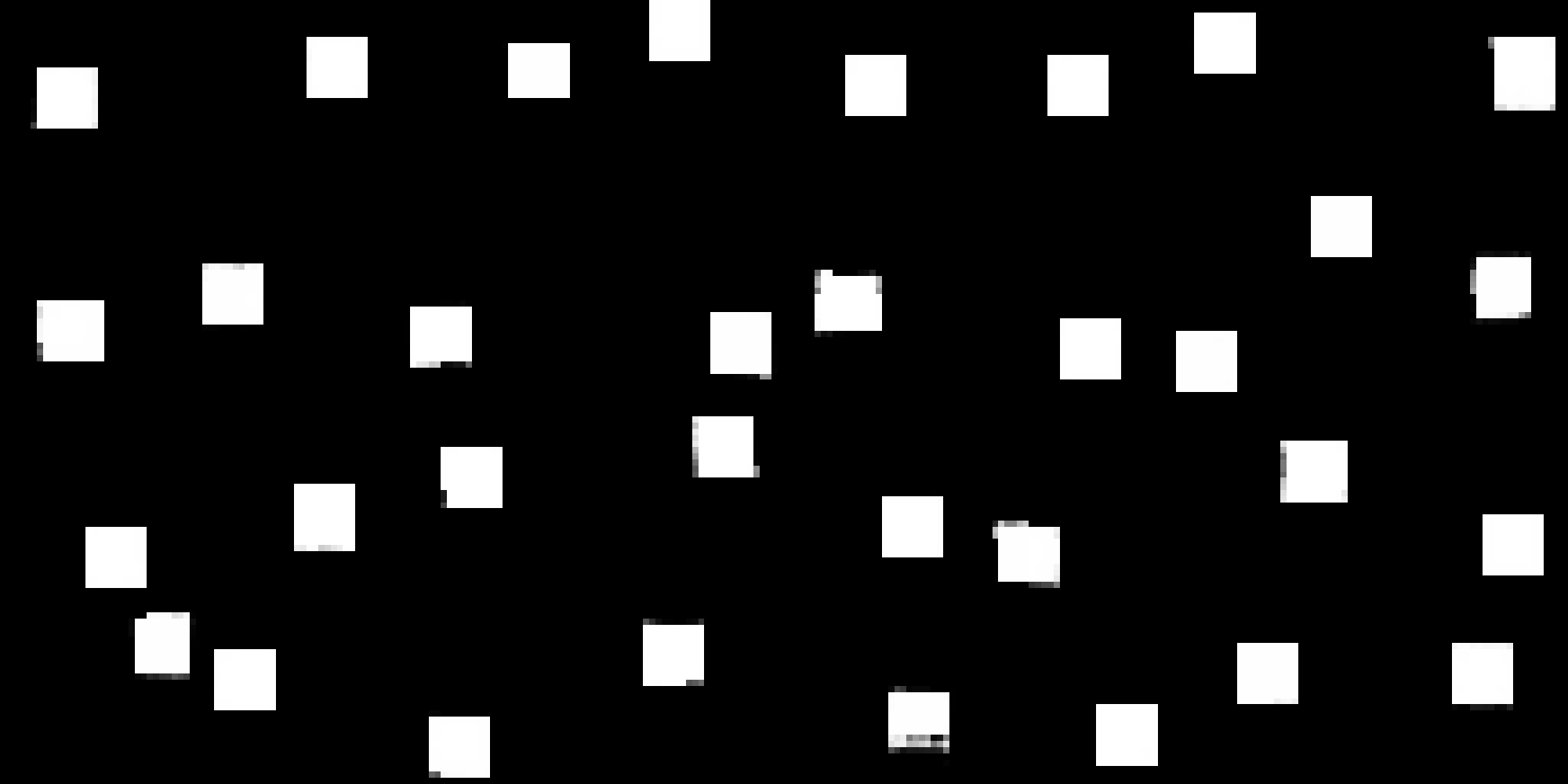}
      };
    \end{tikzpicture}
    \vspace*{-.5em}
    \caption{10$\times$10 block (90\% missing)}
  \end{subfigure}

  \vspace{1em}
  \begin{subfigure}[b]{\textwidth}
    \centering
    \begin{tikzpicture}
      \node[inner sep=0pt,
            label=below:{\small training samples}] (21) {
        \includegraphics[width=\figwidth]{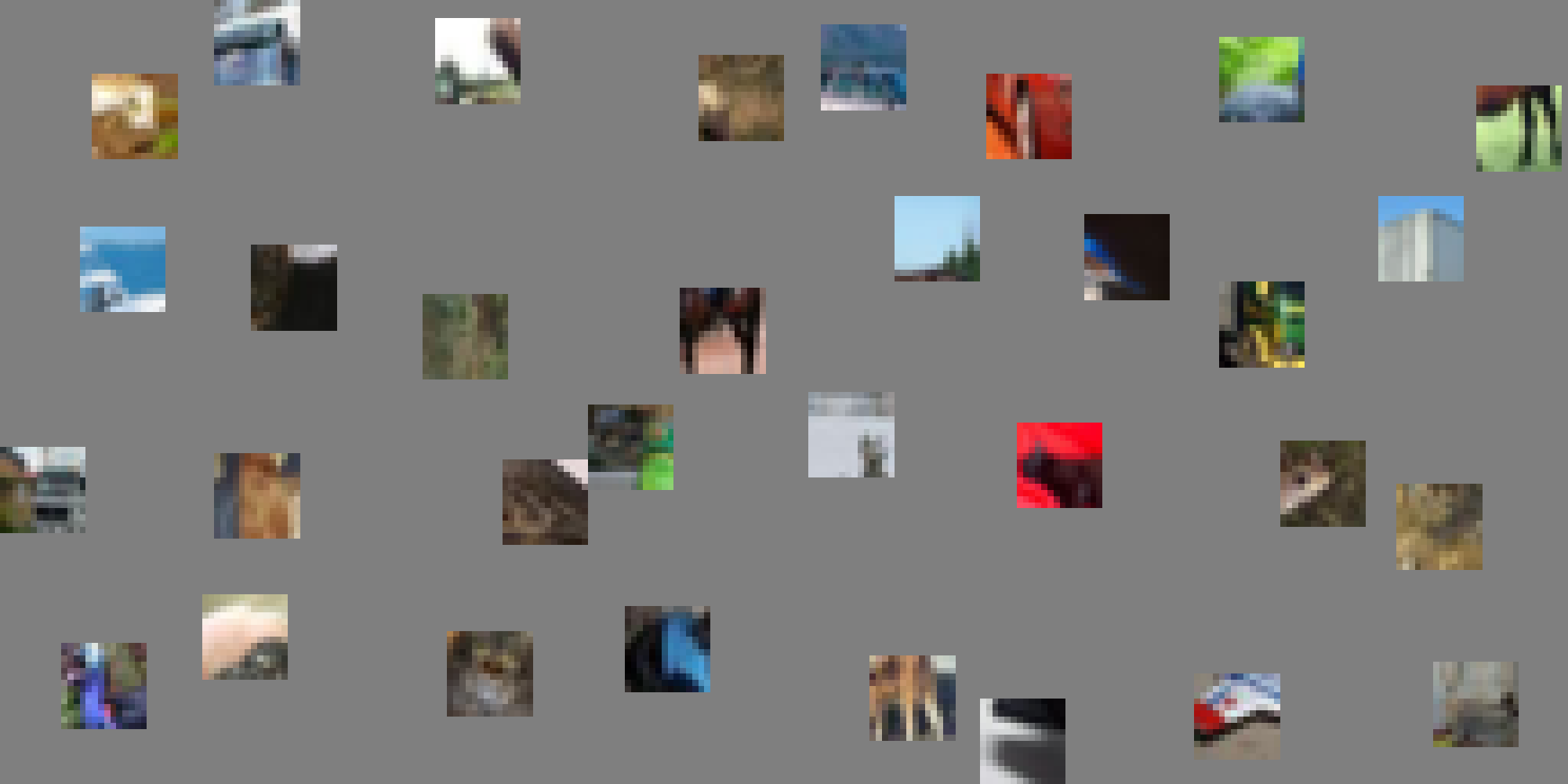}
      };
      \node[inner sep=0pt,right of=21,node distance=\figgap,
            label=below:{\small $G_x$ samples}] (22) {
        \includegraphics[width=\figwidth]{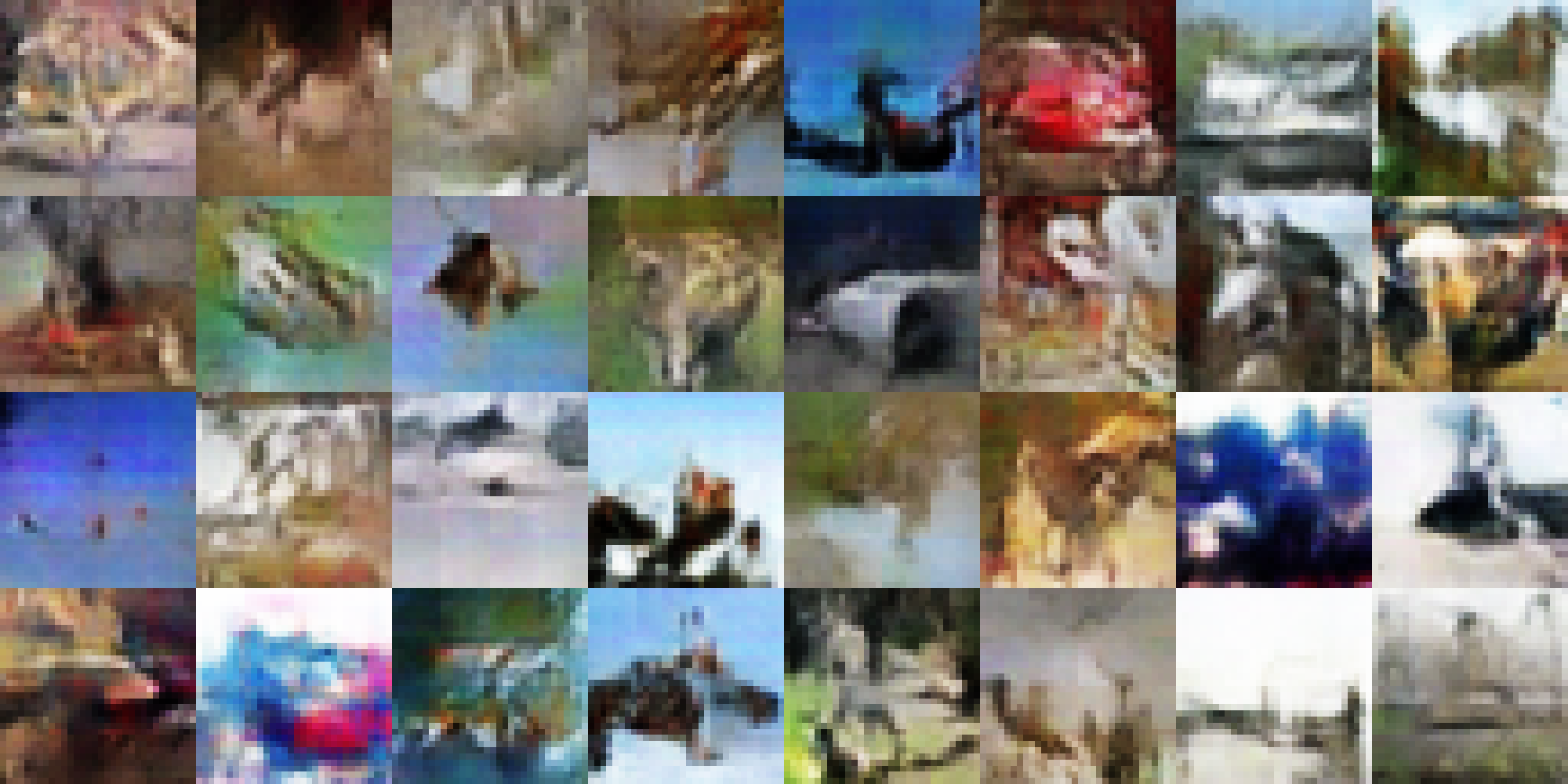}
      };
      \node[inner sep=0pt,right of=22,node distance=\figgap,
            label=below:{\small $G_m$ samples}] (23) {
        \includegraphics[width=\figwidth]{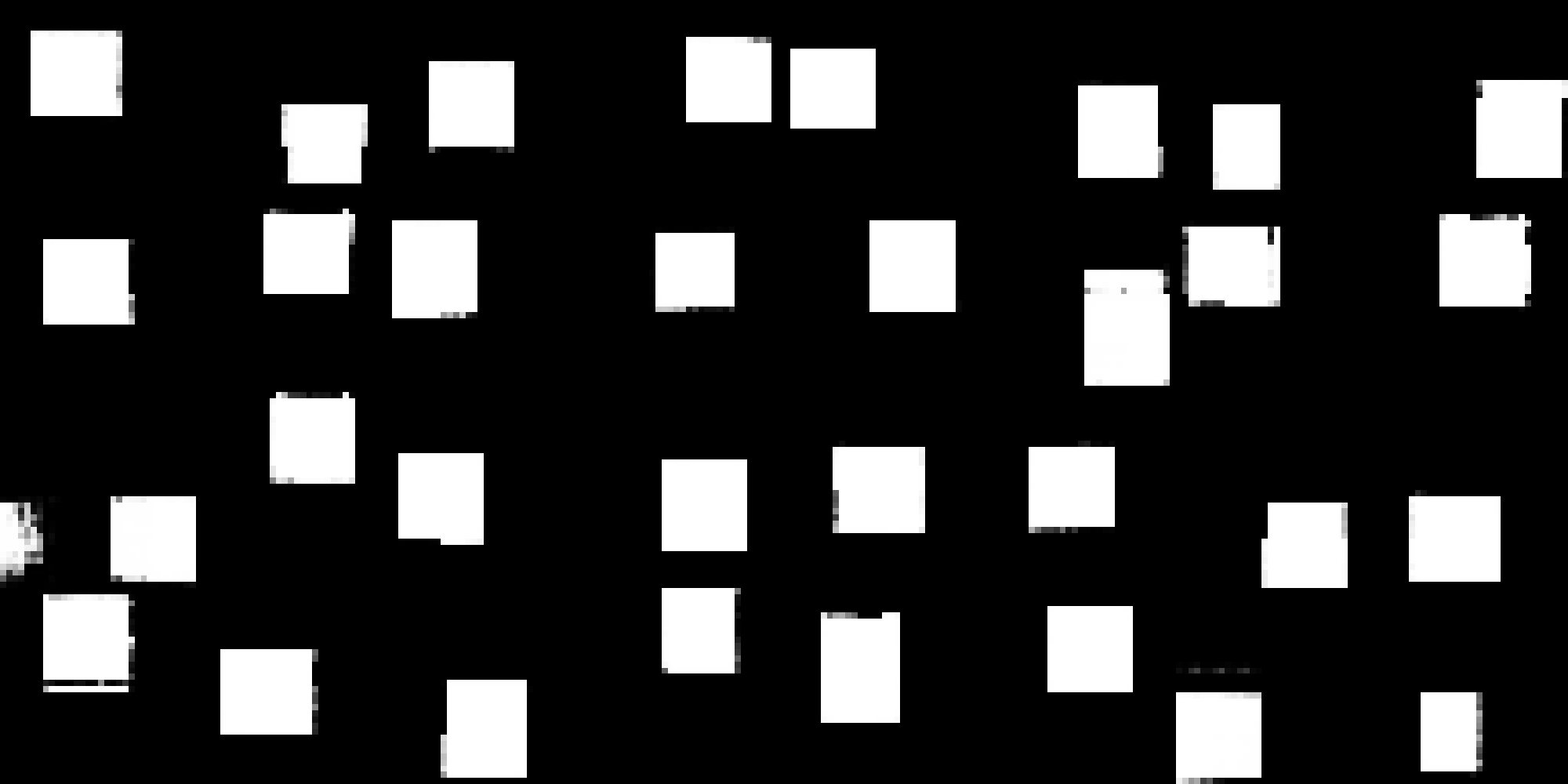}
      };
    \end{tikzpicture}
    \vspace*{-.5em}
    \caption{14$\times$14 block (80\% missing)}
  \end{subfigure}

  \vspace{1em}
  \begin{subfigure}[b]{\textwidth}
    \centering
    \begin{tikzpicture}
      \node[inner sep=0pt,
            label=below:{\small training samples}] (21) {
        \includegraphics[width=\figwidth]{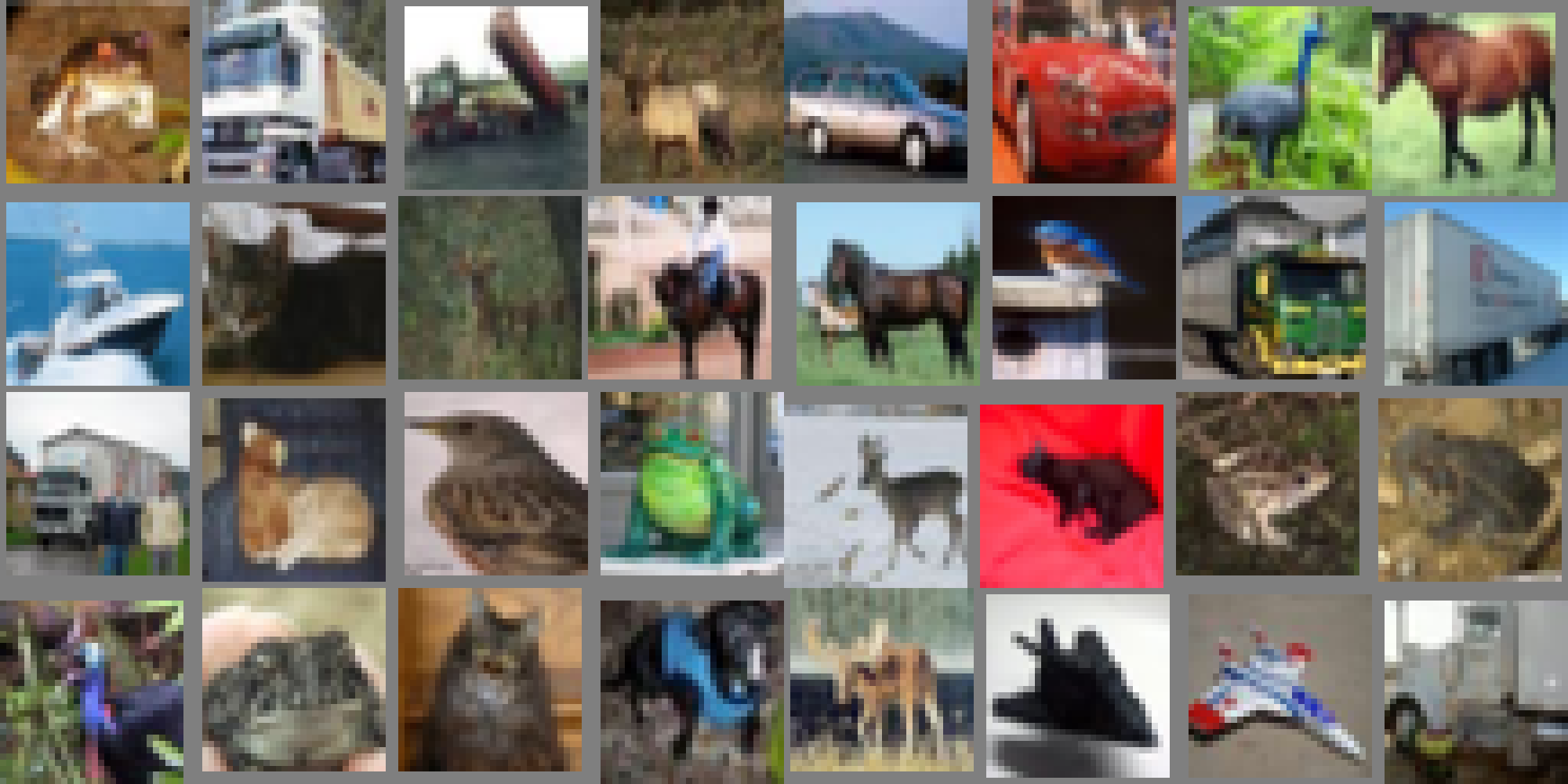}
      };
      \node[inner sep=0pt,right of=21,node distance=\figgap,
            label=below:{\small $G_x$ samples}] (22) {
        \includegraphics[width=\figwidth]{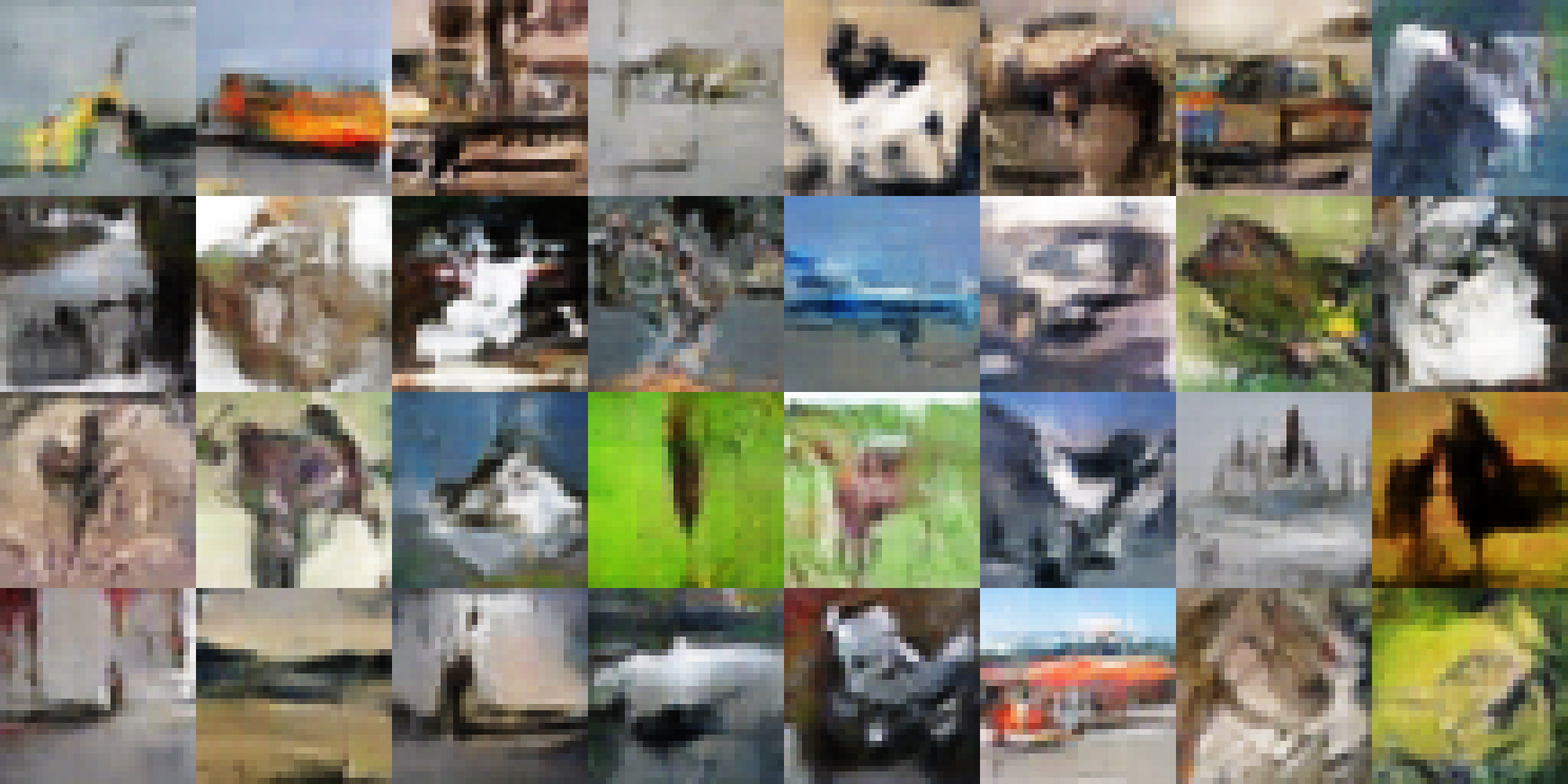}
      };
      \node[inner sep=0pt,right of=22,node distance=\figgap,
            label=below:{\small $G_m$ samples}] (23) {
        \includegraphics[width=\figwidth]{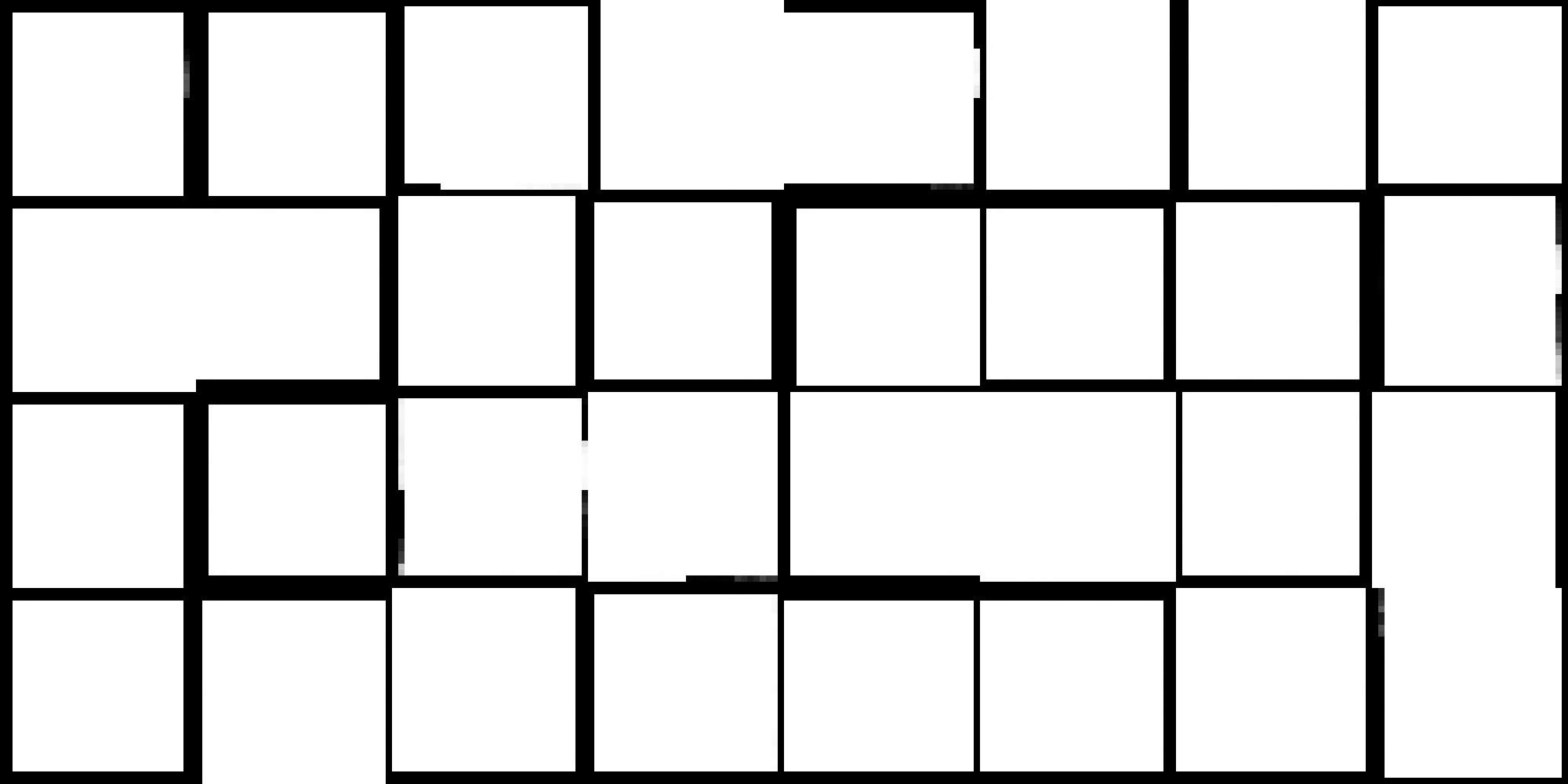}
      };
    \end{tikzpicture}
    \vspace*{-.5em}
    \caption{30$\times$30 block (10\% missing)}
  \end{subfigure}
  \caption{{\misgan} on CIFAR-10 with block observation missingness}
\label{fig:cifar10block}
\end{figure}

\begin{figure}
  \def\figwidth{.32\textwidth}
  \def\figgap{\figwidth+.3em}

  \centering
  \begin{subfigure}[b]{\textwidth}
    \centering
    \begin{tikzpicture}
      \node[inner sep=0pt,
            label=below:{\small training samples}] (11) {
        \includegraphics[width=\figwidth]{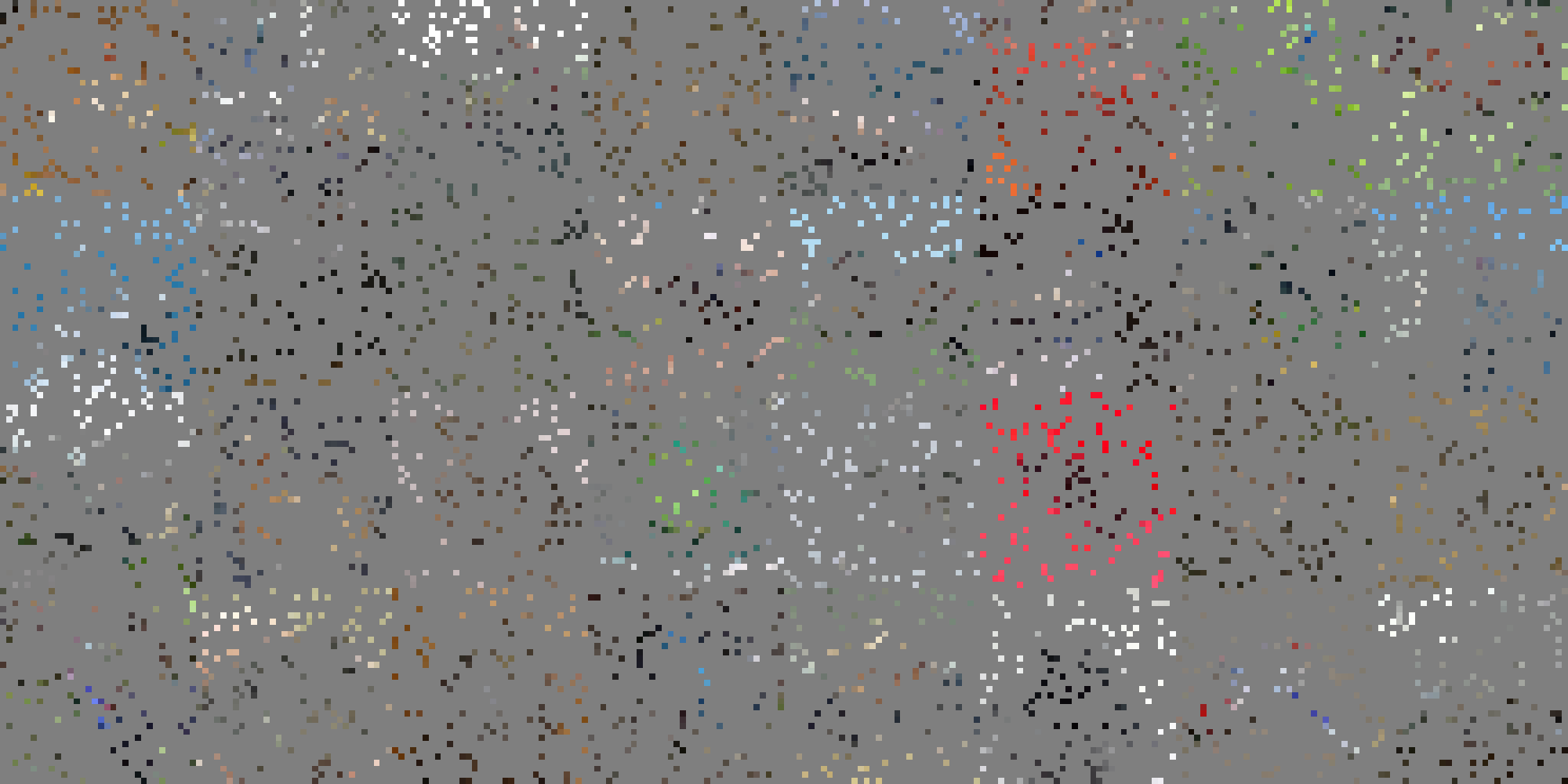}
      };
      \node[inner sep=0pt,right of=11,node distance=\figgap,
            label=below:{\small $G_x$ samples}] (12) {
        \includegraphics[width=\figwidth]{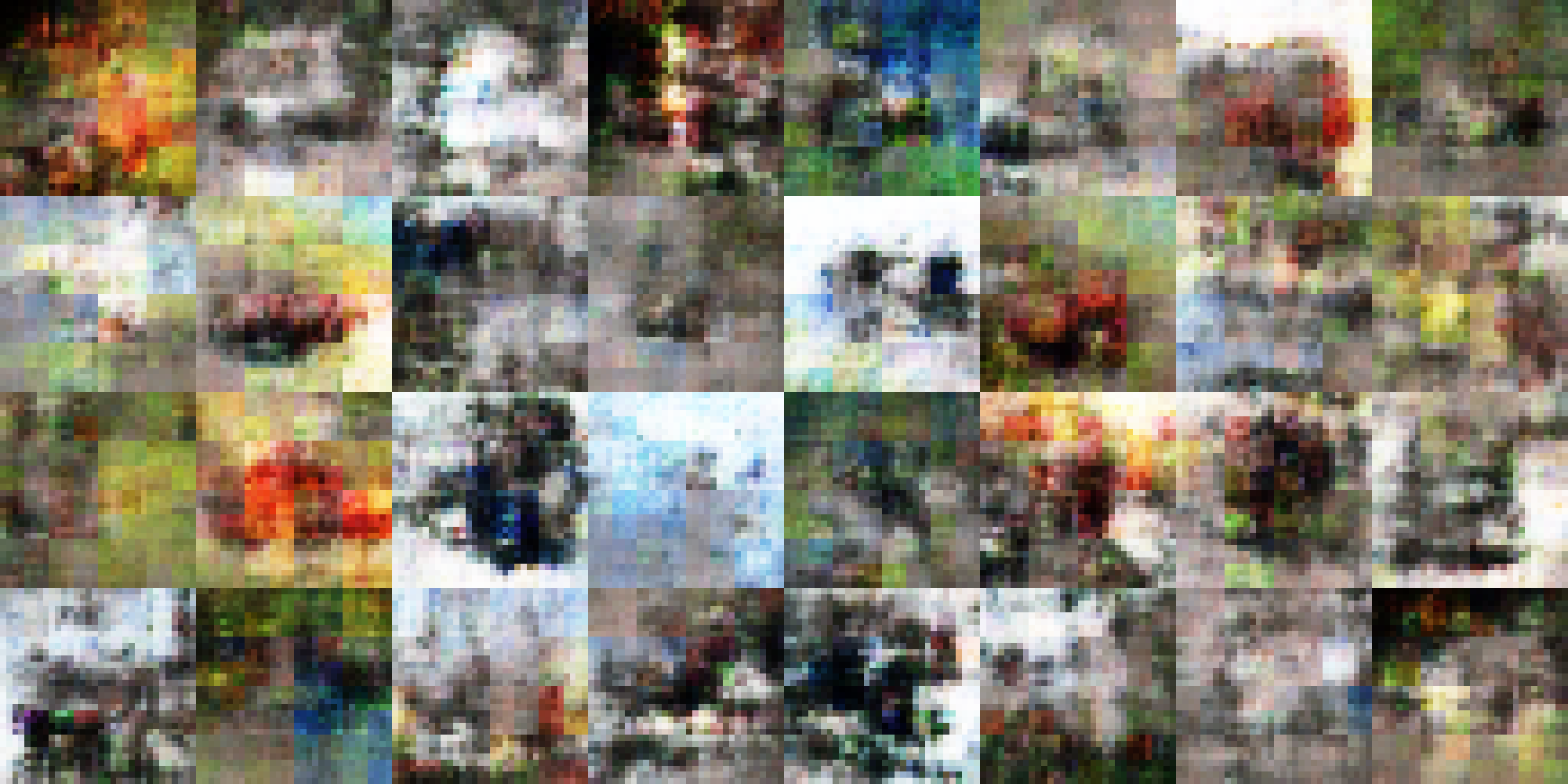}
      };
      \node[inner sep=0pt,right of=12,node distance=\figgap,
            label=below:{\small $G_m$ samples}] (13) {
        \includegraphics[width=\figwidth]{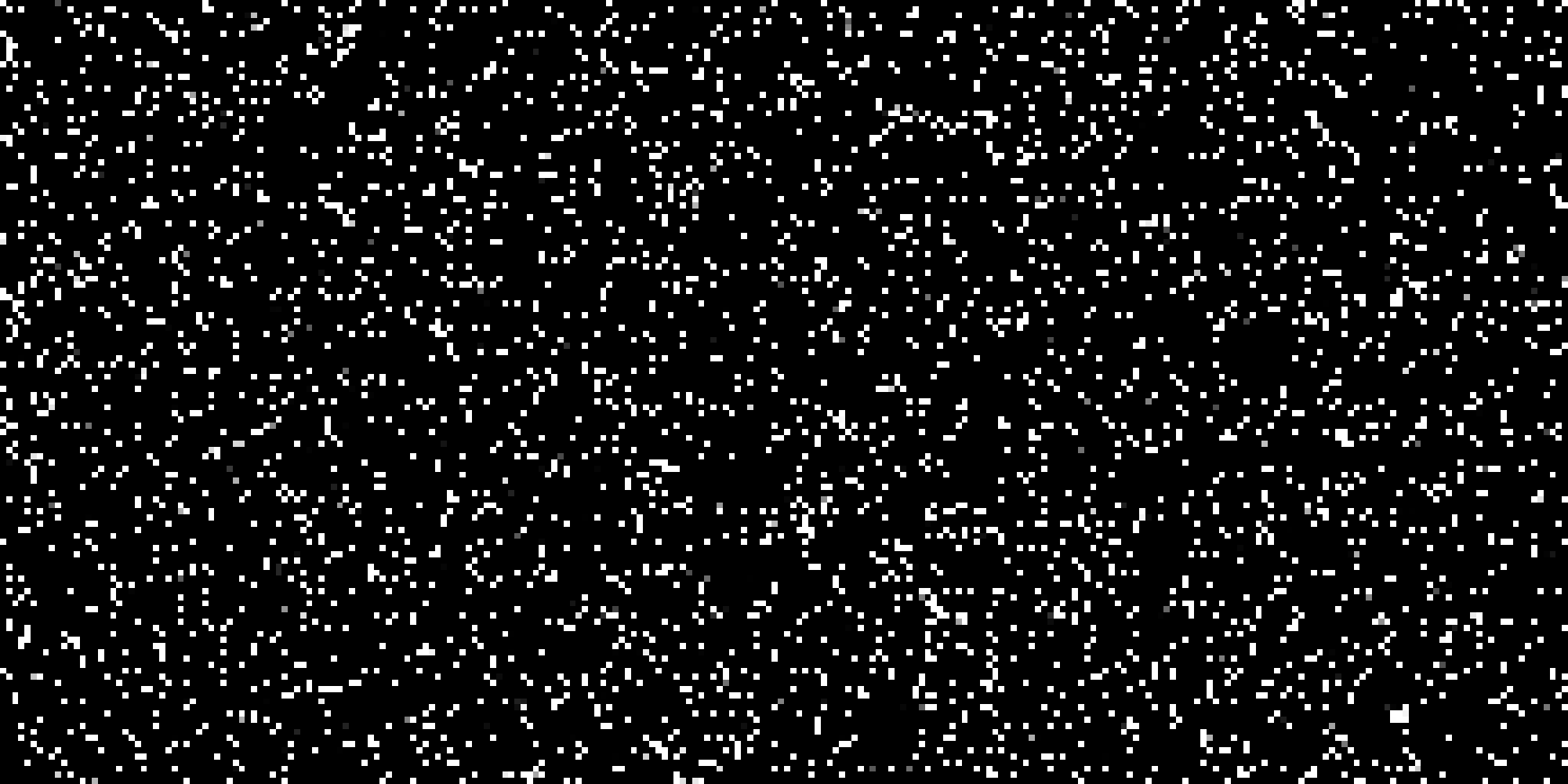}
      };
    \end{tikzpicture}
    \vspace*{-.5em}
    \caption{90\% missing}
  \end{subfigure}

  \vspace{1em}
  \begin{subfigure}[b]{\textwidth}
    \centering
    \begin{tikzpicture}
      \node[inner sep=0pt,
            label=below:{\small training samples}] (21) {
        \includegraphics[width=\figwidth]{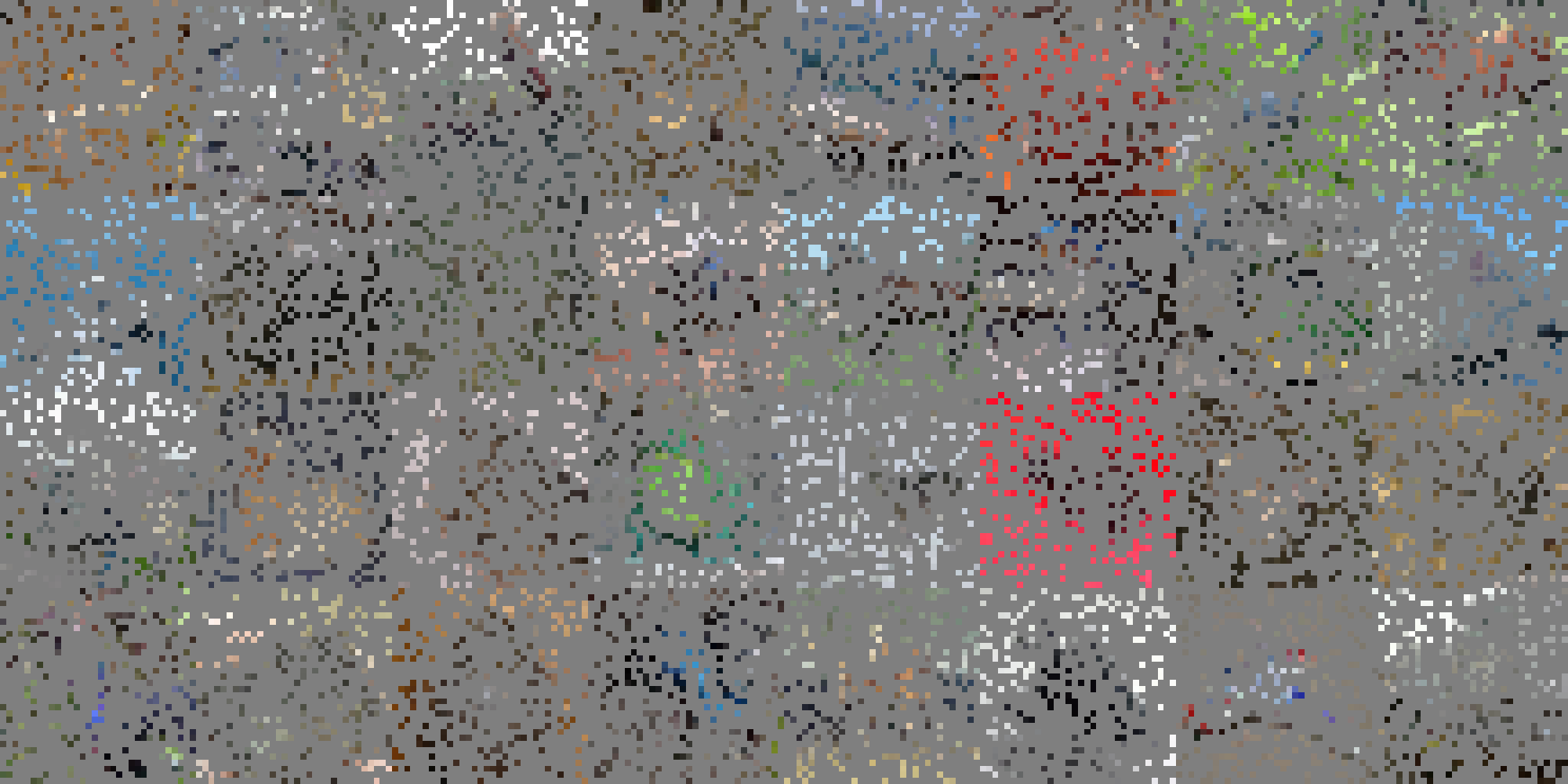}
      };
      \node[inner sep=0pt,right of=21,node distance=\figgap,
            label=below:{\small $G_x$ samples}] (22) {
        \includegraphics[width=\figwidth]{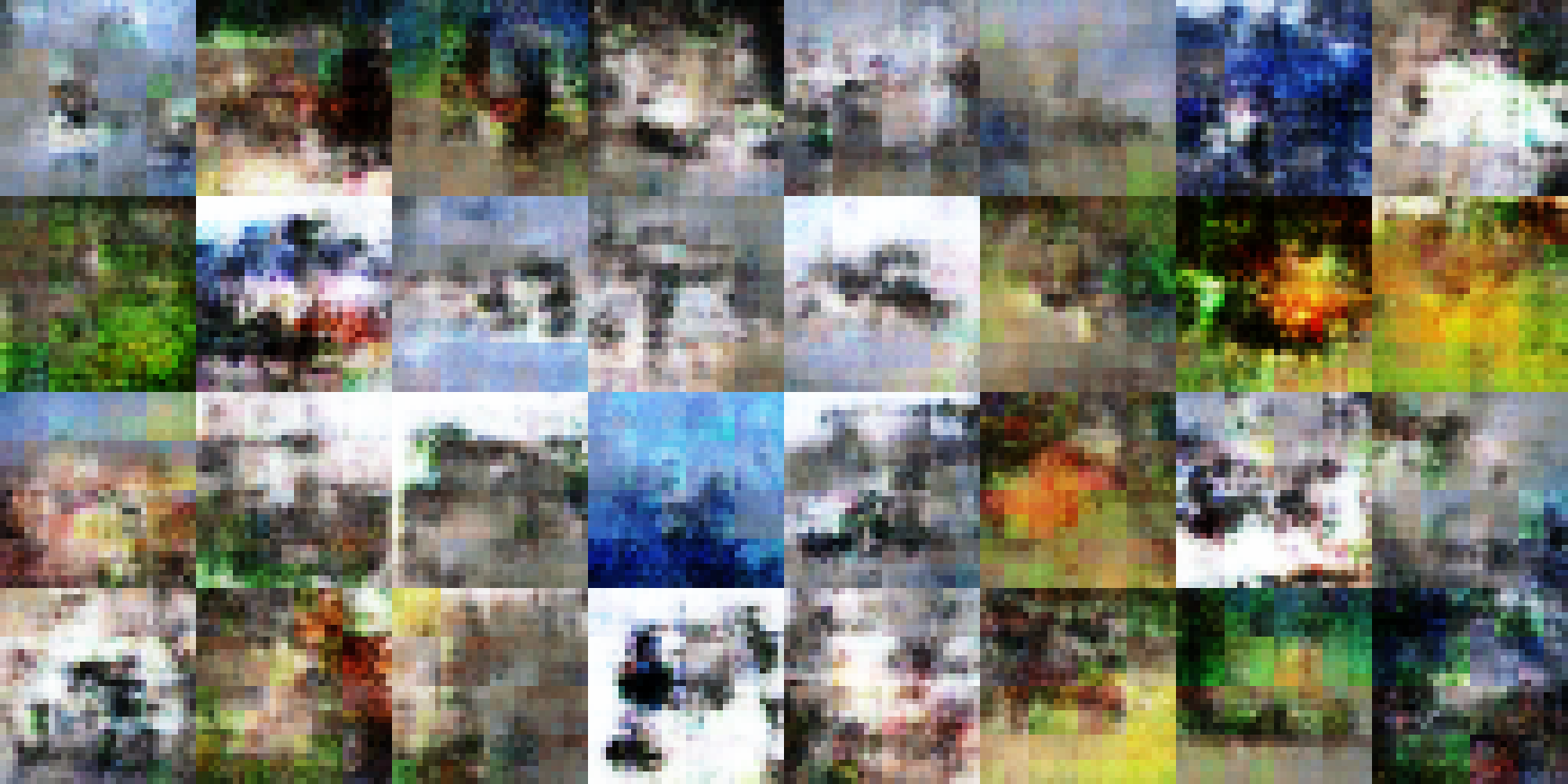}
      };
      \node[inner sep=0pt,right of=22,node distance=\figgap,
            label=below:{\small $G_m$ samples}] (23) {
        \includegraphics[width=\figwidth]{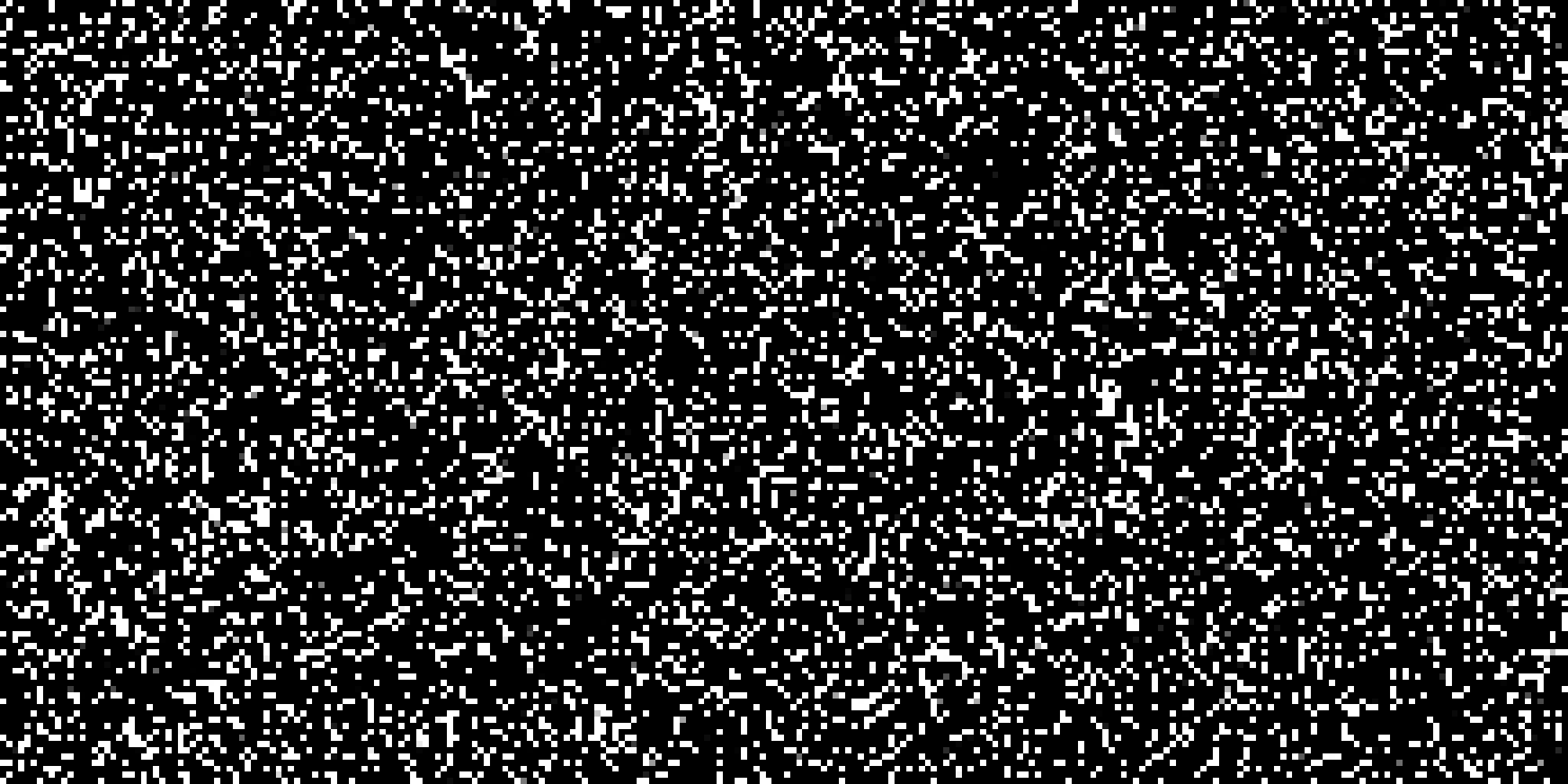}
      };
    \end{tikzpicture}
    \vspace*{-.5em}
    \caption{80\% missing}
  \end{subfigure}

  \vspace{1em}
  \begin{subfigure}[b]{\textwidth}
    \centering
    \begin{tikzpicture}
      \node[inner sep=0pt,
            label=below:{\small training samples}] (21) {
        \includegraphics[width=\figwidth]{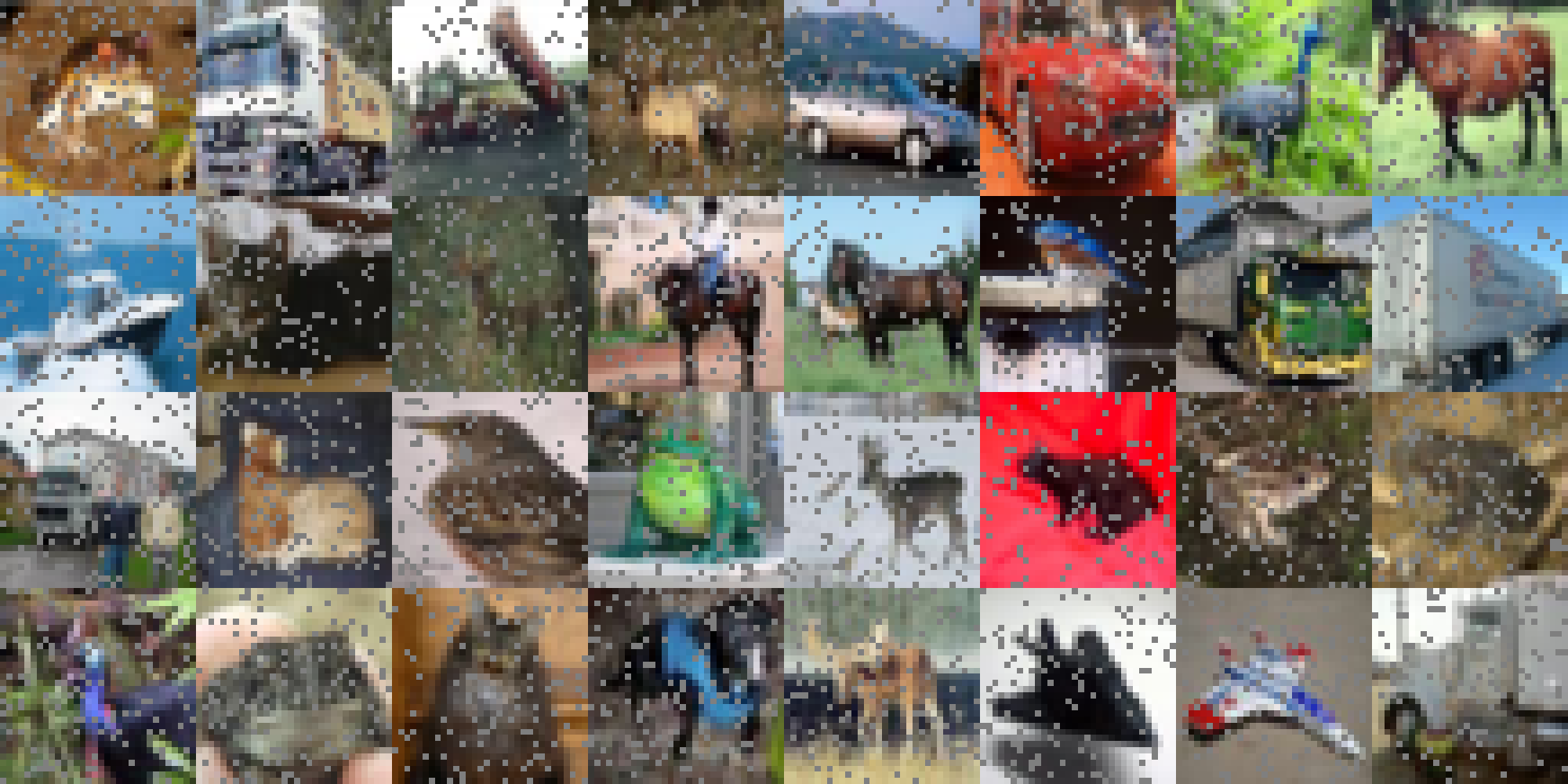}
      };
      \node[inner sep=0pt,right of=21,node distance=\figgap,
            label=below:{\small $G_x$ samples}] (22) {
        \includegraphics[width=\figwidth]{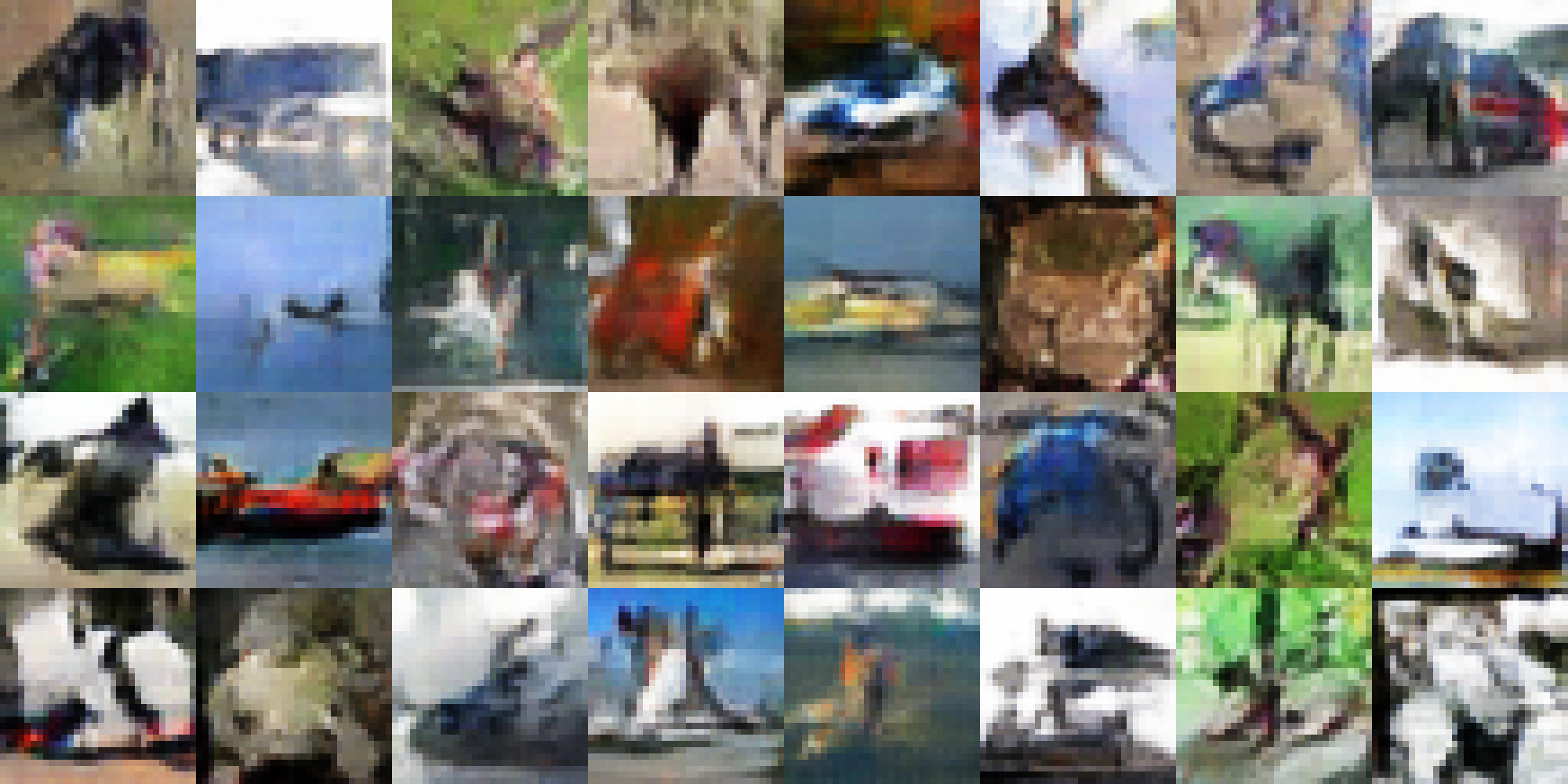}
      };
      \node[inner sep=0pt,right of=22,node distance=\figgap,
            label=below:{\small $G_m$ samples}] (23) {
        \includegraphics[width=\figwidth]{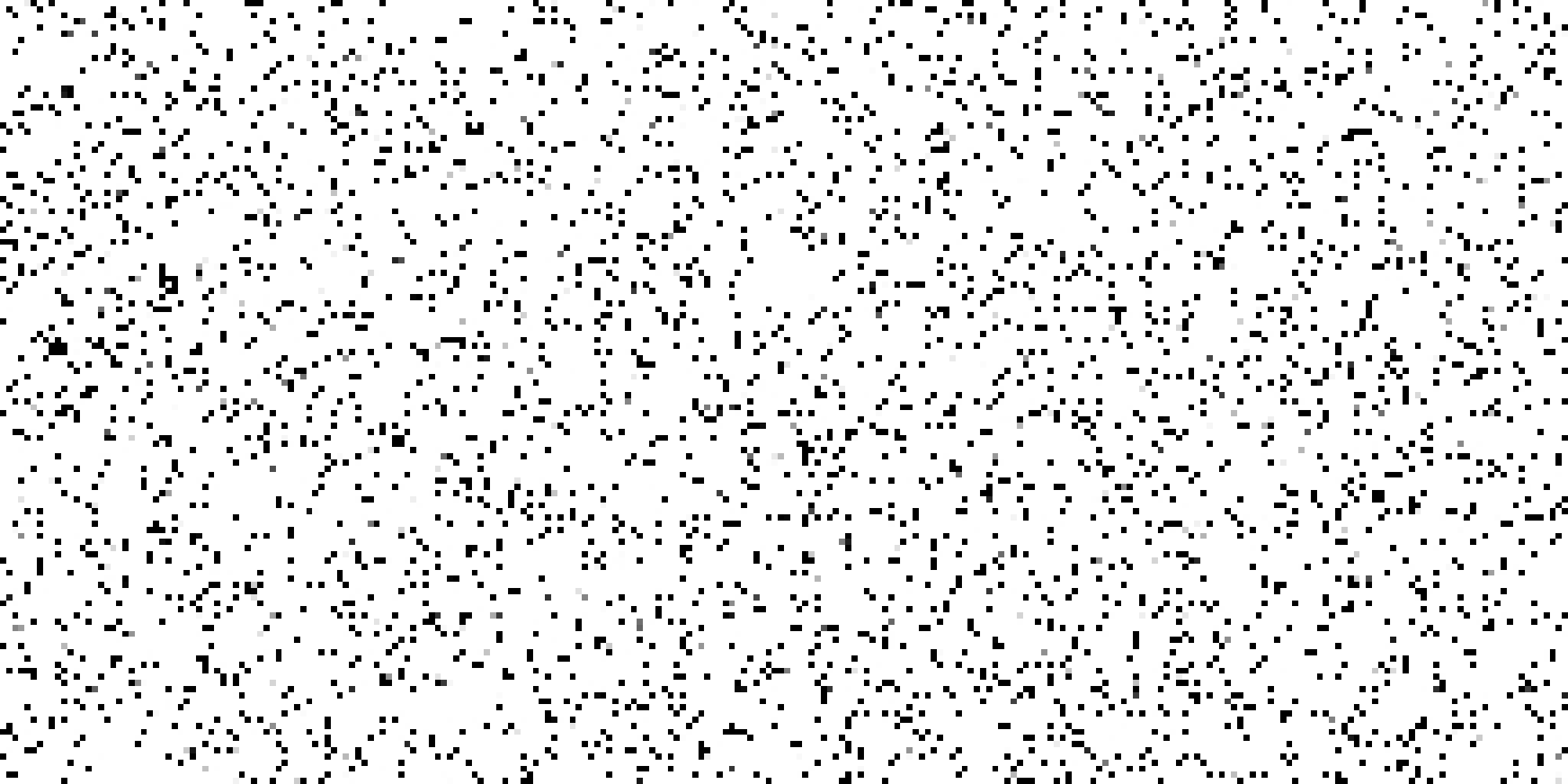}
      };
    \end{tikzpicture}
    \vspace*{-.5em}
    \caption{10\% missing}
  \end{subfigure}
  \caption{{\misgan} on CIFAR-10 with independent dropout missingness}
\label{fig:cifar10indep}
\end{figure}

\begin{figure}
  \def\figwidth{.32\textwidth}
  \def\figgap{\figwidth+.3em}

  \centering
  \begin{subfigure}[b]{\textwidth}
    \centering
    \begin{tikzpicture}
      \node[inner sep=0pt,
            label=below:{\small 10$\times$10 (90\% missing)}] (21) {
        \includegraphics[width=\figwidth]{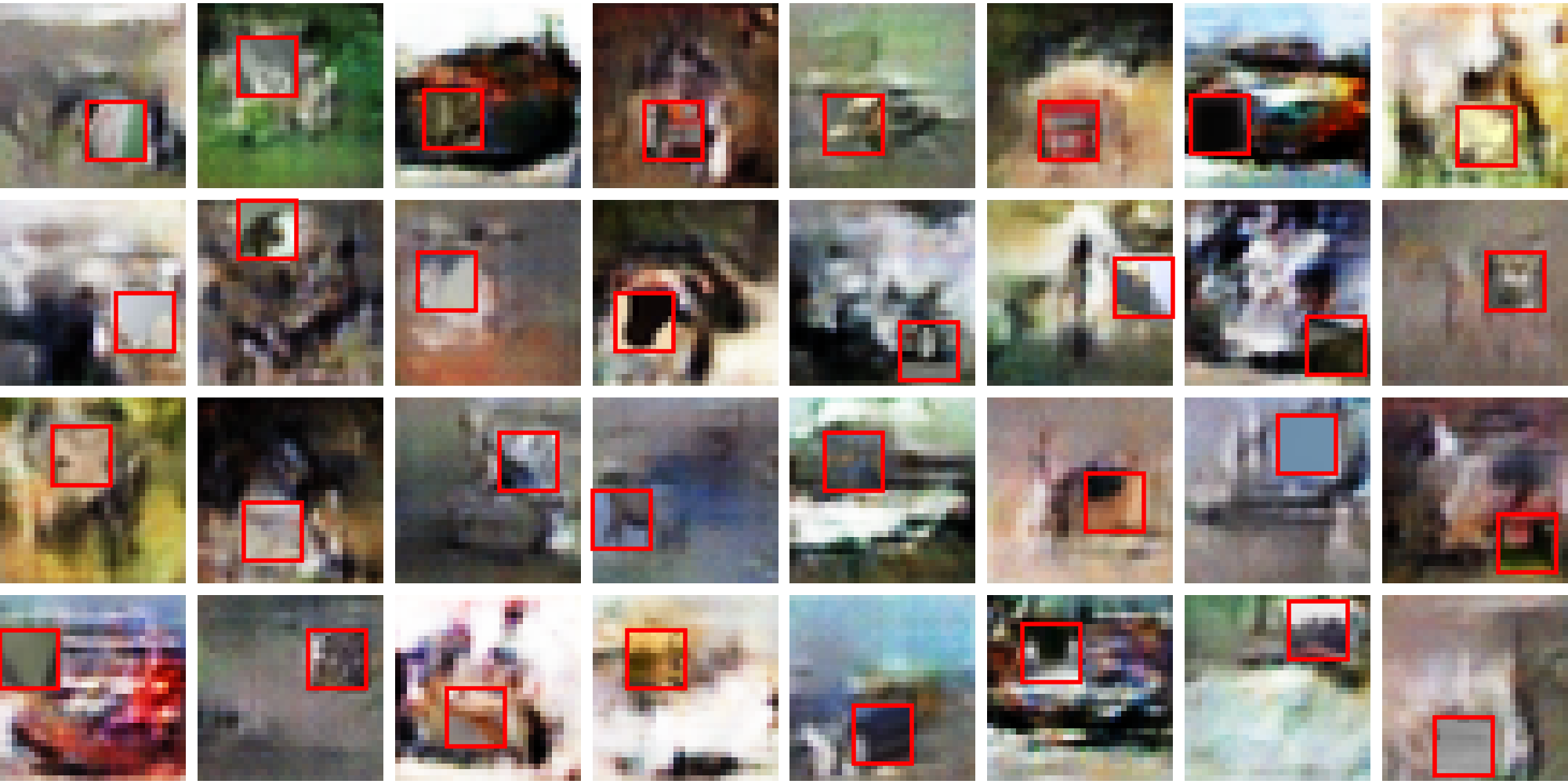}
      };
      \node[inner sep=0pt,right of=21,node distance=\figgap,
            label=below:{\small 14$\times$14 (80\% missing)}] (22) {
        \includegraphics[width=\figwidth]{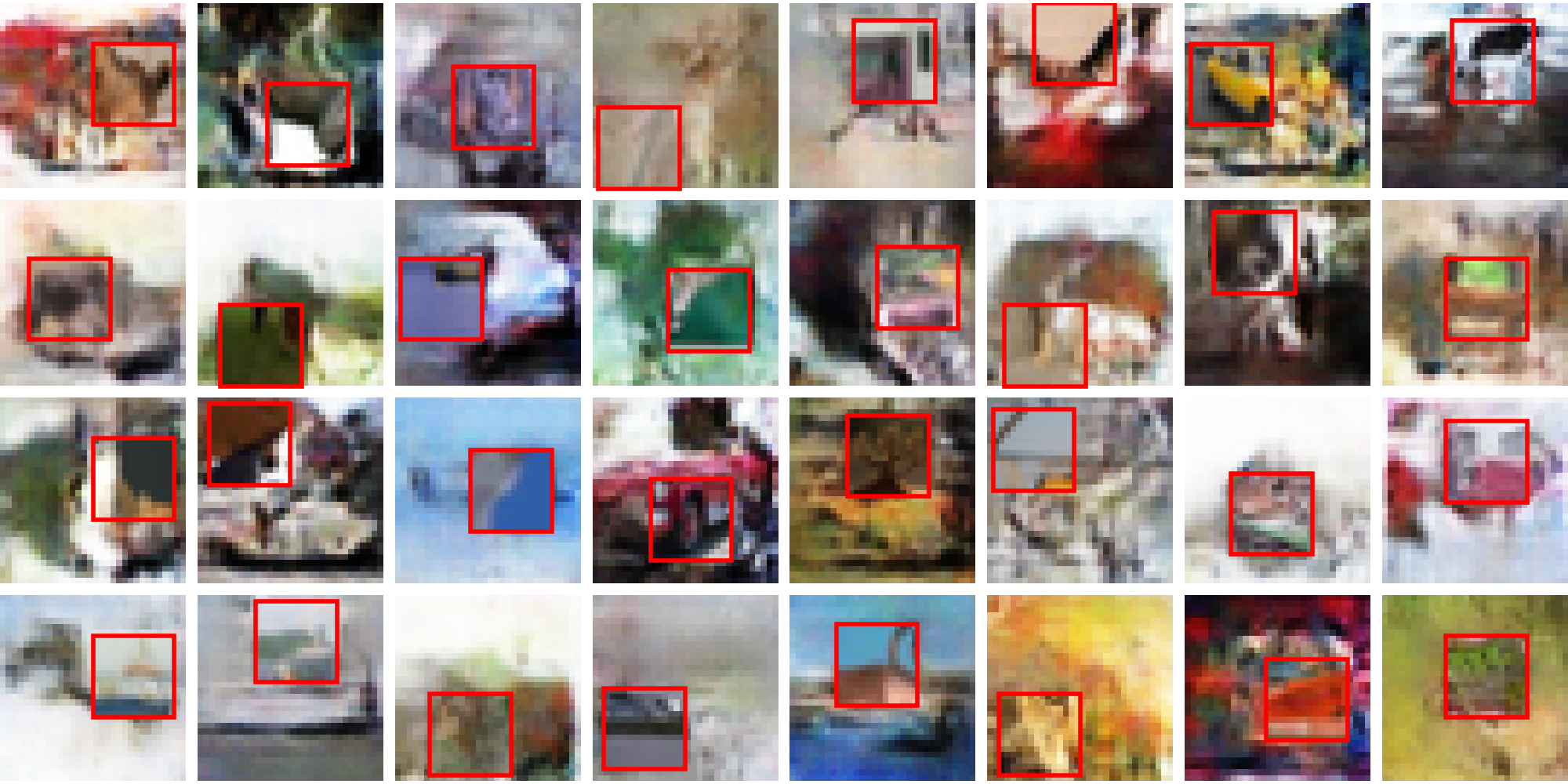}
      };
      \node[inner sep=0pt,right of=22,node distance=\figgap,
            label=below:{\small 30$\times$30 (10\% missing)}] (23) {
        \includegraphics[width=\figwidth]{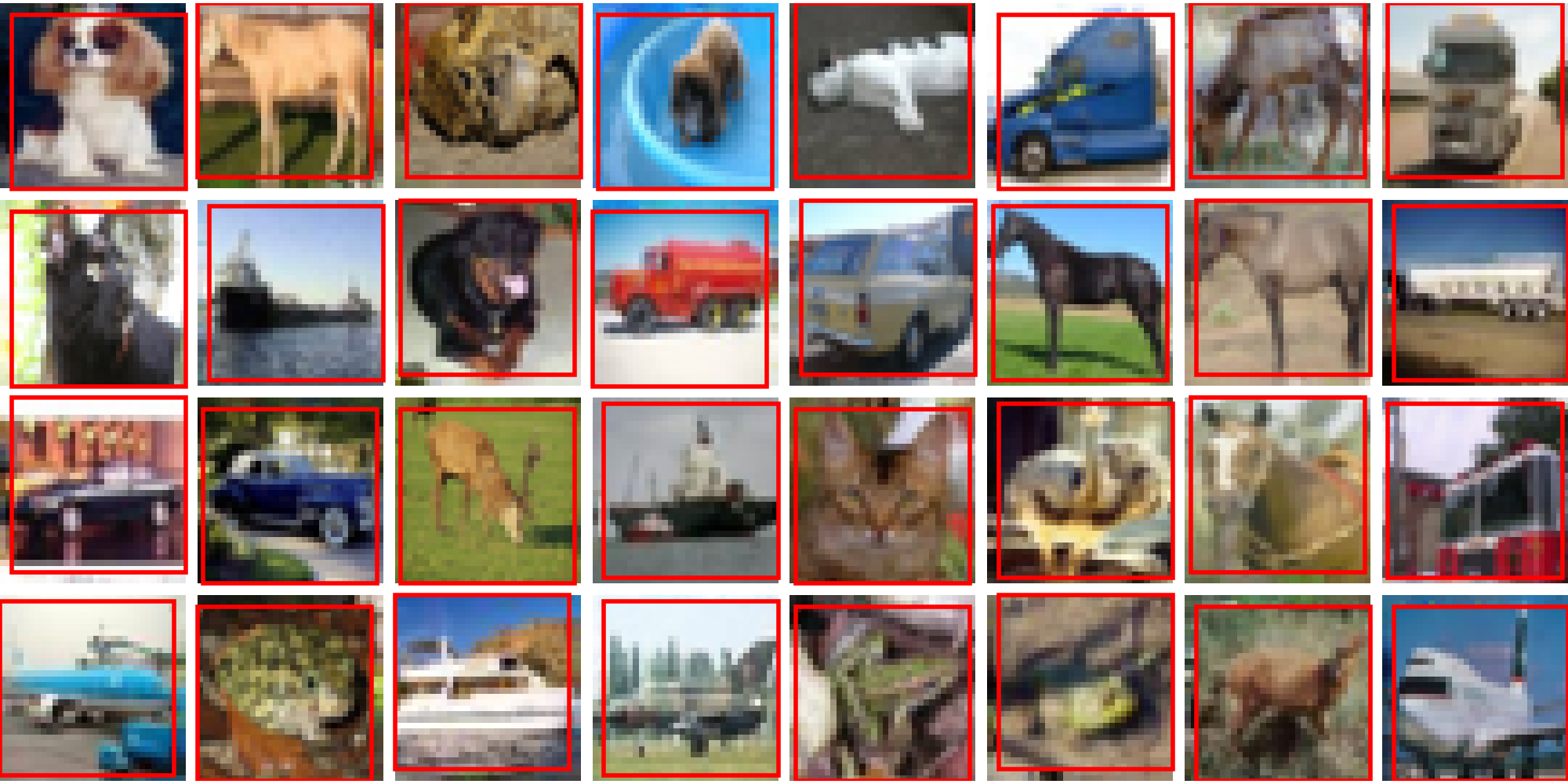}
      };
    \end{tikzpicture}
    \vspace*{-.5em}
    \caption{square observation missingness}
  \end{subfigure}

  \vspace{1em}
  \begin{subfigure}[b]{\textwidth}
    \centering
    \begin{tikzpicture}
      \node[inner sep=0pt,
            label=below:{\small 90\% missing}] (11) {
        \includegraphics[width=\figwidth]{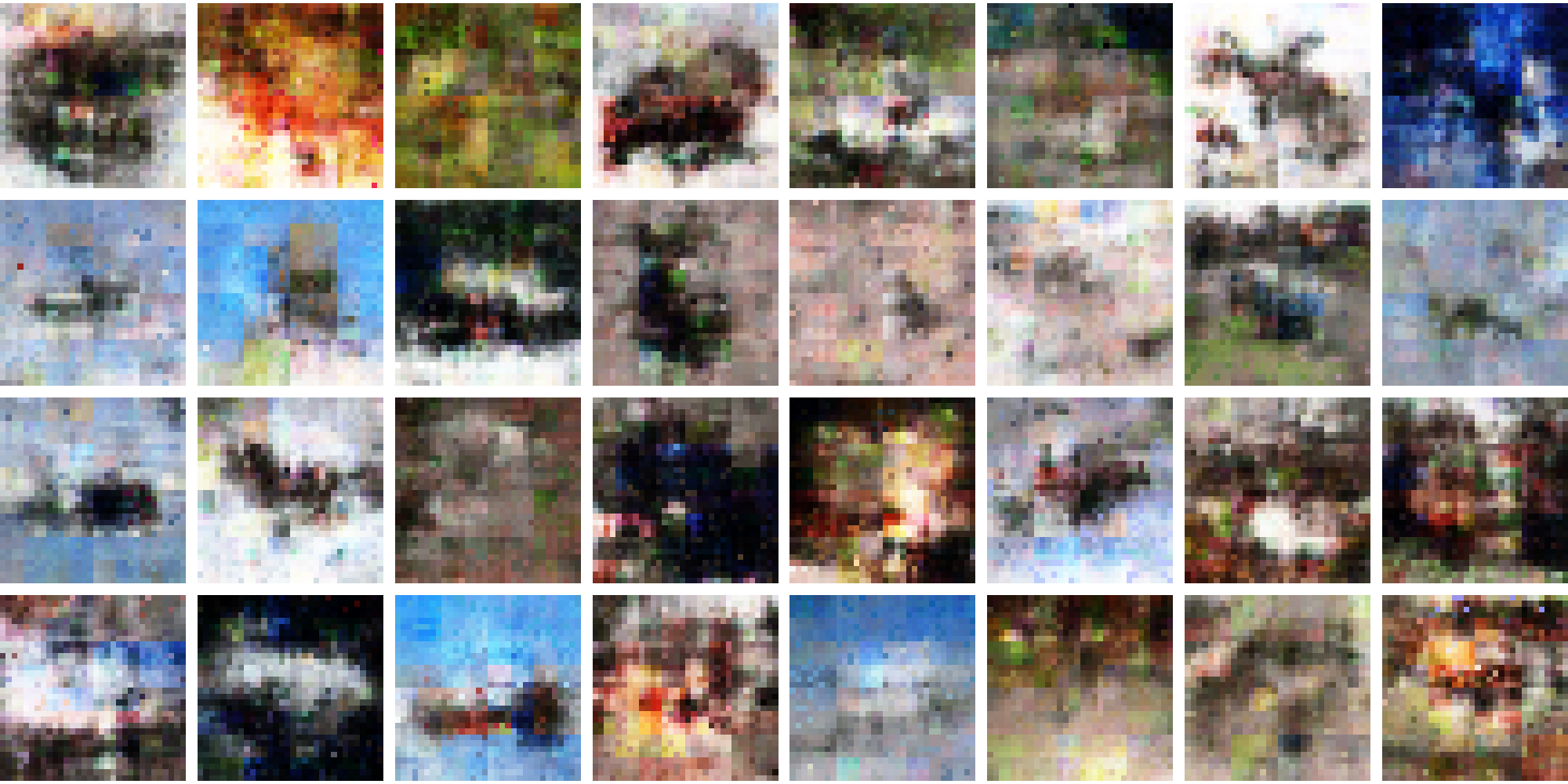}
      };
      \node[inner sep=0pt,right of=11,node distance=\figgap,
            label=below:{\small 80\% missing}] (12) {
        \includegraphics[width=\figwidth]{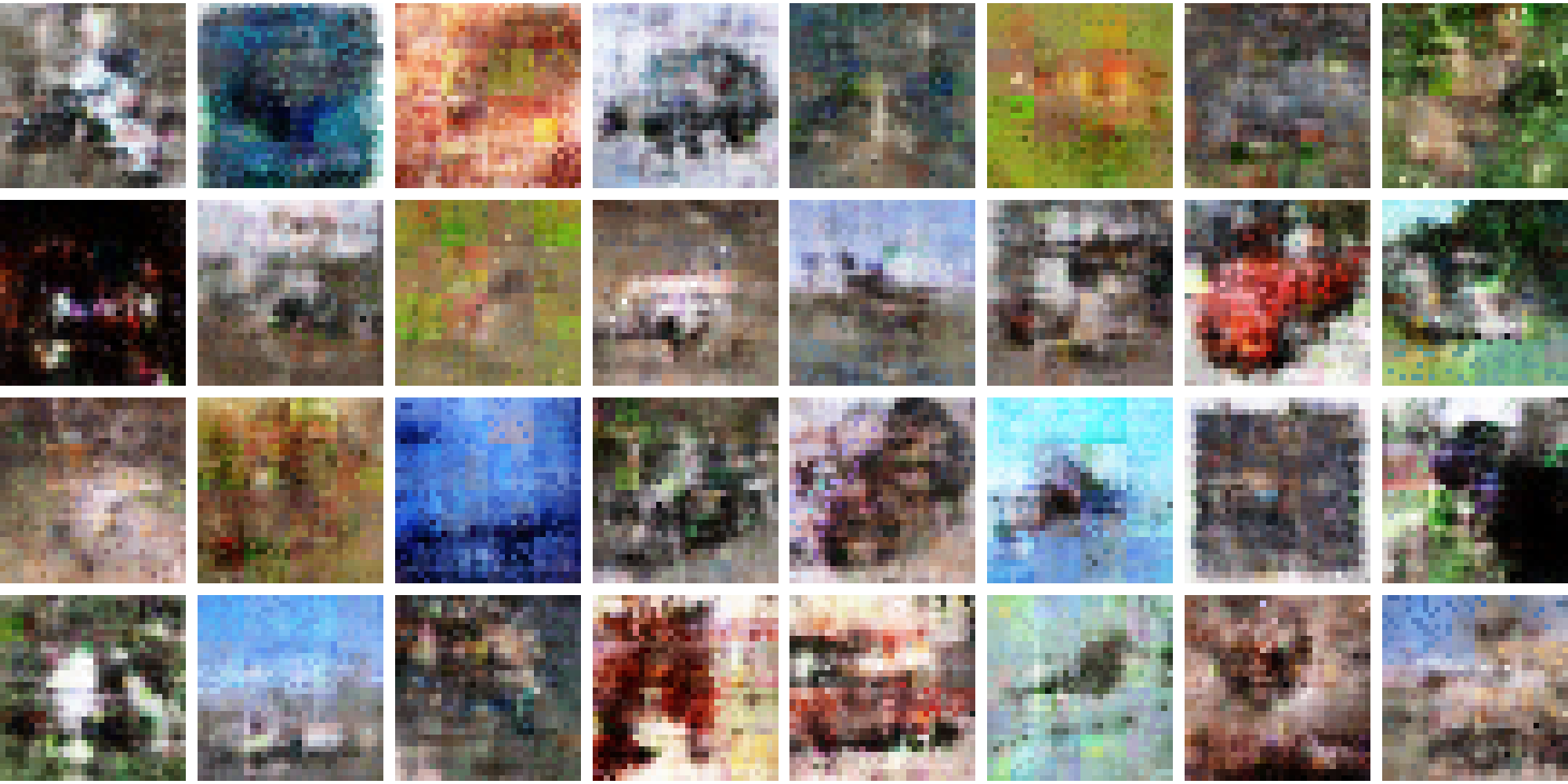}
      };
      \node[inner sep=0pt,right of=12,node distance=\figgap,
            label=below:{\small 10\% missing}] (13) {
        \includegraphics[width=\figwidth]{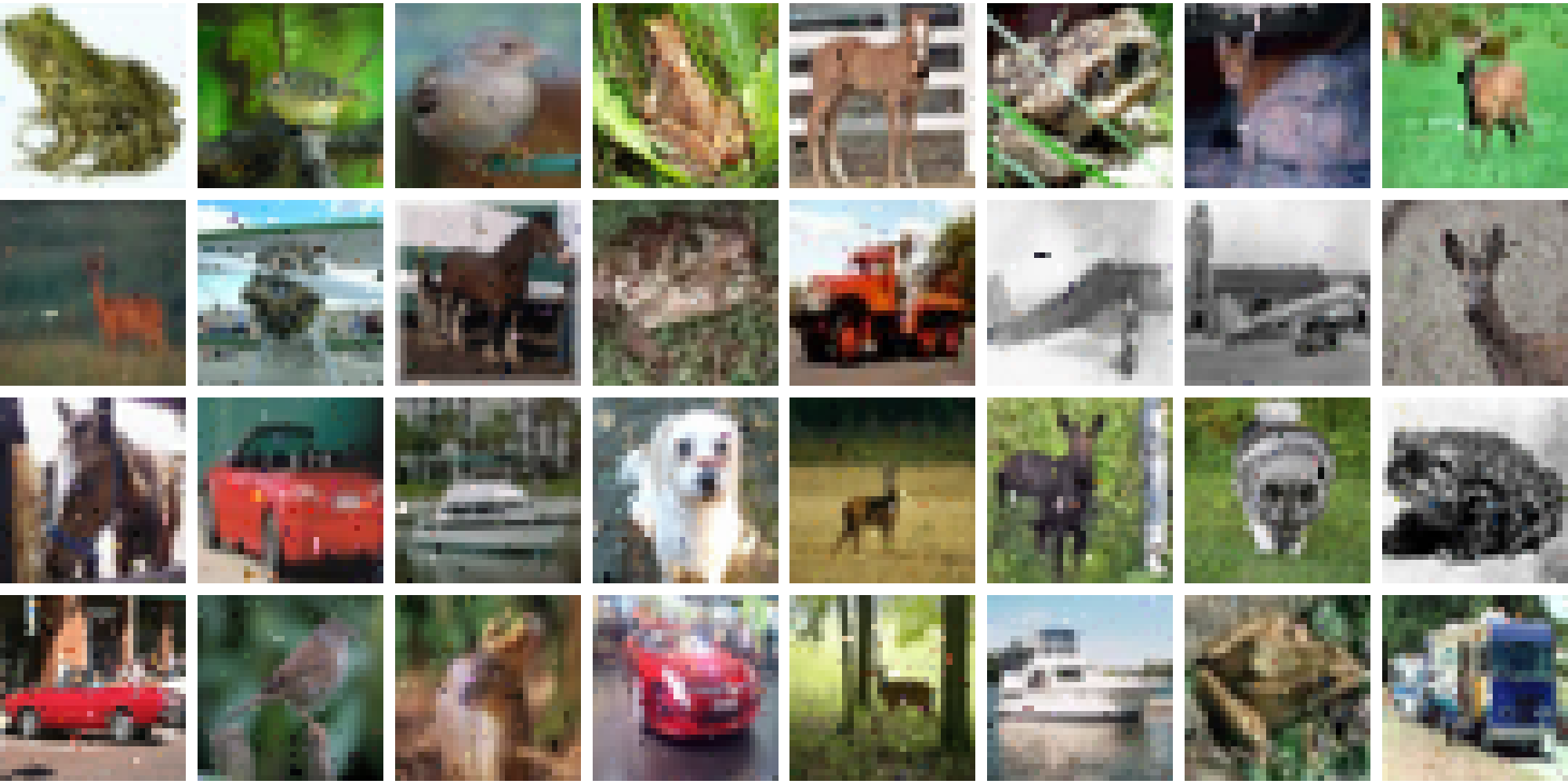}
      };
    \end{tikzpicture}
    \vspace*{-.5em}
    \caption{independent dropout missingness}
  \end{subfigure}

  \caption{{\misgan} imputation on CIFAR-10}
\label{fig:cifar10impute}
\end{figure}

\section{{\misgan} on CelebA}
\label{sec:celebamisgan}
Figure~\ref{fig:celebablock},
\ref{fig:celebaindep}
and \ref{fig:celebaimpute}
show the results of {\misgan} trained on CelebA
for the two extreme missing rates, namely
90\% and 80\%, as well as the case of 10\% that is
close to full observation.

\begin{figure}
  \def\figwidth{.32\textwidth}
  \def\figgap{\figwidth+.3em}

  \centering
  \begin{subfigure}[b]{\textwidth}
    \centering
    \begin{tikzpicture}
      \node[inner sep=0pt,
            label=below:{\small training samples}] (11) {
        \includegraphics[width=\figwidth]{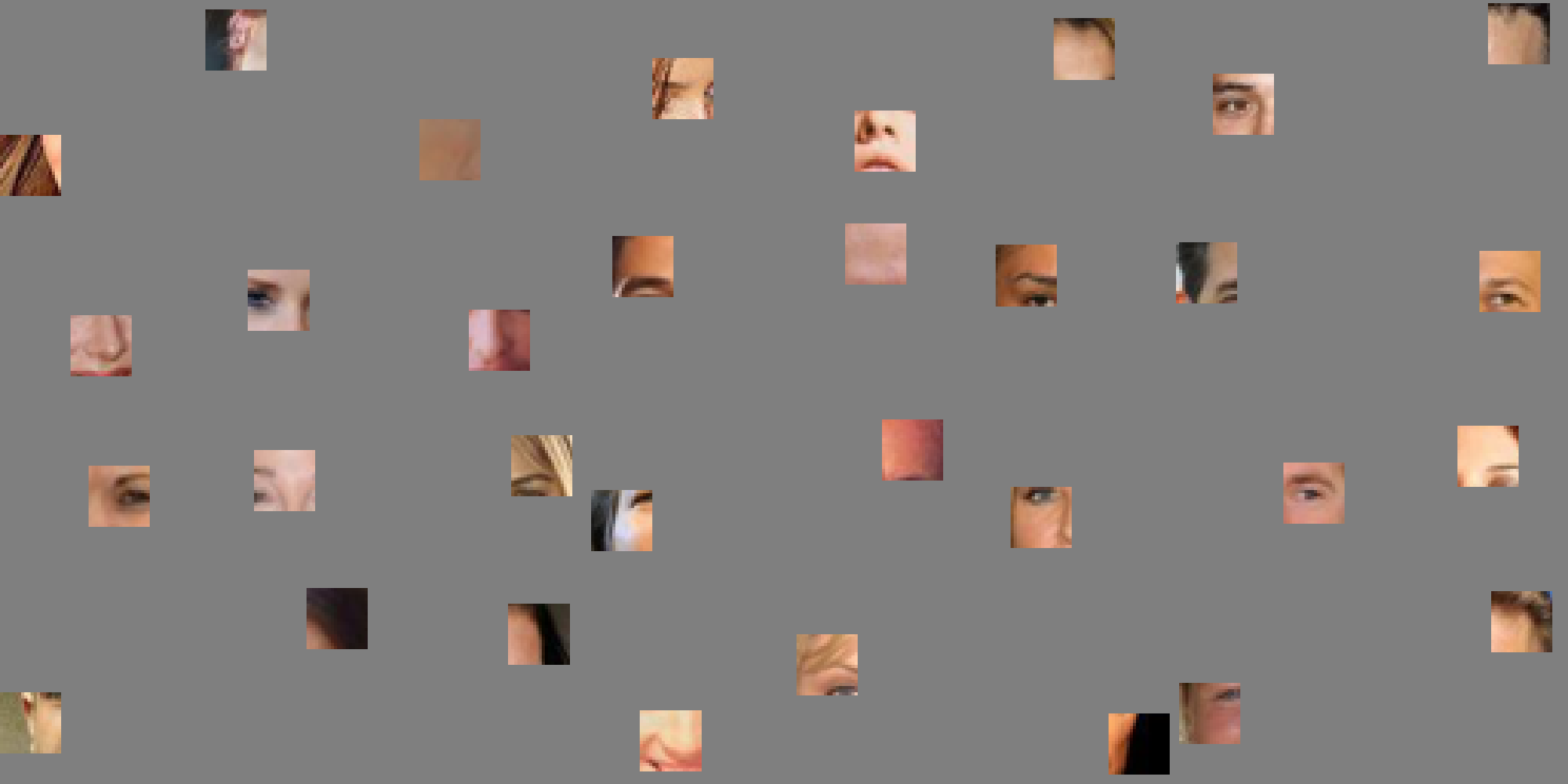}
      };
      \node[inner sep=0pt,right of=11,node distance=\figgap,
            label=below:{\small $G_x$ samples}] (12) {
        \includegraphics[width=\figwidth]{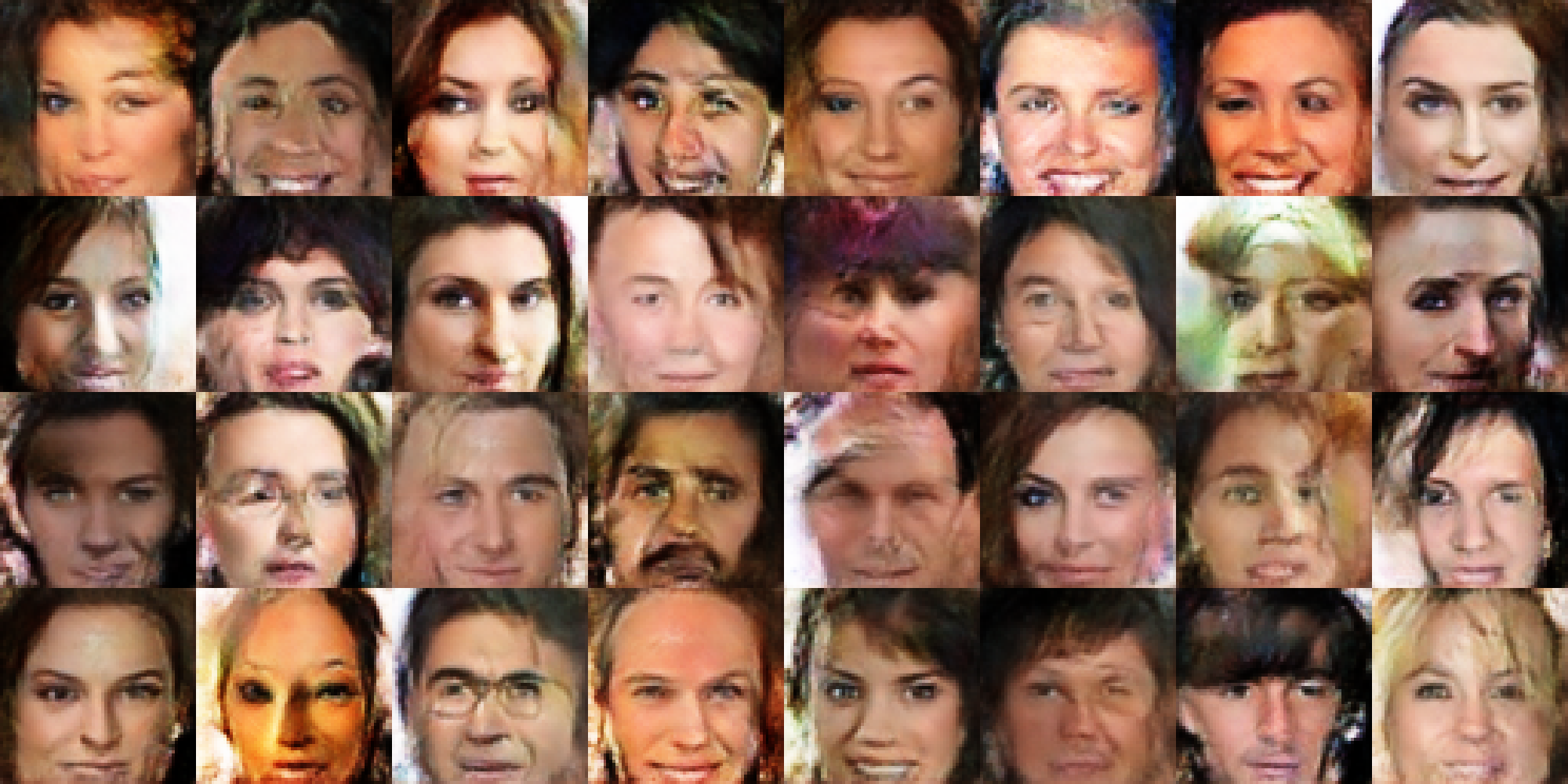}
      };
      \node[inner sep=0pt,right of=12,node distance=\figgap,
            label=below:{\small $G_m$ samples}] (13) {
        \includegraphics[width=\figwidth]{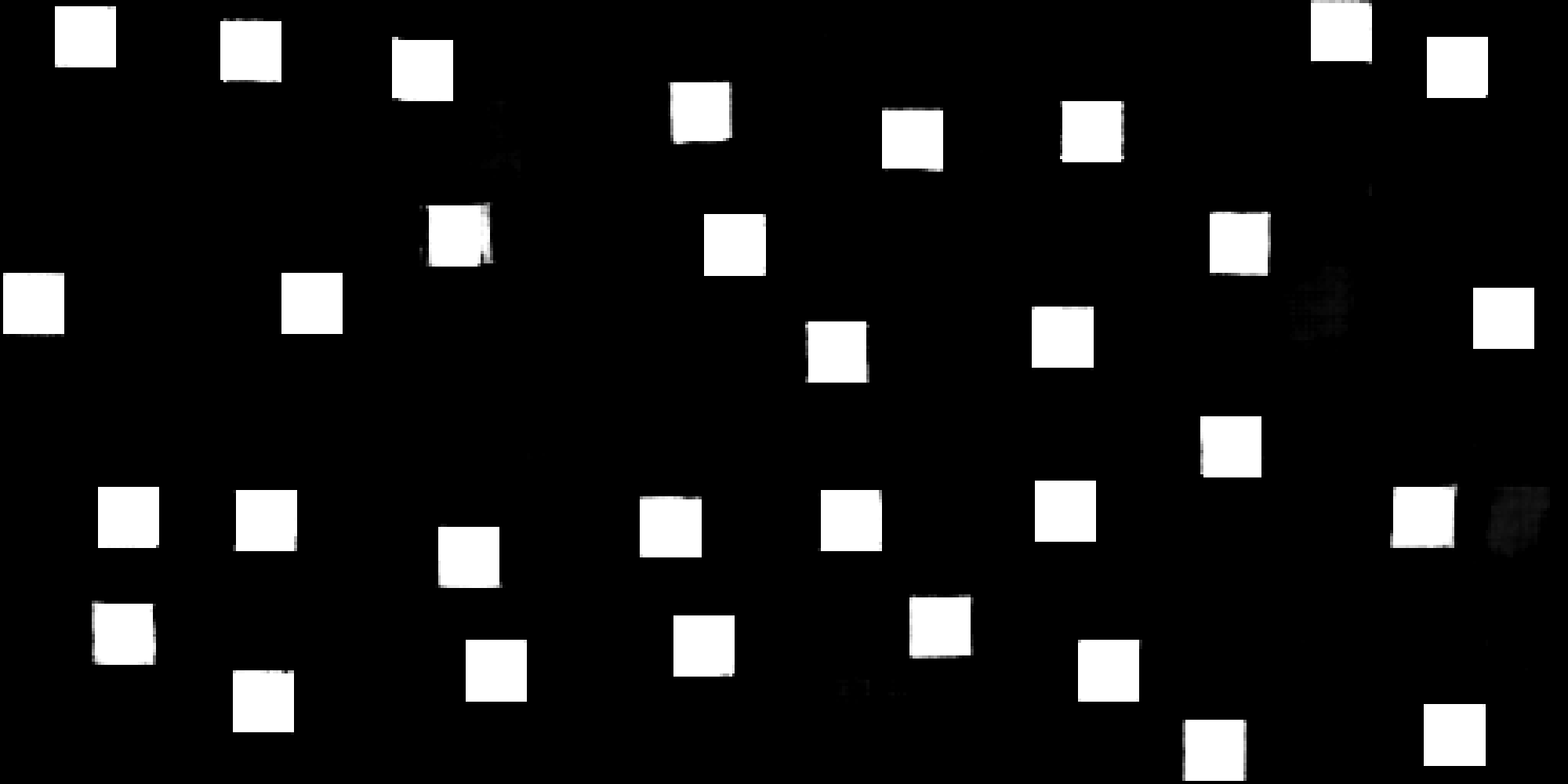}
      };
    \end{tikzpicture}
    \vspace*{-.5em}
    \caption{20$\times$20 block (90\% missing)}
  \end{subfigure}

  \vspace{1em}
  \begin{subfigure}[b]{\textwidth}
    \centering
    \begin{tikzpicture}
      \node[inner sep=0pt,
            label=below:{\small training samples}] (21) {
        \includegraphics[width=\figwidth]{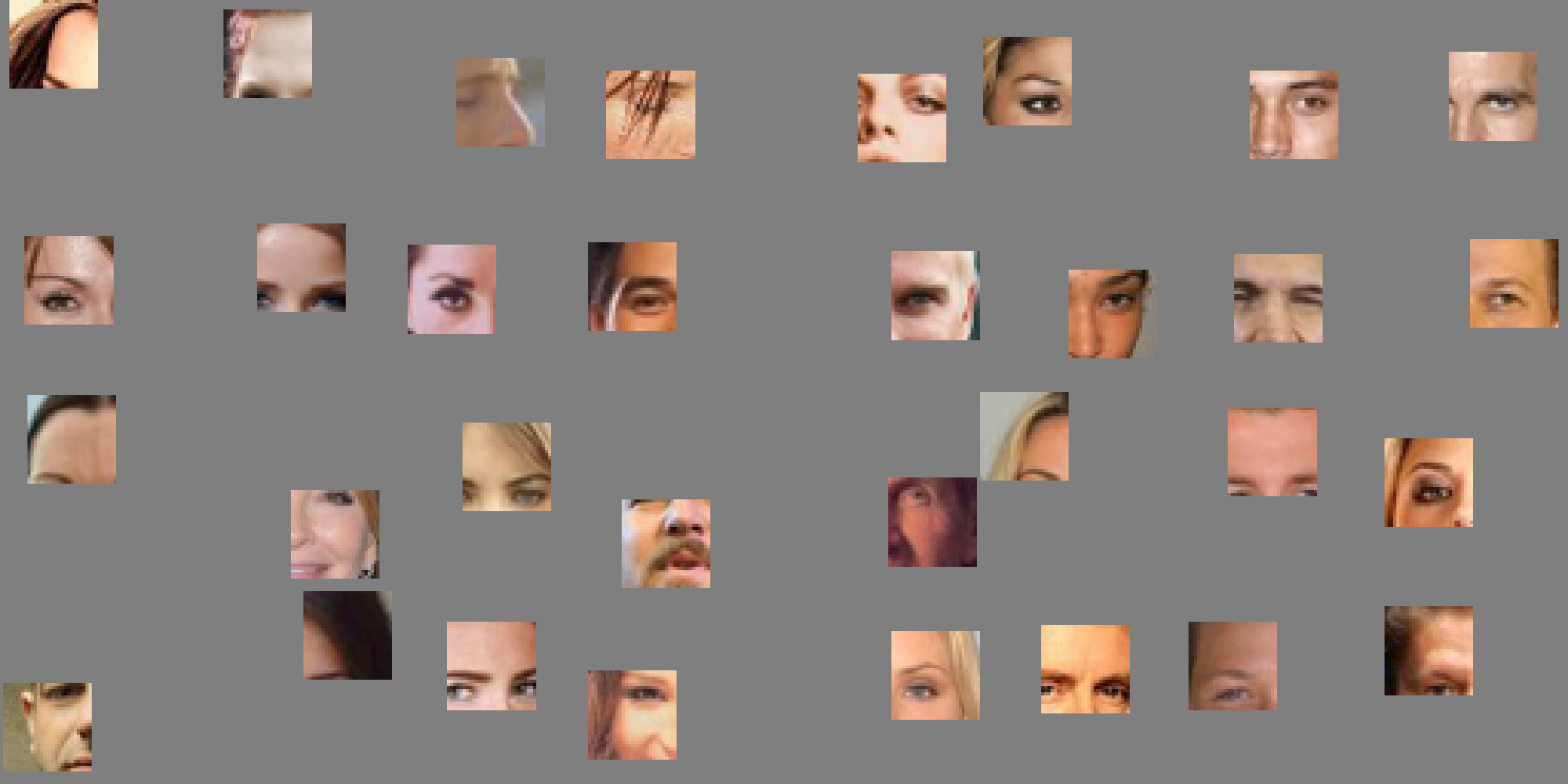}
      };
      \node[inner sep=0pt,right of=21,node distance=\figgap,
            label=below:{\small $G_x$ samples}] (22) {
        \includegraphics[width=\figwidth]{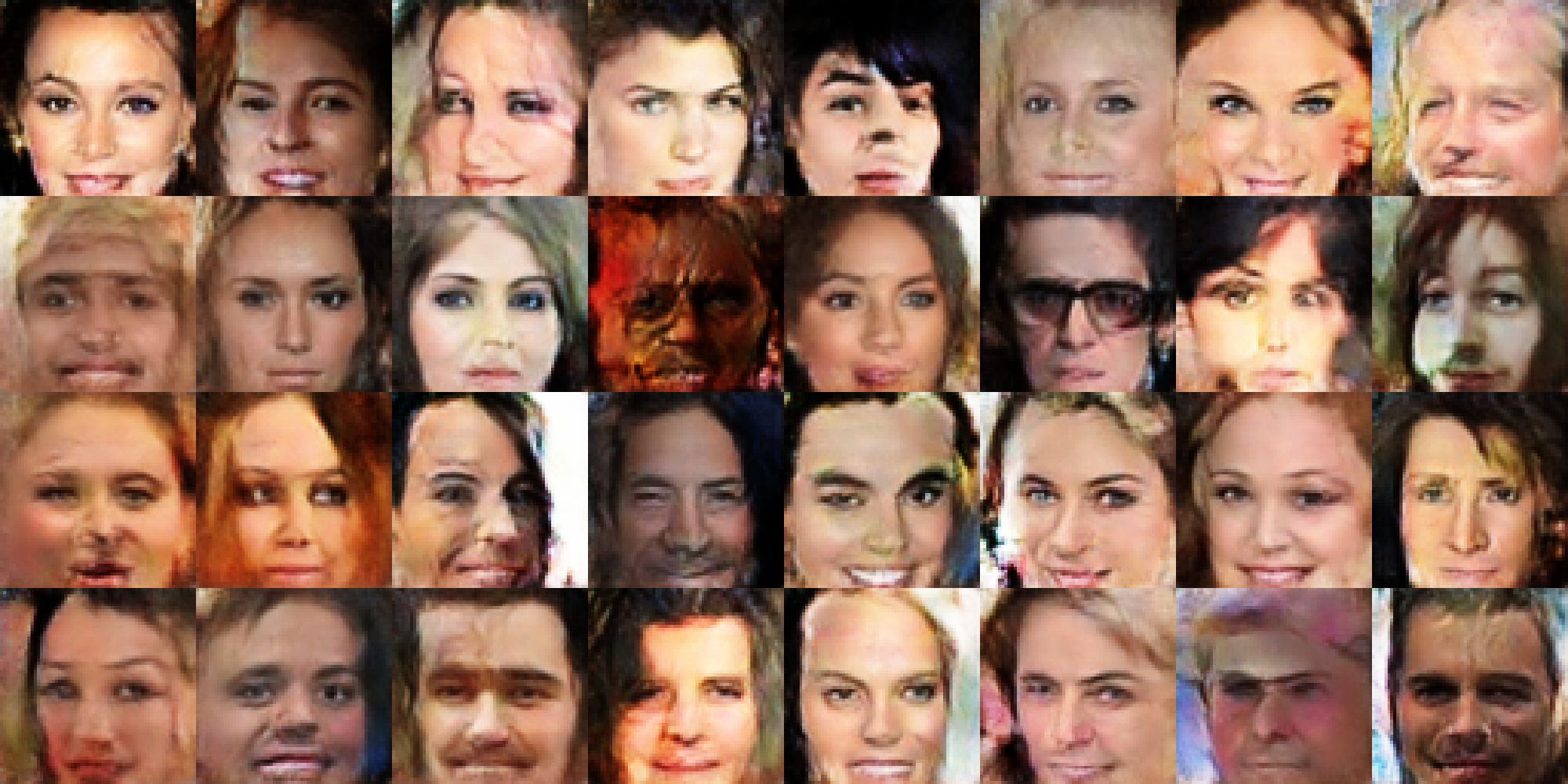}
      };
      \node[inner sep=0pt,right of=22,node distance=\figgap,
            label=below:{\small $G_m$ samples}] (23) {
        \includegraphics[width=\figwidth]{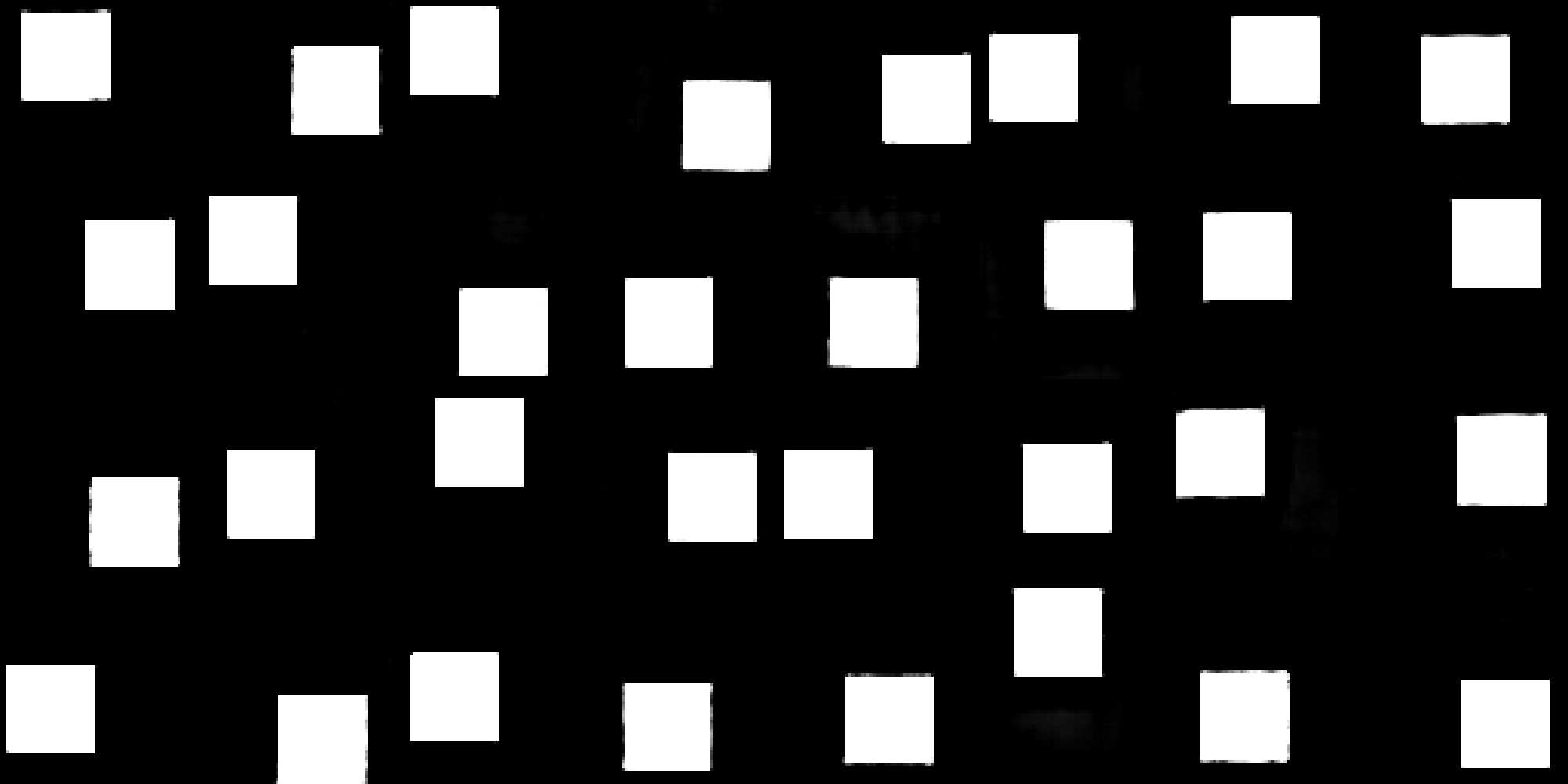}
      };
    \end{tikzpicture}
    \vspace*{-.5em}
    \caption{29$\times$29 block (80\% missing)}
  \end{subfigure}

  \vspace{1em}
  \begin{subfigure}[b]{\textwidth}
    \centering
    \begin{tikzpicture}
      \node[inner sep=0pt,
            label=below:{\small training samples}] (21) {
        \includegraphics[width=\figwidth]{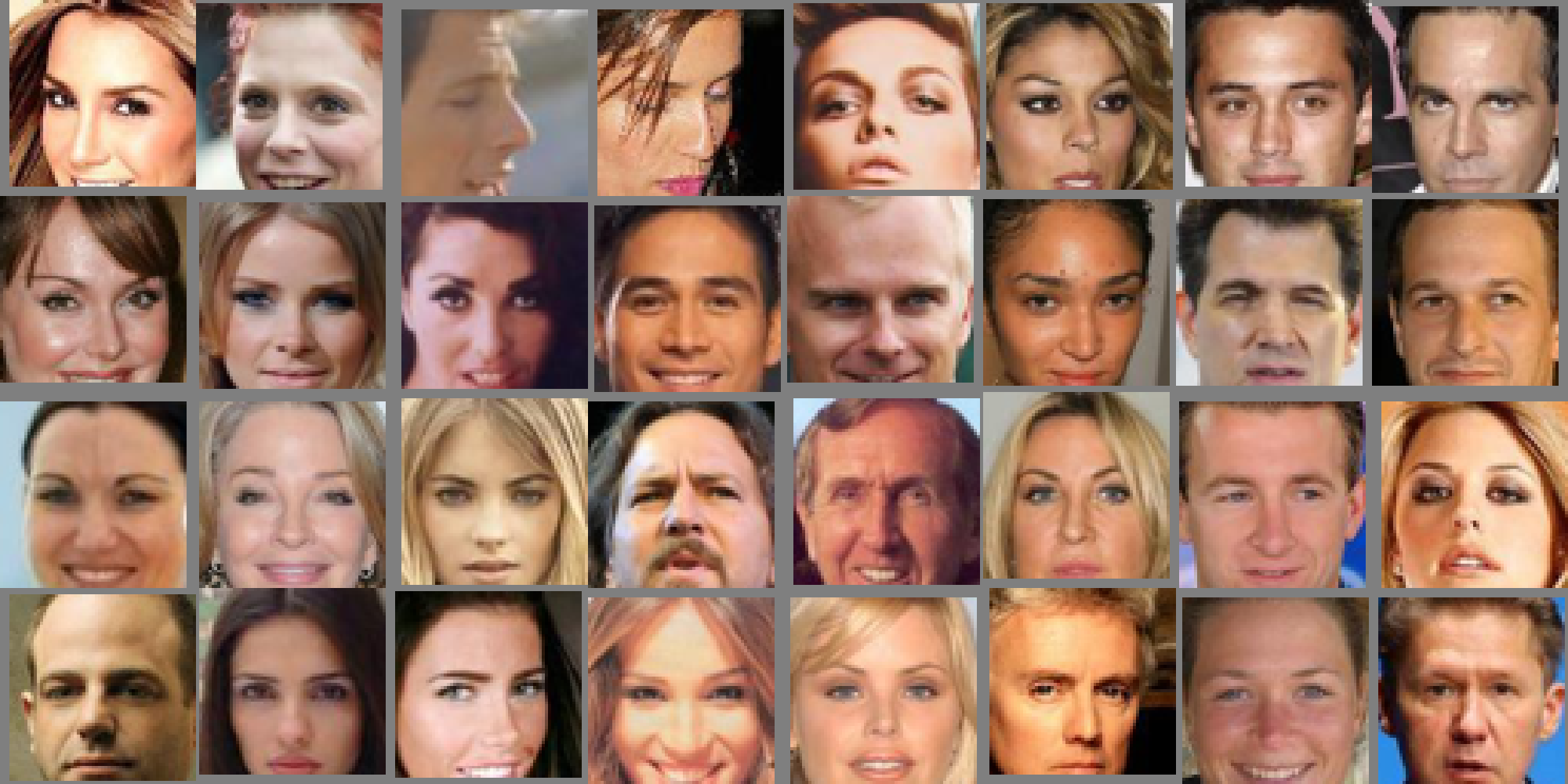}
      };
      \node[inner sep=0pt,right of=21,node distance=\figgap,
            label=below:{\small $G_x$ samples}] (22) {
        \includegraphics[width=\figwidth]{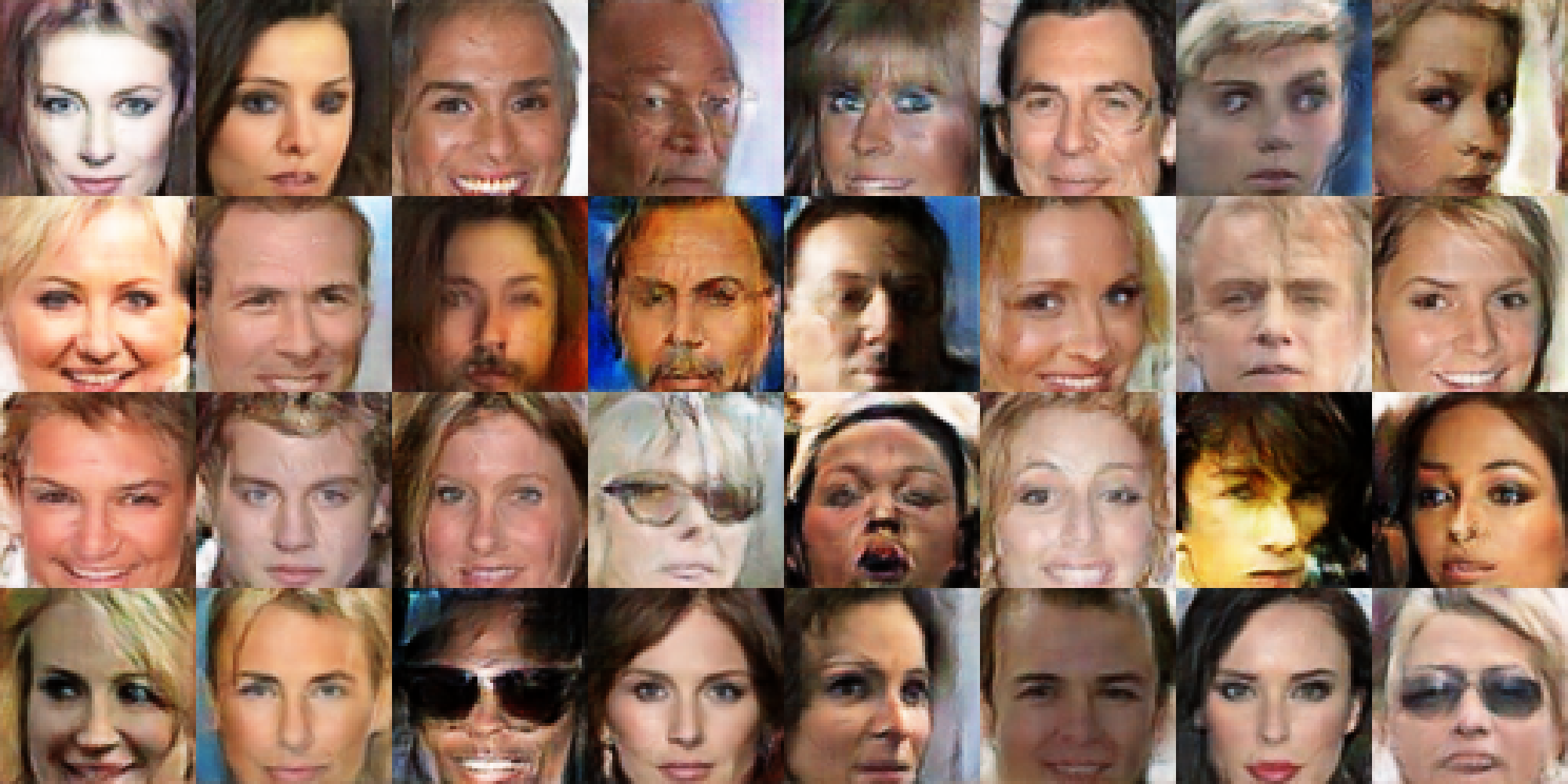}
      };
      \node[inner sep=0pt,right of=22,node distance=\figgap,
            label=below:{\small $G_m$ samples}] (23) {
        \includegraphics[width=\figwidth]{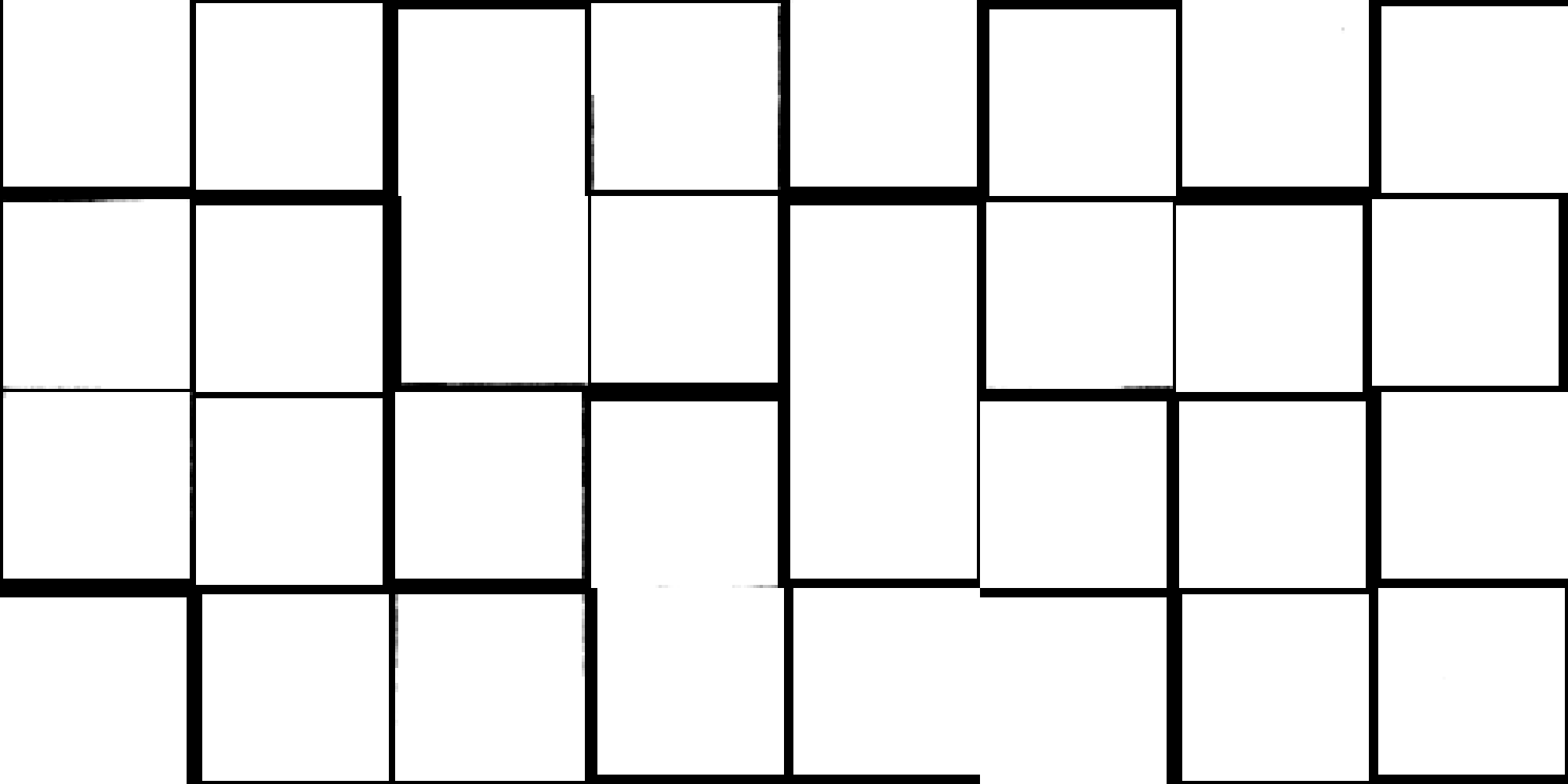}
      };
    \end{tikzpicture}
    \vspace*{-.5em}
    \caption{61$\times$61 block (10\% missing)}
  \end{subfigure}

  \vspace{1em}
  \begin{subfigure}[b]{\textwidth}
    \centering
    \begin{tikzpicture}
      \node[inner sep=0pt,
            label=below:{\small training samples}] (31) {
        \includegraphics[width=\figwidth]{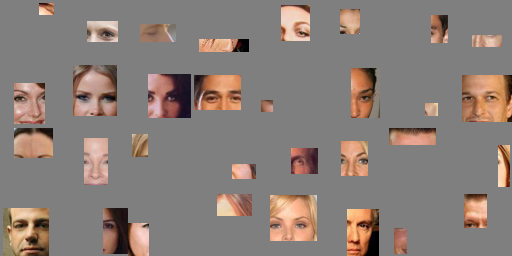}
      };
      \node[inner sep=0pt,right of=31,node distance=\figgap,
            label=below:{\small $G_x$ samples}] (32) {
        \includegraphics[width=\figwidth]{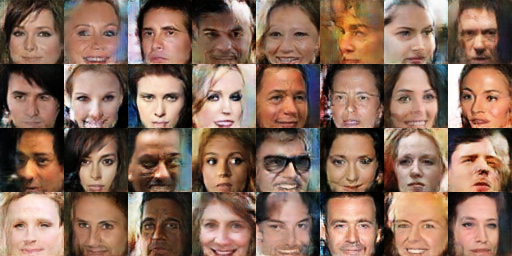}
      };
      \node[inner sep=0pt,right of=32,node distance=\figgap,
            label=below:{\small $G_m$ samples}] (33) {
        \includegraphics[width=\figwidth]{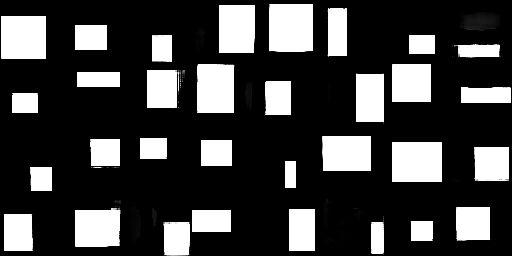}
      };
    \end{tikzpicture}
    \vspace*{-.5em}
    \caption{Variable-size block
    (75\% missing on average)}
  \end{subfigure}
  \caption{{\misgan} on CelebA with block observation missingness}
\label{fig:celebablock}
\end{figure}

\begin{figure}
  \def\figwidth{.32\textwidth}
  \def\figgap{\figwidth+.3em}

  \centering
  \begin{subfigure}[b]{\textwidth}
    \centering
    \begin{tikzpicture}
      \node[inner sep=0pt,
            label=below:{\small training samples}] (11) {
        \includegraphics[width=\figwidth]{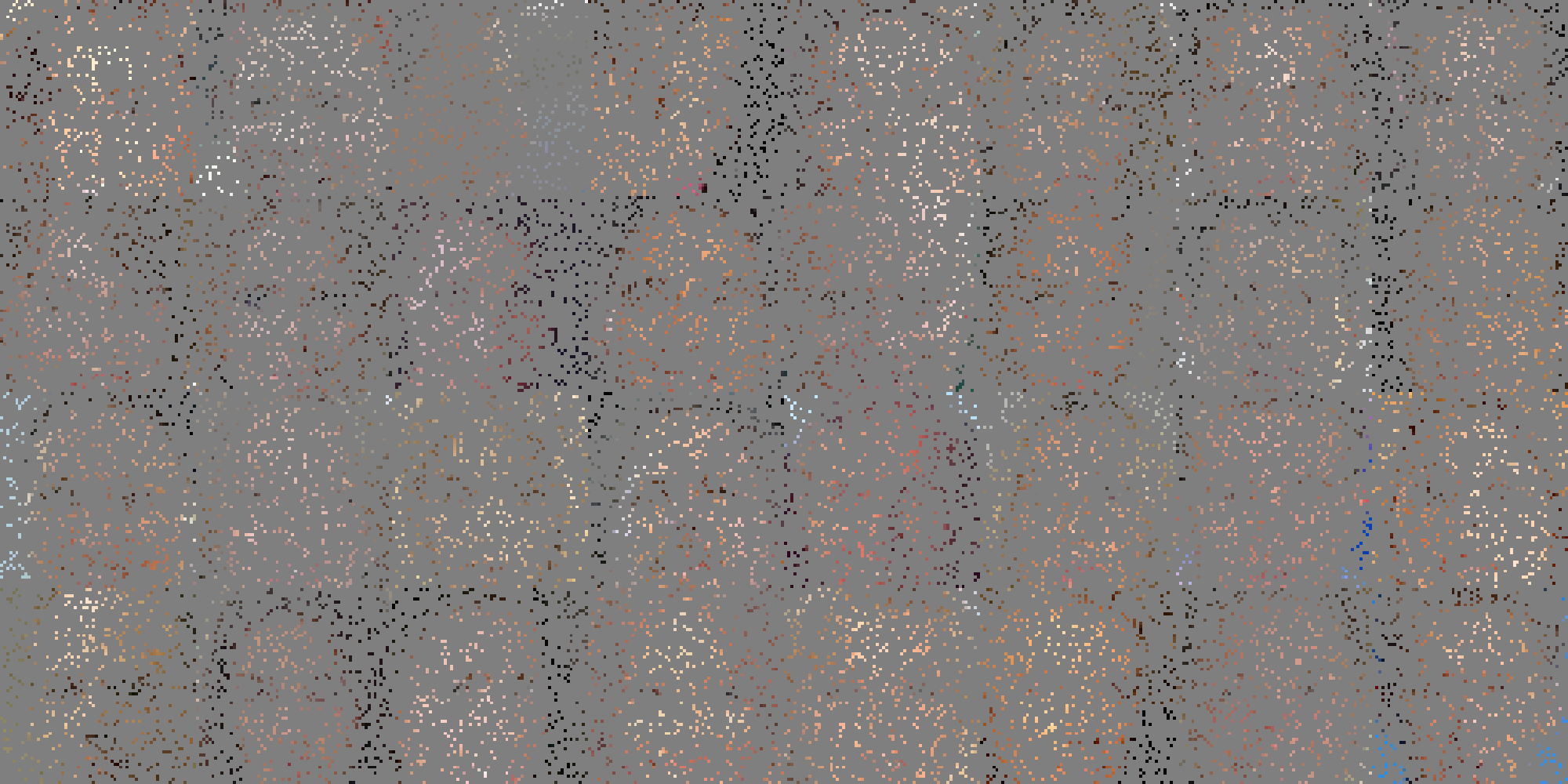}
      };
      \node[inner sep=0pt,right of=11,node distance=\figgap,
            label=below:{\small $G_x$ samples}] (12) {
        \includegraphics[width=\figwidth]{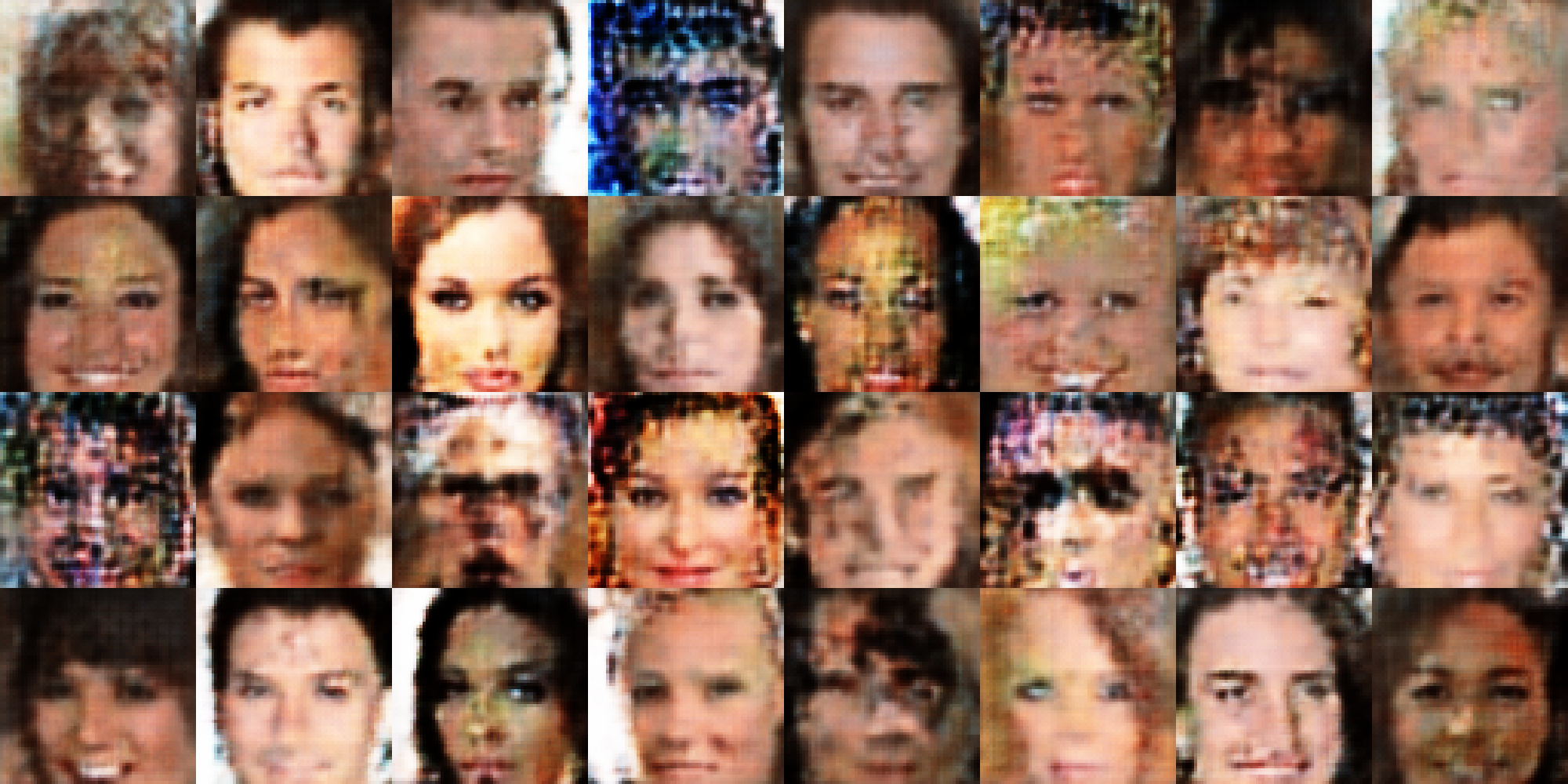}
      };
      \node[inner sep=0pt,right of=12,node distance=\figgap,
            label=below:{\small $G_m$ samples}] (13) {
        \includegraphics[width=\figwidth]{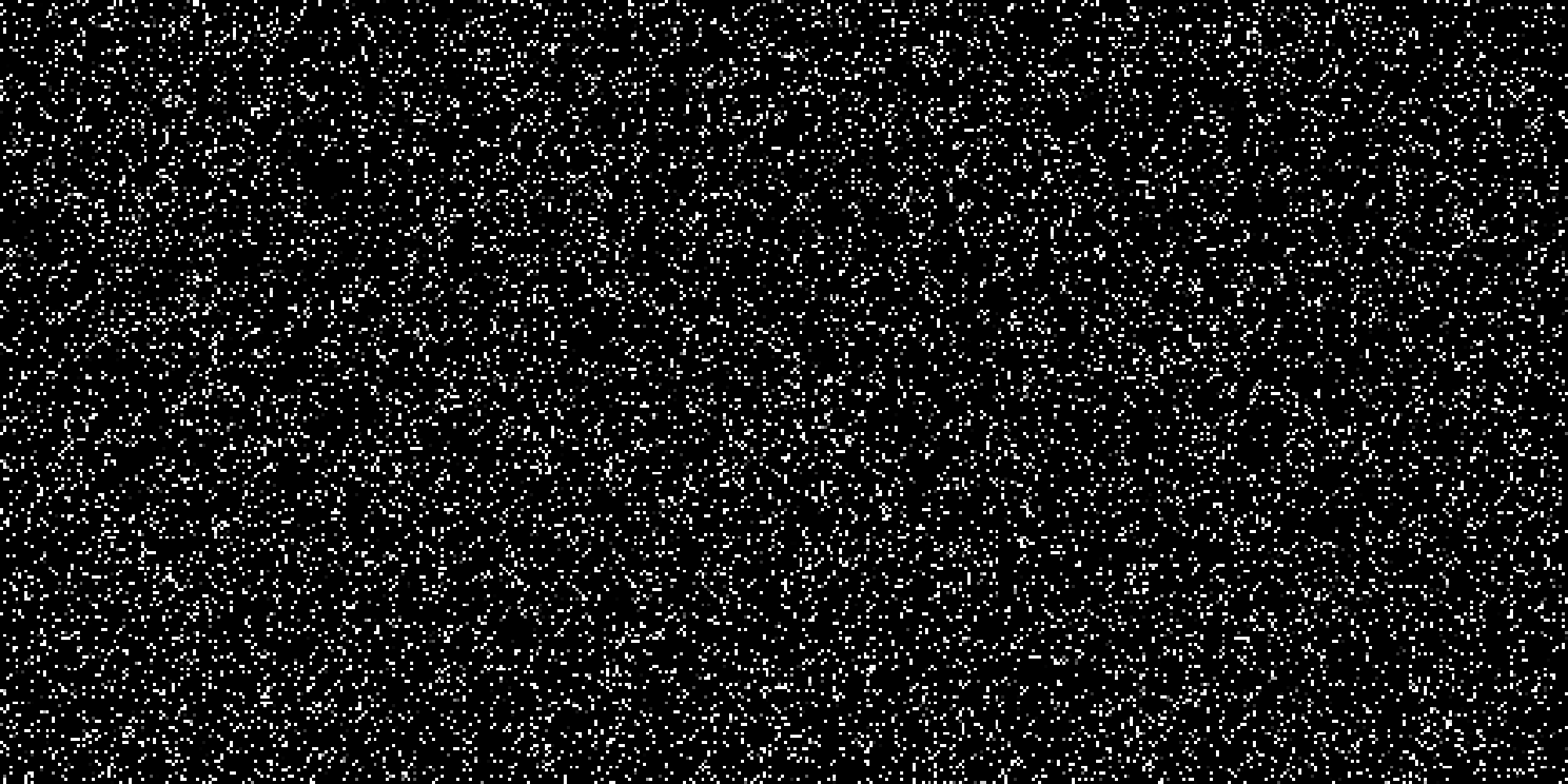}
      };
    \end{tikzpicture}
    \vspace*{-.5em}
    \caption{90\% missing}
  \end{subfigure}

  \vspace{1em}
  \begin{subfigure}[b]{\textwidth}
    \centering
    \begin{tikzpicture}
      \node[inner sep=0pt,
            label=below:{\small training samples}] (21) {
        \includegraphics[width=\figwidth]{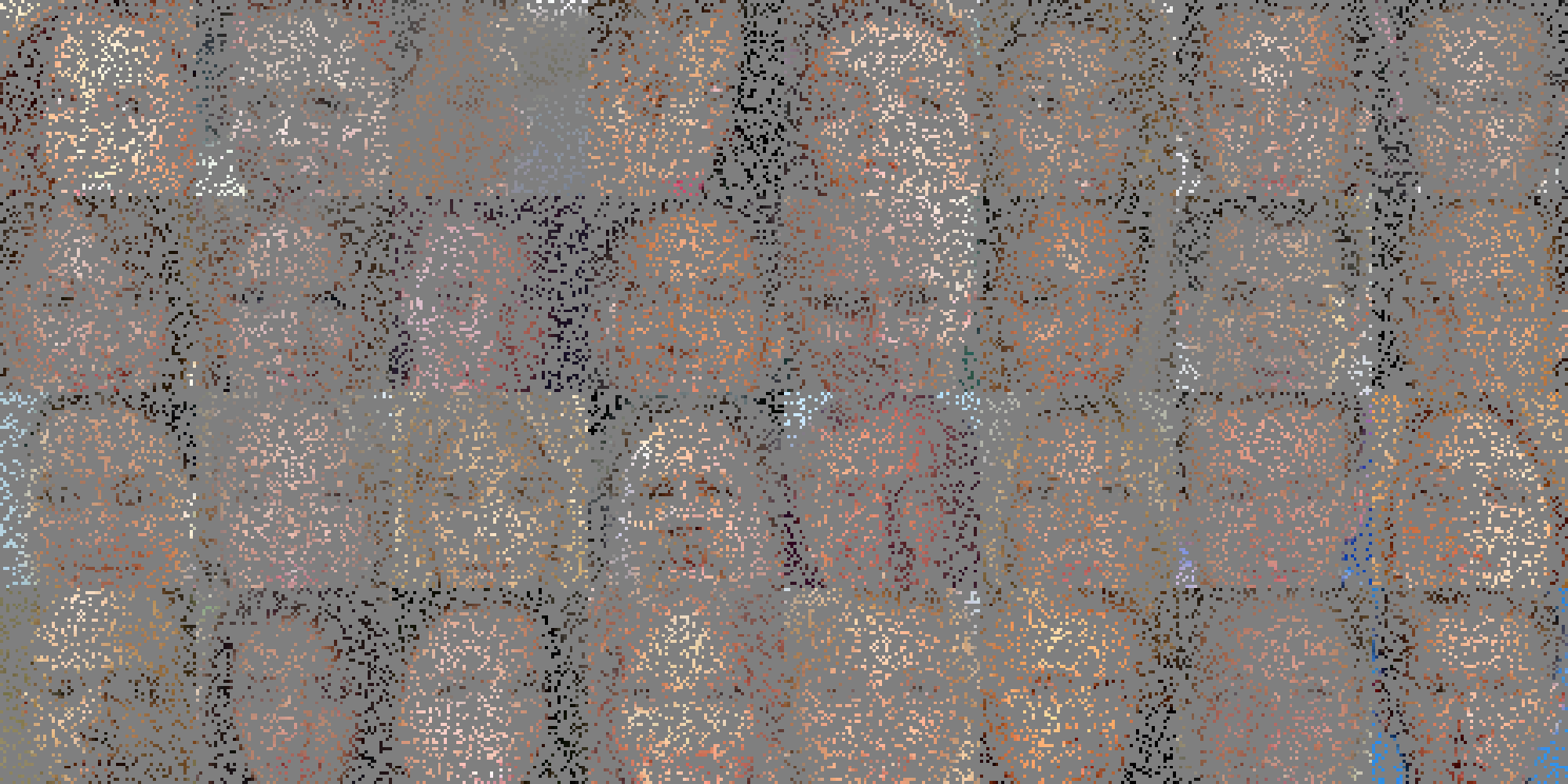}
      };
      \node[inner sep=0pt,right of=21,node distance=\figgap,
            label=below:{\small $G_x$ samples}] (22) {
        \includegraphics[width=\figwidth]{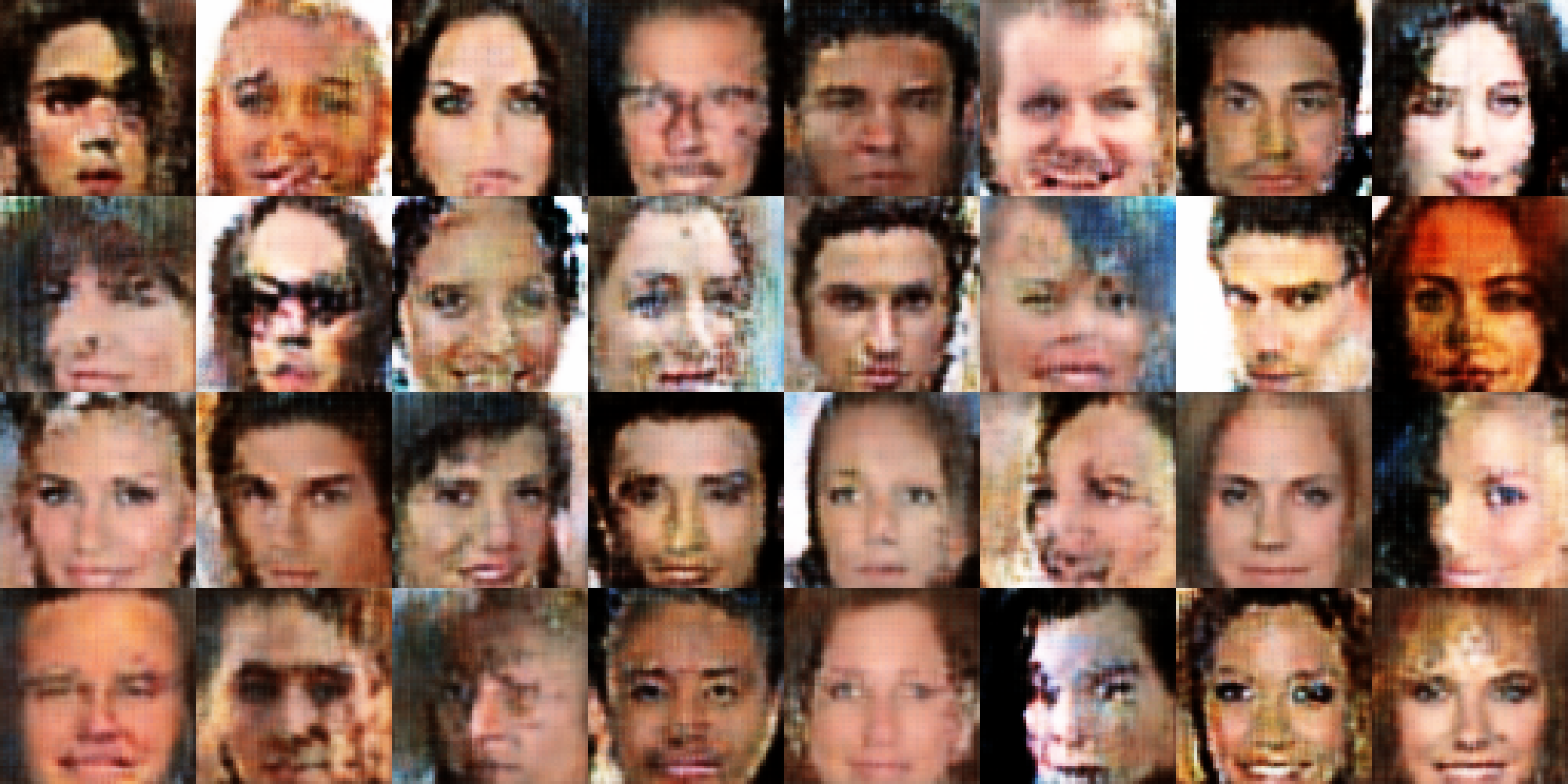}
      };
      \node[inner sep=0pt,right of=22,node distance=\figgap,
            label=below:{\small $G_m$ samples}] (23) {
        \includegraphics[width=\figwidth]{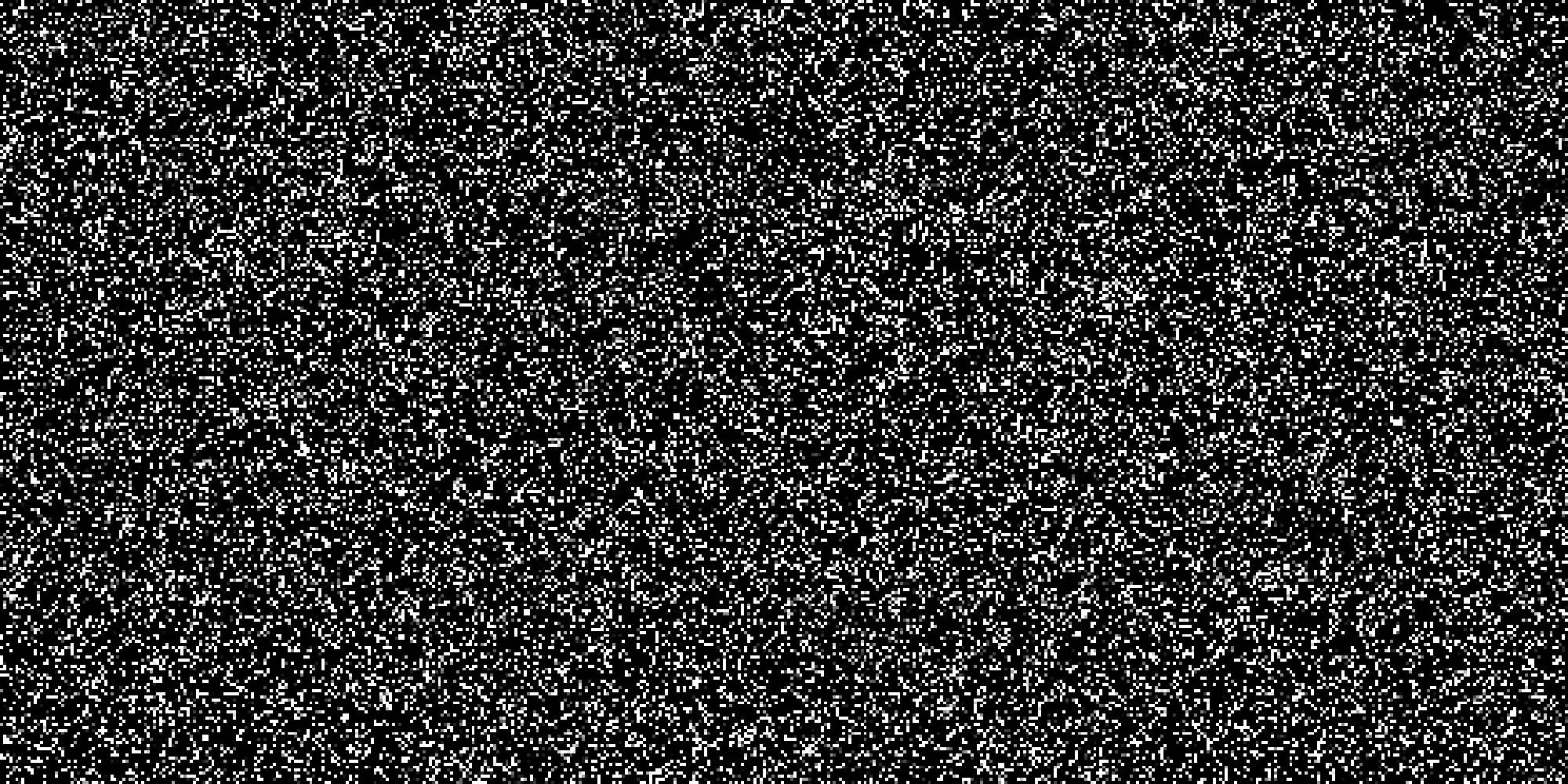}
      };
    \end{tikzpicture}
    \vspace*{-.5em}
    \caption{80\% missing}
  \end{subfigure}

  \vspace{1em}
  \begin{subfigure}[b]{\textwidth}
    \centering
    \begin{tikzpicture}
      \node[inner sep=0pt,
            label=below:{\small training samples}] (21) {
        \includegraphics[width=\figwidth]{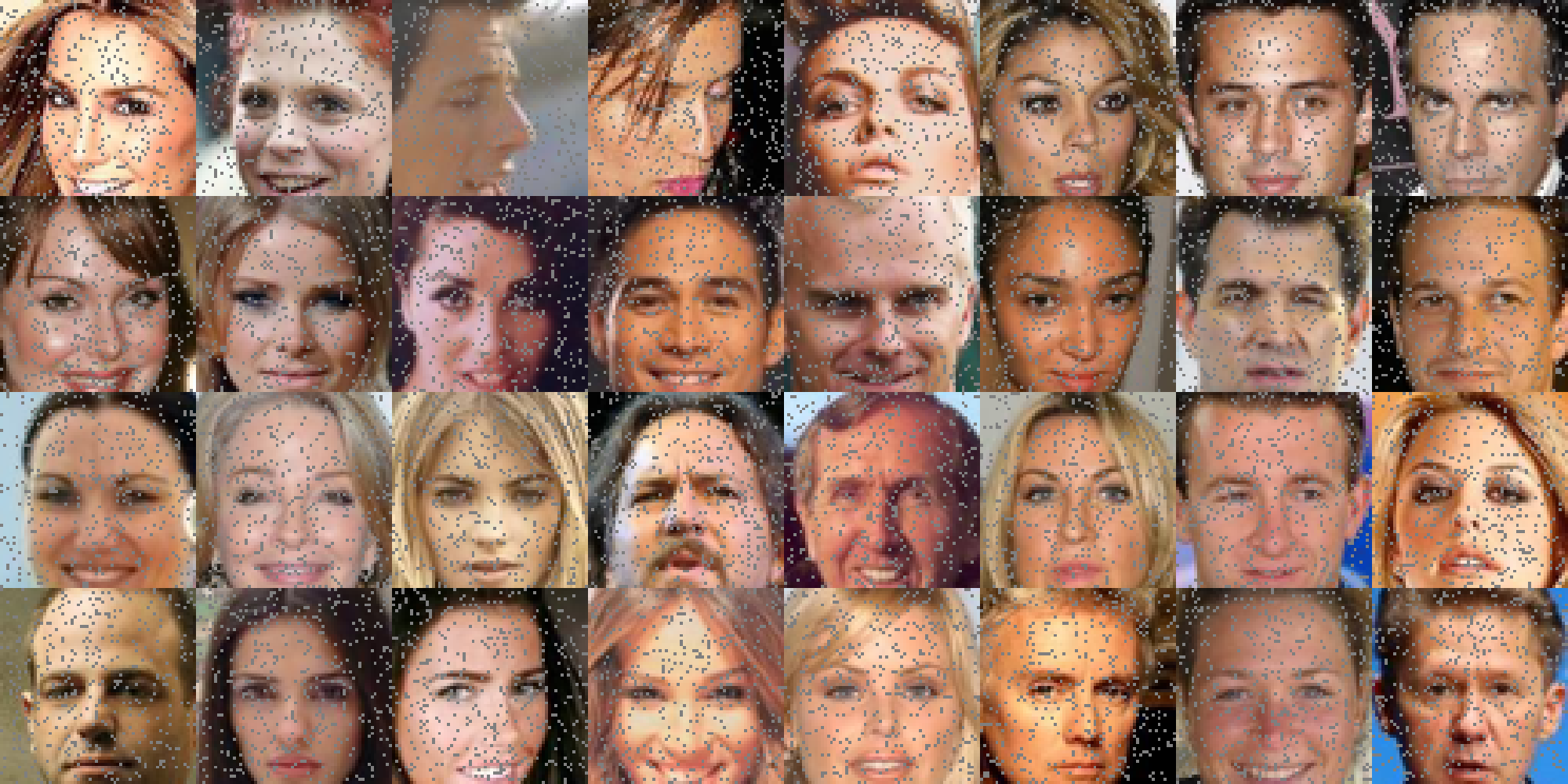}
      };
      \node[inner sep=0pt,right of=21,node distance=\figgap,
            label=below:{\small $G_x$ samples}] (22) {
        \includegraphics[width=\figwidth]{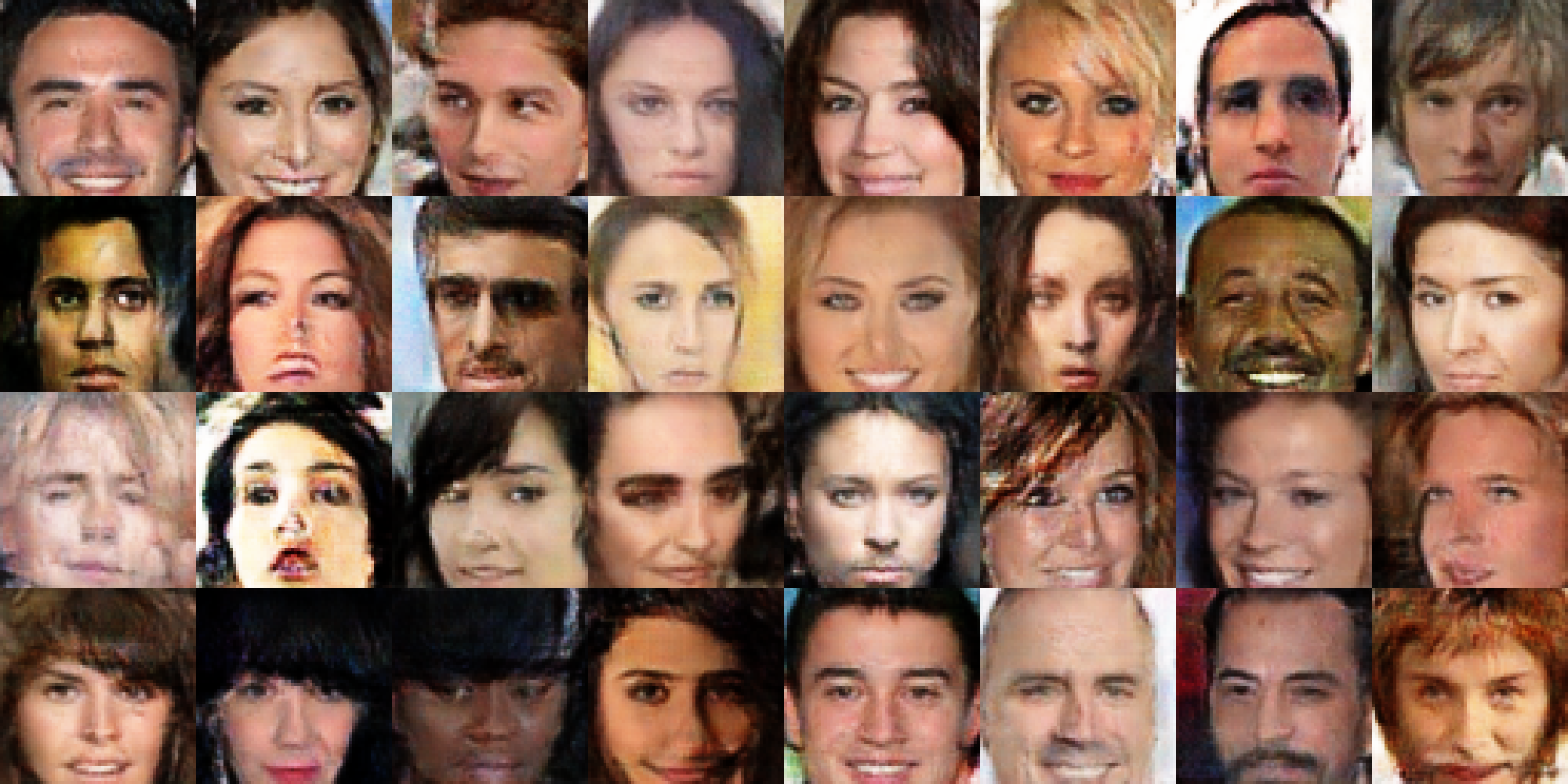}
      };
      \node[inner sep=0pt,right of=22,node distance=\figgap,
            label=below:{\small $G_m$ samples}] (23) {
        \includegraphics[width=\figwidth]{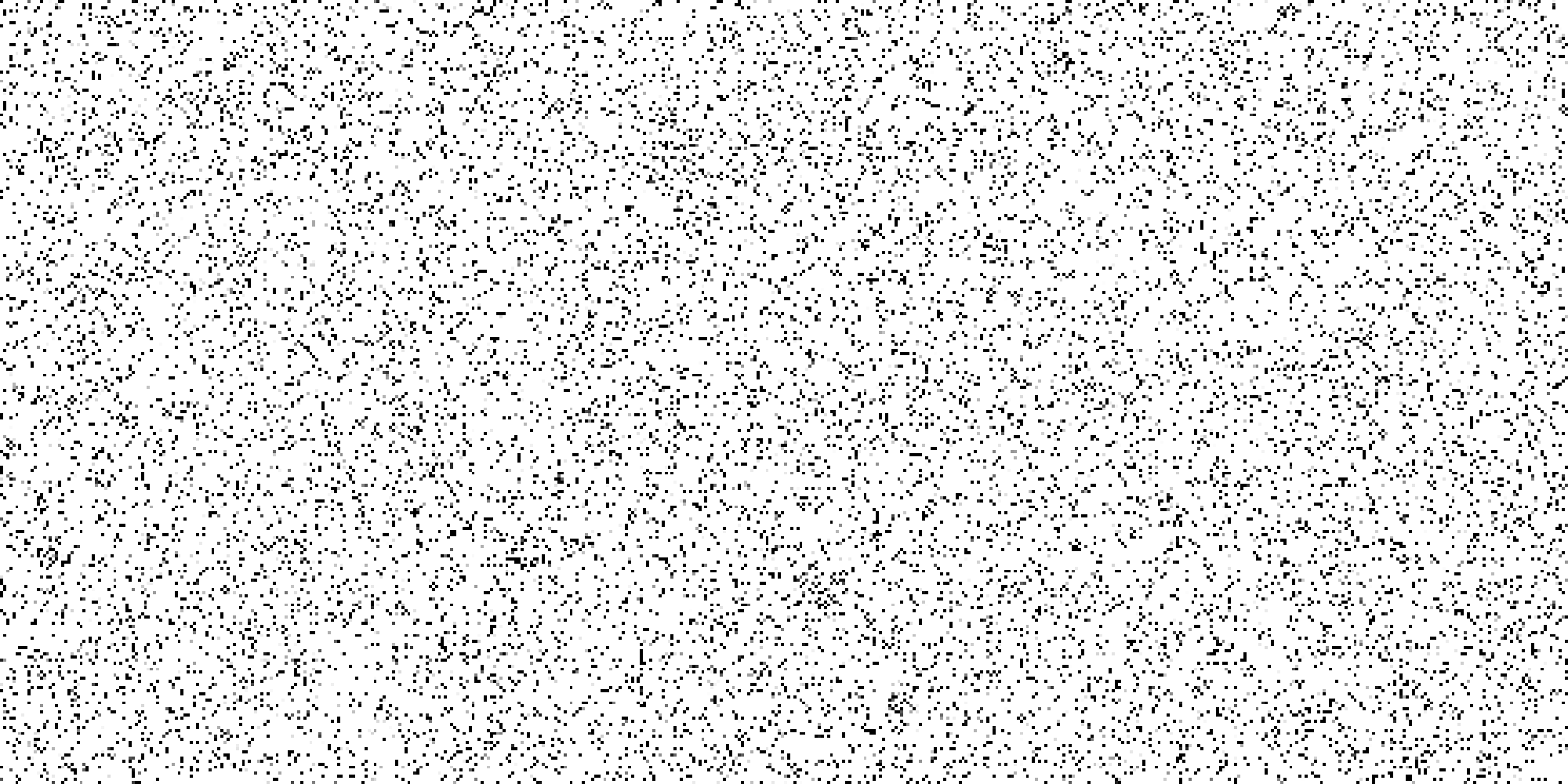}
      };
    \end{tikzpicture}
    \vspace*{-.5em}
    \caption{10\% missing}
  \end{subfigure}
  \caption{{\misgan} on CelebA with independent dropout missingness}
\label{fig:celebaindep}
\end{figure}

\begin{figure}
  \def\figwidth{.32\textwidth}
  \def\figgap{\figwidth+.3em}

  \centering
  \begin{subfigure}[b]{\textwidth}
    \centering
    \begin{tikzpicture}
      \node[inner sep=0pt,
            label=below:{\small 20$\times$20 (90\% missing)}] (21) {
        \includegraphics[width=\figwidth]{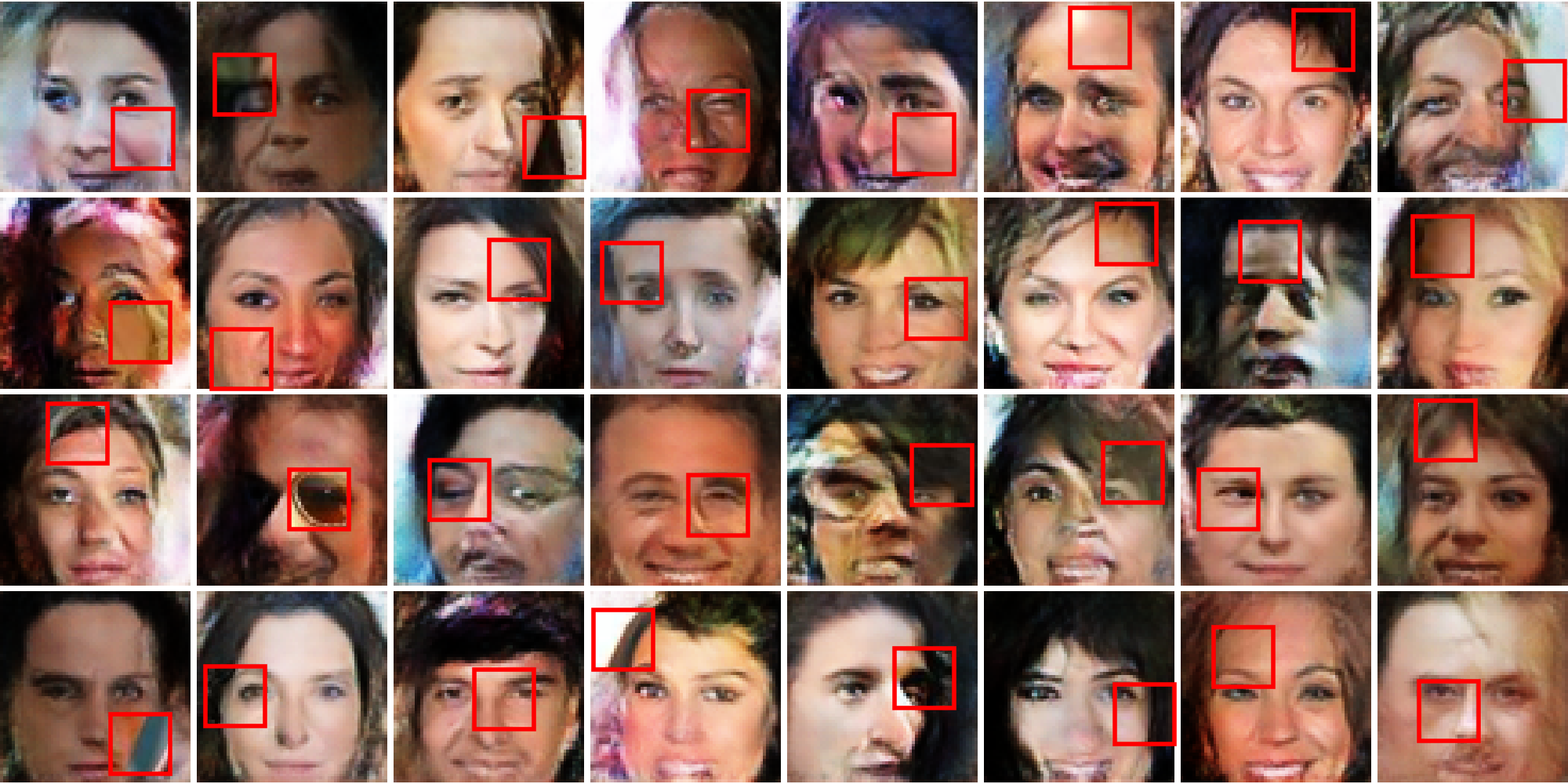}
      };
      \node[inner sep=0pt,right of=21,node distance=\figgap,
            label=below:{\small 29$\times$29 (80\% missing)}] (22) {
        \includegraphics[width=\figwidth]{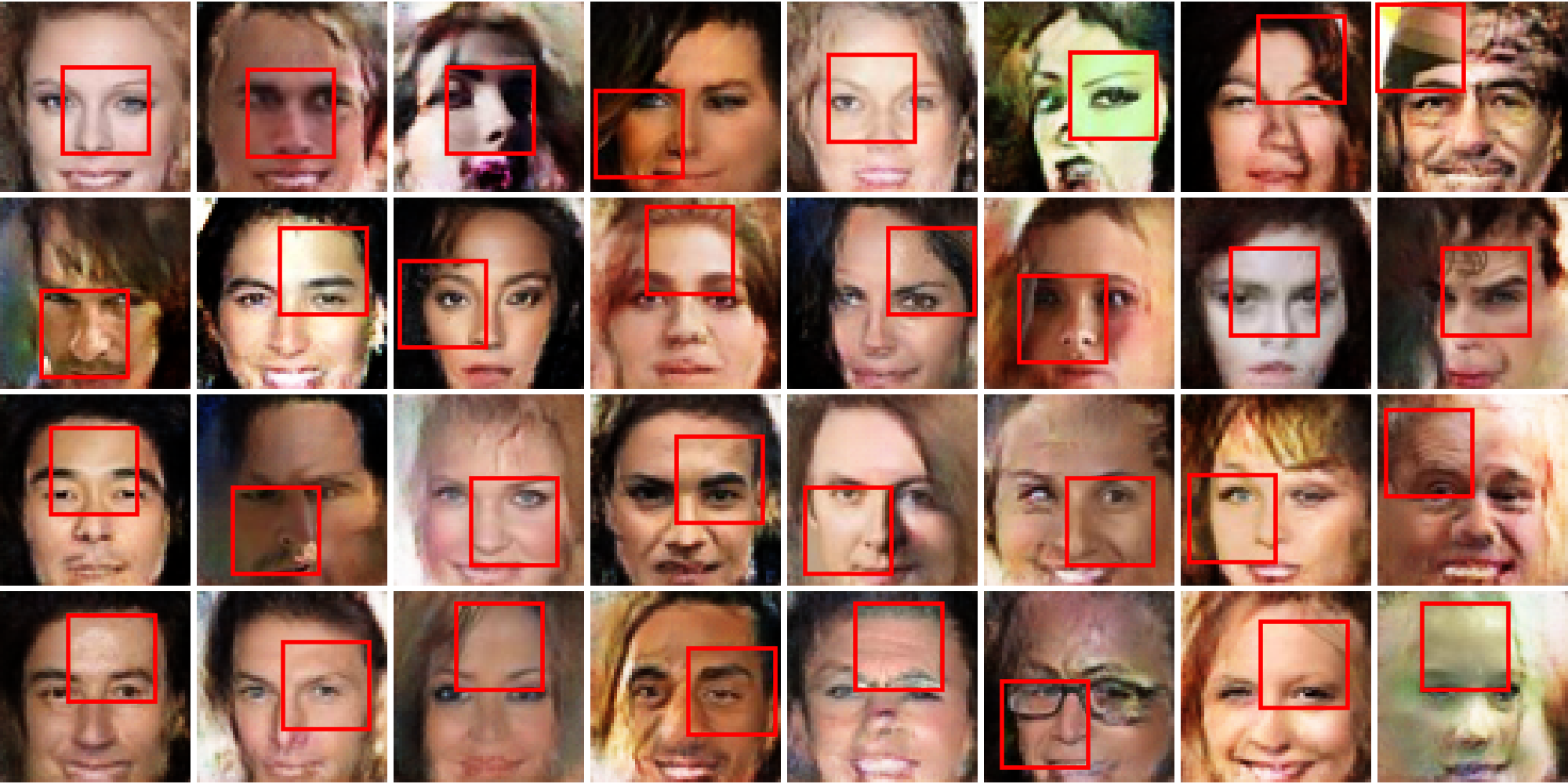}
      };
      \node[inner sep=0pt,right of=22,node distance=\figgap,
            label=below:{\small 61$\times$61 (10\% missing)}] (23) {
        \includegraphics[width=\figwidth]{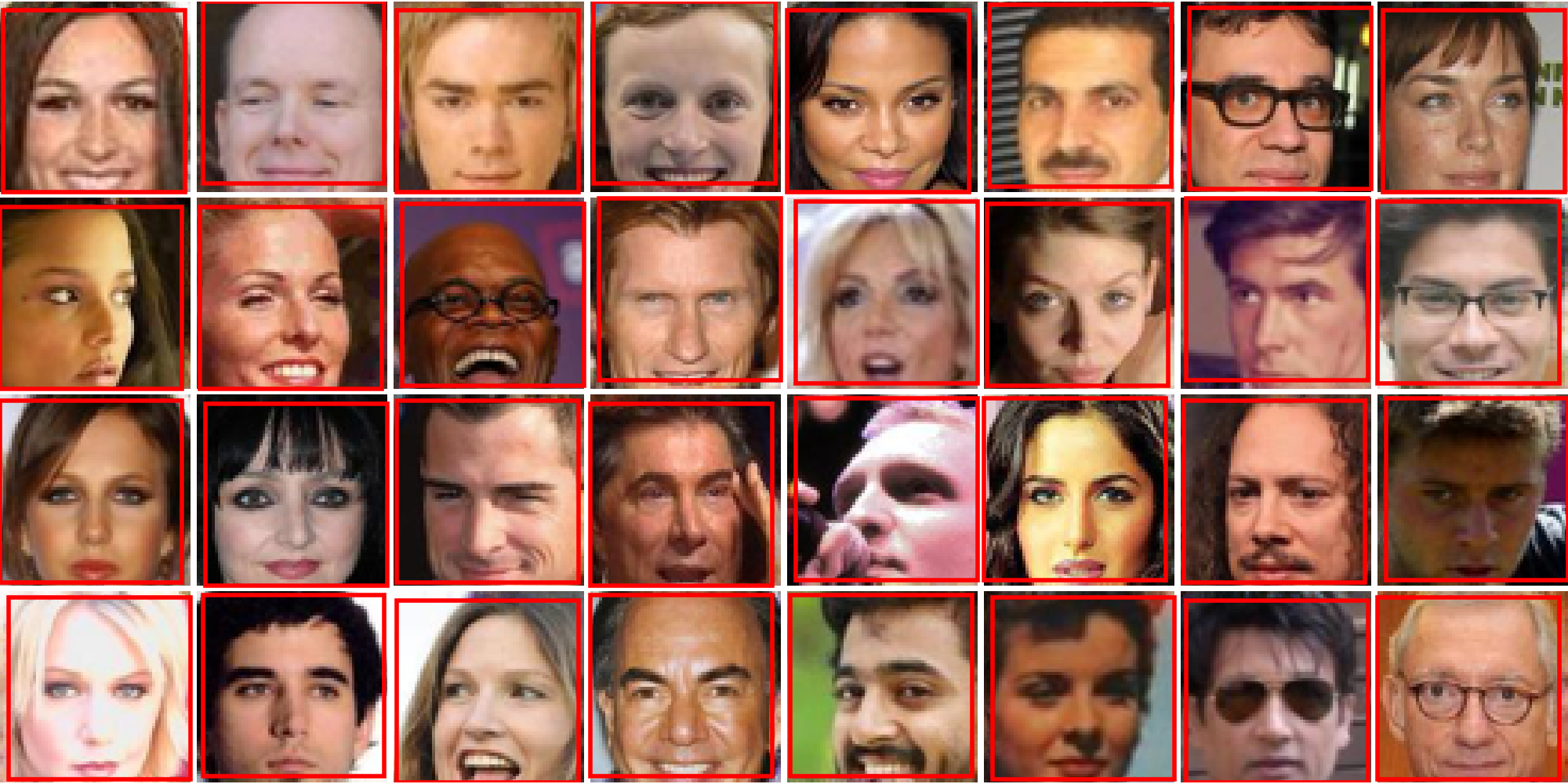}
      };
    \end{tikzpicture}
    \vspace*{-.5em}
    \caption{square observation missingness}
  \end{subfigure}

  \vspace{1em}
  \begin{subfigure}[b]{\textwidth}
    \centering
    \begin{tikzpicture}
      \node[inner sep=0pt,
            label=below:{\small 90\% missing}] (11) {
        \includegraphics[width=\figwidth]{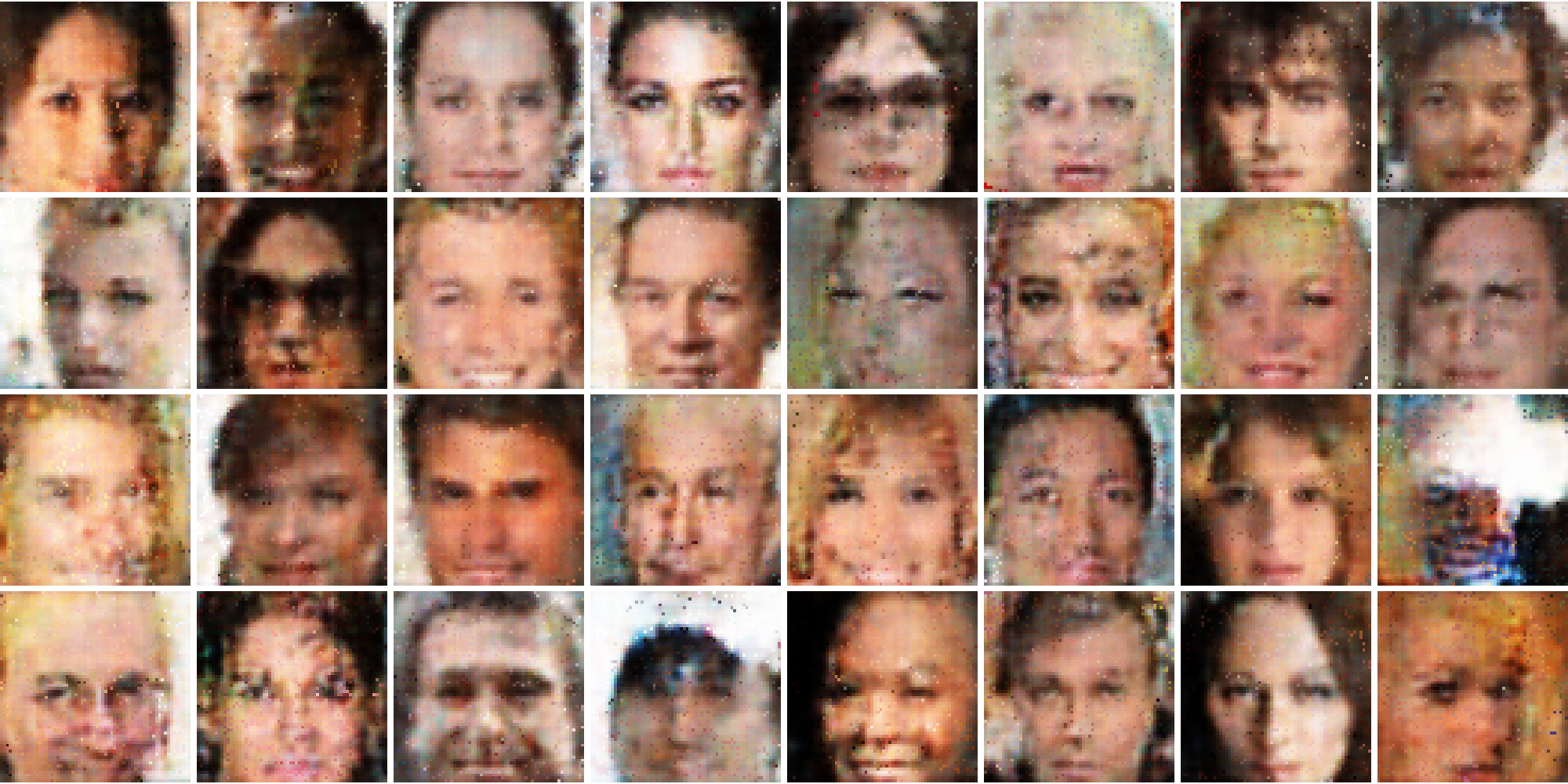}
      };
      \node[inner sep=0pt,right of=11,node distance=\figgap,
            label=below:{\small 80\% missing}] (12) {
        \includegraphics[width=\figwidth]{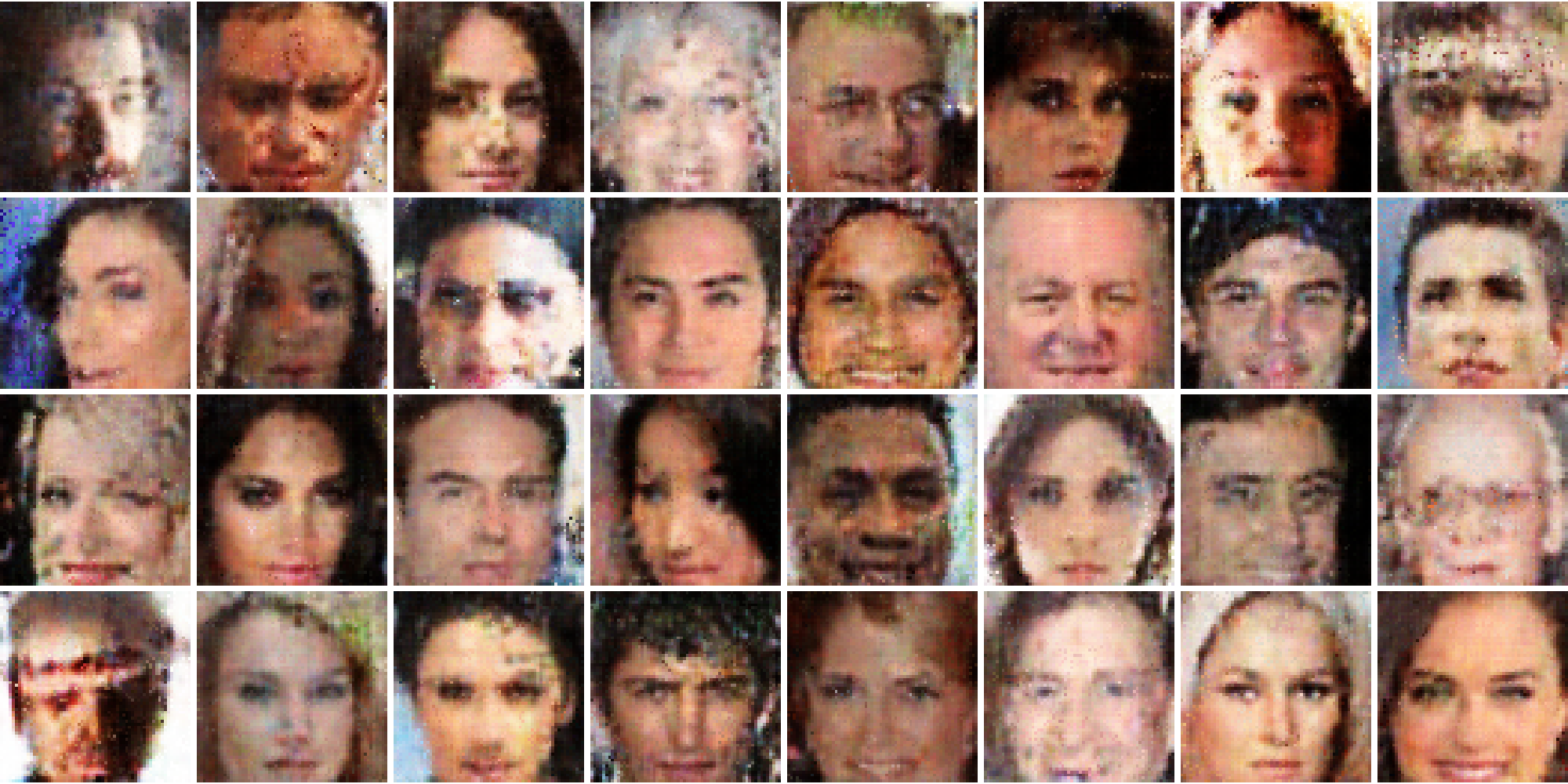}
      };
      \node[inner sep=0pt,right of=12,node distance=\figgap,
            label=below:{\small 10\% missing}] (13) {
        \includegraphics[width=\figwidth]{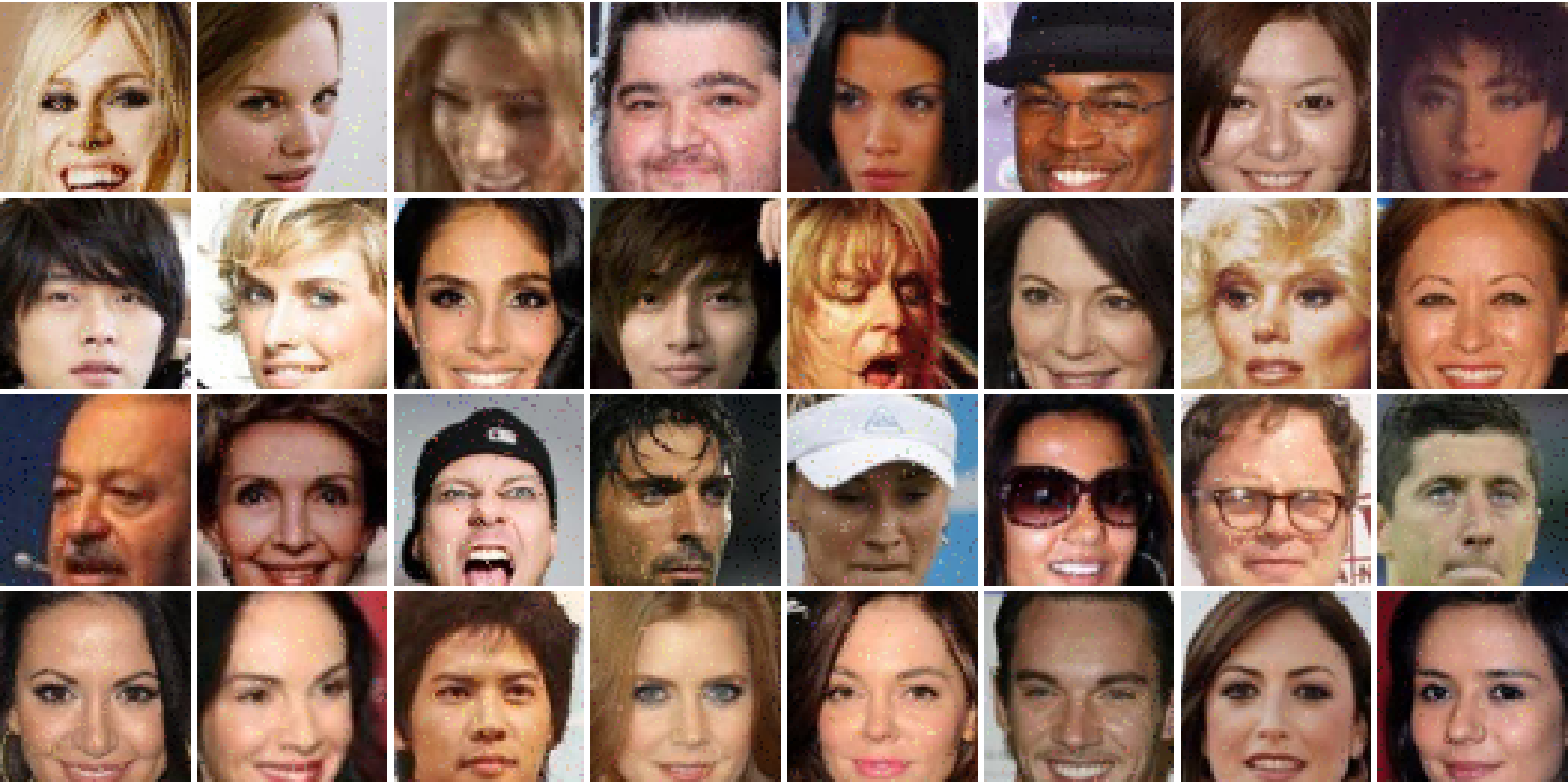}
      };
    \end{tikzpicture}
    \vspace*{-.5em}
    \caption{independent dropout missingness}
  \end{subfigure}

  \caption{{\misgan} imputation on CelebA}
\label{fig:celebaimpute}
\end{figure}

\section{Results of ConvAC}
\label{sec:convac}
Figure~\ref{fig:convac} shows the samples generated by ConvAC trained
with the square observation missing pattern on MNIST.

%However, we find that when training ConvAC on CIFAR-10,
%the generated samples have severe aliasing artifacts and fail to
%capture the details in the training images as shown in
%Figure~\ref{fig:convaccifar10}.
%On the other hand, ConvAC is hard to scale up to data
%like CelebA. For CelebA, if we use 256 channels for the first layer,
%it requires 50,331,648 parameters for the input image of size 64$\times$64.
\def\figwidth{1.8in}
\begin{figure}
  \centering
  \begin{subfigure}[b]{0.48\textwidth}
    \centering
    \includegraphics[width=\figwidth]{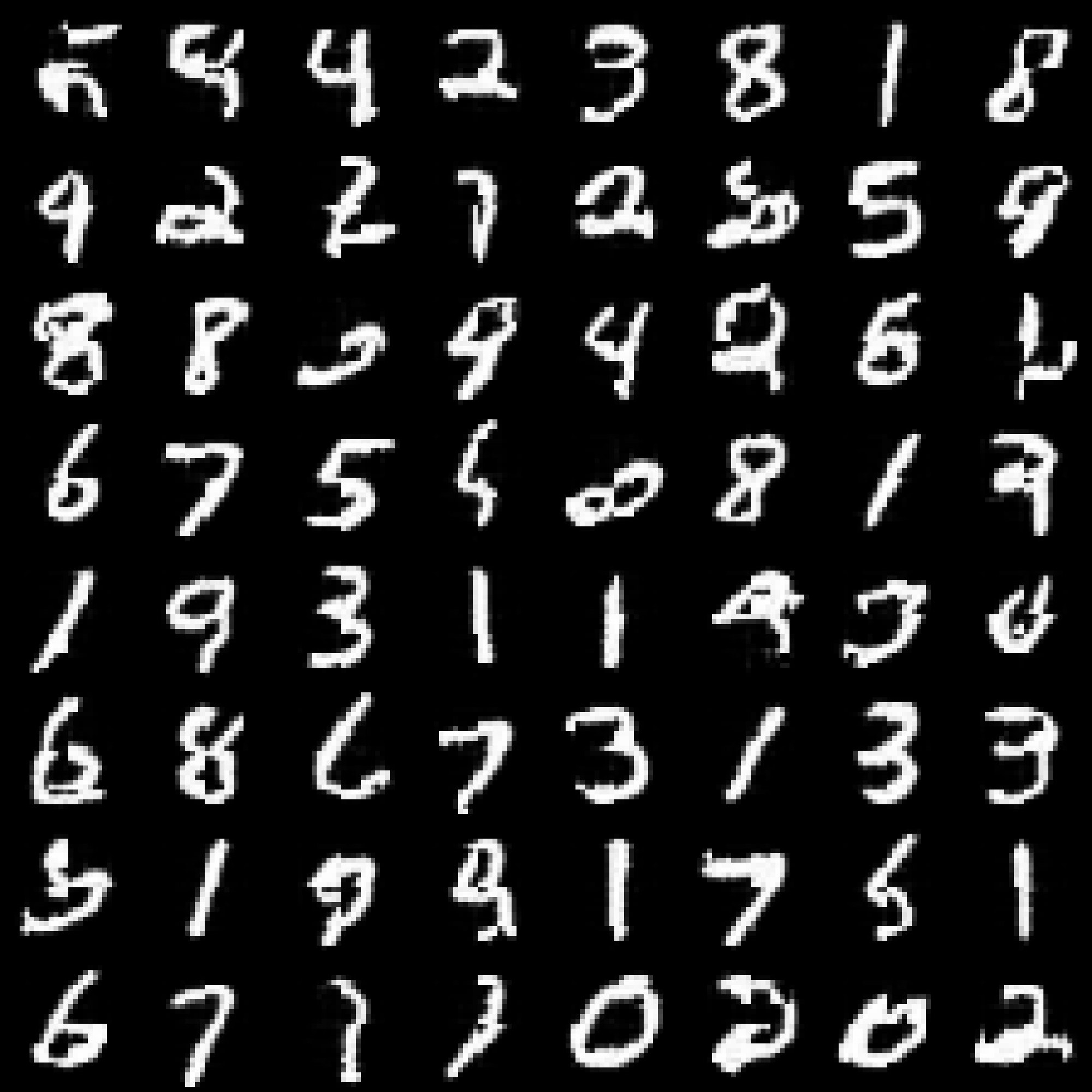}
    \caption{ConvAC: 9$\times$9 (90\%)}
  \end{subfigure}
  \begin{subfigure}[b]{0.48\textwidth}
    \centering
    \includegraphics[width=\figwidth]{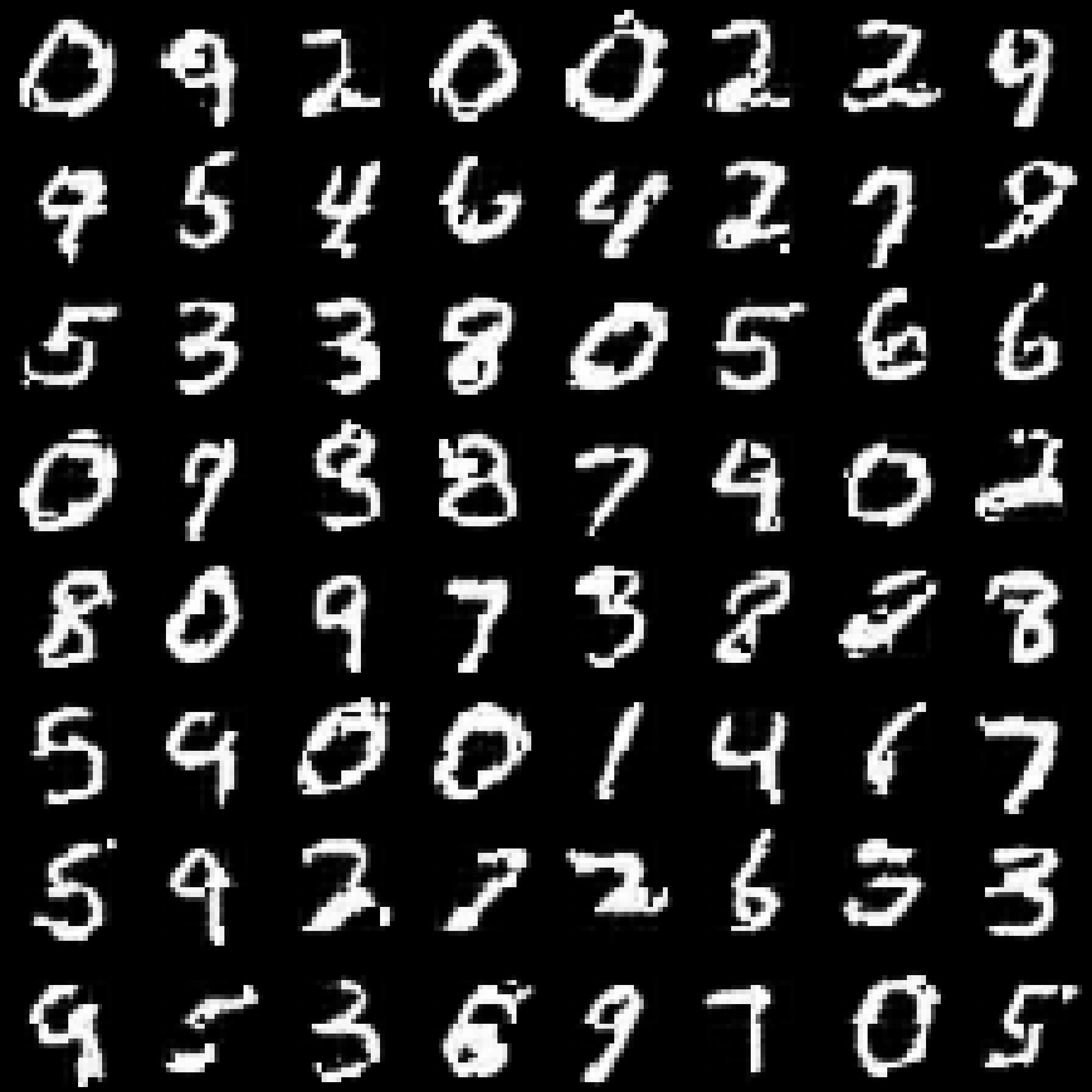}
    \caption{ConvAC: 13$\times$13 (80\%)}
  \end{subfigure}
  \hspace{.5em}
  \caption{Results of ConvAC trained with square observations of
  different sizes on MNIST.}
  \label{fig:convac}
\end{figure}

%\begin{figure}
%  \def\figwidth{.4\textwidth}
%  \def\figgap{\figwidth+3em}
%
%  \centering
%  \begin{subfigure}[b]{\textwidth}
%    \centering
%    \begin{tikzpicture}
%      \node[inner sep=0pt,
%            label=below:{\small 10$\times$10 (90\% missing)}] (11) {
%        \includegraphics[width=\figwidth]{convac/cifar10/block1.png}
%      };
%      \node[inner sep=0pt,right of=11,node distance=\figgap,
%            label=below:{\small 30$\times$30 (10\% missing)}] (12) {
%        \includegraphics[width=\figwidth]{convac/cifar10/block9.png}
%      };
%    \end{tikzpicture}
%    \vspace*{-.5em}
%    \caption{square observation missingness}
%  \end{subfigure}
%
%  \vspace{1em}
%  \begin{subfigure}[b]{\textwidth}
%    \centering
%    \begin{tikzpicture}
%      \node[inner sep=0pt,
%            label=below:{\small 90\% missing}] (11) {
%        \includegraphics[width=\figwidth]{convac/cifar10/indep1.png}
%      };
%      \node[inner sep=0pt,right of=11,node distance=\figgap,
%            label=below:{\small 10\% missing}] (12) {
%        \includegraphics[width=\figwidth]{convac/cifar10/indep9.png}
%      };
%    \end{tikzpicture}
%    \vspace*{-.5em}
%    \caption{independent dropout missingness}
%  \end{subfigure}
%
%  \caption{Samples generated by ConvAC trained on CIFAR-10.}
%\label{fig:convaccifar10}
%\end{figure}

\section{Missing data imputation with GAIN}
\label{sec:gain}

Figure~\ref{fig:gainfail} shows the imputation results of GAIN on
different epochs during training with the 20$\times$20 square observation
missingnss.
We found that this is a common phenomenon for the square observation
missing pattern.
To obtain better results for GAIN, we analyze the FIDs during the course
of training and use the model that achieves the best FID to favorably compare
with {\misgan} for the square observation case.
For CIFAR-10, we use the results from the 500th epoch;
for CelebA, we use the results from the 50th epoch.
Otherwise, we train GAIN for 1000 epochs for CIFAR-10 and 300 epochs for CelebA.
Our implementation is adapted from the code released by the
authors of GAIN.\footnote{
\url{https://github.com/jsyoon0823/GAIN}}

Figure~\ref{fig:gainresults} shows the imputation results of GAIN
for both CIFAR-10 and CelebA.

\begin{figure}
  \def\figwidth{.32\textwidth}
  \def\figgap{\figwidth+.3em}

  \centering
  \begin{subfigure}[b]{\textwidth}
    \centering
    \begin{tikzpicture}
      \node[inner sep=0pt,
            label=below:{\small 30th epoch}] (11) {
        \includegraphics[width=\figwidth]{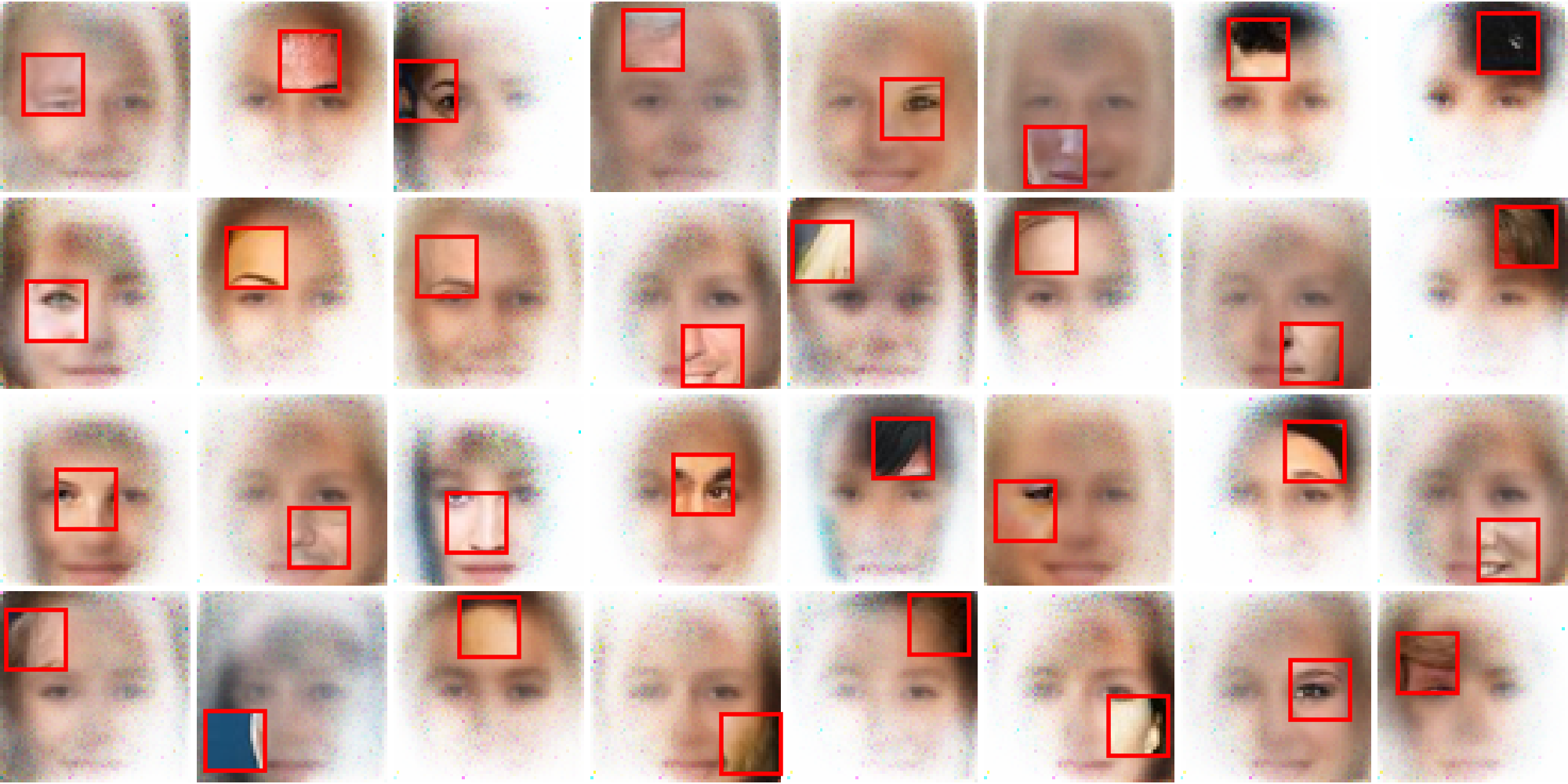}
      };
      \node[inner sep=0pt,right of=11,node distance=\figgap,
            label=below:{\small 60th epoch}] (12) {
        \includegraphics[width=\figwidth]{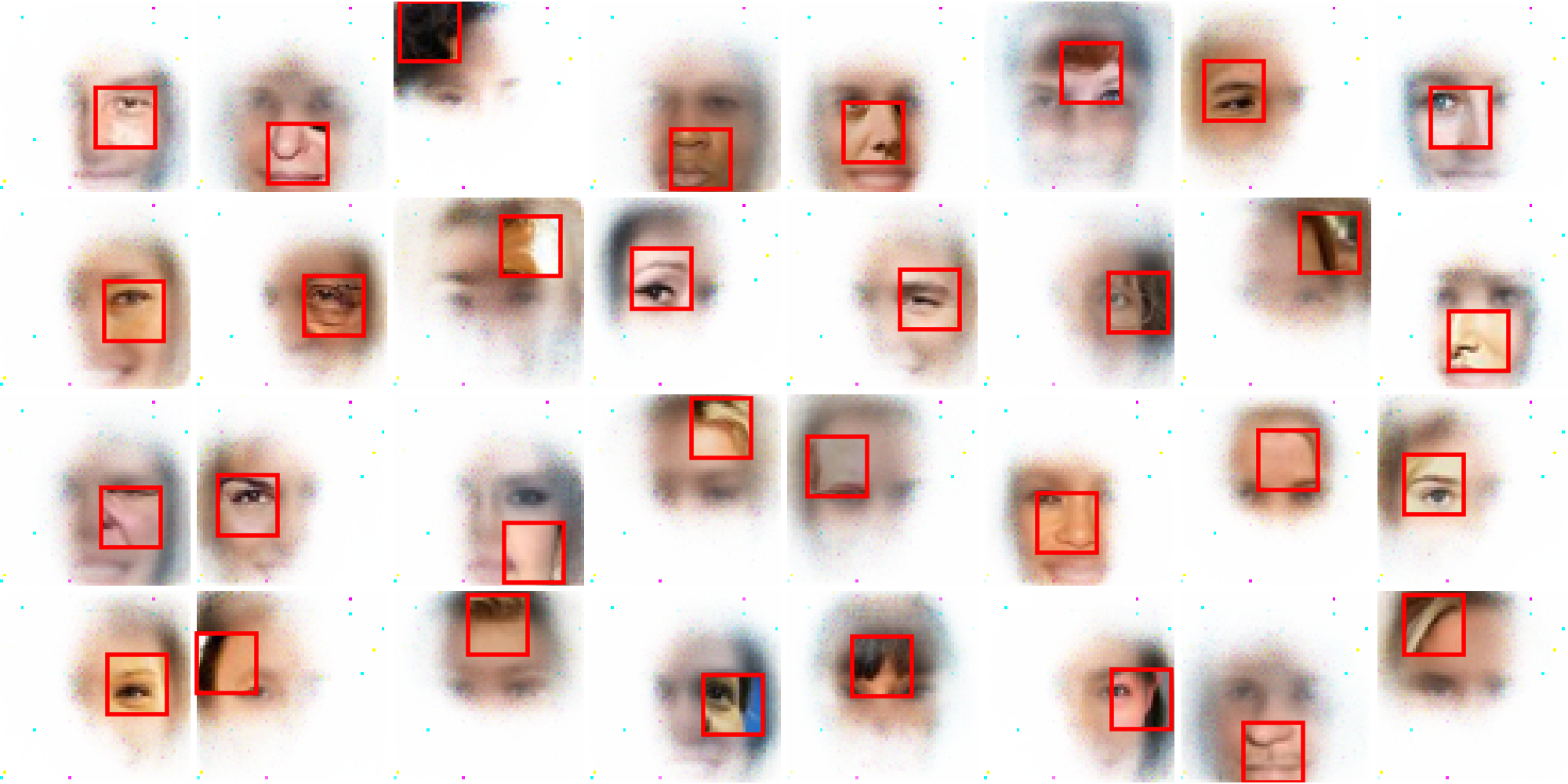}
      };
      \node[inner sep=0pt,right of=12,node distance=\figgap,
            label=below:{\small 90th epoch}] (13) {
        \includegraphics[width=\figwidth]{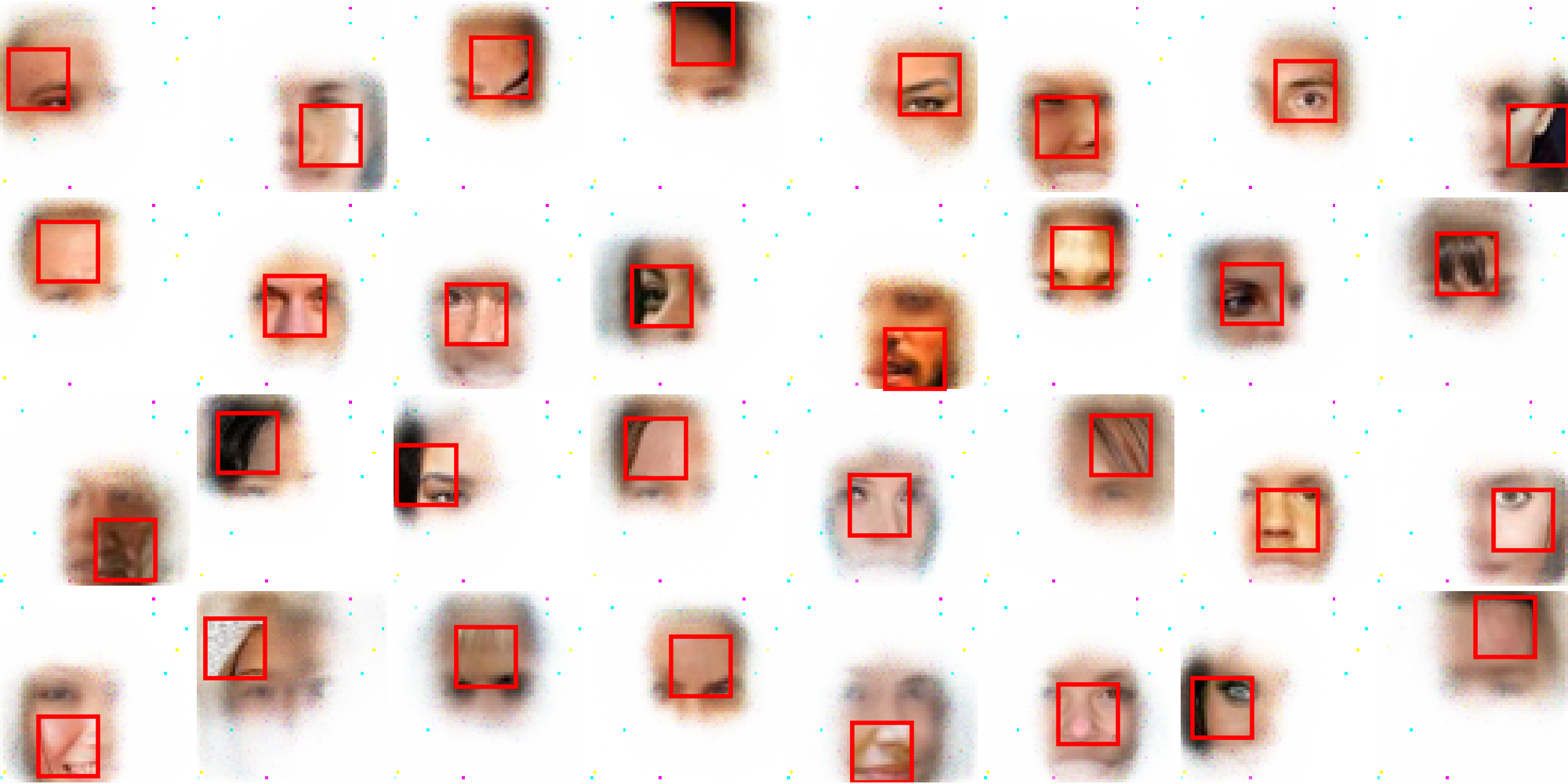}
      };
    \end{tikzpicture}
    %\vspace*{-.5em}
    %\caption{independent dropout missingness}
  \end{subfigure}

  \vspace{1em}
  \begin{subfigure}[b]{\textwidth}
    \centering
    \begin{tikzpicture}
      \node[inner sep=0pt,
            label=below:{\small 120th epoch}] (11) {
        \includegraphics[width=\figwidth]{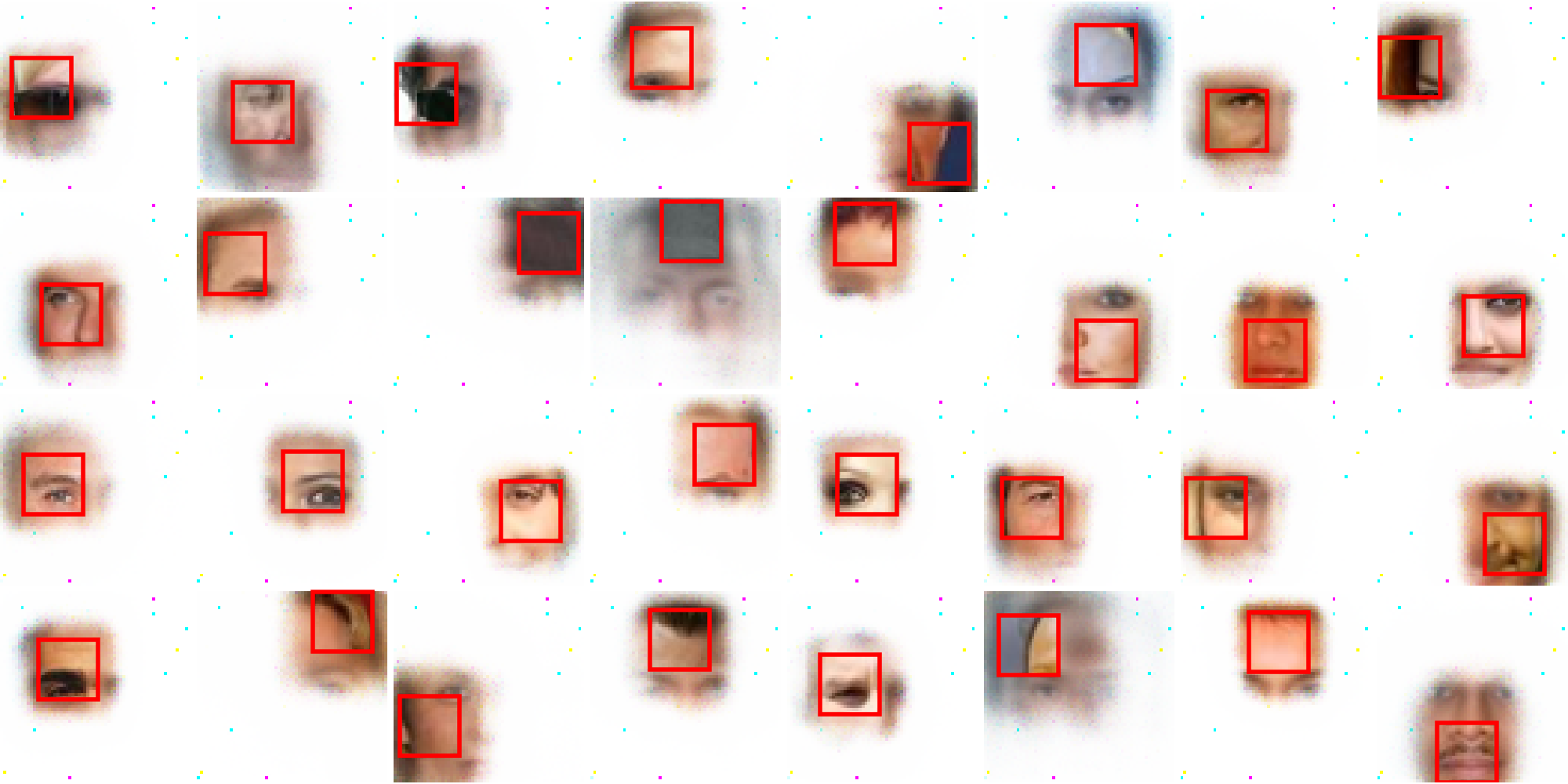}
      };
      \node[inner sep=0pt,right of=11,node distance=\figgap,
            label=below:{\small 150th epoch}] (12) {
        \includegraphics[width=\figwidth]{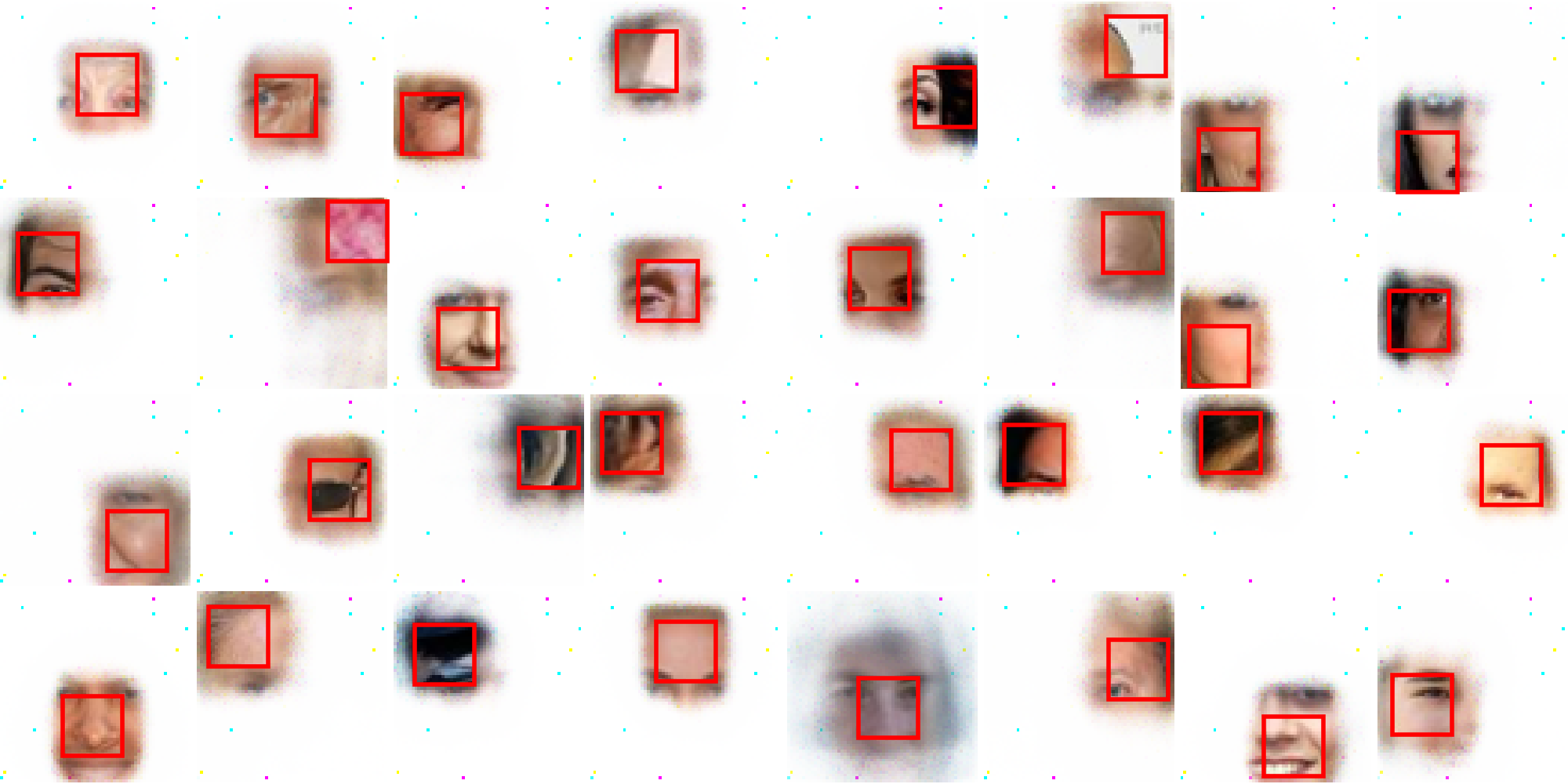}
      };
      \node[inner sep=0pt,right of=12,node distance=\figgap,
            label=below:{\small 180th epoch}] (13) {
        \includegraphics[width=\figwidth]{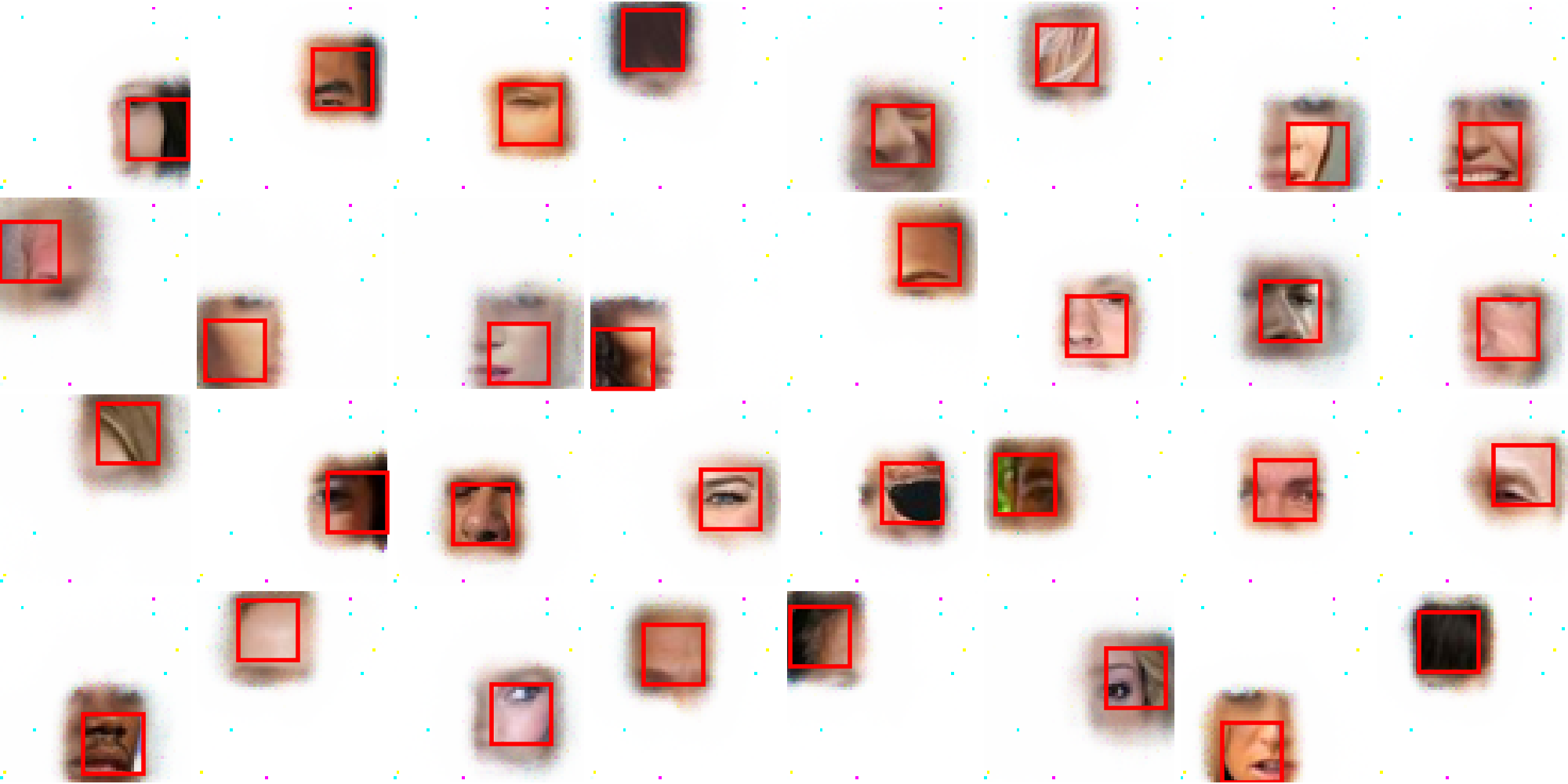}
      };
    \end{tikzpicture}
    %\vspace*{-.5em}
    %\caption{square observation missingness}
  \end{subfigure}

  \vspace{1em}
  \begin{subfigure}[b]{\textwidth}
    \centering
    \begin{tikzpicture}
      \node[inner sep=0pt,
            label=below:{\small 210th epoch}] (11) {
        \includegraphics[width=\figwidth]{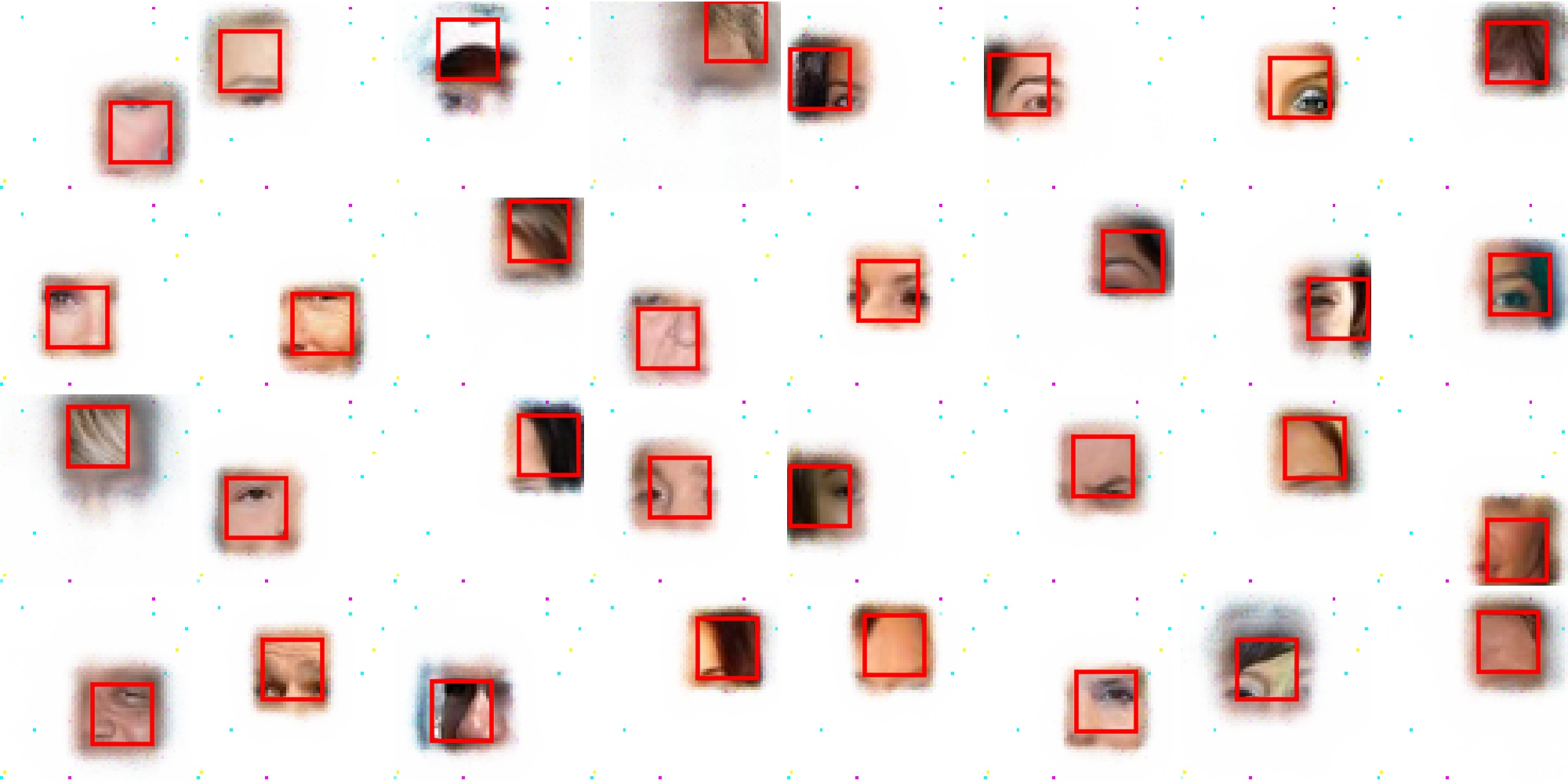}
      };
      \node[inner sep=0pt,right of=11,node distance=\figgap,
            label=below:{\small 240th epoch}] (12) {
        \includegraphics[width=\figwidth]{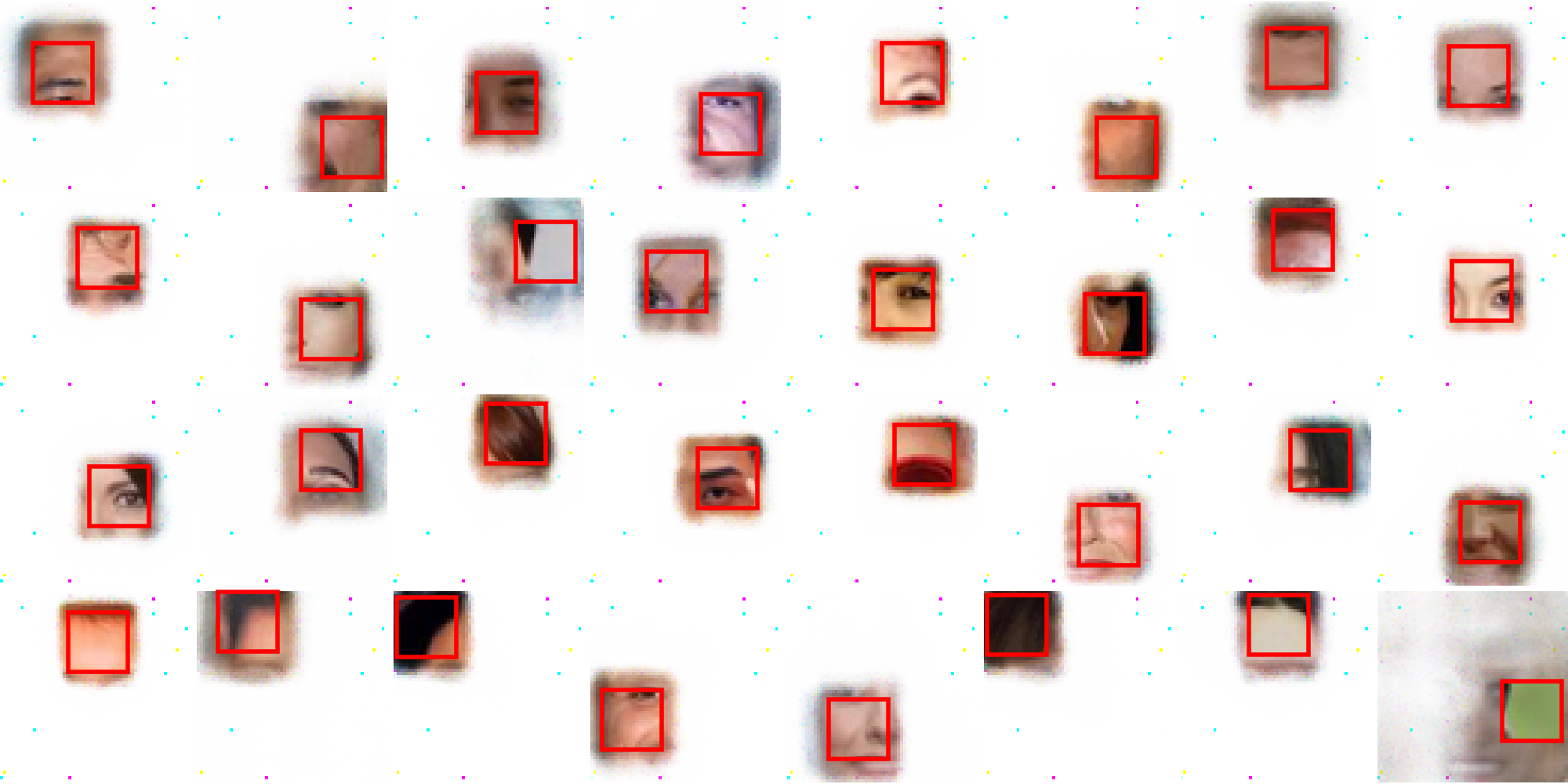}
      };
      \node[inner sep=0pt,right of=12,node distance=\figgap,
            label=below:{\small 270th epoch}] (13) {
        \includegraphics[width=\figwidth]{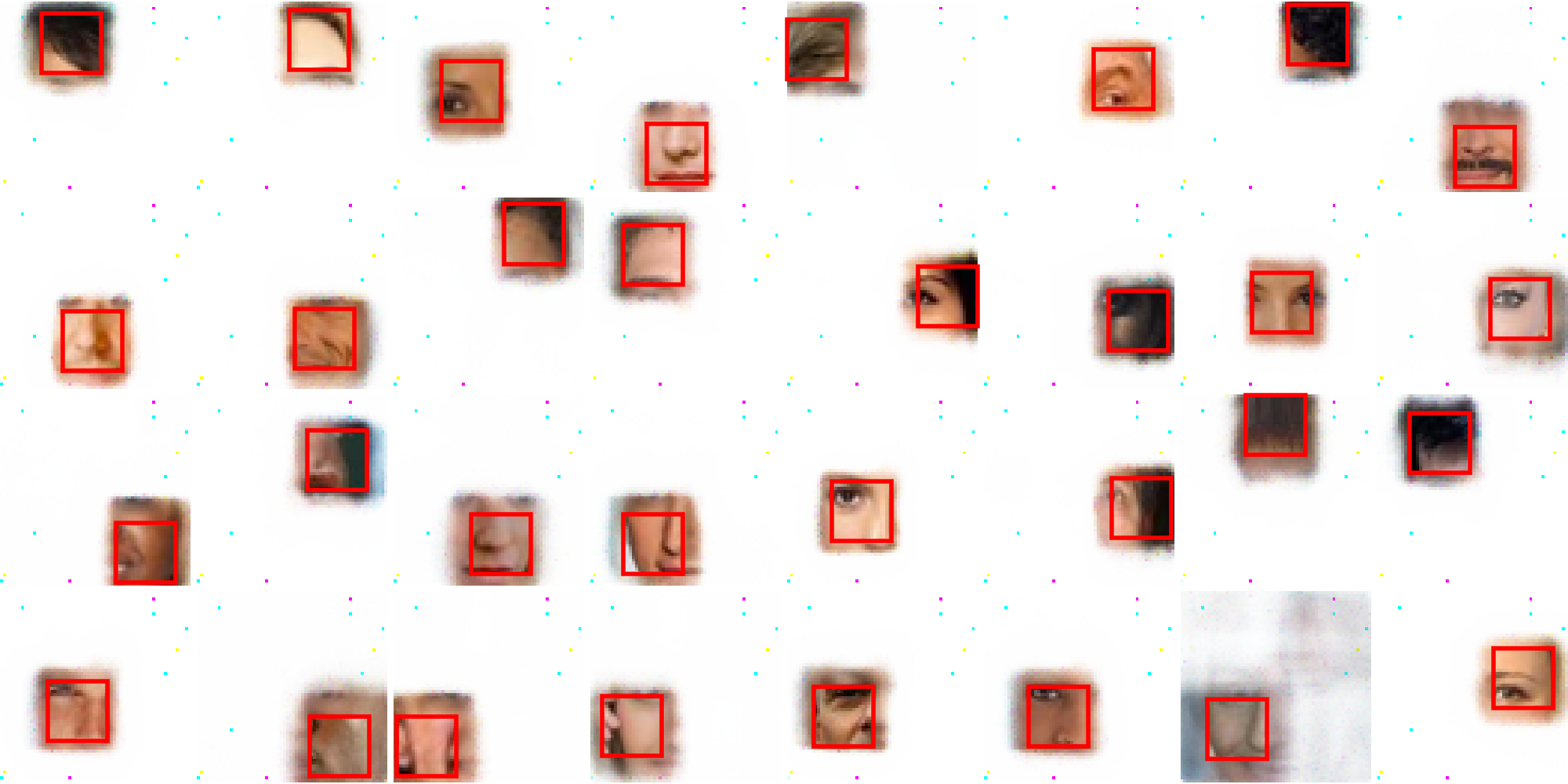}
      };
    \end{tikzpicture}
    %\vspace*{-.5em}
    %\caption{square observation missingness}
  \end{subfigure}

  \caption{Imputation results of GAIN on different epochs during training
    under 20$\times$20 square observation missingness.
    If over-trained,
    the imputation behavior of GAIN will gradually become similar
  to constant imputation.}
\label{fig:gainfail}
\end{figure}

\begin{figure}
  \def\figwidth{.32\textwidth}
  \def\figgap{\figwidth+.3em}

  \centering
  \begin{subfigure}[b]{\textwidth}
    \centering
    \begin{tikzpicture}
      \node[inner sep=0pt,
            label=below:{\small 10$\times$10 (90\% missing)}] (11) {
        \includegraphics[width=\figwidth]{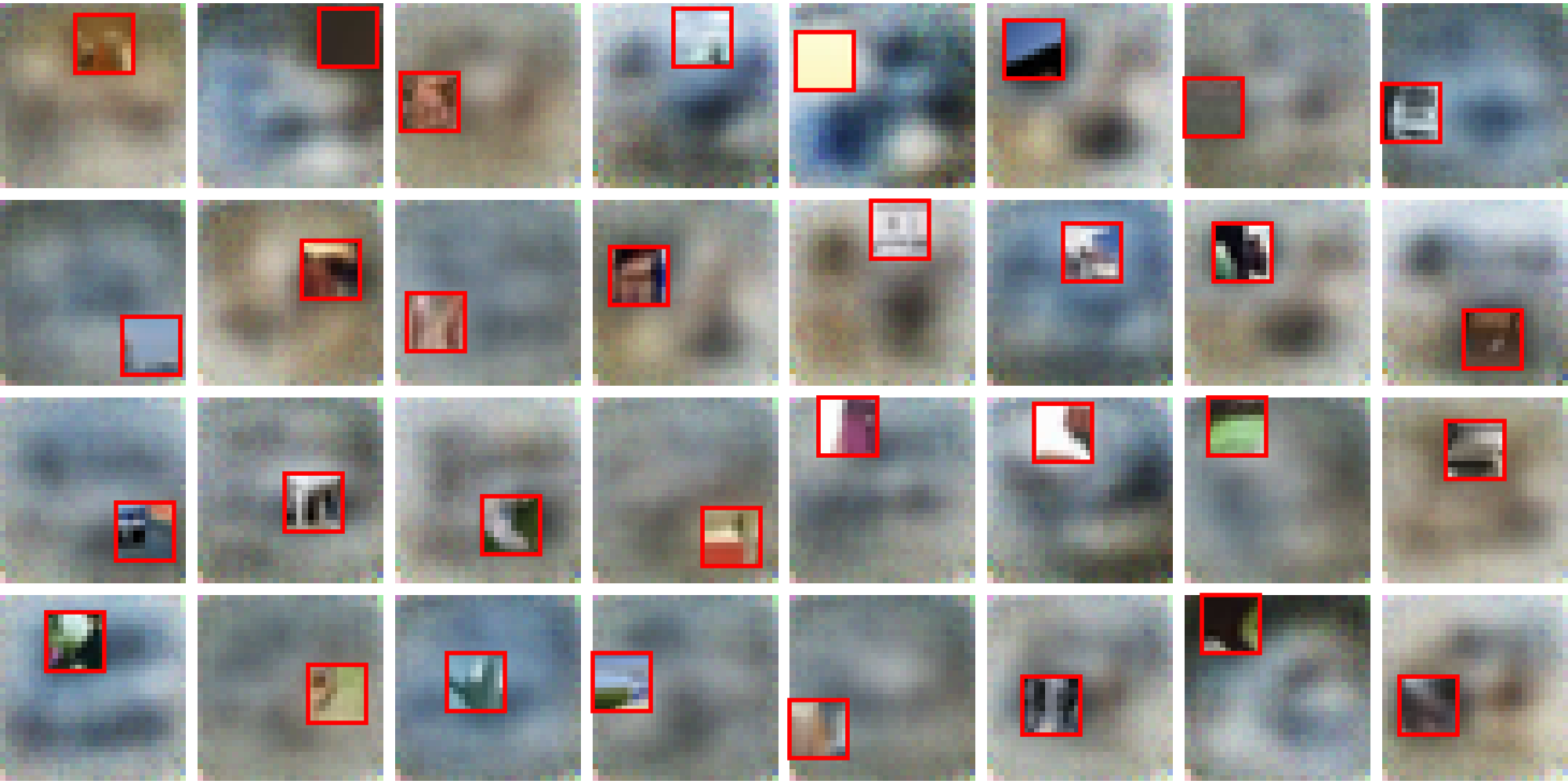}
      };
      \node[inner sep=0pt,right of=11,node distance=\figgap,
            label=below:{\small 14$\times$14 (80\% missing)}] (12) {
        \includegraphics[width=\figwidth]{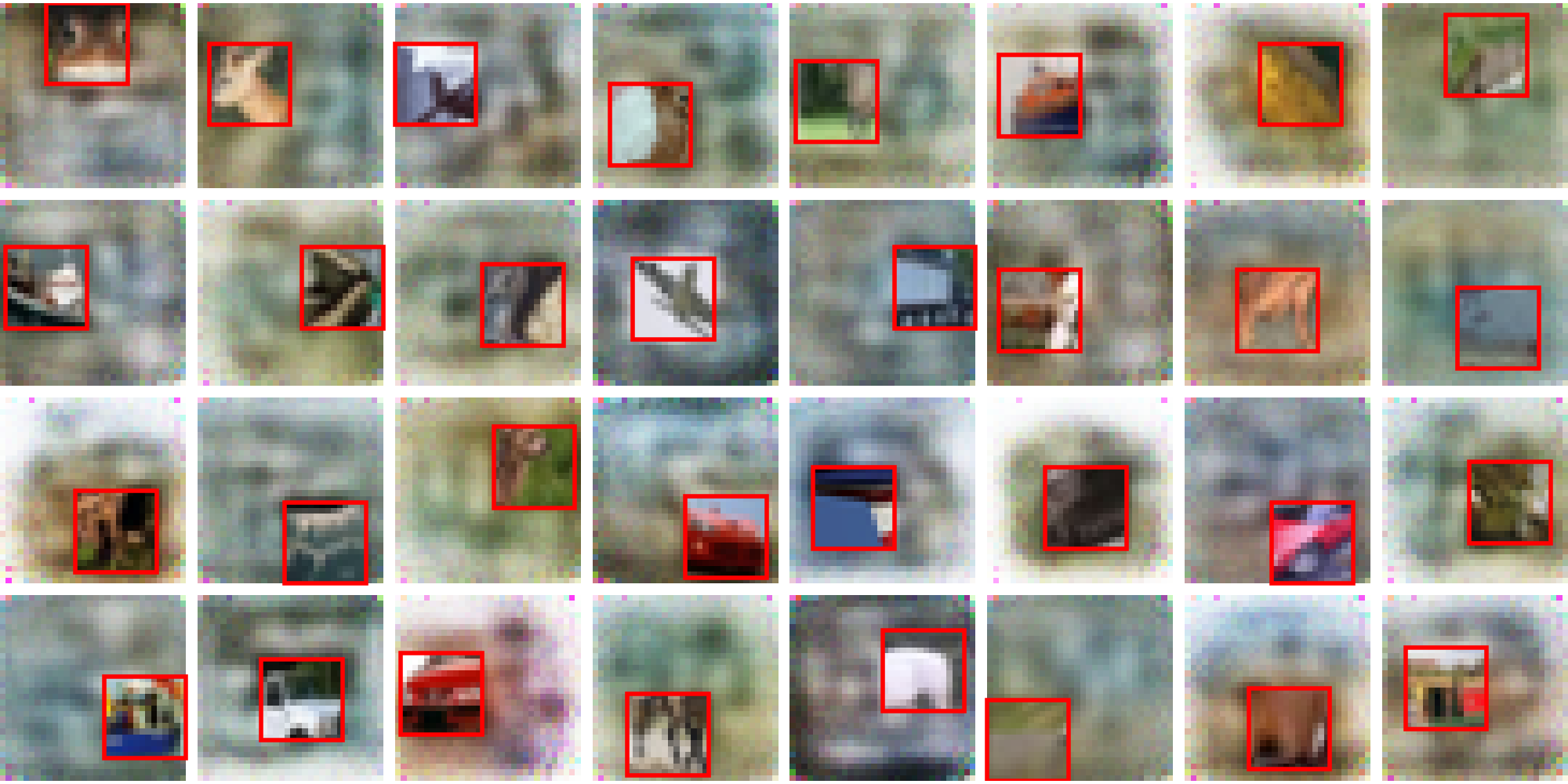}
      };
      \node[inner sep=0pt,right of=12,node distance=\figgap,
            label=below:{\small 30$\times$30 (10\% missing)}] (13) {
        \includegraphics[width=\figwidth]{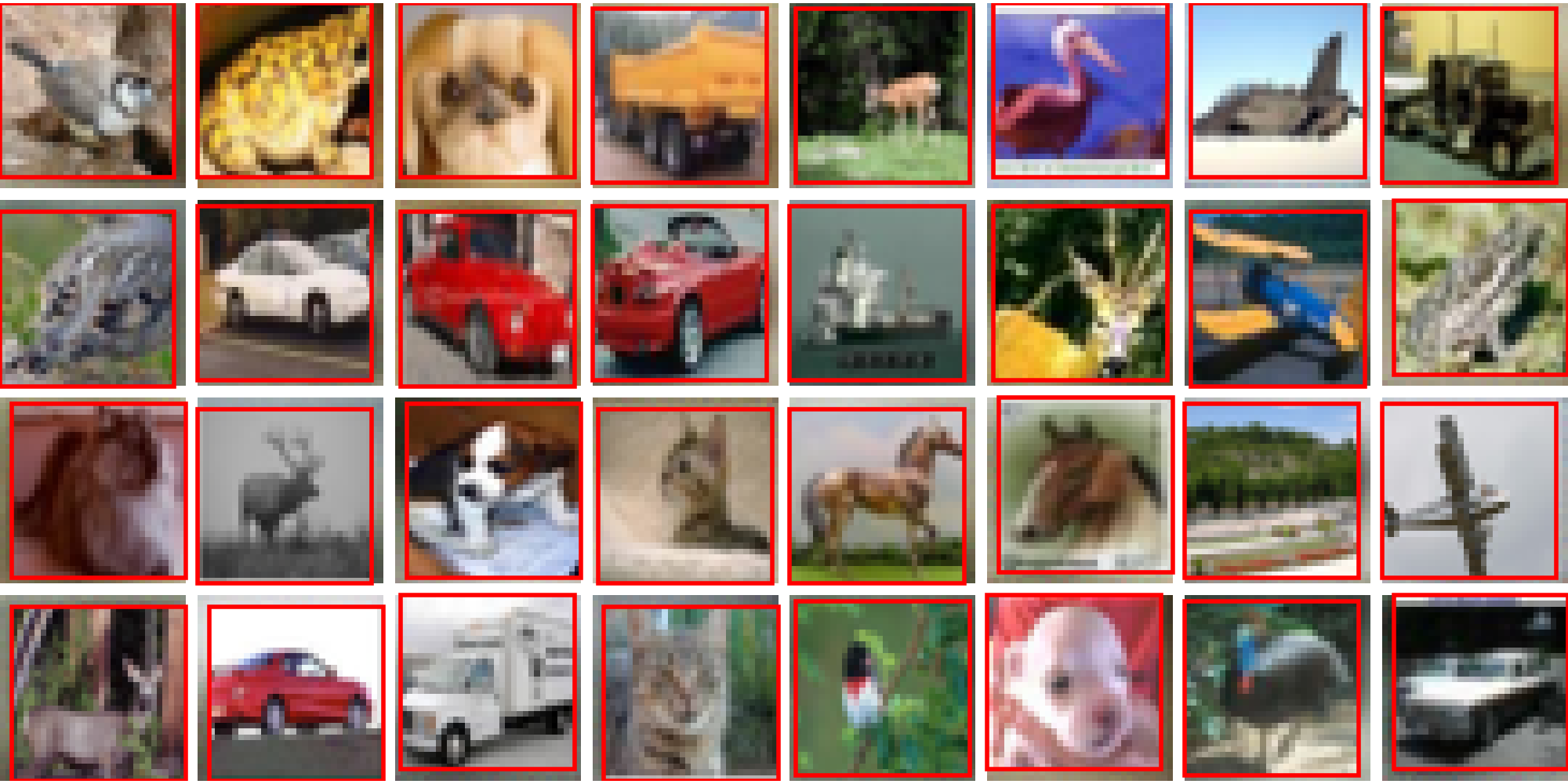}
      };
    \end{tikzpicture}
    \vspace*{-.5em}
    \caption{CIFAR-10 with square observation missingness}
  \end{subfigure}

  \vspace{1em}
  \begin{subfigure}[b]{\textwidth}
    \centering
    \begin{tikzpicture}
      \node[inner sep=0pt,
            label=below:{\small 90\% missing}] (21) {
        \includegraphics[width=\figwidth]{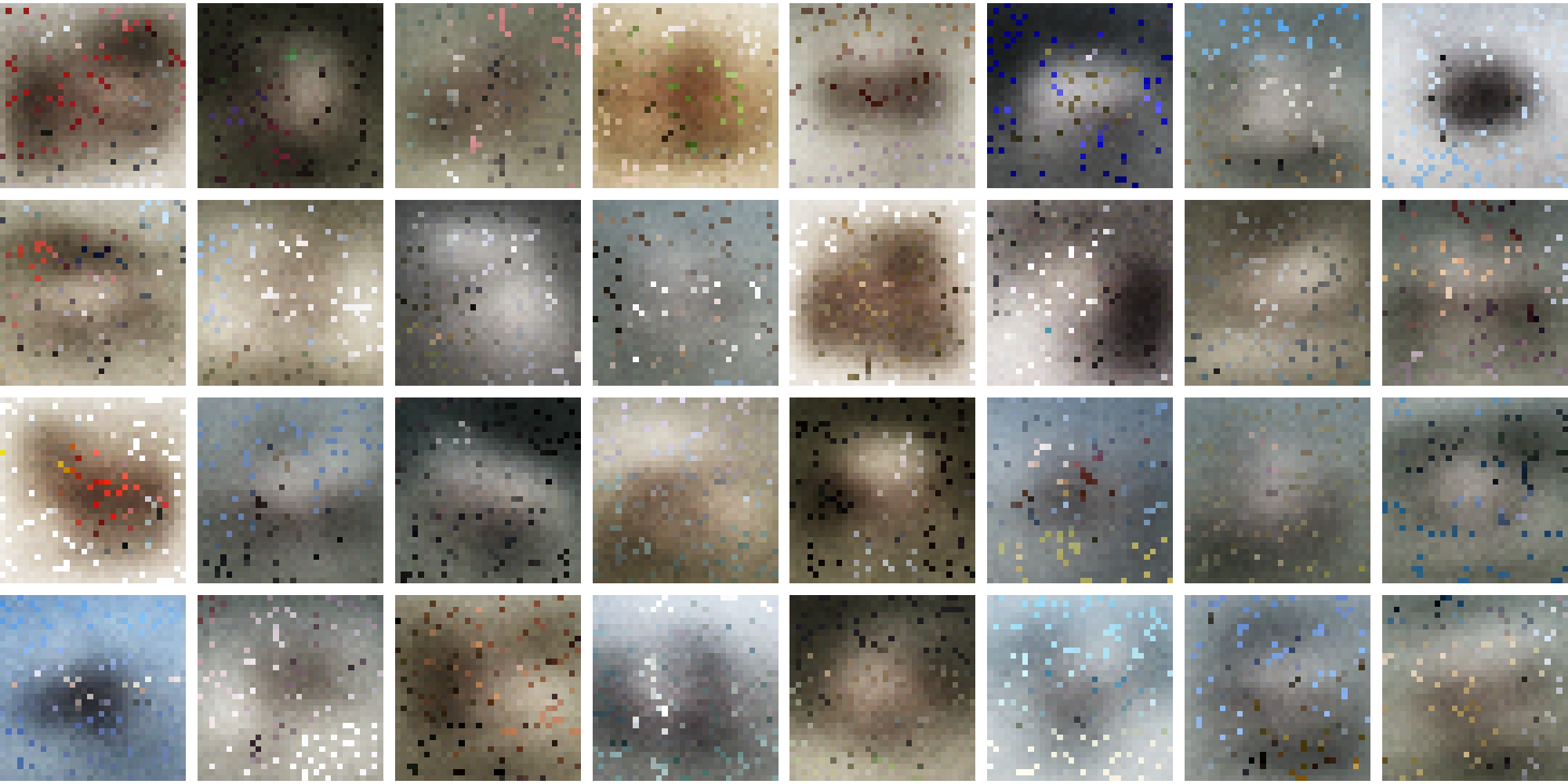}
      };
      \node[inner sep=0pt,right of=21,node distance=\figgap,
            label=below:{\small 80\% missing}] (22) {
        \includegraphics[width=\figwidth]{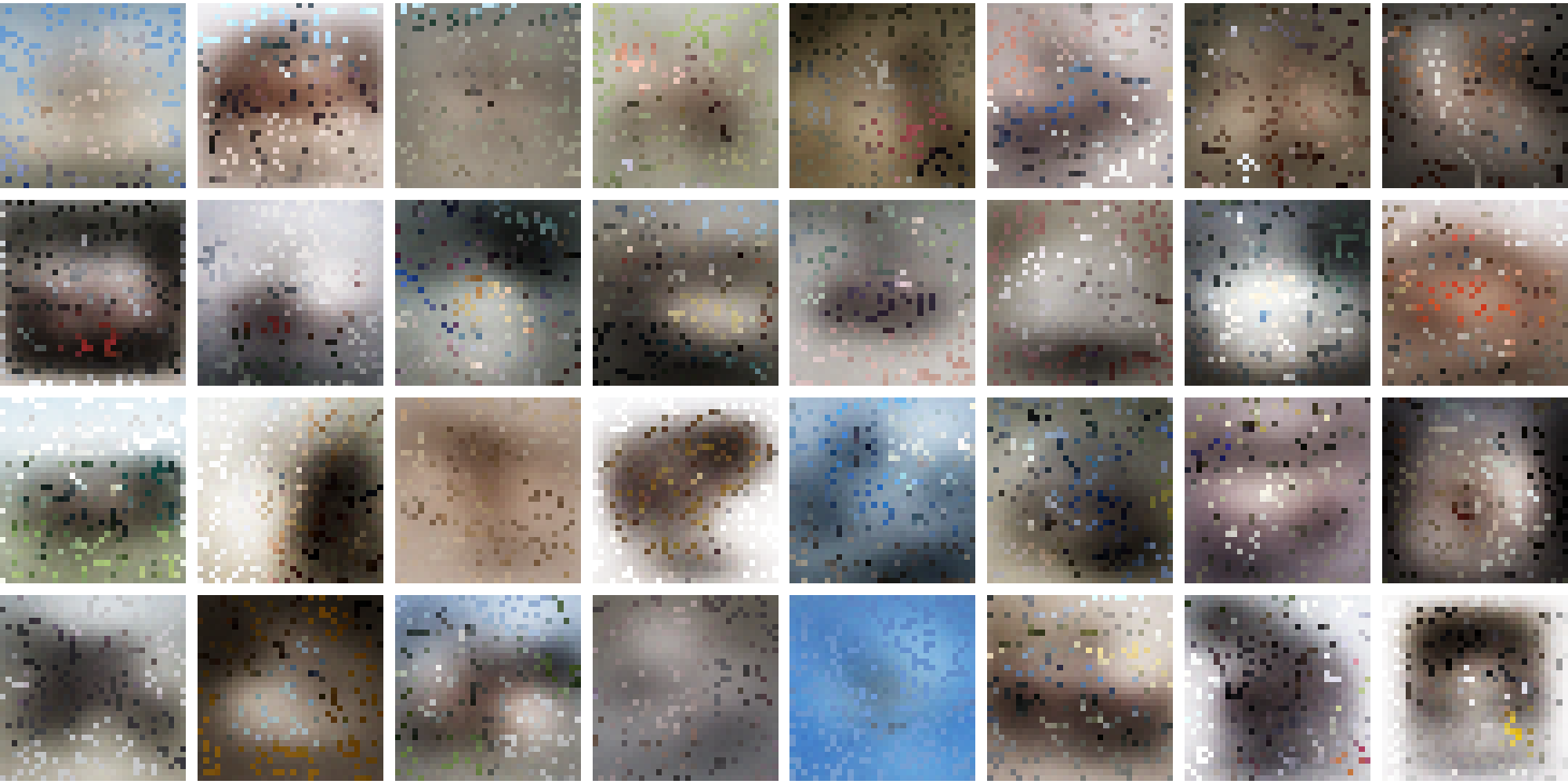}
      };
      \node[inner sep=0pt,right of=22,node distance=\figgap,
            label=below:{\small 10\% missing}] (23) {
        \includegraphics[width=\figwidth]{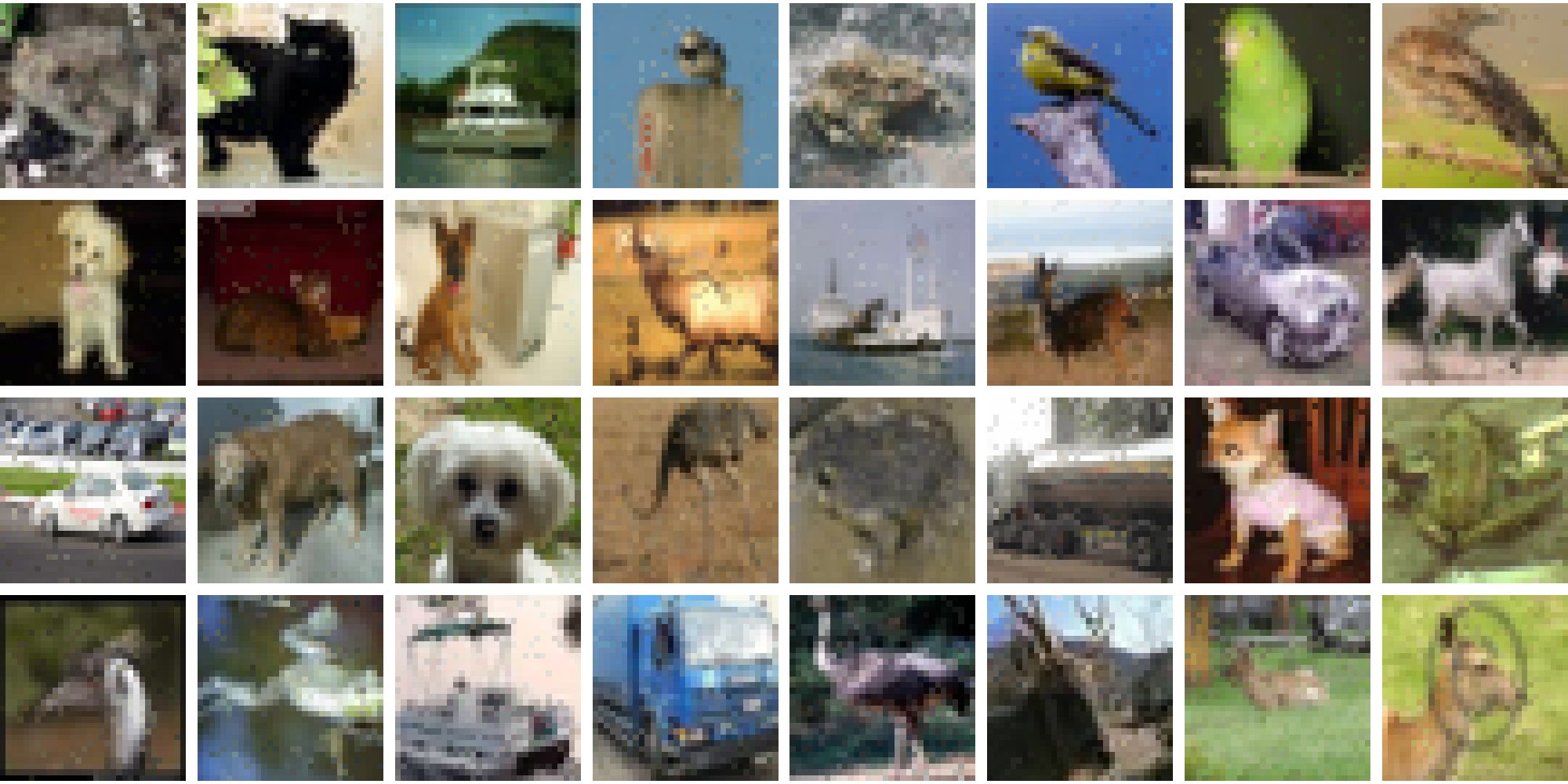}
      };
    \end{tikzpicture}
    \vspace*{-.5em}
    \caption{CIFAR-10 with independent dropout missingness}
  \end{subfigure}

  \vspace{1em}
  \begin{subfigure}[b]{\textwidth}
    \centering
    \begin{tikzpicture}
      \node[inner sep=0pt,
            label=below:{\small 20$\times$20 (90\% missing)}] (11) {
        \includegraphics[width=\figwidth]{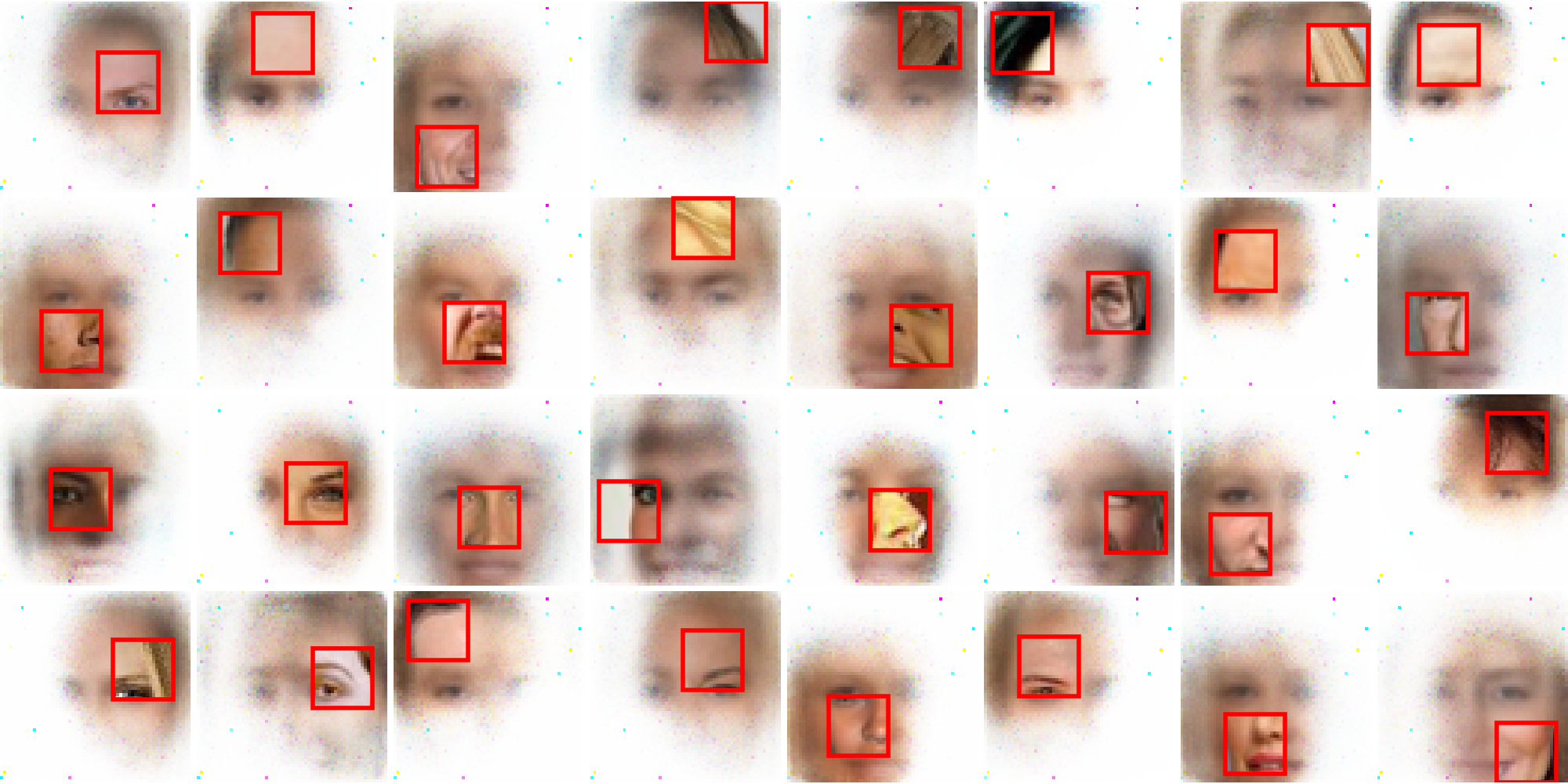}
      };
      \node[inner sep=0pt,right of=11,node distance=\figgap,
            label=below:{\small 29$\times$29 (80\% missing)}] (12) {
        \includegraphics[width=\figwidth]{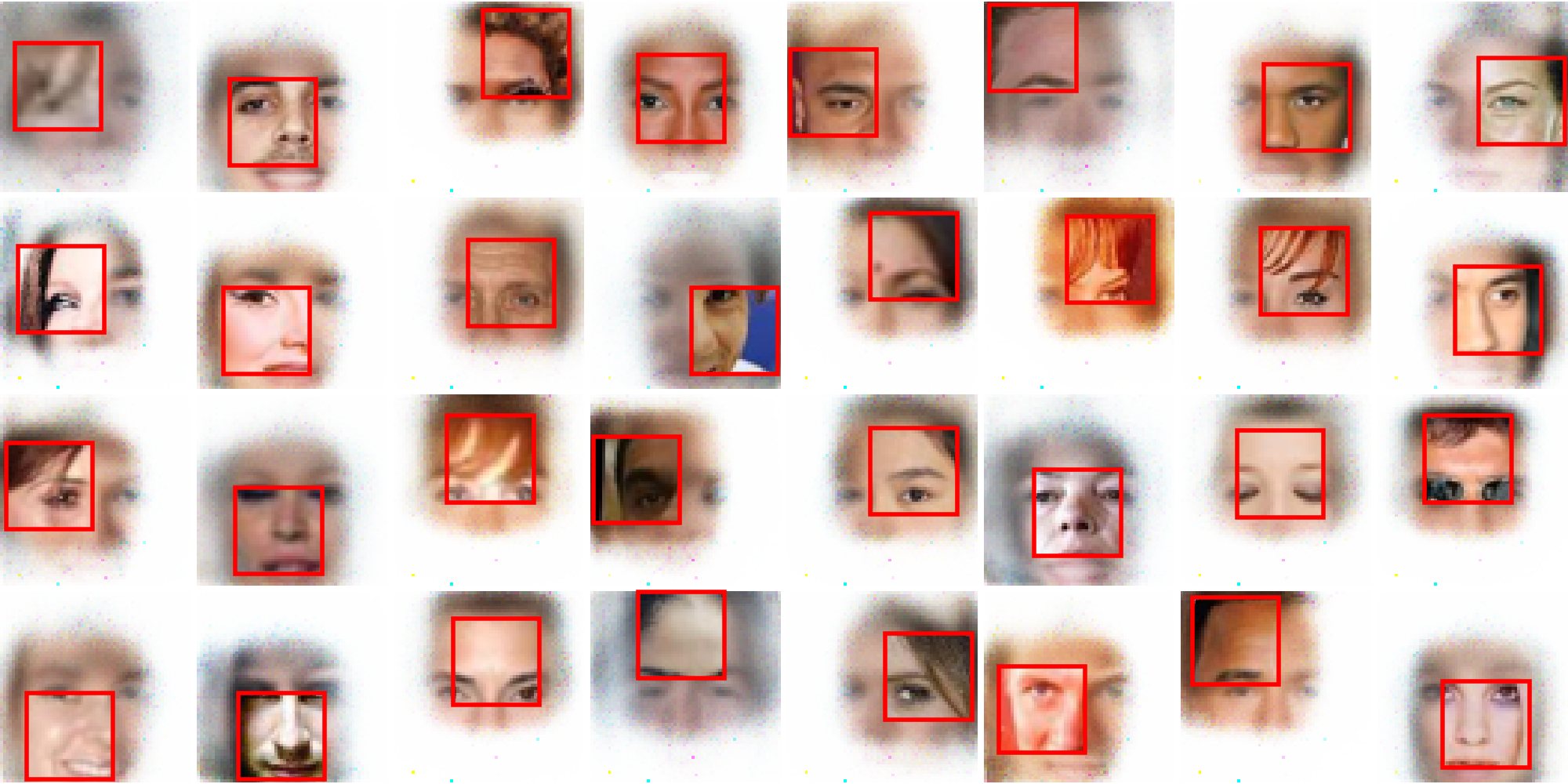}
      };
      \node[inner sep=0pt,right of=12,node distance=\figgap,
            label=below:{\small 61$\times$61 (10\% missing)}] (13) {
        \includegraphics[width=\figwidth]{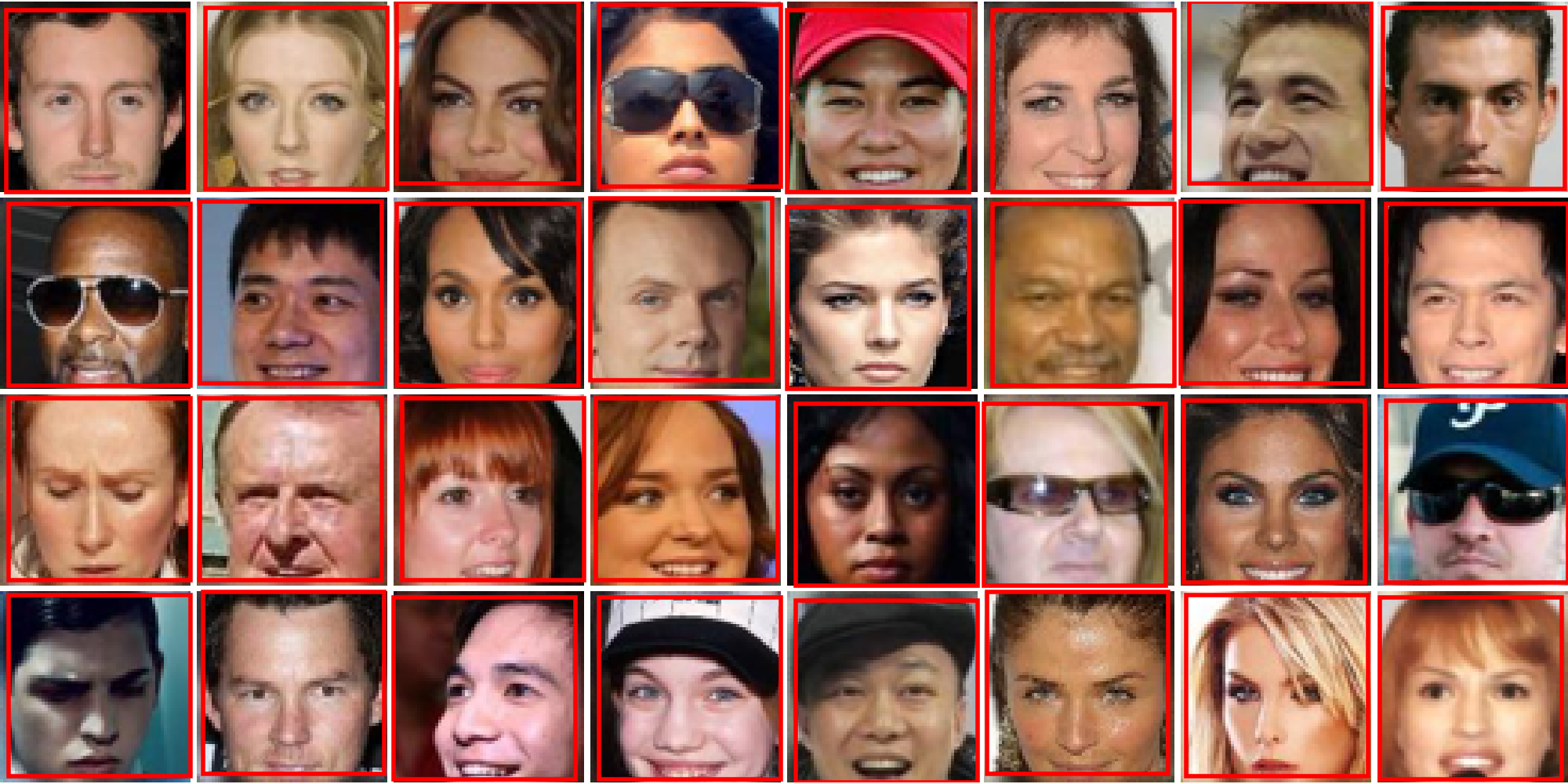}
      };
    \end{tikzpicture}
    \vspace*{-.5em}
    \caption{CelebA with square observation missingness}
  \end{subfigure}

  \vspace{1em}
  \begin{subfigure}[b]{\textwidth}
    \centering
    \begin{tikzpicture}
      \node[inner sep=0pt,
            label=below:{\small 90\% missing}] (21) {
        \includegraphics[width=\figwidth]{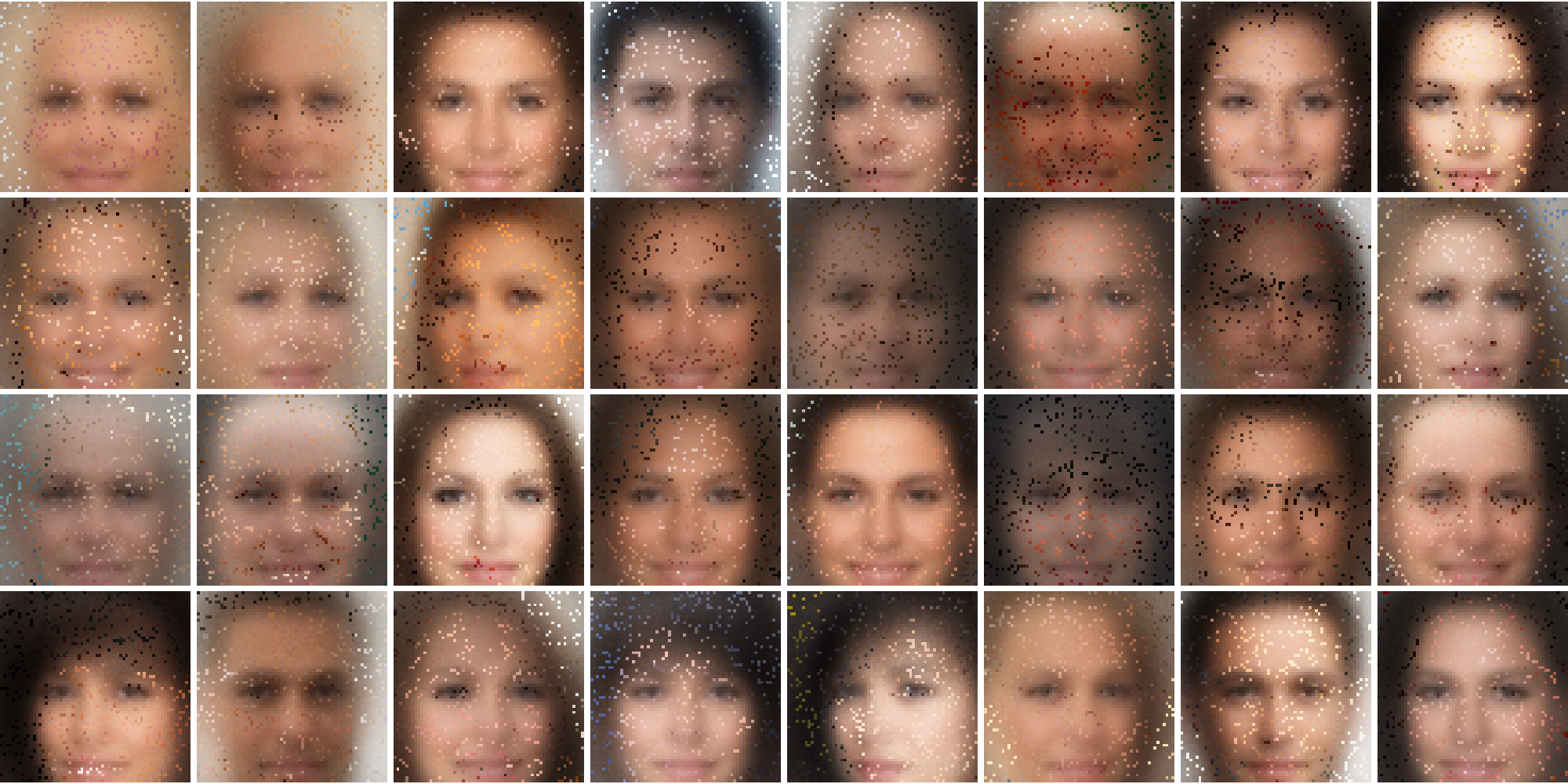}
      };
      \node[inner sep=0pt,right of=21,node distance=\figgap,
            label=below:{\small 80\% missing}] (22) {
        \includegraphics[width=\figwidth]{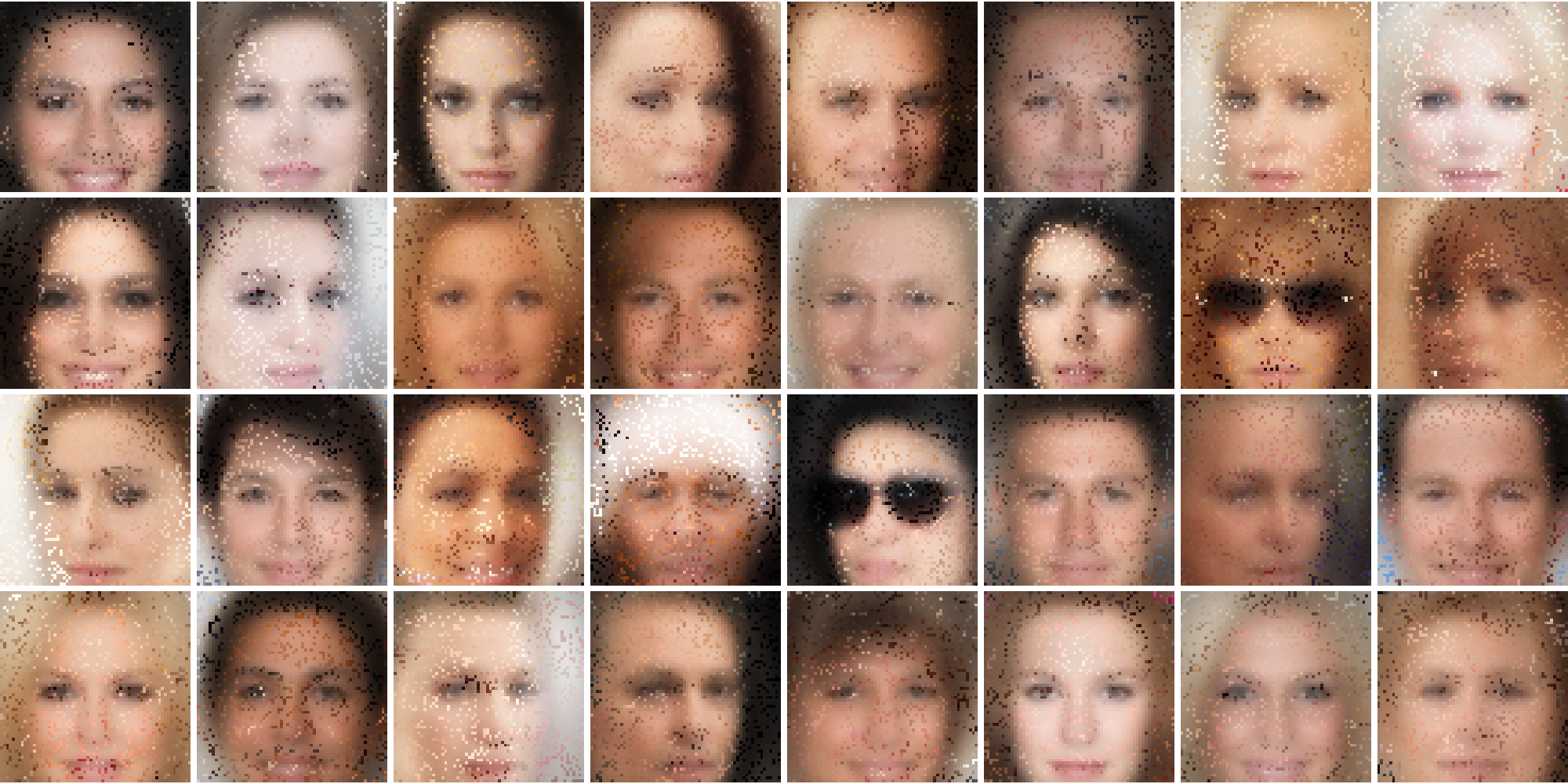}
      };
      \node[inner sep=0pt,right of=22,node distance=\figgap,
            label=below:{\small 10\% missing}] (23) {
        \includegraphics[width=\figwidth]{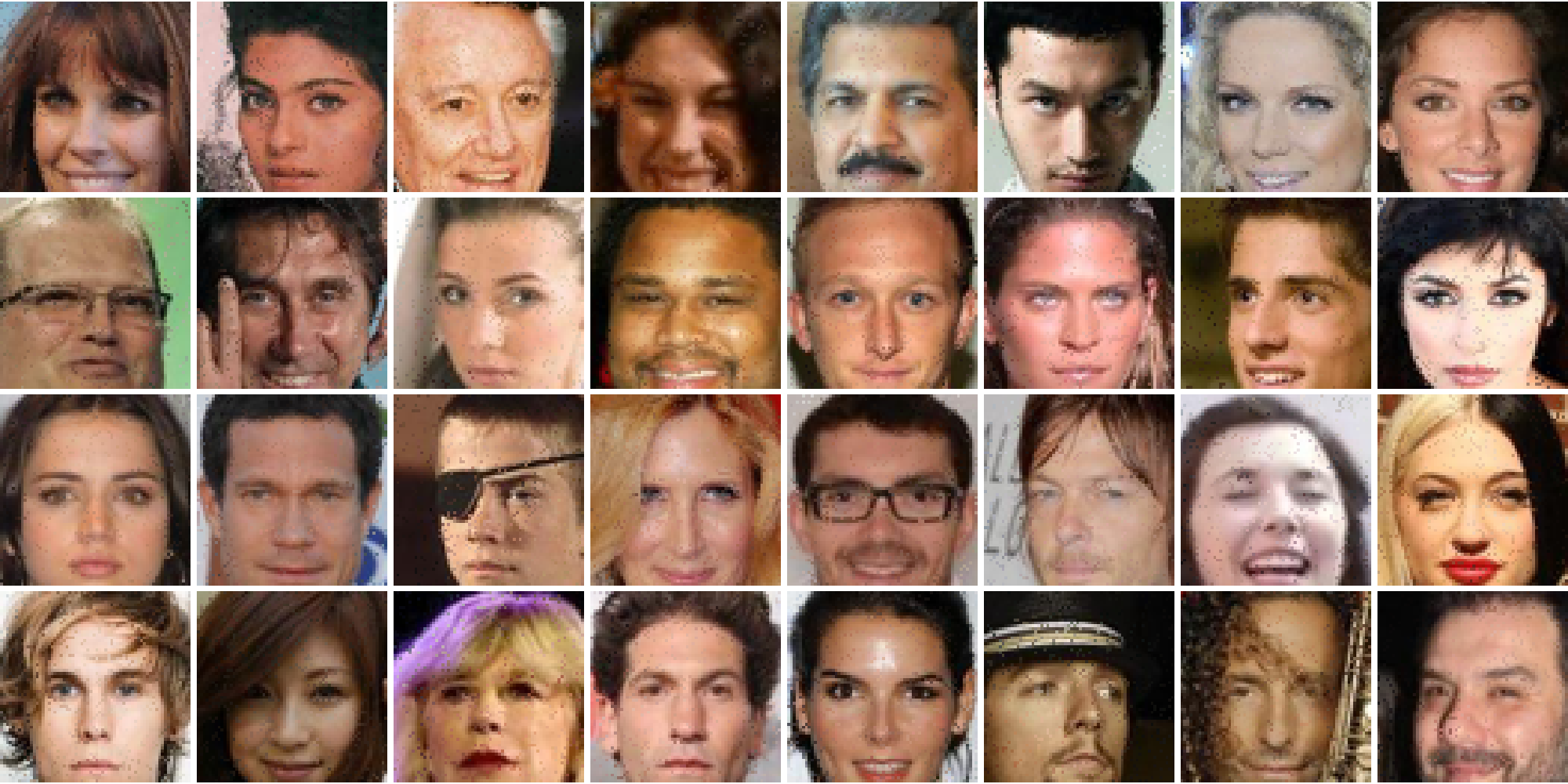}
      };
    \end{tikzpicture}
    \vspace*{-.5em}
    \caption{CelebA with independent dropout missingness}
  \end{subfigure}

  \caption{GAIN imputation}
\label{fig:gainresults}
\end{figure}

% \section{Relationship between imputation and image inpainting}
% Different from most of the image inpainting work that aims at
% completing an incomplete image and
% producing a single visually convincing result,
% the goal of missing data imputation is to model the conditional distribution
% $p(\xx_\mis|\xx_\obs)$.
% In addition, most of the state-of-the-art image inpainting work
% are required to train with complete image datasets
% \citep{pathakCVPR16context,yeh2017semantic,liu2018image}.
%
% We also note that
% although we demonstrate a simple construction of the imputer $G_i$
% in the MisGAN imputation framework as described in Section~\ref{sec:impute},
% it is compatible with many specialized architectures
% that suit the application of interest.
% When it comes to image inpainting, we could use
% the architectures of other image inpainting work
% \citep{pathakCVPR16context,yeh2017semantic,liu2018image}
% as the imputer network within our framework.

\end{document}